\documentclass[lettersize,journal]{cls/IEEEtran}
\usepackage[numbers]{natbib}
\usepackage{amsmath,amssymb,amsfonts,amstext,dsfont,mathtools}
\usepackage{multicol}
\usepackage{graphicx}
\usepackage[dvipsnames]{xcolor}
\usepackage{algorithm}
\usepackage[noend]{algpseudocode}
\usepackage{subcaption}
\usepackage{array}
\usepackage{float}
\usepackage{multirow}
\usepackage[bookmarks=true, colorlinks, breaklinks=true, pagebackref=true]{hyperref}

\usepackage{amsmath,amssymb,amsfonts,amsthm,amscd,dsfont} 
\usepackage{bbm}

\newcommand{\bfc}{\mathbf{c}}

\newcommand{\bfe}{\mathbf{e}}

\newcommand{\bfm}{\mathbf{m}}
\newcommand{\bfn}{\mathbf{n}}

\newcommand{\bfp}{\mathbf{p}}
\newcommand{\bfq}{\mathbf{q}}
\newcommand{\bfr}{\mathbf{r}}
\newcommand{\bfs}{\mathbf{s}}
\newcommand{\bft}{\mathbf{t}}

\newcommand{\bfv}{\mathbf{v}}
\newcommand{\bfw}{\mathbf{w}}
\newcommand{\bfx}{\mathbf{x}}
\newcommand{\bfy}{\mathbf{y}}
\newcommand{\bfz}{\mathbf{z}}

\newcommand{\bftheta}{\boldsymbol{\theta}}

\newcommand{\bfrho}{\boldsymbol{\rho}}

\newcommand{\bfxi}{\boldsymbol{\xi}}

\newcommand{\bfB}{\mathbf{B}}

\newcommand{\bfI}{\mathbf{I}}

\newcommand{\bfL}{\mathbf{L}}

\newcommand{\bfQ}{\mathbf{Q}}
\newcommand{\bfR}{\mathbf{R}}

\newcommand{\bfT}{\mathbf{T}}
\newcommand{\bfU}{\mathbf{U}}

\newcommand{\bfW}{\mathbf{W}}
\newcommand{\bfX}{\mathbf{X}}

\newcommand{\bfLambda}{\boldsymbol{\Lambda}}

\newcommand{\bbR}{\mathbb{R}}
\newcommand{\bbS}{\mathbb{S}}

\newcommand{\calD}{\mathcal{D}}
\newcommand{\calE}{\mathcal{E}}

\newcommand{\calG}{\mathcal{G}}

\newcommand{\calJ}{\mathcal{J}}
\newcommand{\calK}{\mathcal{K}}
\newcommand{\calL}{\mathcal{L}}

\newcommand{\calN}{\mathcal{N}}
\newcommand{\calO}{\mathcal{O}}
\newcommand{\calP}{\mathcal{P}}

\newcommand{\calR}{\mathcal{R}}

\newcommand{\calU}{\mathcal{U}}
\newcommand{\calV}{\mathcal{V}}

\newcommand{\calX}{\mathcal{X}}

\newcommand{\calZ}{\mathcal{Z}}

\newcommand{\diag}[1]{\operatorname{diag}\left({#1}\right)}

\newtheorem{proposition}{Proposition}

\theoremstyle{definition}
\newtheorem{definition}{Definition}

\newtheorem*{assumption*}{Assumption}
\newtheorem*{problem*}{Problem}

\theoremstyle{remark}

\def\papertitle{Learning Scene-Level Signed Directional Distance Function with Ellipsoidal Priors and Neural Residuals}
\def\sdfbaseline{SDF-Instant-NGP}

\begin{document}

\title{\huge\papertitle}

\markboth{
IEEE TRANSACTIONS ON PATTERN ANALYSIS AND MACHINE INTELLIGENCE, VOL. XXX, NO. XXX, Month 20XX}%
{\papertitle}

\author{Zhirui Dai, Hojoon Shin, Yulun Tian, Ki Myung Brian Lee, Nikolay Atanasov%
\thanks{We gratefully acknowledge support from NSF FRR CAREER 2045945 and ARL DCIST CRA W911NF-17-2-0181.}%
\thanks{Zhirui Dai, Ki Myung Brian Lee, and Nikolay Atanasov are with the Department of Electrical and Computer Engineering, University of California San Diego, La Jolla, CA 92093, US (e-mails: {\tt\small \{zhdai,kmblee,natanasov\}@ucsd.edu}).}%
\thanks{Hojoon Shin is with Brain Corporation, San Diego, CA 92121, US (email: hojoon.shin@braincorp.com).}%
\thanks{Yulun Tian is with the Robotics Department, University of Michigan, Ann Arbor, MI 48109, US (email: yulunt@umich.edu).}%
}
\IEEEpubid{0000--0000/00\$00.00~\copyright~2021 IEEE} %

\maketitle

\begin{abstract}
Dense reconstruction and differentiable rendering are fundamental tightly connected operations in 3D vision and computer graphics.
Recent neural implicit representations demonstrate compelling advantages in reconstruction fidelity and differentiability over conventional discrete representations such as meshes, point clouds, and voxels.
However, many neural implicit models, such as neural radiance fields (NeRF) and signed distance function (SDF) networks, are inefficient in rendering due to the need to perform multiple queries along each camera ray. Moreover, NeRF and Gaussian Splatting methods offer impressive photometric reconstruction but often require careful supervision to achieve accurate geometric reconstruction. To address these challenges, we propose a novel representation called signed directional distance function (SDDF).
Unlike SDF and similar to NeRF, SDDF has a position and viewing direction as input. Like SDF and unlike NeRF, SDDF directly provides distance to the observed surface rather than integrating along the view ray. As a result, SDDF achieves accurate geometric reconstruction and efficient differentiable directional distance prediction.
To learn and predict scene-level SDDF efficiently, we develop a differentiable hybrid representation that combines explicit ellipsoid priors and implicit neural residuals. This allows the model to handle distance discontinuities around obstacle boundaries effectively while preserving the ability for dense high-fidelity distance prediction.
Through extensive evaluation against state-of-the-art representations, we show that SDDF achieves (i) competitive SDDF prediction accuracy, (ii) faster prediction speed than SDF and NeRF, and (iii) superior geometric consistency compared to NeRF and Gaussian Splatting.
\end{abstract}


\begin{figure*}[t]
    \begin{center}
        \begin{subfigure}[t]{0.32\linewidth}
            \centering
            \includegraphics[width=\linewidth,trim={0pt 0pt 0pt 40pt},clip]{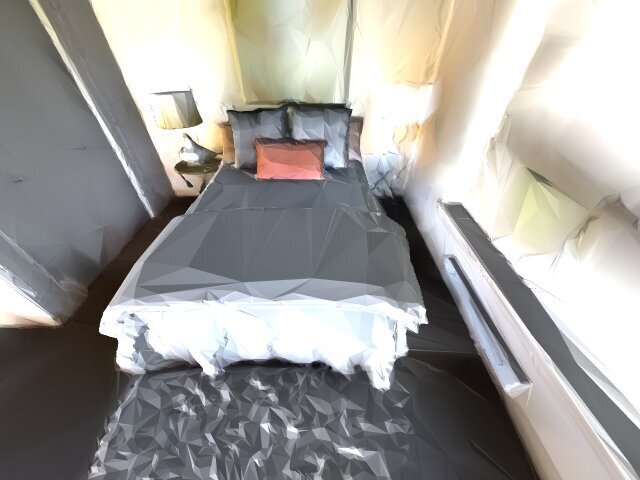}
            \caption{Gibson Allensville}
        \end{subfigure}
        \begin{subfigure}[t]{0.32\linewidth}
            \centering
            \includegraphics[width=\linewidth,trim={0pt 0pt 0pt 40pt},clip]{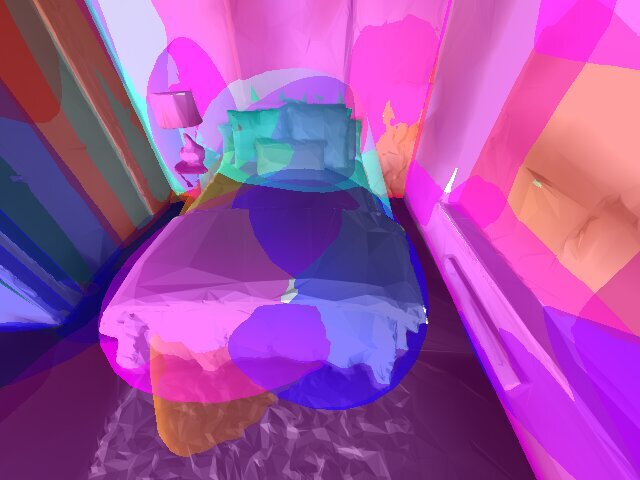}
            \caption{Ellipsoid Approximation}\label{fig:teaser_ellipsoids}
        \end{subfigure}
        \begin{subfigure}[t]{0.32\linewidth}
            \centering
            \includegraphics[width=\linewidth,trim={0pt 0pt 0pt 40pt},clip]{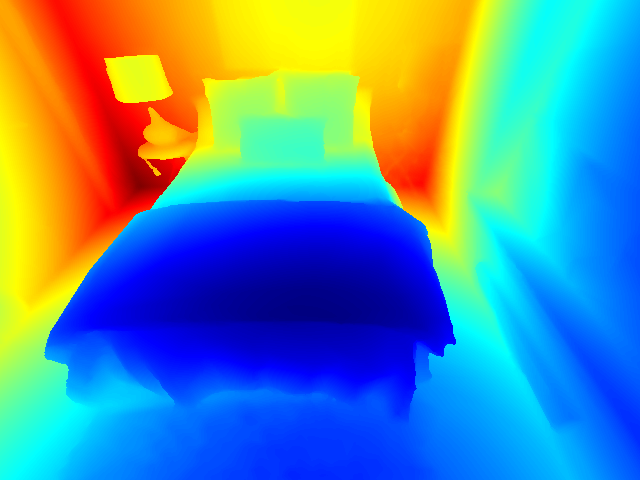}
            \caption{SDDF Ground Truth}
        \end{subfigure}
        \begin{subfigure}[t]{0.24\linewidth}
            \centering
            \includegraphics[width=\linewidth,trim={0pt 0pt 0pt 40pt},clip]{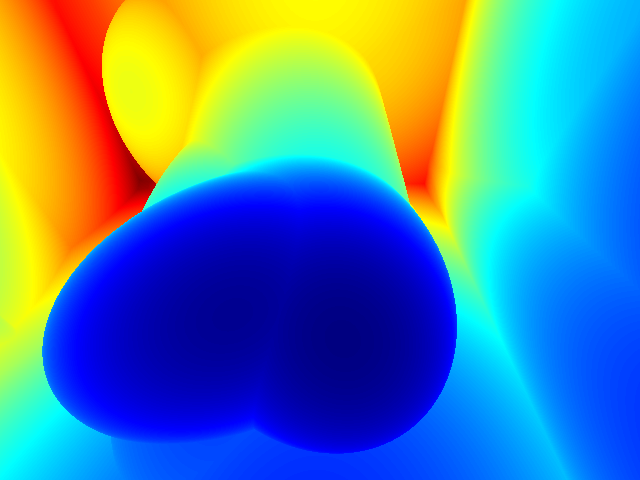}
            \caption{Ellipsoid-based SDDF Prior}
        \end{subfigure}
        \begin{subfigure}[t]{0.24\linewidth}
            \centering
            \includegraphics[width=\linewidth,trim={0pt 0pt 0pt 40pt},clip]{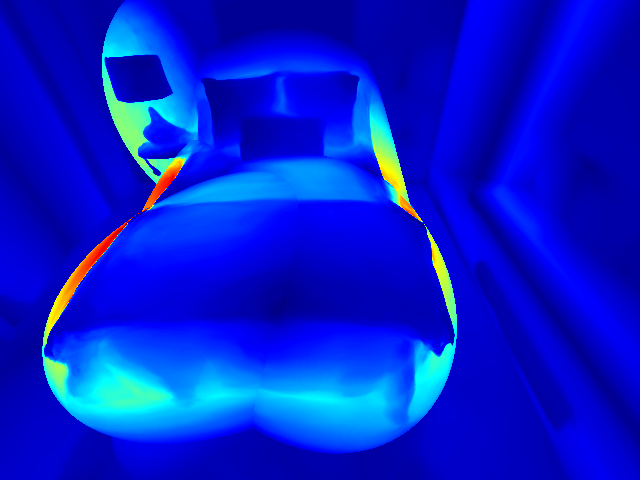}
            \caption{SDDF Residual Prediction}
        \end{subfigure}
        \begin{subfigure}[t]{0.24\linewidth}
            \centering
            \includegraphics[width=\linewidth,trim={0pt 0pt 0pt 40pt},clip]{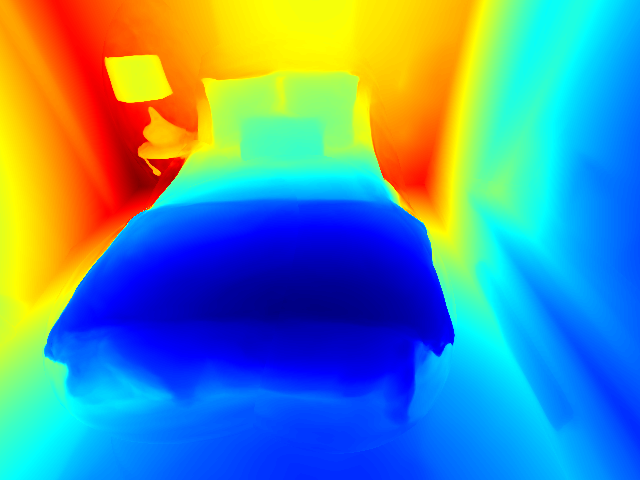}
            \caption{SDDF Prediction}
        \end{subfigure}
        \begin{subfigure}[t]{0.24\linewidth}
            \centering
            \includegraphics[width=\linewidth,trim={0pt 0pt 0pt 40pt},clip]{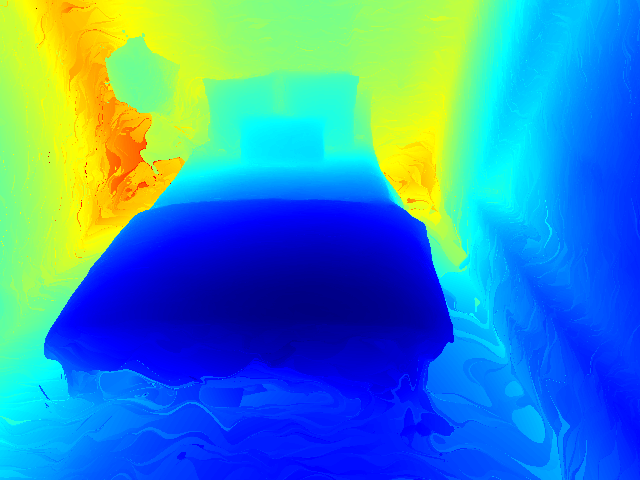}
            \caption{Prediction By RaDe-GS~\cite{radegs2024}}
        \end{subfigure}
    \end{center}
    \caption{
        \vspace{-0.5ex}
        (a), (c): We present a method to learn scene-level signed directional distance function (SDDF). (a), (b), (d): Our method uses ellipsoids as an initial coarse approximation of the shapes of objects in the environment.
        (e), (f): The ellipsoid prior is refined by a latent feature network and a shared decoder to predict the surface reconstruction residual. (f), (g): Our SDDF learning method offers single-query differentiable novel distance image synthesis without RGB supervision as an alternative to Gaussian Splat distance rendering (e.g., RaDe-GS~\cite{radegs2024}) or signed distance function sphere tracing (e.g., InstantNGP~\cite{instantngp2022}).
        \vspace{-1.5ex}
    }
    \label{fig:teaser}
\end{figure*}

\section{Introduction}
\label{sec:introduction}

\IEEEPARstart{F}inding the best 3D scene representation is a challenging problem in computer vision. The appropriate representation varies between different applications depending on the required operations, such as novel view synthesis, surface reconstruction, occupancy estimation, and occlusion checking. While explicit scene representations, e.g., based on meshes~\cite{voxblox2017,wang2018pixel2mesh,kimera2020}, point clouds~\cite{surfelslam2018,sparse2dense2023,dynamic2024}, and voxels~\cite{octomap2013}, are widely used, they are not continuous and do not support differentiation. The former hurts reconstruction accuracy and novel view synthesis, while the latter hinders their uses within downstream tasks requiring end-to-end differentiable rendering.

Recent work has focused on implicit scene representations that support differentiable geometry reconstruction and novel view synthesis.
For example, occupancy networks~\cite{occupancynet2019} and DeepSDF~\cite{deepsdf2019} have shown impressive results by representing surfaces as the zero level set of an occupancy probability function or a signed distance function (SDF).
Neural radiance field (NeRF)~\cite{nerf2020} and Gaussian Splatting (GS)~\cite{gs2023} models learn geometry implicitly through 2D image rendering supervision.
Although these implicit differentiable methods offer superior fidelity, they require multiple network forward passes, complicated calculations per pixel/ray, and high memory use.

A promising recent approach that overcomes these limitations is the signed directional distance function (SDDF). SDDF is a directional formulation of SDF that takes a position and a viewing direction as input (like NeRF and unlike SDF) and provides the distance to the observed surface (like SDF and unlike NeRF).
The benefits of SDDF as a scene representation are three-fold.
First, SDDF models can provide fast, single forward-pass directional distance queries in a differentiable way, supporting operations such as novel view synthesis and differentiable view optimization.
Second, SDDF models can be trained from different kinds of sensor data, including depth images and LiDAR scans, as long as they can be converted to ray distances. This is in contrast with SDF models, which require processing of sensor measurements to obtain SDF supervision data, and with NeRF and GS models that commonly expect camera images.
Third, an SDDF model learns a geometric representation in the space of positions and directions, which allows arbitrary view synthesis and efficient occlusion queries. This is in contrast with an SDF model which requires an iterative sphere tracing algorithm~\cite{spheretracing1996} to compute distance in a desired direction.

\IEEEpubidadjcol  

However, learning scene-level geometry in the space of positions and directions is challenging. Compared to SDF, the introduction of direction as an additional input requires additional training data with diverse viewing directions to provide sufficient supervision. Another difficulty is that SDDF is sensitive to the ray positions and directions because occlusions introduce discontinuities in the observed distance. For these reasons, previous methods for learning directional distance models~\cite{zobeidi2021,pddf2022} are only applicable to single-object shape modeling. In contrast, we consider learning \emph{scene-level} SDDF, which is necessary for many applications ranging from mixed reality to robotics, where rendering and directional distance measurement are required at the scene level.

To address these challenges, we propose a method that combines the advantages of explicit and implicit representations to learn SDDF at the scene level.
As shown in Fig.~\ref{fig:teaser}, our method first constructs an explicit ellipsoid-based prior to capture the coarse structure and occlusions of the environment. Then, an implicit residual neural network model corrects the coarse ellipsoid predictions with precise details that capture the fine structure of the environment. We guarantee that both the ellipsoid prior and the residual network are differentiable, and hence, our SDDF model supports single-query novel view synthesis and differentiable view optimization. While NeRF and GS focus on photometric rendering, our focus in this paper is on geometric reconstruction with efficient depth rendering, which can benefit various applications relying on directional distances. In summary, this paper makes the following contributions.
\begin{itemize}
    \item We introduce a new definition of SDDF suitable for scene-level representation, which extends the original SDDF definition for single objects~\cite{zobeidi2021} to handle occlusions in complex large-scale environments.

    \item We design a hybrid explicit-implicit model to approximate SDDF, consisting of an ellipsoid-based prior and an implicit neural residual, and show that the full model satisfies an Eikonal-like constraint by construction.

    \item We develop an algorithm to initialize the ellipsoid prior and derive the gradients with respect to its parameters in closed-form to accelerate training. Code is available at \url{https://github.com/existentialrobotics/neural_sddf}.

\end{itemize}
Our experiments show that our SDDF reconstruction method achieves competitive results against state-of-the-art, highly optimized implementations of SDF, GS, and NeRF in Replica~\cite{replica19arxiv} and Gibson~\cite{gibson2018} scenes, in terms of reconstruction accuracy, rendering speed, and GPU memory usage.

\section{Related Work}
\label{sec:related_work}
Three bodies of work are related to ours: NeRFs, implicit geometric representations, and directional distance functions.

\subsection{Neural Radiance Fields}

The formulation of radiance fields is a recent direction in scene representation that has led to significant improvements in the synthesis of novel photometric views. The original NeRF model~\cite{nerf2020,nerfstudio2023} employs volumetric rendering with a neural network trained to predict the color and opacity along a view direction at a given position, which are then integrated for rendering.
Such volumetric rendering, however, is computationally expensive because many samples are required for along-ray integration. To mitigate this, subsequent approaches~\cite{nerfpp2020,mipnerf2021,fastnerf2021,bakingnerf2021} introduce explicit data structures for faster and more accurate rendering with a smaller memory footprint. For example, PlenOctree~\cite{plenoctree2021} and Plenoxel~\cite{plenoxels2022} use voxelized data structures such as an octree~\cite{plenoctree2021} or a regular grid~\cite{plenoxels2022} to store spherical harmonics coefficients for better efficiency. Similarly, neural point methods~\cite{npbg2020,npbgpp2022,pointnerf2022,pointRF2022} use point clouds to store SH coefficients or neural features.
Gaussian Splatting (GS)~\cite{gs2023} and its variants such as 2D-GS~\cite{2dgs2024} and rasterizing depth GS (RaDe-GS)~\cite{radegs2024} explore this direction further by using a large number of Gaussians with SH coefficients that can be explicitly rasterized to an image through a projective transform.
These methods show that a hybrid model combining an explicit representation (e.g., point cloud, octree, or Gaussians) with an implicit representation (e.g., SH coefficients, or neural features) can offer a better trade-off between rendering quality and speed.
Drawing from these methods, we also design a hybrid explicit-implicit scene representation.
However, we focus on geometric accuracy instead of photometric accuracy.
One of our key contributions in this paper is a new definition of SDDF suitable for scene-level representation. We leave the application of our SDDF representation to photometric rendering as future work.

\subsection{Implicit Geometric Representations}
Many previous scene representations consider geometric reconstruction. Distance representations such as SDF~\cite{deepsdf2019,neuraltps2025}, unsigned distance function (UDF)~\cite{neuraludf2023,neudf2024}, and truncated SDF (TSDF)~\cite{tsdf2013}, model the scene geometry explicitly and offer fast proximity queries, which are valuable in applications, such as mapping and trajectory optimization in robotics~\cite{safe_bubble_cover,loggpismop,chomp}.
To scale to larger environments, an effective strategy is to use voxel hashing and incrementally estimate and store SDF in a hash table~\cite{voxelhashing2013,voxblox2017}.
Gaussian process (GP) methods~\cite{gpis2019,loggpis2021} use octrees to store incrementally estimated oriented surface points and the corresponding GP models to regress SDF.
Although these methods achieve real-time incremental SDF estimation, they fail to capture small-scale details accurately.

Recent approaches such as DeepSDF~\cite{deepsdf2019} and implicit geometric regularization (IGR)~\cite{igr2020} demonstrate that neural networks can learn SDF accurately, with supervision from oriented surface points and Eikonal regularization~\cite{igr2020}.
To learn a geometric representation from images instead of surface points, geometric representations can be incorporated in a radiance field and used as a means to improve photometric accuracy.
For example, NeuS~\cite{neus2021} replaces the opacity with an SDF so that the geometry (i.e., SDF) and the radiance field are learned together.
RaDe-GS~\cite{radegs2024} proposes a geometric regularization term that improves both geometric and photometric accuracy. These results show that an accurate geometric representation that we pursue here is fundamentally important for both geometric and photometric accuracy.
As part of our contributions, the combination of explicit geometric prior and implicit neural features achieves good balance of accuracy and efficiency in SDDF prediction, which is an important conclusion for future research on implicit geometric representations.

For geometric supervision, SDFDiff~\cite{sdfdiff2020} and DIST~\cite{dist2020} use depth, surface normals, or silhouette for supervision by differentiating through SDF queries in sphere tracing iterations, although the sphere tracing process~\cite{spheretracing1996} itself is not differentiable.
Sphere tracing on SDFs can also be used to compute directional distance as we do but its iterative and cumulative nature leads not only to slower training and rendering but also to error accumulation.

\subsection{Directional Distance Functions}
Directional distance functions (DDFs) provide directional distance in a single query, and are hence more efficient and free from error accumulation compared to SDFs. Directional TSDF~\cite{dtsdf2019} extends TSDF in 3D to six voxel grids, each along or against the $X$, $Y$, and $Z$ axes.
Since the method uses only six discrete directions, its reconstruction accuracy may be limited, but the importance of directional information when modeling thin objects is evident.
A signed ray distance function (SRDF)~\cite{srdf2023} is defined with respect to a given camera pose, as the distance of a 3D point to the scene surface along the viewing lines of the camera.
This definition is useful for multiview stereo because SRDF is zero only when the query point is on the surface.
Volrecon~\cite{volrecon2023} uses the same term, SRDF, to instead denote simply the distance between a point and the scene surface along a ray, which is closer to our definition.
This SRDF is used in volumetric rendering for multiview stereo and high-quality rendering, which shows the utility of the directional distance in capturing geometric details.
However, Volrecon uses multiple transformers, and many samples along each ray, which is prohibitively slow for scaling to larger scenes.

Other methods focus on learning DDFs at the object level.
NeuralODF~\cite{neuralodf2022} and RayDF~\cite{raydf2024} learn ray-surface distances for an object bounded by a sphere that is used to parameterize the ray.
Sphere-based ray parameterization cannot apply to scene-level DDF reconstruction because scenes may be unbounded.
FIRe~\cite{fire2024} combines SDF and DDF to reconstruct object shapes, where DDF renders object shapes efficiently.
However, FIRe is difficult to scale to scene-level representation because it requires silhouettes for good performance, which cannot be directly captured by sensors.
Pointersect~\cite{pointersect2023} uses a transformer to predict the ray travel distance, the surface normal vector, and RGB color given a dense point cloud generated from multiple posed RGB-D images.
Due to the dense point cloud and the transformer, Pointersect exhibits slow rendering and lower output quality when multiple surfaces exist along the ray, which limits its scalability.
The network architecture of~\cite{zobeidi2021} is specialized in learning the SDDF at the object level. It satisfies the directional Eikonal constraint by construction. However, their definition and design are suitable only for learning a single object.
The probabilistic DDF (PDDF)~\cite{pddf2022} introduces an additional ``probability'' network to learn the discontinuities caused by occlusions. The method is designed to learn PDDF at the object level.
Although multiple objects' PDDFs can be fused into a scene PDDF, its scalability is limited, since many independent networks are needed for each object in larger scenes.
Moreover, the omission of sign information obfuscates whether a query position is inside or outside the objects.
Instead, as shown in Fig.~\ref{fig:sddf_examples}, our definition of SDDF has fewer discontinuities than DDF, and can encode complicated geometric structures ranging from small objects to large scenes.
Moreover, we design a hybrid explicit-implicit model to approximate SDDF. The explicit ellipsoidal SDDF prior not only learns a coarse SDDF approximation but also helps to efficiently handle discontinuities. The implicit neural residual further refines the SDDF prediction to capture fine geometric details.

\begin{figure*}[t]
    \centering
    \begin{subfigure}[t]{0.2\linewidth}
        \centering
        \includegraphics[width=\linewidth,trim={0 0pt 0 0pt},clip]{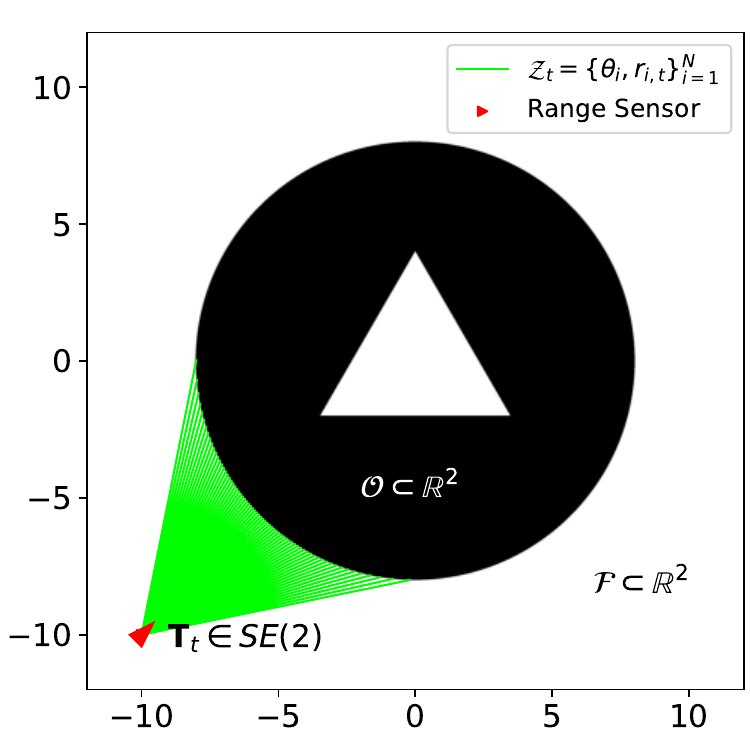}
        \caption{2D Scene} \label{fig:sddf_examples_a}
    \end{subfigure}
    \begin{subfigure}[t]{0.26\linewidth}
        \centering
        \includegraphics[width=\linewidth,trim={0 40pt 0 20pt},clip]{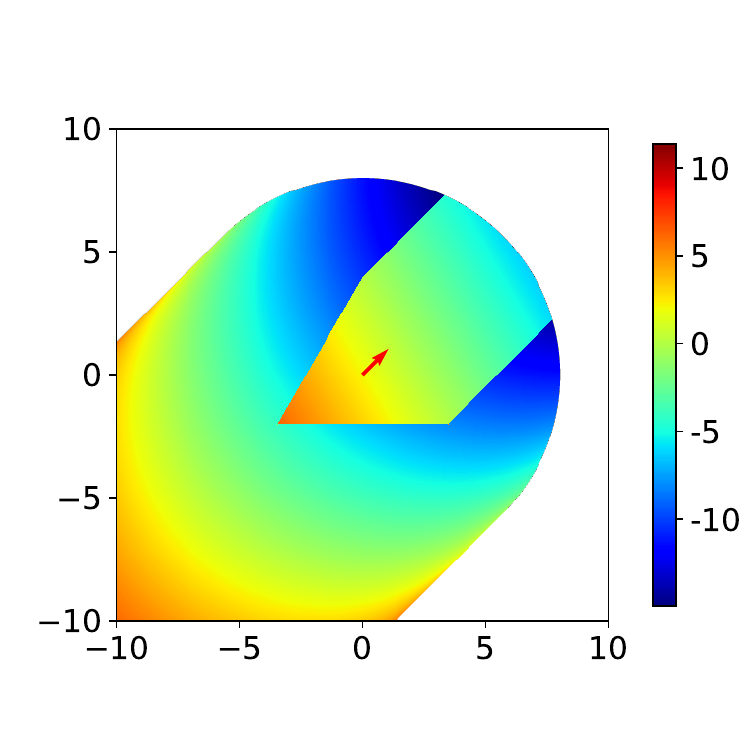}
        \caption{SDDF ($45^\circ$)} \label{fig:sddf_examples_b}
    \end{subfigure}
    \begin{subfigure}[t]{0.26\linewidth}
        \centering
        \includegraphics[width=\linewidth,trim={0 40pt 0 20pt},clip]{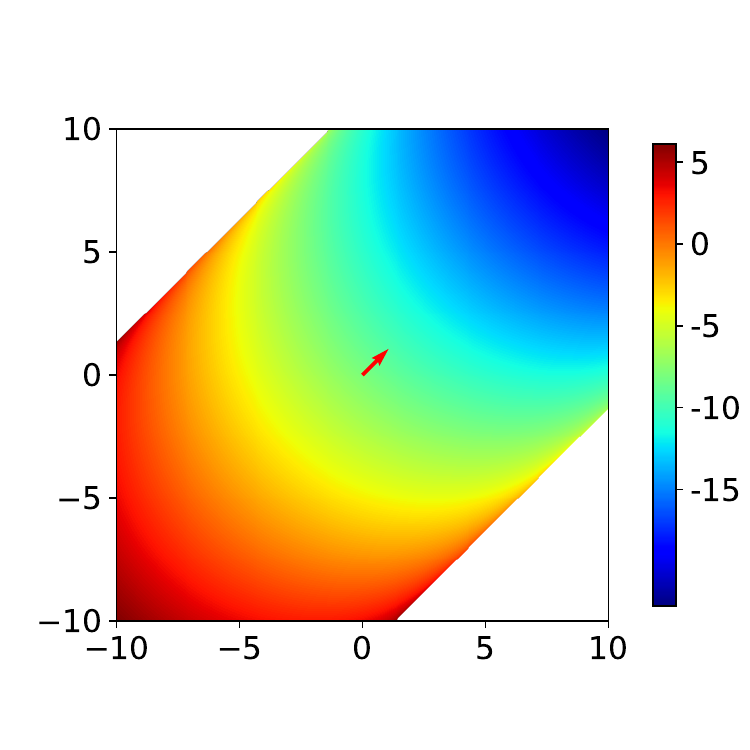}
        \caption{Object SDDF ($45^\circ$)~\cite{zobeidi2021}} \label{fig:sddf_examples_c}
    \end{subfigure}
    \begin{subfigure}[t]{0.26\linewidth}
        \centering
        \includegraphics[width=\linewidth,trim={0 40pt 0 20pt},clip]{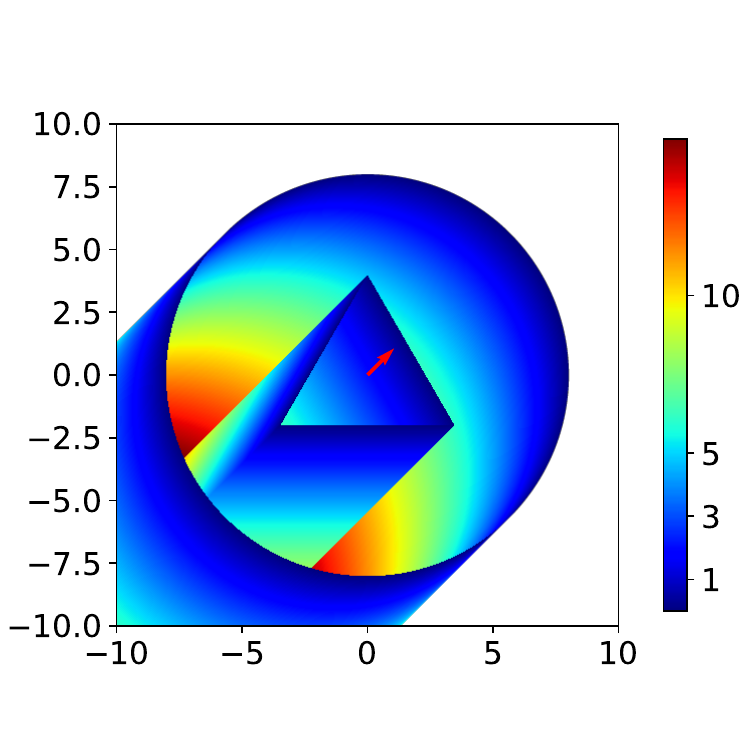}
        \caption{DDF ($45^\circ$)~\cite{pddf2022}} \label{fig:sddf_examples_d}
    \end{subfigure}
    \caption{Example of our scene-level SDDF, the object-level SDDF of~\cite{zobeidi2021}, and the DDF of~\cite{pddf2022} in a 2D synthetic environment. (a): a range sensor (\textcolor{red}{red triangle}) with pose $\bfT_t \in SE(2)$ is measuring the distance to a doughnut-like obstacle $\calO$ (\textcolor{black}{black}) with a triangular hole in the middle (left plot). At time $t$, the sensor measurement $\calZ_t=\{\theta_i,r_{i,t}\}_{i=1}^N$ consists of $N$ range measurements $r_i$ obtained along rays (\textcolor{green}{green lines}) cast at angles $\theta_i$. Measurements from multiple time steps $\{\calZ_t\}_{t=0}^{T-1}$ are collected. (b) to (d): the red arrow in the three plots on the right labels the viewing direction. The white region in the right three plots indicates invalid/infinite field values. Unlike DDF, our SDDF definition is continuous when transitioning from free to occupied space along the viewing direction. Compared with object SDDF~\cite{zobeidi2021}, our SDDF definition reflects the geometry well, allowing scene-level reconstruction.}
    \label{fig:sddf_examples}
\end{figure*}

\section{Method Overview} \label{sec:method_overview}

We aim to design a method for range sensors, such as LiDARs or depth cameras, to learn an environment model capable of efficient and differentiable synthesis of arbitrary distance views. Let the occupied space in the environment be represented by a set $\calO \subset \bbR^n$, where $n$ is the dimension ($n \in \{2,3\}$ in practice). Consider a set of measurements $\{\bfT_t, \calZ_t\}_{t=1}^T$, where $\bfT_t \in SE(n)$ is the sensor pose at time $t$ and $\calZ_t = \{\bfv_{i}, r_{t,i}\}_{i=1}^N$ are the measurements at time $t$, consisting of $N$ viewing directions $\bfv_{i} \in \bbS^{n-1}$ (e.g., unit vector in the direction of each depth camera pixel or unit vector in the direction of each LiDAR scan ray) and the corresponding range measurements $r_{t,i} \in \bbR_{>0}$. Our objective is to learn a representation of the occupied space $\calO$ in the form of a signed directional distance function.

\begin{definition} \label{def:sddf}
    The \emph{signed directional distance function} (SDDF) of a set $\calO \subset \bbR^n$ is a function $f: \bbR^n \times \bbS^{n-1} \to \bbR \cup \{\pm \infty\}$ that measures the signed distance from a point $\bfp \in \bbR^n$ to the set boundary $\partial \calO$ along a direction $\bfv \in \bbS^{n-1}$, defined as:
    \begin{equation}\label{eq:SDDF_def}
        f(\bfp, \bfv ; \calO) :=\!\begin{cases}
             \min\{ d > 0 \mid \bfp + d\bfv \in \partial \calO\}, & \bfp  \not \in \calO, \\
             \max\{ d \leq 0 \mid \bfp + d\bfv \in \partial \calO\}, & \bfp \in \calO.
        \end{cases}
    \end{equation}
\end{definition}

The SDDF definition is illustrated in Fig.~\ref{fig:sddf_examples}. For comparison, the signed distance function (SDF) of a set $\calO$ is defined as the shortest distance from $\bfp \in \bbR^n$ to the boundary $\partial\calO$:
\begin{equation}
    f_\text{SDF}(\bfp; \calO) := \!\begin{cases}
        \phantom{+}\min_{\bfy \in \partial\calO} \left\|\bfp - \bfy\right\|_2, & \bfp \not \in \calO, \\
        -\min_{\bfy \in \partial\calO} \left\|\bfp - \bfy\right\|_2, & \bfp \in \calO.
    \end{cases} \label{eq:sddf_definition}
\end{equation}
The SDF and SDDF of $\calO$ are related as follows:
\begin{equation}
    f_\text{SDF}(\bfp; \calO) = \min_{\bfv \in \bbS^{n-1}} f(\bfp,\bfv; \calO). \label{eq:sdf_definition}
\end{equation}

It is well known~\cite{igr2020,loggpis2021} that SDFs satisfy an Eikonal equation $\left\|\nabla_\bfp f_\text{SDF}(\bfp; \calO)\right\|_2=1$, which is useful for regularizing or designing the structure of models for estimating SDF. The next proposition shows that SDDF satisfies a similar property.

\begin{proposition} \label{prop:sddf_eikonal_equation}
Suppose an SDDF $f(\bfp,\bfv;\calO)$ is differentiable at $\bfp \in \bbR^n$. Then, it satisfies a directional Eikonal equation:
\begin{equation} \label{eq:sddf_eikonal_equation}
\bfv^\top \nabla_\bfp f(\bfp,\bfv;\calO) = -1.
\end{equation}
\end{proposition}

\begin{proof}\renewcommand{\qedsymbol}{}
    When the ray hits the same surface point $\bfq=\bfp + d\bfv$ as $\bfp$ moves along $\bfv$, we have:
    \begin{align*}
        & \bfv^\top \nabla_\bfp f(\bfp, \bfv;\calO) = \lim_{\delta \rightarrow 0} \frac{f(\bfp+\delta\bfv,\bfv) - f(\bfp,\bfv)}{\delta} \\
        & =\lim_{\delta \rightarrow 0} \frac{(d - \delta) - d}{\delta} = -1. \tag*{\IEEEQEDclosed}
    \end{align*}
    \vspace{-2ex}
\end{proof}

This directional Eikonal equation indicates that as the position $\bfp$ moves towards (away from) the intersected surface along the direction $\bfv$, the SDDF value decreases (increases) at unit rate. This property is useful for designing the structure or regularizing a neural network representation of SDDF.

\begin{figure*}[t]
    \centering
    \includegraphics[width=\linewidth]{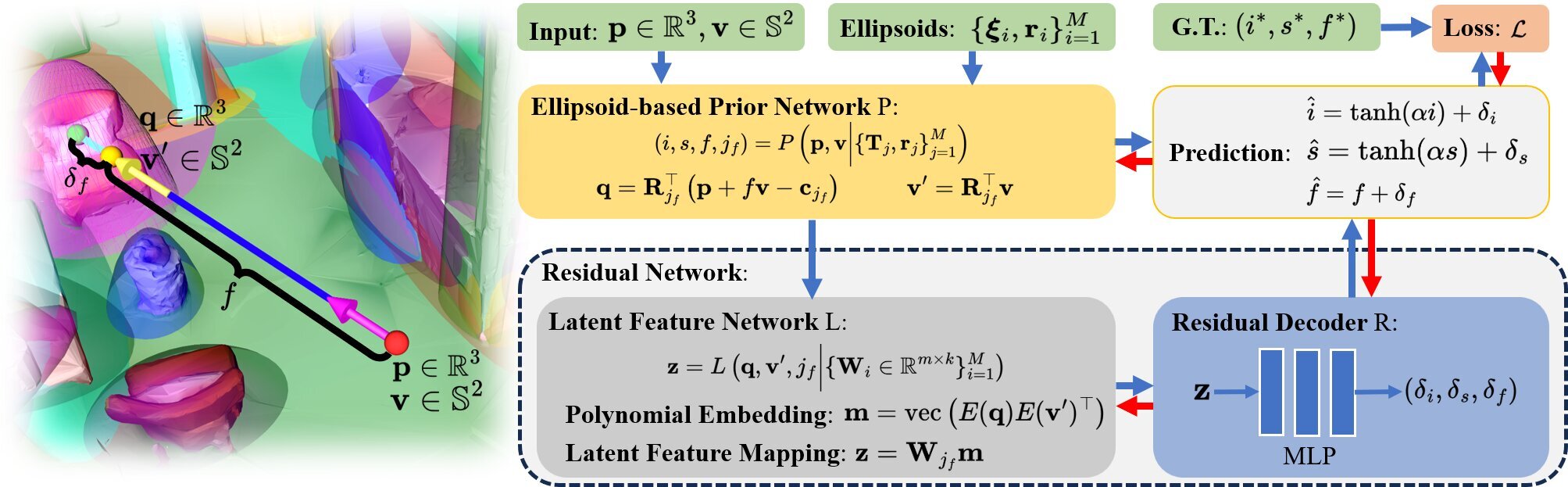}
    \caption{\textbf{Method overview.} Given a query ray from position $\bfp\in\bbR^3$ in direction $\bfv\in\bbS^2$, an ellipsoid-based \textbf{P}rior network $P$ uses $M$ ellipsoids $\{\bfxi_i,\bfr_i\}_{i=1}^M$ to learn the rough shape of the environment such that it can determine the closest ellipsoid intersected by the ray and predict an SDDF prior. Then, with the intersection point $\bfq\in\bbR^3$ and ray direction $\bfv'\in\bbS^2$ in the ellipsoid's local frame, a \textbf{L}atent network $L$ generates a latent feature $\bfz\in\bbR^m$, which is decoded by the \textbf{R}esidual decoder $R$ into residual predictions $\left(\delta_i,\delta_s,\delta_f\right)$, i.e., the difference between the ground truth and the prior. Finally, we compose the SDDF prediction as $\hat{f}=f+\delta_f$. Blue arrows show the data flow in the forward pass, while red arrows represent the backward pass.}
    \label{fig:overview}
\end{figure*}

Previous work~\cite{zobeidi2021} proposed an object SDDF that is defined as $\min\{d\in\bbR \mid \bfp + d\bfv \in \partial\calO\}$.
However, this definition is only suitable for learning the shape of a single object because when there is another object behind the viewing position $\bfp$, the object SDDF will ignore the object that the viewing direction $\bfv$ is pointing at. Fig.~\ref{fig:sddf_examples_c} shows an example of this issue.

PDDF~\cite{pddf2022} addresses this problem by learning directional distance without sign, defined as $\min\{d \ge 0 \mid \bfp + d\bfv \in \partial\calO\}$. However, as shown in Fig.~\ref{fig:sddf_examples_d}, this definition introduces several discontinuities, which is not favorable for learning a directional distance model from data.

Our definition of SDDF, shown in Fig.~\ref{fig:sddf_examples_b}, guarantees that the distance continuously changes from positive to negative values when a viewing ray $\bfv$ enters the occupied space $\calO$ from \emph{outside in}. However, when the ray leaves $\calO$ from \emph{inside out}, there is a discontinuity in the distance. Fortunately, such a discontinuity does not occur in practice because sensors remain within free space, only observing $\calO$ from \emph{outside in}. Thus, our definition can accurately model occlusions while retaining the sign to distinguish whether the query position is in free space, unlike previous formulations of DDF~\cite{pddf2022}.

Yet, learning a solely continuous, implicit representation of our new SDDF definition remains challenging.
Unlike an SDF, SDDF is sensitive to both position and direction, as small input perturbations may lead to hitting a different obstacle surface, resulting in a significant change in SDDF value.
The inherent discontinuity and the added input dimensions for the ray direction demand more training data over different ray directions.
Explicit representations have advantages in handling such discontinuities with less data,
but purely explicit representations struggle to achieve high-fidelity reconstruction and differentiable view synthesis, as they are discrete.
In contrast, implicit representations are good at learning geometric details and allow differentiation.

Therefore, as shown in Fig.~\ref{fig:overview}, our model combines both explicit and implicit representations.
First, in Sec.~\ref{sec:geometry_prior_net}, an explicit ellipsoid-based \textbf{P}rior network $P(\bfp, \bfv)$ is introduced to predict a coarse SDDF prior $f(\bfp,\bfv)$.
Then, in Sec.~\ref{sec:residual_net}, we present a residual network consisting of a \textbf{L}atent feature network $L$ and a \textbf{R}esidual decoder $R$ that predicts an SDDF correction $\delta_f(\bfp,\bfv)$, so that the combination accurately models the true SDDF $f^*(\bfp,\bfv)$ as $\hat{f}(\bfp,\bfv) = f(\bfp,\bfv) + \delta_f(\bfp,\bfv)$.

\section{Ellipsoid-based Prior Network}\label{sec:geometry_prior_net}

To take advantage of an explicit representation for occlusion modeling, we design an ellipsoid-based prior network $P$. The prior uses a set of ellipsoids to approximate the structure of the environment based on the range measurements and leaves the task of learning fine details to the residual network $R$.

\subsection{SDDF of a Single Ellipsoid} \label{sec:ellipsoid}

First, for simplicity, consider a single ellipsoid given by:
\begin{equation}
    \calE = \{ \bfy \in \bbR^3 \mid (\bfy-\bfc)^\top \bfR \bfQ_0^{-2} \bfR^\top (\bfy - \bfc) \le 1 \},
\end{equation}
where $\bfc\in\bbR^3$ and $\bfR\in SO(3)$ are the position and orientation, $\bfQ_0=\diag{\bfr}$, and $\bfr \in \bbR_+^3$ are the radii of the ellipsoid.

We parameterize the ellipsoid pose as:
\begin{equation}\label{eq:pose_parametrization}
    \bfT = \begin{bmatrix}
        \bfR & \bfc \\ \mathbf{0}^\top & 1
    \end{bmatrix} = \bfT_0 \exp\left(\bfxi^\wedge\right), \quad \bfxi^\wedge = \begin{bmatrix}
        \bftheta^\wedge & \bfrho \\
        \mathbf{0}^\top & 0
    \end{bmatrix},
\end{equation}
where $\bfT_0 \in SE(3)$ is initialized and fixed, and $\bfxi = (\bfrho, \bftheta) \in \bbR^6$ is learnable. In~\eqref{eq:pose_parametrization}, the function $\bftheta^\wedge$ maps a vector $\bftheta \in \bbR^3$ to a corresponding skew-symmetric matrix, and $\exp$ is the matrix exponential function.

To ensure that the radii $\bfr\in\bbR_+^3$, we parameterize it as $\bfr=\bfr_0\exp\left(\bfs\right)$, where $\bfr_0\in\bbR_+^3$ is initialized and fixed, $\bfs \in \bbR^3$ is learnable, and $\exp$ is applied element-wise.

Then, we derive the SDDF of an ellipsoid $\calE$ in closed form.
\begin{proposition} \label{prop:ellipsoid_sddf}
    Consider a ray from position $\bfp \in \bbR^3$ in direction $\bfv \in \bbS^2$ and ellipsoid $\calE \subset \bbR^3$. If the ray does not intersect $\calE$, the ellipsoid SDDF is $f(\bfp,\bfv;\calE) = \infty$. Otherwise:
    \begin{equation} \label{eq:intersection_sol}
        f(\bfp,\bfv; \calE) = -\frac{\det{\bfQ_0} \sqrt{i(\bfp,\bfv)} + {\bfp'}^\top \bfQ_1^2 \bfv'}{{\bfv'}^\top \bfQ^2_1 \bfv'},
    \end{equation}
    where
    \begin{equation}\label{eq:ellipsoid_intersection}
        i(\bfp,\bfv) = \bfv'^\top \bfQ_1^2 \bfv'-\bfw'^\top \bfQ_0^2 \bfw'
    \end{equation}
    is an intersection indicator,
    $\bfQ_1=\det(\bfQ_0) \bfQ_0^{-1}$, $\bfp' =\bfR^\top\left(\bfp-\bfc\right)$, $\bfv' = \bfR^\top\bfv$, and $\bfw' = \bfp' \times\bfv'$.
\end{proposition}

The proof of Proposition~\ref{prop:ellipsoid_sddf} is provided in Supplemental~\ref{sec:appdx_proof_prop_ellipsoid_sddf}. The expression in~\eqref{eq:intersection_sol} is one of the two solutions of the quadratic equation formed by combining the ray equation $\bfq = \bfR^\top (\bfp + f(\bfp,\bfv;\calE)\bfv - \bfc)$ and the ellipsoid equation $\bfq^\top \bfQ_0^{-2} \bfq = 1$. It corresponds to the first intersection point along the ray direction, which aligns with the SDDF definition in Definition~\ref{def:sddf}.
The intersection indicator $i(\bfp,\bfv)$ in~\eqref{eq:ellipsoid_intersection} is positive when the line through $\bfp$ in direction $\bfv$ intersects $\calE$ and negative otherwise. When the line is tangent to the ellipsoid, we have $i(\bfp,\bfv)=0$. We also introduce a sign indicator:
\begin{equation} \label{eq:ellipsoid_sign}
    s(\bfp,\bfv) = {\bfp'}^\top \bfQ_1^2 \bfp' - \det{\bfQ^2_0},
\end{equation}
which is negative when $\bfp \in \calE$ and positive otherwise.
We use the above equation instead of $\bfp'^\top \bfQ_0^{-2}\bfp' - 1$ as the sign indicator for better numerical stability.

In order to specify a reasonable SDDF prior when the line does not intersect the ellipsoid, i.e., when $i(\bfp,\bfv)<0$, we change the SDDF from $\infty$ to:
\begin{equation} \label{eq:prior_no_intersection}
    f(\bfp,\bfv;\calE) = -\frac{{\bfp'}^\top \bfQ_1^2 \bfv'}{{\bfv'}^\top \bfQ^2_1 \bfv'}.
\end{equation}
The expression in~\eqref{eq:prior_no_intersection} gives the distance from $\bfp'$ along $\bfv'$ to a virtual plane at the ellipsoid origin with normal vector $\bfQ_1^2\bfv'/\left\| \bfQ_1^2\bfv' \right\|_2$.
This guarantees that $f(\bfp,\bfv;\calE)$ changes smoothly when $i(\bfp,\bfv)$ changes sign.

The expression in~\eqref{eq:intersection_sol} is valid only when the ray from $\bfp$ in the direction of $\bfv$ intersects $\calE$. If the ellipsoid $\calE$ is behind the ray, i.e., $\bfp \not \in \calE$ and $s(\bfp,\bfv) > 0$, then~\eqref{eq:intersection_sol} is negative but $f(\bfp,\bfv;\calE) = \infty$. We introduce a validity function $v(\bfp,\bfv)$, which is negative when $\calE$ is behind the view ray. This allows us to combine~\eqref{eq:intersection_sol} and~\eqref{eq:prior_no_intersection} with consideration of validity to obtain a modified SDDF prior for a single ellipsoid:
\begin{equation} \label{eq:single_ellipsoid_sddf_prior}
    f(\bfp,\bfv; \calE) = \!\begin{cases}
        -\frac{\det{\bfQ_0} \sqrt{\beta} + {\bfp'}^\top \bfQ_1^2 \bfv'}{{\bfv'}^\top \bfQ^2_1 \bfv'}, & v(\bfp,\bfv) \ge 0, \\
        \infty, & v(\bfp,\bfv) < 0,
    \end{cases}
\end{equation}
where $\beta=\max(i(\bfp,\bfv),0) + \epsilon$, $\epsilon > 0$ is a small value introduced for numerical stability of backward propagation, and the validity indicator function is defined as:
\begin{equation} \label{eq:prior_validity}
    v(\bfp,\bfv) = -\frac{\det{\bfQ_0} \sqrt{\beta} + {\bfp'}^\top \bfQ_1^2 \bfv'}{{\bfv'}^\top \bfQ^2_1 \bfv'} s(\bfp,\bfv),
\end{equation}
which requires that the sign of a valid SDDF and the sign of the indicator in~\eqref{eq:ellipsoid_sign} should agree.
A 2D example in Fig.~\ref{fig:prior_sddf_example} shows that the ellipsoid SDDF prior $f(\bfp,\bfv; \calE)$ in~\eqref{eq:single_ellipsoid_sddf_prior} reflects Definition~\ref{def:sddf} correctly.

\begin{figure*}
    \centering
    \begin{subfigure}[t]{0.24\linewidth}
        \centering
        \includegraphics[width=\linewidth,trim={0 15pt 20 10pt},clip]{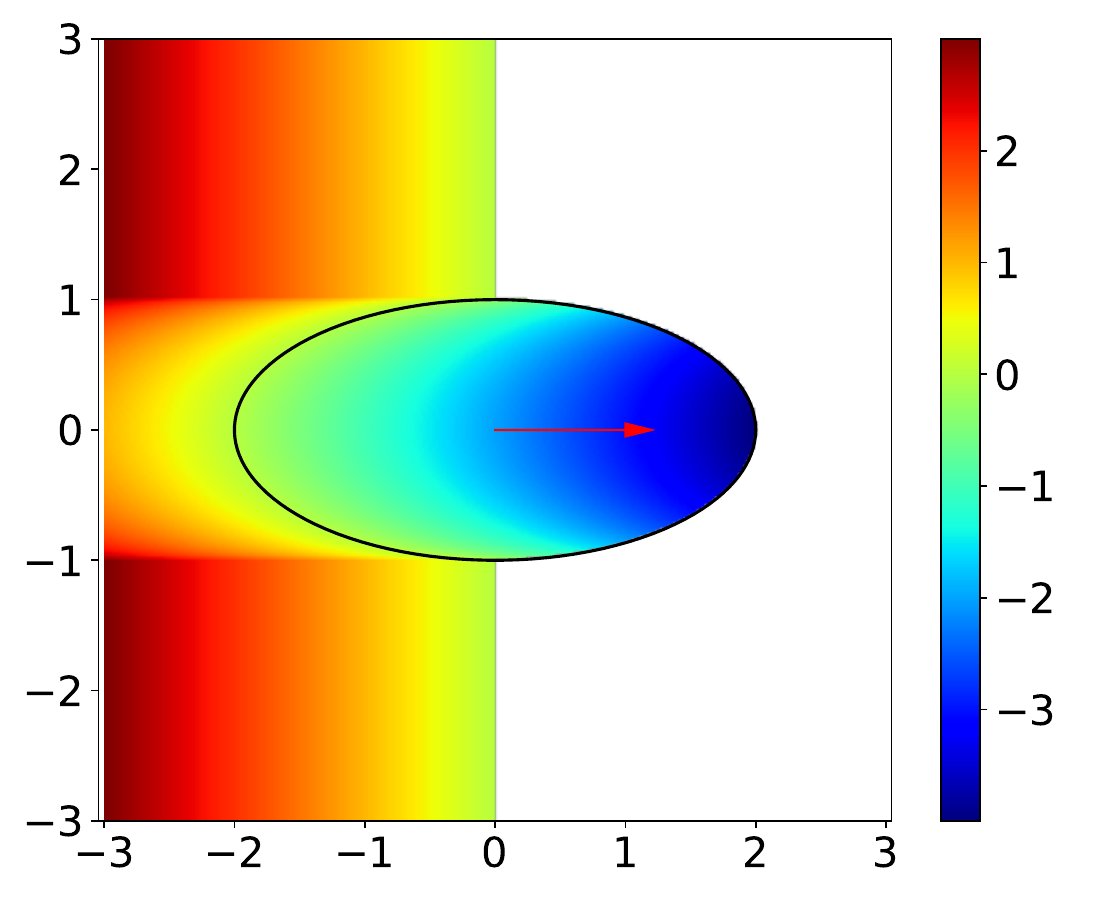}
        \caption{$0^\circ$}
    \end{subfigure}
    \begin{subfigure}[t]{0.24\linewidth}
        \centering
        \includegraphics[width=\linewidth,trim={0 15pt 20 10pt},clip]{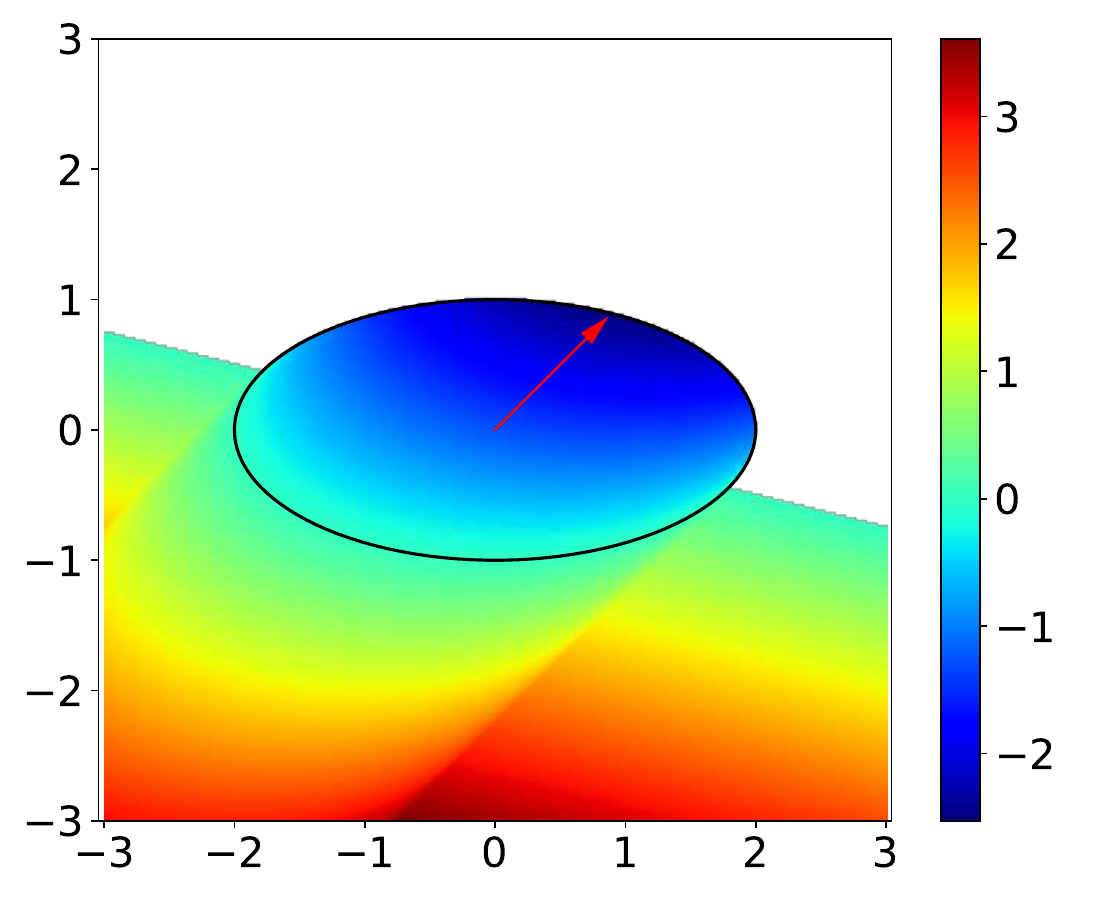}
        \caption{$45^\circ$}
    \end{subfigure}
    \begin{subfigure}[t]{0.24\linewidth}
        \centering
        \includegraphics[width=\linewidth,trim={0 15pt 20 10pt},clip]{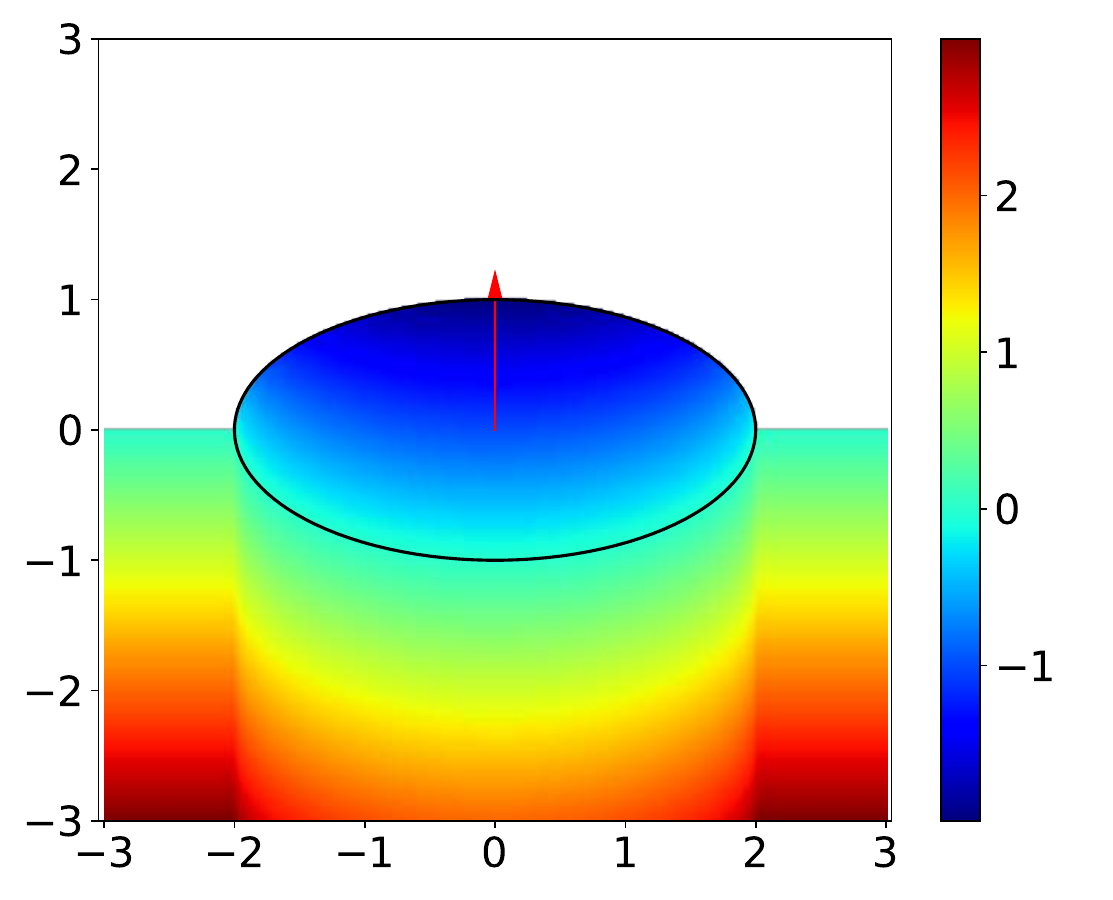}
        \caption{$90^\circ$}
    \end{subfigure}
    \begin{subfigure}[t]{0.24\linewidth}
        \centering
        \includegraphics[width=\linewidth,trim={0 15pt 20 10pt},clip]{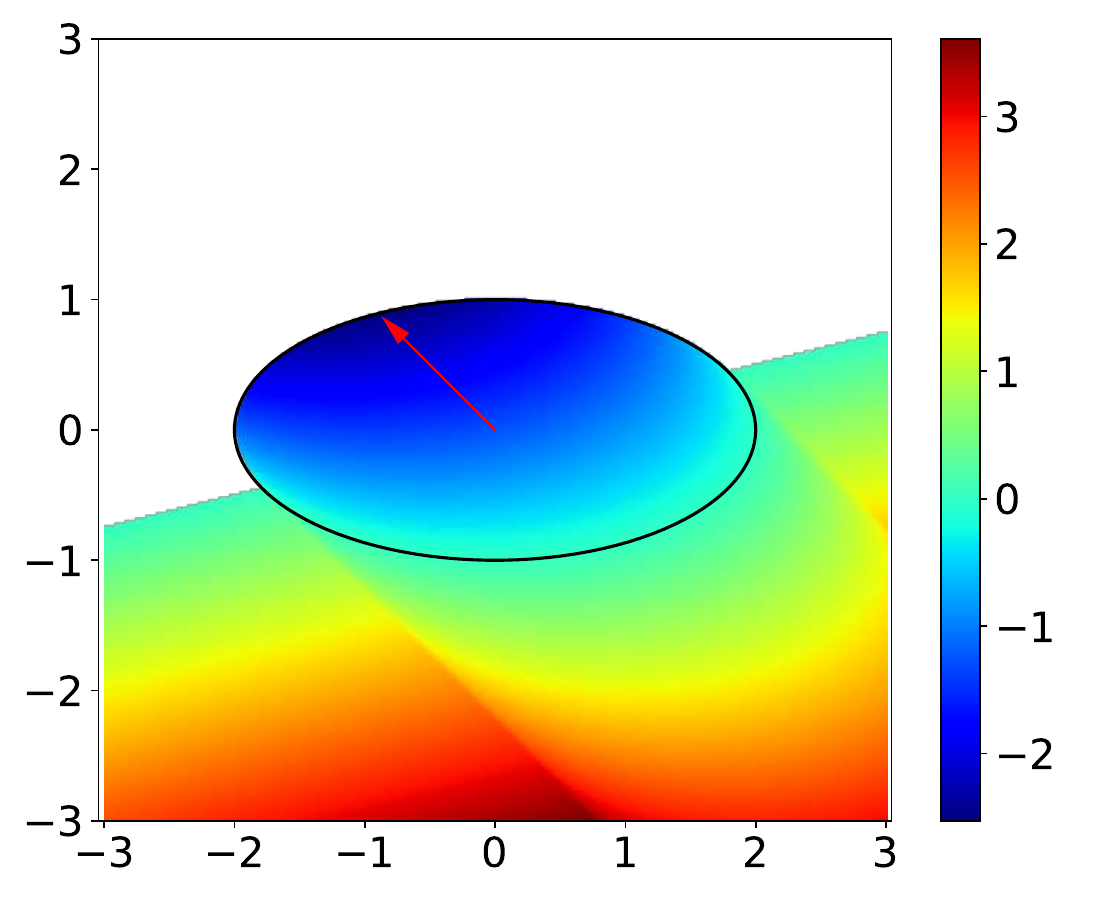}
        \caption{$135^\circ$}
    \end{subfigure}
    \vspace{-0.5ex}
    \caption{2D visualization of the single ellipsoid SDDF $f(\bfp,\bfv; \calE)$ in~\eqref{eq:single_ellipsoid_sddf_prior} for fixed $\bfv$ and varying $\bfp$. According to~\eqref{eq:single_ellipsoid_sddf_prior}, the SDDF prior for a single ellipsoid is finite when $\bfp$ is inside or when the ray intersects the ellipsoid or the virtual plane when $\bfp$ is outside the ellipsoid. Otherwise, the SDDF prior is infinite, indicating that the ellipsoid is behind the ray. The virtual plane is placed at the ellipsoid center with normal vector $\bfQ_1^2\bfv/\left\| \bfQ_1^2\bfv \right\|_2$, which varies with the viewing direction $\bfv$.}
    \vspace{-1.5ex}
    \label{fig:prior_sddf_example}
\end{figure*}

The ellipsoid prior is also compatible with unsigned DDF, which is advantageous when dealing with non-watertight or thin objects where sign is ill-defined. In this case, one of the axis radii may be set to 0 so that the 3D ellipsoid becomes a 2D disk. Additional details are presented in Supplemental~\ref{sec:appdx_ddf_extension}.

\subsection{Fusing Multiple Ellipsoid SDDFs}
\label{sec:fuse_ellipsoid}

To model scenes with multiple objects at different locations, we consider a set of $M$ ellipsoids $\calE_j$ for $1\le j \le M$. Sec.~\ref{sec:ellipsoid} showed how to determine the SDDF $f_j(\bfp,\bfv)$, intersection indicator $i_j(\bfp,\bfv)$, and sign indicator $s_j(\bfp,\bfv)$ for a single ellipsoid in~\eqref{eq:single_ellipsoid_sddf_prior},~\eqref{eq:ellipsoid_intersection} and~\eqref{eq:ellipsoid_sign}, respectively. If $\bfp \in \calE_j$, then $s_j(\bfp,\bfv) \leq 0$. Hence, to determine whether $\bfp$ is contained in any ellipsoid, we can find the minimum $s_j(\bfp,\bfv)$. When a ray $(\bfp, \bfv)$ intersects an ellipsoid, we have $i_j(\bfp,\bfv) \geq 0$. Hence, to determine whether any ellipsoid is intersected, we can find the maximum $i_j(\bfp,\bfv)$. Finally, the SDDF of a union of ellipsoids is equal to the minimum of the individual ellipsoid SDDFs but with the intersected ellipsoids prioritized. In summary, the intersection indicator, the sign indicator, and the SDDF value of a union of $M$ ellipsoids are:
\begin{align}
    i(\bfp,\bfv; \cup_j \calE_j) &= \max_j i_j(\bfp,\bfv), \label{eq:fuse_prior_i} \\
    s(\bfp,\bfv; \cup_j \calE_j) &= \min_j s_j(\bfp,\bfv), \label{eq:fuse_prior_s} \\
    f(\bfp,\bfv; \cup_j \calE_j) &= \!\begin{cases}
        \displaystyle{\min_{j : i_j(\bfp,\bfv) \geq 0}} f_{j}(\bfp,\bfv), & \exists i_j(\bfp,\bfv) \ge 0, \\
        \min_j f_{j}(\bfp,\bfv), & \text{otherwise.}
    \end{cases} \label{eq:fuse_prior_f}
\end{align}

Thus, given $M$ ellipsoids with pose and radii $\bfT_j$, $\bfr_j$ for $1\le j \le M$, we define our ellipsoid-based prior network as:
\begin{equation}\label{eq:prior_net}
    \left(i, s, f, j_f\right) = P\left(\bfp,\bfv \;\big\vert\; \left\lbrace \bfT_j,\bfr_j \right\rbrace_{j=1}^M\right),
\end{equation}
where $j_f$ is the index of the ellipsoid selected by~\eqref{eq:fuse_prior_f}. This ellipsoid index is used to condition the latent feature computation, which will be described in Sec.~\ref{sec:residual_net}. We refer to $P$ in~\eqref{eq:prior_net} as a network because its outputs are differentiable with respect to $\bfp$, $\bfv$, $\bfT_j$, and $\bfr_j$. We provide the analytical gradients in Supplemental~\ref{sec:appdx_bp_ellipsoid_sddf}.

Because our prior network uses ellipsoids to approximate SDDF, it satisfies the Eikonal equation in~\eqref{eq:sddf_eikonal_equation} by construction.

\begin{proposition} \label{prop:prior_eikonal_equation}
    The SDDF $f(\bfp,\bfv;\cup_j \calE_j)$ in~\eqref{eq:fuse_prior_f} computed by the ellipsoid-based prior network satisfies the SDDF directional Eikonal equation in~\eqref{eq:sddf_eikonal_equation}.
\end{proposition}

We provide the proof of Proposition~\ref{prop:prior_eikonal_equation} in Supplemental~\ref{sec:appdx_proof_prop_eikonal_equation}.
Proposition~\ref{prop:prior_eikonal_equation} shows that the SDDF of a single ellipsoid satisfies the Eikonal equation, and the fusion operations in~\eqref{eq:fuse_prior_i}--\eqref{eq:fuse_prior_f} preserve this property.
The network $P$ provides a coarse geometric prior but does not yield accurate  predictions. We address this next.

\section{Residual Network} \label{sec:residual_net}

In this section, we design a residual network to predict a correction term $\delta_f(\bfp,\bfv)$ so that the SDDF prediction of our combined prior and residual model, $\hat{f}(\bfp,\bfv) = f(\bfp,\bfv) + \delta_f(\bfp,\bfv)$, is accurate even for sets with complicated shapes. Intuitively, we expect the prior network $P$ to learn the rough shape of the environment, encoded in $f(\bfp,\bfv)$, while the residual network $R$ captures the details, encoded in $\delta_f(\bfp,\bfv)$.

\subsection{Latent Feature Network}\label{sec:latent_net}

The residual network needs to estimate the SDDF correction term $\delta_f(\bfp,\bfv)$ based on the local shape of the surface that the ray $(\bfp,\bfv)$ interacts with. We convert $\bfp$ and $\bfv$ to the local coordinate frame of the first intersected ellipsoid and train a latent feature network whose output $\bfz$ is decoded into a SDDF correction by a residual decoder network $\delta_f = R(\bfz)$.

Given a ray $(\bfp,\bfv)$, the prior network in~\eqref{eq:prior_net} provides $(i, s, f)$ and the index $j_f$ of the selected ellipsoid. Using the pose of the $j_f$-th ellipsoid $\bfR_{j_f}\in SO(3)$, $\bfc_{j_f} \in \bbR^3$, we obtain the intersection point $\bfq$ in the ellipsoid frame:
\begin{equation}
    \bfq = \bfR_{j_f}^\top(\bfp + f\bfv - \bfc_{j_f}) = \bfp' + f\bfv'.
\end{equation}
We train a latent feature network $\bfz = L(\bfq,\bfv',j_f)$ with the intersection point $\bfq$, the local viewing direction $\bfv'$, and the ellipsoid index $j_f$ as inputs and a latent feature $\bfz$ as output.

We observe that the ellipsoid SDDF in~\eqref{eq:intersection_sol} can be interpreted as a nonlinear function of multiple multivariate polynomials involving $\bfq$ and $\bfv'$, where each variable's maximum degree is $2$. While we use the ellipsoid prior network to exactly compute the SDDF value for ellipsoids, we hypothesize that the residual SDDF correction for complex shapes can also be approximated as a function of similar polynomial terms of $\bfq$ and $\bfv'$, which learns certain implicit features of second-order surfaces suitable for residual correction.
Inspired by this observation, we design the network $L$ to transform a polynomial embedding vector into a latent feature vector:
\begin{align}
    \bfz &= L\left(\bfq,\bfv',j_f \;\big\vert\; \{\bfW_i\}_{i=1}^M \right) = \bfW_{j_f} \bfm \in \bbR^m, \label{eq:polynomial_embedding}\\
    \bfm &= \operatorname{vec}\left(E(\bfq) E(\bfv')^\top\right) \in \bbR^{100}, \label{eq:pv_monomial_gen}\\
    E(\bfp) &= \begin{bmatrix}
        p_x^2, p_x p_y, p_x p_z, p_y^2, p_y p_z, p_z^2, p_x, p_y, p_z, 1
    \end{bmatrix}^\top, \label{eq:monomial_gen}
\end{align}
where $E: \bbR^3 \rightarrow \bbR^{10}$ is a degree-$2$ monomial embedding, $\operatorname{vec}(\cdot)$ concatenates the columns of the input matrix, $\bfm$ is a vector of degree-$2$ monomials, and $\bfW_i \in \bbR^{m \times 100}$.
The expression in~\eqref{eq:monomial_gen} constructs a $10$-dimensional vector of all degree-$2$ monomials of the input $\bfq \in \bbR^3$ and $\bfv' \in SO(3)$ correspondingly. Then,~\eqref{eq:pv_monomial_gen} uses the outer product of the two monomial vectors to generate a $100$-dimensional vector of degree-$2$ monomials involving both $\bfq$ and $\bfv'$. Finally,~\eqref{eq:polynomial_embedding} uses a learnable weight matrix $\bfW_{j_f}$ specific to the ellipsoid index $j_f$ to transform the vector of monomials $\bfm$ into a latent feature vector $\bfz \in \bbR^m$.

\subsection{Residual Decoder} \label{sec:residual_decoder}

The residual decoder is a multi-layer perceptron $R: \bbR^m \rightarrow \bbR^3$ that decodes the latent feature vector $\bfz \in \bbR^m$ into three residual predictions $(\delta_i, \delta_s, \delta_f)$. Then, the final predictions of the intersection indicator, sign indicator, and SDDF value are
\begin{align}
    \hat{i}(\bfp,\bfv) &= \tanh\left(\alpha i(\bfp,\bfv) \right) + \delta_i(\bfp,\bfv), \label{eq:final_intersection_indicator} \\
    \hat{s}(\bfp,\bfv) &= \tanh\left(\alpha s(\bfp,\bfv) \right) + \delta_s(\bfp,\bfv), \label{eq:final_sign_indicator} \\
    \hat{f}(\bfp,\bfv) &= f(\bfp,\bfv) + \delta_f(\bfp,\bfv),  \label{eq:final_sddf}
\end{align}
where $\alpha > 0$ is a hyperparameter. Since the prior network produces $i(\bfp,\bfv) \in \bbR$ and $s(\bfp,\bfv) \in \bbR$ but, as described later in Sec.~\ref{sec:dataset}, we provide supervision $i^*(\bfp,\bfv) \in \{-1,1\}$ and $s^*(\bfp,\bfv)\in \{-1,1\}$ for~\eqref{eq:final_intersection_indicator} and~\eqref{eq:final_sign_indicator}, and use $\tanh$ to squash the output of the prior intersection indicator and sign indicator. The residuals $\delta_i$ and $\delta_s$ are applied after the squashing to prevent clipping their gradients during training.

Complementing Proposition~\ref{prop:prior_eikonal_equation}, we show that the joint prior-residual SDDF prediction $\hat{f}$ in~\eqref{eq:final_sddf} still satisfies the SDDF directional Eikonal equation by construction.

\begin{proposition} \label{prop:eikonal_equation_with_residual}
    The SDDF $\hat{f}(\bfp,\bfv)$ in~\eqref{eq:final_sddf} computed by the combination of the prior and the residual networks satisfies the SDDF directional Eikonal equation in~\eqref{eq:sddf_eikonal_equation}.
\end{proposition}

The proof of Proposition~\ref{prop:eikonal_equation_with_residual} is presented in Supplemental~\ref{sec:appdx_proof_prop_eikonal_equation_with_residual}. The intuition behind the proof is that the input position $\bfq$, which is the intersection point between the ray and the ellipsoid, does not change when $\bfp$ moves along the ray direction $\bfv$ as long as the same ellipsoid is selected by the prior network. Therefore, the residual SDDF correction $\delta_f(\bfp,\bfv)$ remains unchanged when $\bfp$ moves along $\bfv$, leading to a zero directional derivative along $\bfv$.
Because our SDDF model satisfies the Eikonal equation by construction, we do not need an extra loss term to regularize the network and can use fewer parameters in the model, making it more efficient to train and evaluate.

\section{Training} \label{sec:training}
\subsection{Dataset Generation and Augmentation}\label{sec:dataset}

We convert the sensor poses and range measurements $\{\bfT_t, \calZ_t\}_{t=1}^T$, $\calZ_t = \{\bfv_{i}, r_{t,i}\}_{i=1}^N$, described in Sec.~\ref{sec:method_overview}, into a dataset $\calD=\{\bfp_j,\bfv_j,f_j^*,i_j^*,s_j^*\}_j$ suitable for training our SDDF model. Here, $\bfp_j\in\bbR^3$ is the origin of the ray, $\bfv_j\in\bbS^2$ is the direction of the ray provided in $\calZ_t$, $f_j^*$ is an SDDF measurement, $i_j^* \in \{-1,1\}$ is an intersection indicator measurement, and $s_j^* \in \{-1,1\}$ is a sign indicator measurement. From each ray measurement $\bfv_i, r_{t,i}$, obtained from sensor position $\bfp_t$ and orientation $\bfR_t$, we generate a corresponding data sample as $\bfp_j=\bfp_t$, $\bfv_j=\bfR_t\bfv_i$, $f_j^*=r_{t,i}$, $i_j^*=1$, $s_j^*=1$.
These data are unbalanced in the sign indicator and the SDDF sign because all the samples have $s_j^*=1$ and $f_j^* \ge 0$.
Therefore, we need to augment the data with negative samples.
For each sample $(\bfp_j,\bfv_j,f_j^*,i_j^*,s_j^*)$, we generate a corresponding negative sample $(\bfp',\bfv_j,-\epsilon,1,-1)$, where we extend the ray to a point $\bfp' = \bfp_j + (f_j^* + \epsilon)\bfv_j$ slightly behind the observed surface with a small offset $\epsilon > 0$.
Although the intersection indicator $i_j^*$ is also unbalanced, it is inefficient to generate non-intersecting rays, and our experiments show that augmenting the intersection indicator does not improve our reconstruction results.

\subsection{Ellipsoid Initialization} \label{sec:ellipsoid_initialization}

We present an initialization strategy for the $M$ ellipsoids, needed by the ellipsoid-based prior network $P$ in Sec.~\ref{sec:geometry_prior_net}. Using the augmented dataset $\calD$ constructed above, we obtain a point cloud:
\begin{equation*}
    \calX = \{\bfx_k\}_{k=1}^K = \{\bfp_j+f_j^*\bfv_j \mid f_j^* \geq 0\}_j \cup \{\bfp_j \mid f_j^* < 0\}_j,
\end{equation*}
which is the collection of surface points and points inside obstacles.
Leaving the number of ellipsoids $M$ as a hyperparameter, we use K-means++~\cite{kmeanspp2007} to divide $\calX$ into $M$ clusters $\{\calX_m\}_{m=1}^M$. For each cluster, we initialize an ellipsoid using principal component analysis of the points in the cluster as shown in Alg.~\ref{alg:single_ellipsoid_init}. In practice, we set a minimum allowed radius $r_{\min}=0.005$ to avoid numerical problems and scale the ellipsoids by $\alpha=3$ to ensure that the points $\calX_m$ are mostly covered by the $m$-th ellipsoid.

\begin{algorithm}[t]
\caption{Single Ellipsoid Initialization}\label{alg:single_ellipsoid_init}
\small
\begin{algorithmic}[1]
\Procedure{SingleEllipsoidInit}{$\calX = \{\bfx_k\}_{k=1}^K$}
    \State $\bfc \leftarrow \frac{1}{K}\sum_{k=1}^K \bfx_i$ \Comment{center}
    \State $\bfX \leftarrow \begin{bmatrix}\bfx_1 \cdots \bfx_K\end{bmatrix}^\top - \mathbf{1}\bfc^\top$
    \State Eigen decomposition on $\frac{1}{K}\bfX^\top \bfX = \bfQ \bfLambda \bfQ^{-1}$
    \State $\bfR \leftarrow \bfQ \diag{\begin{bmatrix}1 \cdots \det{\bfQ}\end{bmatrix}}$ \Comment{rotation}
    \State $\begin{bmatrix}\lambda_1 \cdots \lambda_n\end{bmatrix} \leftarrow \diag{\bfLambda}$ \Comment{$n$ is the space dimension}
    \State $r_i \leftarrow \max\left(r_{\min}, \alpha \sqrt{|\lambda_i|}\right)$\Comment{radii}
    \State \Return $\bfR, \bfc, \bfr$
\EndProcedure
\end{algorithmic}
\end{algorithm}

\begin{figure}[t]
    \centering
    \begin{subfigure}[t]{0.48\linewidth}
        \centering
        \includegraphics[width=\linewidth,trim={60pt 50pt 60pt 50pt},clip]{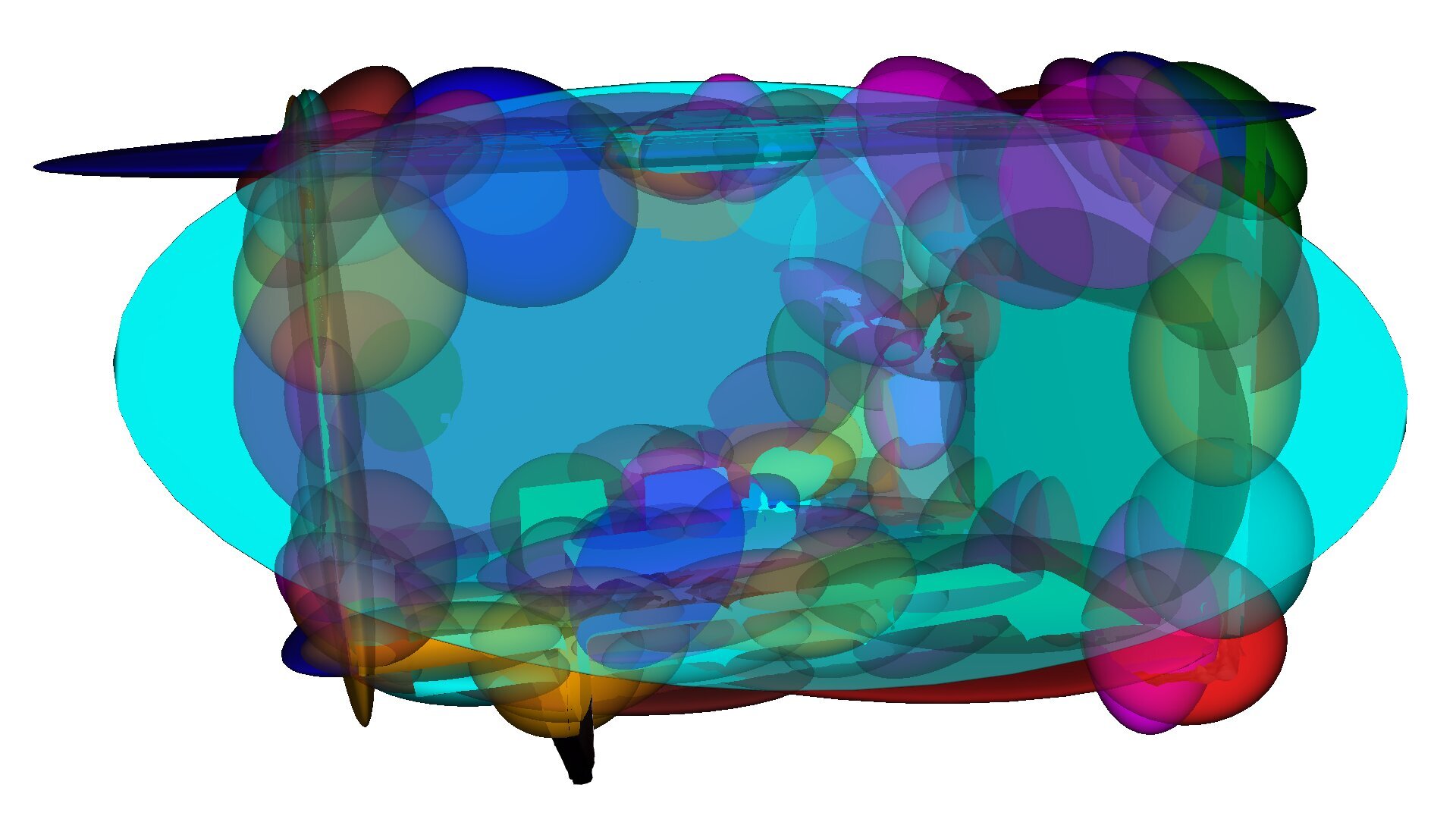}
        \caption{By Alg.~\ref{alg:multi_ellipsoid_init}} \label{fig:ellipsoid_init_a}
    \end{subfigure}
    \begin{subfigure}[t]{0.48\linewidth}
        \centering
        \includegraphics[width=\linewidth,trim={60pt 50pt 60pt 50pt},clip]{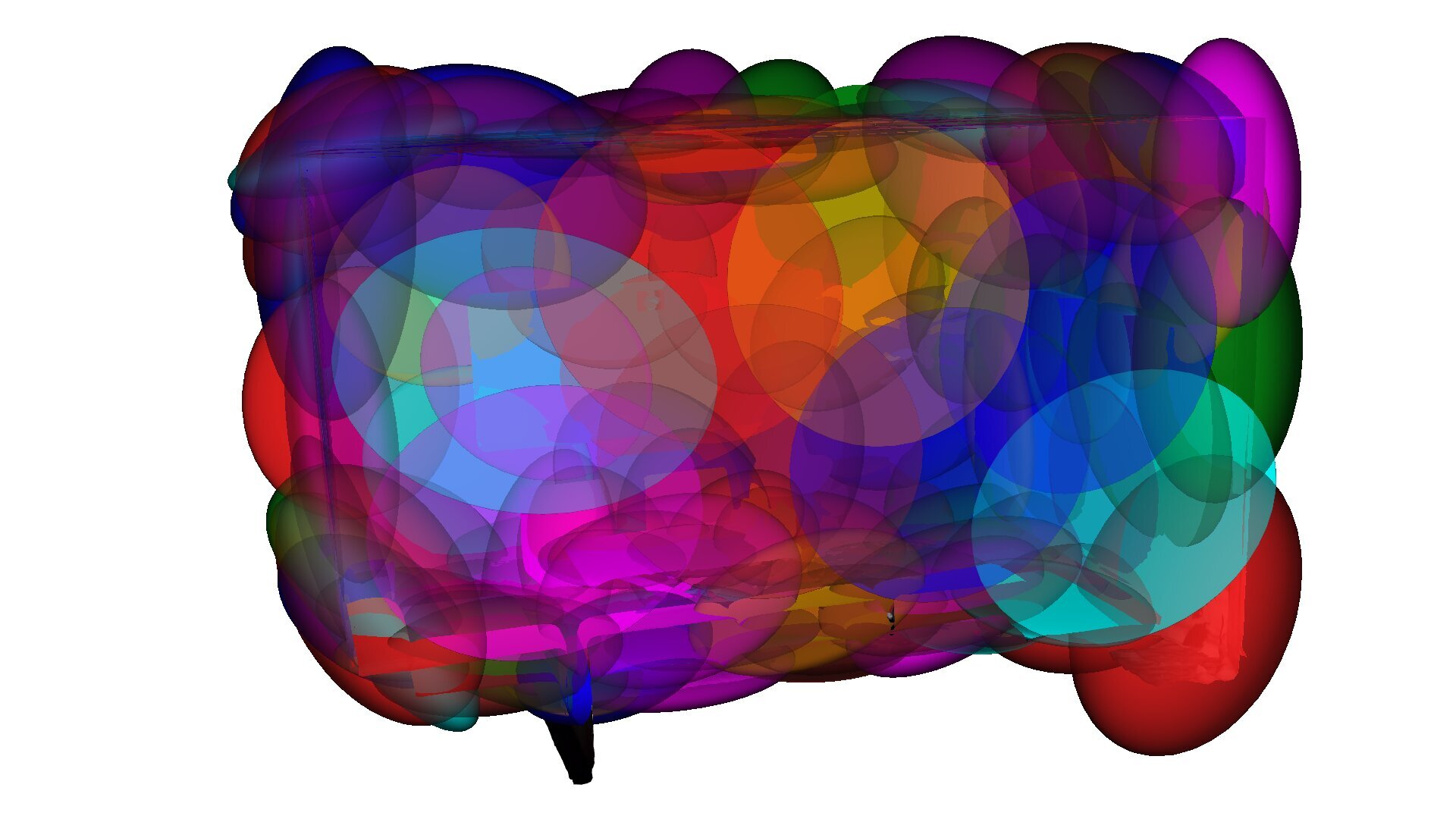}
        \caption{By K-means++~\cite{kmeanspp2007} only} \label{fig:ellipsoid_init_b}
    \end{subfigure}
    \begin{subfigure}[t]{0.48\linewidth}
        \centering
        \includegraphics[width=\linewidth]{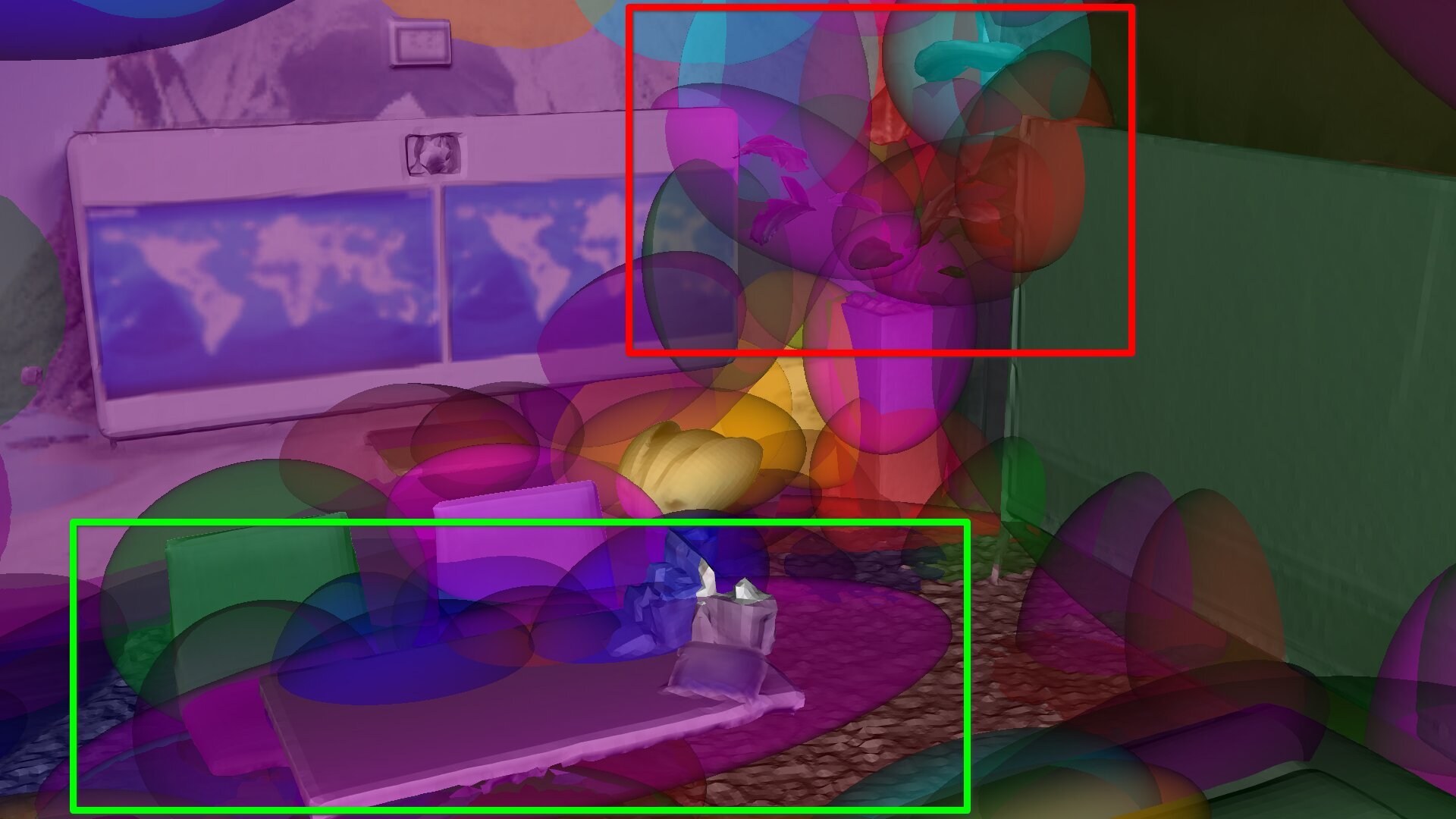}
        \caption{Closer view of (a)} \label{fig:ellipsoid_init_c}
    \end{subfigure}
    \begin{subfigure}[t]{0.48\linewidth}
        \centering
        \includegraphics[width=\linewidth]{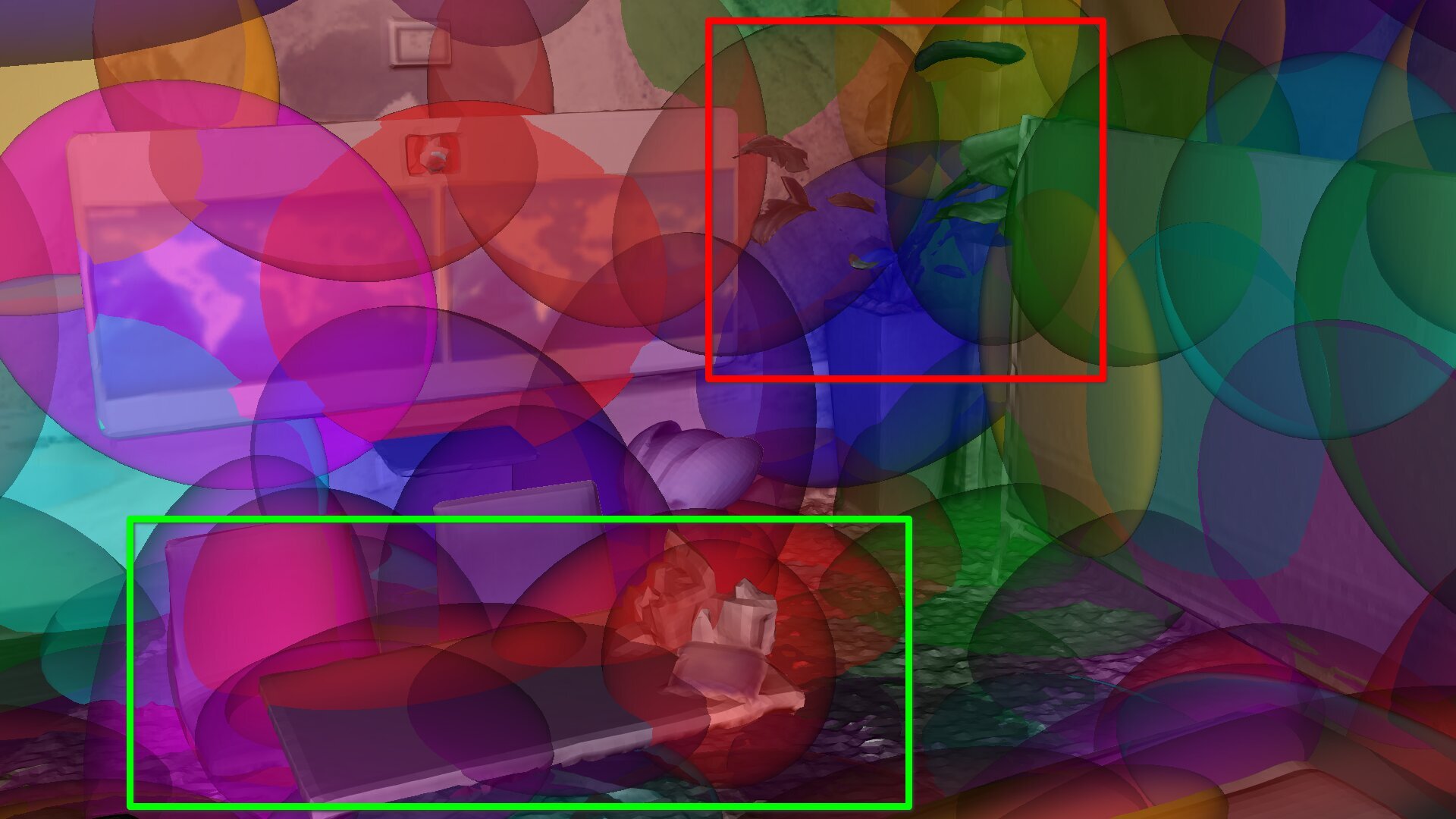}
        \caption{Closer view of (b)}
    \end{subfigure}
    \caption{Comparison of ellipsoid initialization algorithms. The left column is generated by Alg.~\ref{alg:multi_ellipsoid_init} and the right column is by the K-means++~\cite{kmeanspp2007} algorithm only (i.e., L3 to L5 of Alg.~\ref{alg:multi_ellipsoid_init}). (a) Using Alg.~\ref{alg:multi_ellipsoid_init}, a few ellipsoids are used to approximate planar surfaces like ceiling, wall, and ground. (b) Using K-means++~\cite{kmeanspp2007} only, too many ellipsoids are used to approximate planar surfaces. (c) and (d) show close-ups of the indoor objects. The table (in the {\color{green}green} box) and the plant (in the {\color{red}red} box) are better approximated by ellipsoids from Alg.~\ref{alg:multi_ellipsoid_init}.}
    \label{fig:ellipsoid_init}
\end{figure}

However, in typical indoor environments, as shown in Fig.~\ref{fig:ellipsoid_init},  this approach allocates too many ellipsoids to planar surfaces, which could be fit by a single flat ellipsoid. Hence, we propose Alg.~\ref{alg:multi_ellipsoid_init} to create a better multi-ellipsoid initialization. After using K-means++~\cite{kmeanspp2007} to divide $\calX$ into $M$ clusters and create an ellipsoid $\left(\bfR_m,\bfc_m,\bfr_m\right)$ for each cluster $m$ by Alg.~\ref{alg:single_ellipsoid_init}, we build an undirected graph $\calG=(\calV,\calE)$ where each node is a flat ellipsoid $i\in\calV$ (i.e., the mean projection $\beta$ of $\bfx_k \in \calX_m$ to the shortest axis is smaller than threshold $\beta_{\max}$), and two nodes $(i, j)$ are connected only when they are coplanar neighbors.
Then, we find the number of ellipsoids, $C$, for these planar surfaces by finding the connected components $\{\calV_i\}_{i=1}^C$ of $\calG$. We merge the ellipsoids in each component $\calV_i$ running Alg.~\ref{alg:single_ellipsoid_init} on $\bigcup_{m\in\calV_i}\calX_m$. Finally, the remaining points $\bigcup_{m\not\in\calV}\calX_m$ are divided into $M-C$ clusters by K-means++~\cite{kmeanspp2007} and the ellipsoid of each cluster is initialized by Alg.~\ref{alg:single_ellipsoid_init}.
As shown in Fig.~\ref{fig:ellipsoid_init}, the planar surfaces are correctly detected and approximated by only a few ellipsoids, leaving more ellipsoids for other more complex shapes.

\begin{algorithm}[t]
\caption{Multi-Ellipsoid Initialization}\label{alg:multi_ellipsoid_init}
\small
\begin{algorithmic}[1]
\Procedure{MultiEllipsoidInit}{$\calX = \{\bfx_k\}_{k=1}^K, M, S$}
    \State \Comment{$S$ is the number of ellipsoid neighbors}
    \State $\{\calX_m\}_{m=1}^M \leftarrow \Call{K-means++}{\calX, M}$
    \For{$m=1\cdots M$}
        \State $\bfR_m, \bfc_m, \bfr_m \leftarrow$\Call{SingleEllipsoidInit}{$\calX_m$}
        \State $j \leftarrow \arg\min_{1\le j\le n} r_{m,j}$
        \State $\bfn_m \leftarrow \bfR_m \bfe_j$
        \State $\beta_m \leftarrow \frac{1}{|\calX_m|}\sum_{\bfp \in \calX_m} \bfn_m^\top \left(\bfp - \bfc_m\right)$
    \EndFor
    \State $\calV \leftarrow \emptyset, \calE \leftarrow \emptyset$ \Comment{undirected graph}
    \For{$i=1\cdots M, \beta_i < \beta_{\max}$}\Comment{flat ellipsoid}
        \State $\calK_i \leftarrow$ \Call{K-NearestNeighbor}{$\{\bfc_m\}_{m=1}^M, i, S$}
        \For{$j \in \calK_i, \beta_i < \beta_{\max}$}
            \State $\eta \leftarrow \frac{1}{2}\left(|(\bfc_i-\bfc_j)^\top\bfn_i| + |(\bfc_i-\bfc_j)^\top\bfn_j|\right)$
            \If{$\eta < \eta_{\max}$} \Comment{$i$ and $j$ are coplanar}
                \State $\calV \leftarrow \calV \cup \{i, j\}$ and $\calE \leftarrow \calE \cup \{(i, j)\}$
            \EndIf
        \EndFor
    \EndFor
    \State $\{\calV_i\}_{i=1}^C \leftarrow$ \Call{FindConnectedComponents}{$\calV$, $\calE$}
    \State $\calR \leftarrow \emptyset, \calX' \leftarrow \emptyset$
    \For{$i=1\cdots C$} \Comment{merge ellipsoids}
        \State $\bfR, \bfc, \bfr \leftarrow$ \Call{SingleEllipsoidInit}{$\bigcup_{m\in\calV_i}\calX_m$}
        \State $\calR \leftarrow \calR \cup \{(\bfR,\bfc,\bfr)\}$
        \State $\calX' \leftarrow \calX' \cup \bigcup_{m\in\calV_i}\calX_m$
    \EndFor
    \State $\{\calX'_m\}_{m=1}^{M-C} \leftarrow$ \Call{K-means++}{$\calX-\calX'$, $M-C$}
    \For{$m=1\cdots M-C$} \Comment{initialize other ellipsoids}
        \State $\bfR, \bfc, \bfr \leftarrow$ \Call{SingleEllipsoidInit}{$\calX'_m$}
        \State $\calR \leftarrow \calR \cup \{(\bfR,\bfc,\bfr)\}$
    \EndFor
    \State \Return $\calR$ \Comment{$M$ ellipsoid poses and radii}
\EndProcedure
\end{algorithmic}
\end{algorithm}

\subsection{Loss Function for Ellipsoid-based Prior Network}\label{sec:loss_function_for_prior}

Given a prediction $\bfx=(i,s,f)$ from the ellipsoid-based prior network $P$ and the corresponding label $\bfx^*=(i^*,s^*,f^*)$ from the measurement data $\calD$, we compute a weighted sum of Huber losses:
\begin{equation} \label{eq:loss_function}
\begin{aligned}
    &l_0(x,y)=\left\{\begin{array}{ll}
        0.5(x-y)^2, & \text{if}\ |x-y|<1, \\
        |x-y|-0.5,  & \text{otherwise},
        \end{array}\right. \\
    &l(x, y) = \left(w^+\mathbbm{1}(y\ge0) + w^-\mathbbm{1}(y<0)\right) l_0(x, y),\\
    &\calL_P(\bfx,\bfx^*) = l(i,i^*) + l(s,s^*) + l(f,f^*).
\end{aligned}
\end{equation}

The role of $P$ is to make the ellipsoids cover the objects in the scene so that the occlusions between objects are captured. Thus, larger weights are used for negative $f^*$ and negative $s^*$: $w_f^-=1.65$ and $w_s^-=10$. The other weights are set to $1$.

\subsection{Loss Function for Residual Network}\label{sec:loss_function_for_residual}

The residual network is trained using the same loss function as the prior network in~\eqref{eq:loss_function} but with different weights.
Unlike the prior network, the residual network focuses on learning the geometric details.
Thus, higher weights are used for the SDDF prediction $\hat{f}$.
Given the prediction $\hat{\bfx}=(\hat{i},\hat{s},\hat{f})$ and the corresponding label $\bfx^*=(i^*,s^*,f^*)$ from the measurement data $\calD$, we have
\begin{equation}
\begin{aligned}
    &\calL_R(\hat{\bfx},\bfx^*) = l(\hat{i},i^*) + l(\hat{s},s^*) +l(\hat{f},f^*), \\
    &\calL(\hat{\bfx},\bfx,\bfx^*) = \calL_P(\bfx,\bfx^*) + \calL_R(\hat{\bfx},\bfx^*),
\end{aligned}
\end{equation}
where for $\calL_R$, $w_{\hat{f}}^+=1.0,w_{\hat{f}}^-=1.1$, and other weights $0.1$.

\section{Application to Viewpoint Optimization}
\label{sec:viewpoint_optimization}

Our SDDF model is differentiable and, hence, enables continuous viewpoint optimization. This property is useful for several applications, e.g., to guide navigation in virtual reality or to enable a robot to explore an unknown environment. We first consider determining the next-best view and then scale up to optimization of a trajectory of several views. Our SDDF model can predict a point cloud measurement from any desired sensor pose  $(\bfp_t,\bfR_t)$ as:
\begin{equation}
    \calP_t = \{\bfp_t + \hat{f}_i \bfR_t \bfv_i\}_{i=1}^N,
\end{equation}
where $\{\bfv_i\}_{i=1}^N$ are the ray directions in the sensor frame and $\hat{f}_i$ are the SDDF predictions for each ray. The utility of a point-cloud measurement for the purpose of exploration or environment coverage can be evaluated in terms of the visible region volume. We use the following loss to measure the (negative) size of the visible volume:
\begin{equation}
    \calL_v\left(\{\hat{f}_i\}_{i=1}^N\right) = -\frac{1}{2N} \sum_{i=1}^N \left(\max\{\hat{f}_i, 0\}\right)^2.
\end{equation}
Moreover, consecutive measurements $\calP_t$ and $\calP_{t+1}$ should have a small overlap in order to observe a large area. Therefore, we design the following overlap loss:
\begin{equation} \label{eq:overlap_loss}
    \calL_o \left(\calP_t,\calP_{t+1}\right) = -\frac{\sum_{\bfp \in \calP_t, \bfq \in \calP_{t+1}} \min\{\left\|\bfp - \bfq\right\|_2, d_{\max}\}}{|\calP_t||\calP_{t+1}|},
\end{equation}
where $d_{\max} > 0$ is a distance threshold of no overlap.

In practice, we must also ensure that the sensor is not in collision with any obstacles. To do so, we introduce a set of risk detection rays, which are uniformly sampled from the sphere that contains the robot, and obtain their SDDF predictions $\{\hat{f}^r_i\}_{i=1}^M$.
The risk loss is defined as:
\begin{equation}
    \calL_r\left( \{ \hat{f}_i^r \}_{i=1}^M \right) = \frac{1}{M}\sum_{i=1}^M \max{\{d_\text{safe} - \hat{f}^r_i, 0\}},
\end{equation}
where $d_\text{safe} > 0$ is a safe distance threshold, chosen such that $\hat{f}^r_i < d_\text{safe}$ implies a potential collision.

We optimize the camera pose $(\bfp_{t+1},\bfR_{t+1})$ at time $t+1$ by minimizing the weighted sum $\calL = w_o \calL_o + w_v \calL_v + w_r \calL_r$, where $w_o$, $w_v$, and $w_r$ are weights.

Considering multi-view trajectory optimization, it is inefficient to optimize every pose in a continuous trajectory because two views that are close to each other are likely to overlap significantly, i.e., with minimal change in the viewpoint. Therefore, we consider optimizing certain waypoints on the trajectory. We incrementally optimize $n$ poses $\{\bfp_i,\bfR_i\}_{i=0}^n$ generated by an off-the-shelf planning algorithm, such as RRT*~\cite{rrtstar2011}, such that for each pose $\{\bfp_i, \bfR_i\}_{i > 0}$, we optimize the loss:
\begin{equation}
    \calL' = w_o \calL_o\biggl(\calP_0 \bigcup_{j=1}^{j<i}\calP'_j, \calP_i\biggr) + w_v \calL_v + w_r \calL_r,
\end{equation}
where $\calP'_j$ is the predicted point cloud at the optimized waypoint $(\bfp'_j, \bfR'_j)$. During the optimization, we downsample $\calP_0 \bigcup_{j=1}^{j<i}\calP'_j$ with a stride of $i-j$, labeled as $\Tilde{\calP}_j$, such that $\Tilde{\calP}_j$ has a constant size.
This incremental optimization strategy provides two benefits.
First, it reduces the use of GPU memory and makes the along-trajectory multi-view optimization problem solvable. Second, it allows to parallelize the trajectory optimization and execution.

\section{Evaluation}
\label{sec:evaluation}

We evaluate our SDDF model in terms of accuracy, computational efficiency, and usefulness for viewpoint optimization.
First, we compare against three state-of-the-art baselines on synthetic and real datasets.
Then, an ablation study examines the effectiveness of each module in our model by varying the number of ellipsoids, data augmentation and initialization strategies, and the use of polynomial embedding.
Finally, we demonstrate the utility of SDDF in viewpoint optimization.

\subsection{Comparison with Baselines}
\textbf{Baselines:} We compare our method against three baselines: Nerfacto (with and without depth loss, check Supplemental~\ref{sec:appdx_training_of_nerfacto} for details)~\cite{nerfstudio2023}, RaDe-GS~\cite{radegs2024}, and \sdfbaseline~\cite{instantngp2022}.
These methods were not initially designed for SDDF prediction but can be used to predict SDDF as follows.
For SDF, we implement sphere tracing~\cite{spheretracing1996} to find the closest point on the surface along the query direction.
For Nerfacto and RaDe-GS, we render a depth image at the query viewpoint, project the depth image to a point cloud, and compute the distance.
We compare the mean absolute error (MAE) in SDDF prediction and the computational cost in terms of GPU memory, model size, and inference time.

\textbf{Datasets:}
A total of 15 datasets (14 synthesized and 1 real) are used for comparison.
We obtained data from six scenes from Replica (``Hotel'' and ``Office 0--5'')~\cite{replica19arxiv}, the Allensville scene from Gibson~\cite{gibson2018}, and scene 0000-00 from ScanNet~\cite{dai2017scannet}.

From the seven synthetic scenes, we synthesize two sensor types, LiDAR and RGB-D, for a total of 14 datasets.
For both sensor types, measurements are synthesized from a grid of free space positions in the scene.
LiDAR data is synthesized with 360 horizontal and 180 vertical angular increments, ranging between $[-\pi,\pi) \times [-\pi/2,\pi/2]$.
RGB-D data is synthesized from six uniform camera orientations, with an image resolution of $640 \times 480$, and horizontal and vertical fields of view of $94^\circ \times 77^\circ$.
For each scene, we sample 20 random test viewpoints (40 for Allensville), which were visually confirmed to be representative of possible camera poses in real applications.

ScanNet~\cite{dai2017scannet} provides real RGB-D images captured along a trajectory.
During training, the RGB images are linearly interpolated to align with the depth images of resolution $640\times 480$.
We found that the quality of the reference mesh from ScanNet is insufficient for generating SDDF ground truths; thus, we used this dataset for qualitative comparison only.
The views for comparison are generated by randomly selecting and perturbing waypoints of the trajectory ($\calU_{[-0.2, 0.2]}$ for each dimension of the translation and $\calU_{[-\pi,\pi]}$ for the yaw angle) so that the rendered views are representative.

\textbf{Experiment Setup:} The LiDAR and RGB-D datasets are used as follows.
For our method, we consider both the LiDAR and RGB-D cases.
When using the RGB-D datasets with our method, we only use the depth images and apply $20\times$ downsampling ($5\times$ for Replica Hotel) to achieve a similar ray density as the LiDAR dataset.
Similarly, for \sdfbaseline, we consider both LiDAR and RGB-D datasets and use all range measurements in either the LiDAR or depth images, along with additional data augmentations of near-surface samples and free-space samples necessary for stable training~\cite{instantngp2022}.
For Nerfacto and RaDe-GS, we only consider the RGB-D case, as they do not support rendering LiDAR data.
RaDe-GS is trained with and without RGB (i.e., either depth-only or RGB-D), whereas Nerfacto is only trained with RGB-D, as RGB is necessary for training Nerfacto.
RaDe-GS is trained with $2\times$ downsampling due to GPU memory limitations, while Nerfacto is trained using the full dataset.

Our method uses 128 ellipsoids (256 for Gibson Allensville) to fit an SDDF prior of the scene.
In contrast, RaDe-GS uses $5\%$ of the range measurements ($3\%$ for Allensville due to GPU memory limits) to initialize one Gaussian for each measurement, which yields many more Gaussians than our ellipsoids. Further training details are provided in Supplemental~\ref{sec:appdx_training_details}.

\begin{figure}[t]
    \centering
    \begin{subfigure}[t]{0.24\linewidth}
        \centering
        \includegraphics[width=\linewidth]{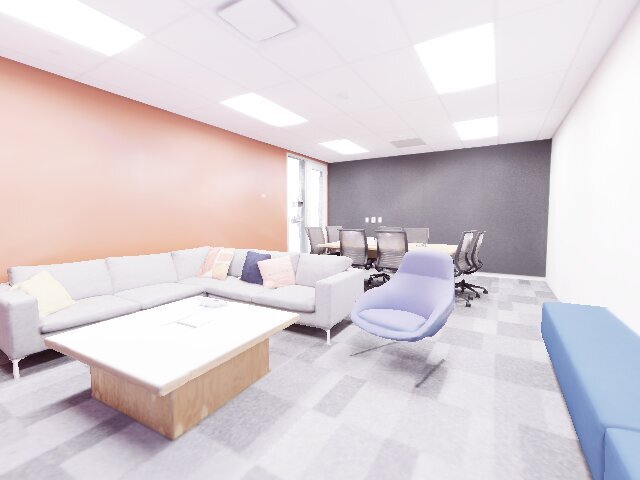}
    \end{subfigure}
    \begin{subfigure}[t]{0.24\linewidth}
        \centering
        \includegraphics[width=\linewidth]{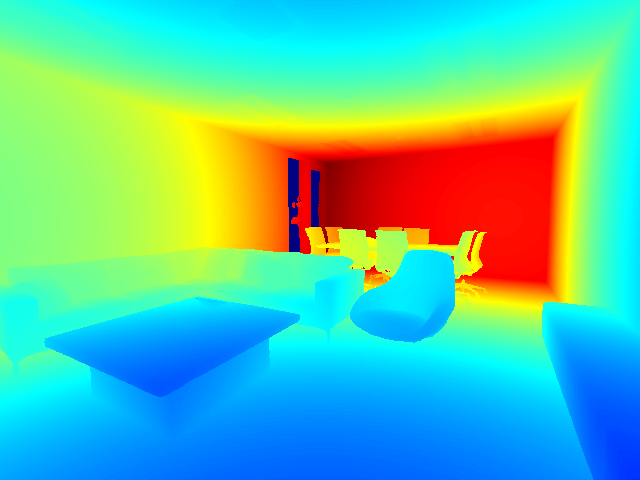}
    \end{subfigure}
    \begin{subfigure}[t]{0.24\linewidth}
        \centering
        \includegraphics[width=\linewidth]{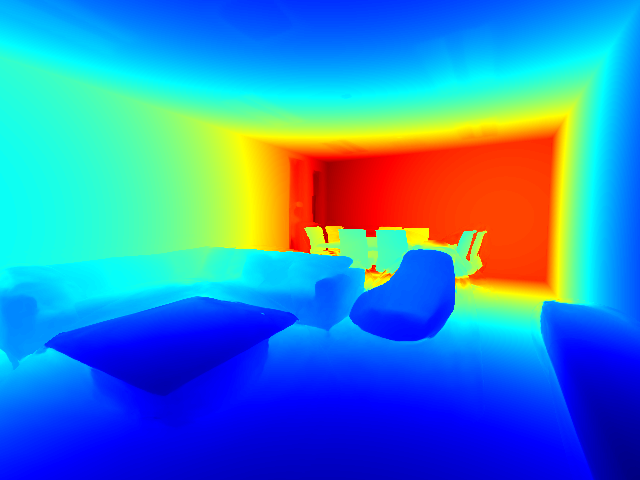}
    \end{subfigure}
    \begin{subfigure}[t]{0.24\linewidth}
        \centering
        \includegraphics[width=\linewidth]{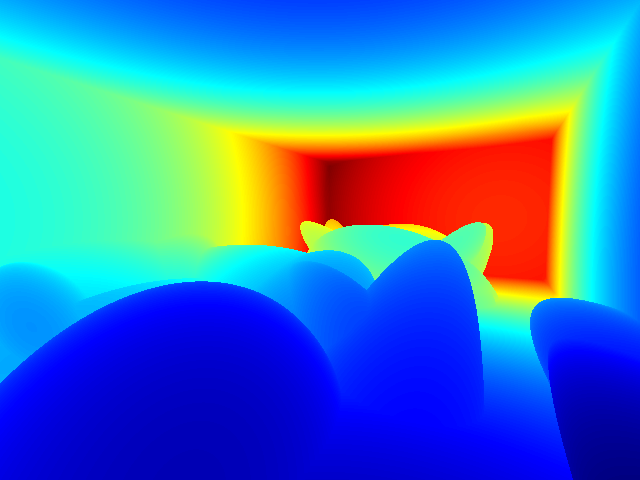}
    \end{subfigure}
    \begin{subfigure}[t]{0.24\linewidth}
        \centering
        \includegraphics[width=\linewidth]{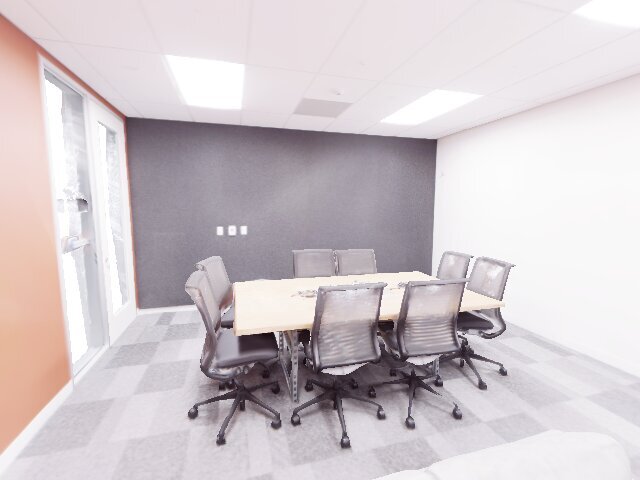}
    \end{subfigure}
    \begin{subfigure}[t]{0.24\linewidth}
        \centering
        \includegraphics[width=\linewidth]{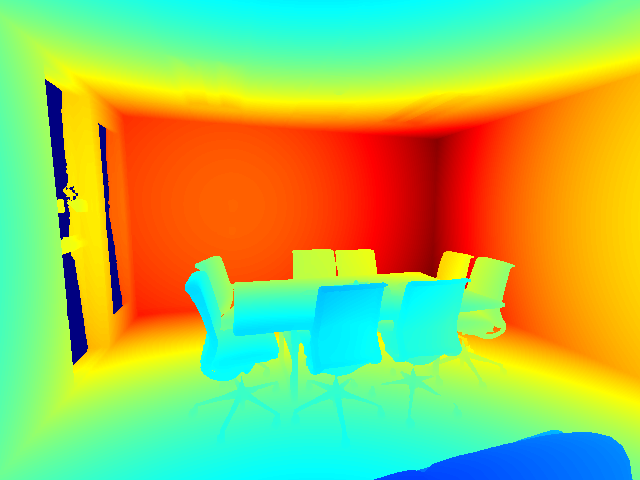}
    \end{subfigure}
    \begin{subfigure}[t]{0.24\linewidth}
        \centering
        \includegraphics[width=\linewidth]{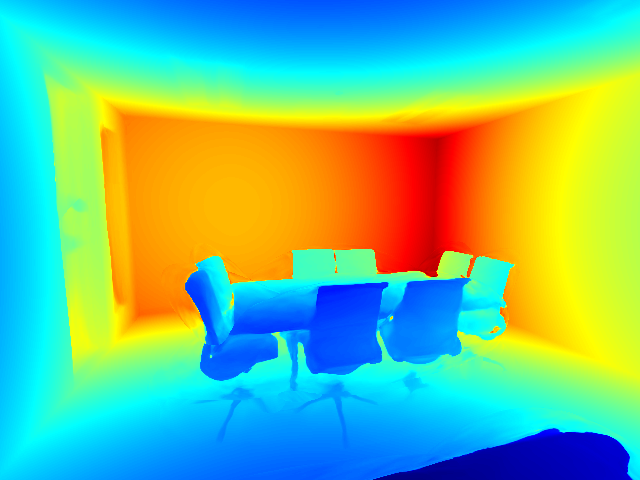}
    \end{subfigure}
    \begin{subfigure}[t]{0.24\linewidth}
        \centering
        \includegraphics[width=\linewidth]{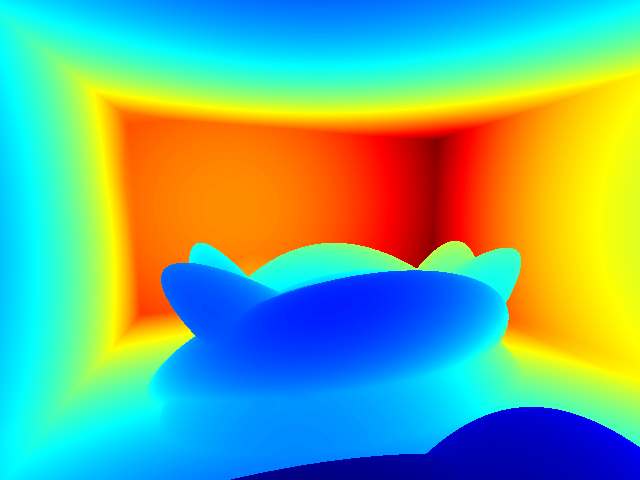}
    \end{subfigure}
    \begin{subfigure}[t]{0.24\linewidth}
        \centering
        \includegraphics[width=\linewidth]{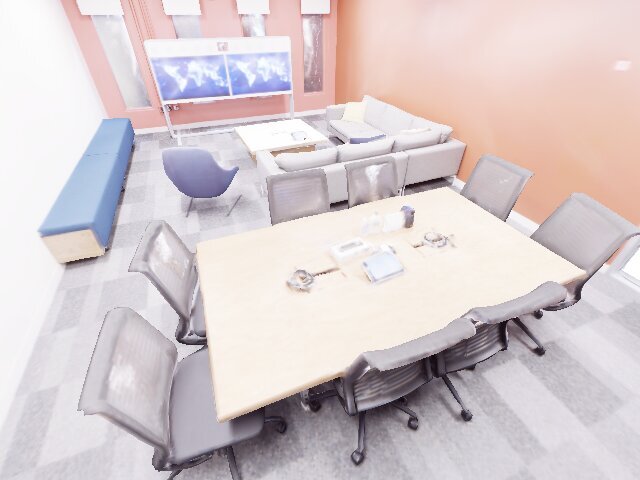}
        \caption{Scene}
    \end{subfigure}
    \begin{subfigure}[t]{0.24\linewidth}
        \centering
        \includegraphics[width=\linewidth]{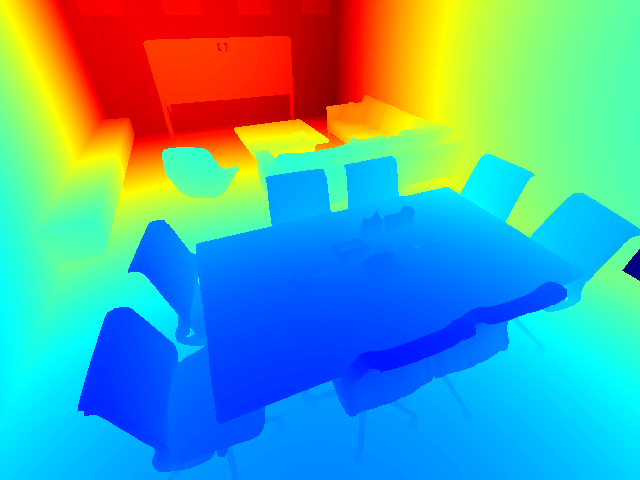}
        \caption{SDDF G.T.}
    \end{subfigure}
    \begin{subfigure}[t]{0.24\linewidth}
        \centering
        \includegraphics[width=\linewidth]{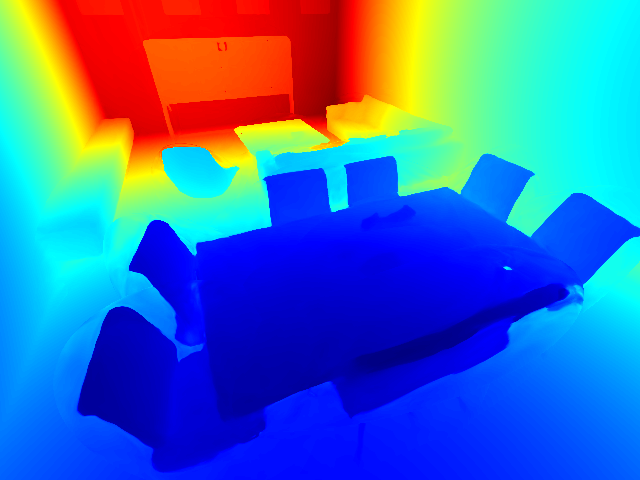}
        \caption{Prediction} \label{fig:sddf_prediction}
    \end{subfigure}
    \begin{subfigure}[t]{0.24\linewidth}
        \centering
        \includegraphics[width=\linewidth]{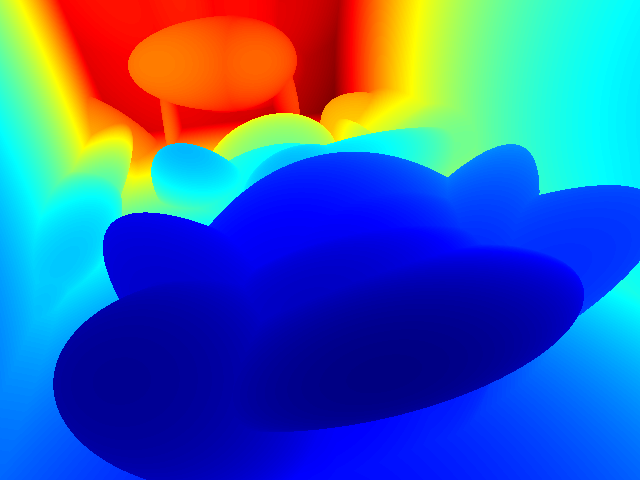}
        \caption{SDDF Prior} \label{fig:sddf_prior_prediction}
    \end{subfigure}
    \caption{SDDF prediction by our method. Columns a) and b) show RGB images and ground-truth SDDF. As shown in c), our model remains accurate with varying view distance due to satisfying the directional Eikonal constraint. As shown in d), the ellipsoidal prior provides a coarse approximation of the shape of objects in the scene, such as walls, chairs, and tables to be used by the residual network.
    \vspace{-1ex}}
    \label{fig:qualitative_sddf_results}
\end{figure}

\begin{figure*}
    \centering
    \begin{subfigure}[t]{0.16\linewidth}
        \centering
        \includegraphics[width=\linewidth]{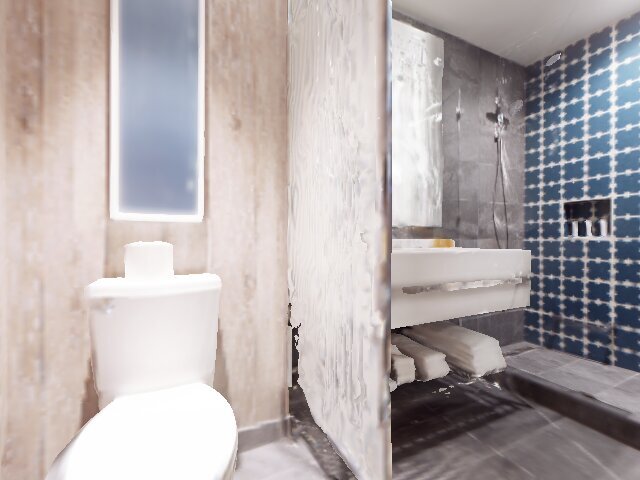}
    \end{subfigure}
    \begin{subfigure}[t]{0.16\linewidth}
        \centering
        \includegraphics[width=\linewidth]{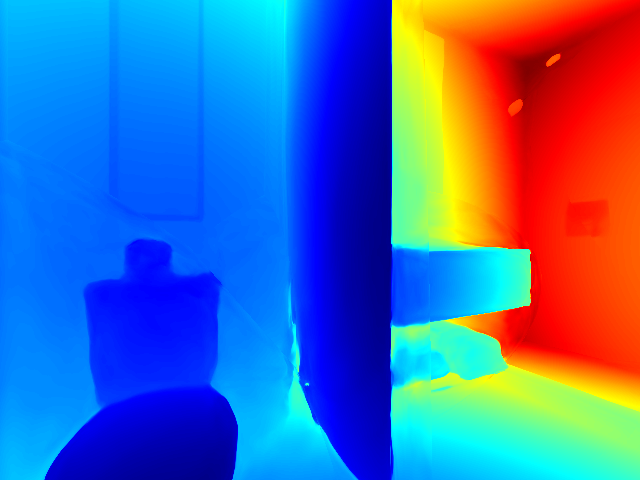}
    \end{subfigure}
    \begin{subfigure}[t]{0.16\linewidth}
        \centering
        \includegraphics[width=\linewidth]{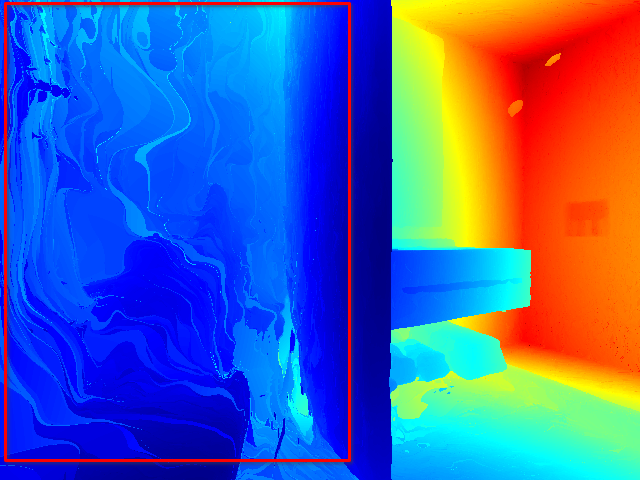}
    \end{subfigure}
    \begin{subfigure}[t]{0.16\linewidth}
        \centering
        \includegraphics[width=\linewidth]{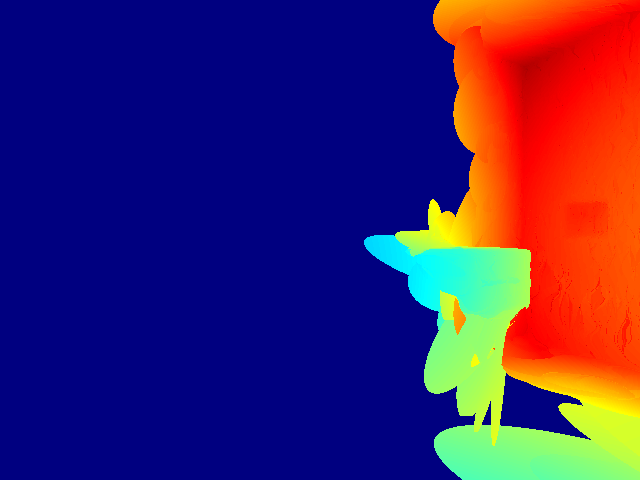}
    \end{subfigure}
    \begin{subfigure}[t]{0.16\linewidth}
        \centering
        \includegraphics[width=\linewidth]{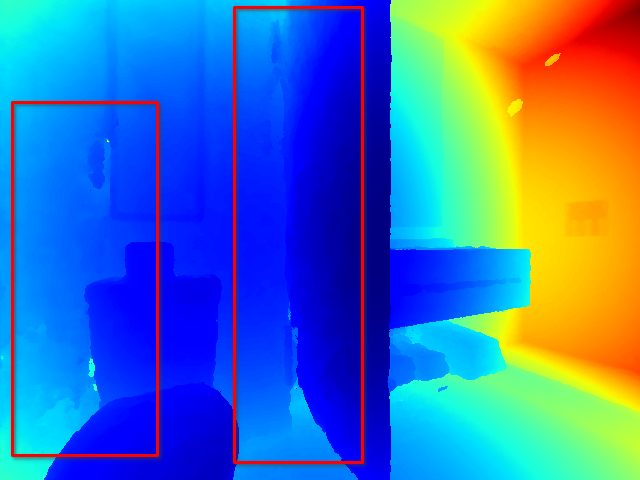}
    \end{subfigure}
    \begin{subfigure}[t]{0.16\linewidth}
        \centering
        \includegraphics[width=\linewidth]{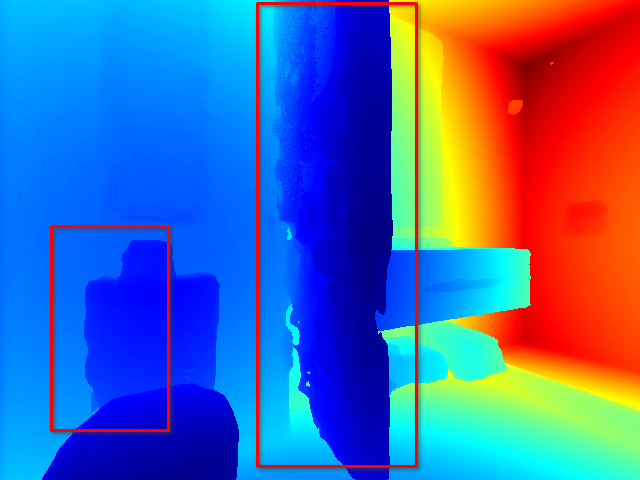}
    \end{subfigure}
    \begin{subfigure}[t]{0.16\linewidth}
        \centering
        \includegraphics[width=\linewidth]{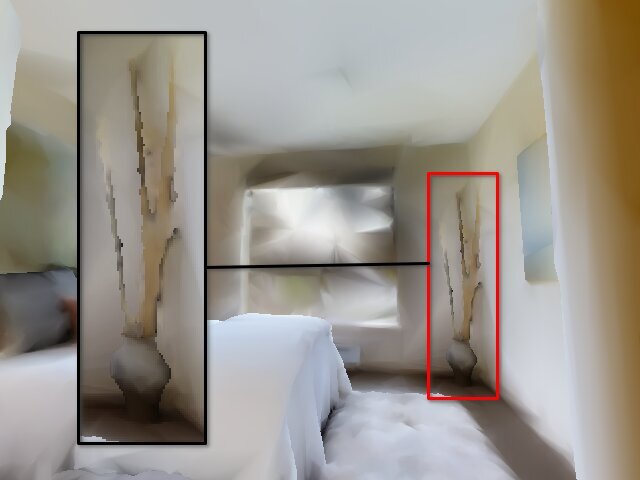}
    \end{subfigure}
    \begin{subfigure}[t]{0.16\linewidth}
        \centering
        \includegraphics[width=\linewidth]{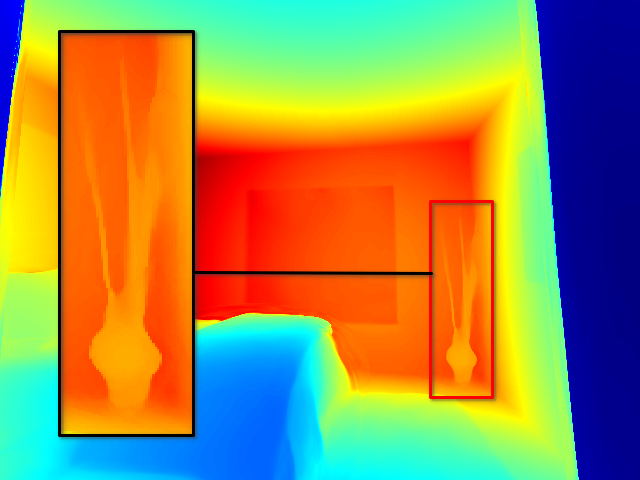}
    \end{subfigure}
    \begin{subfigure}[t]{0.16\linewidth}
        \centering
        \includegraphics[width=\linewidth]{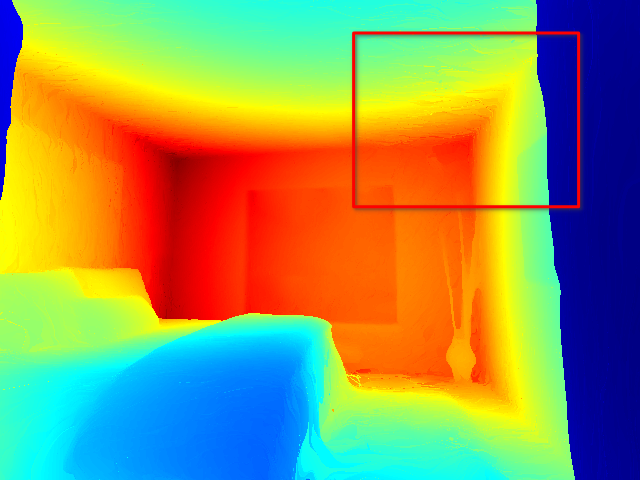}
    \end{subfigure}
    \begin{subfigure}[t]{0.16\linewidth}
        \centering
        \includegraphics[width=\linewidth]{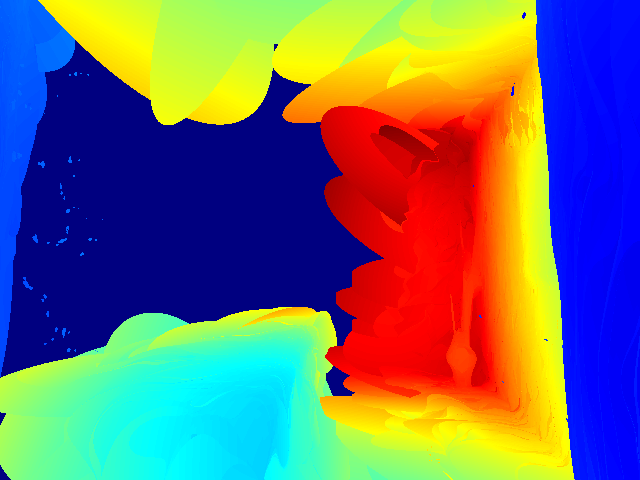}
    \end{subfigure}
    \begin{subfigure}[t]{0.16\linewidth}
        \centering
        \includegraphics[width=\linewidth]{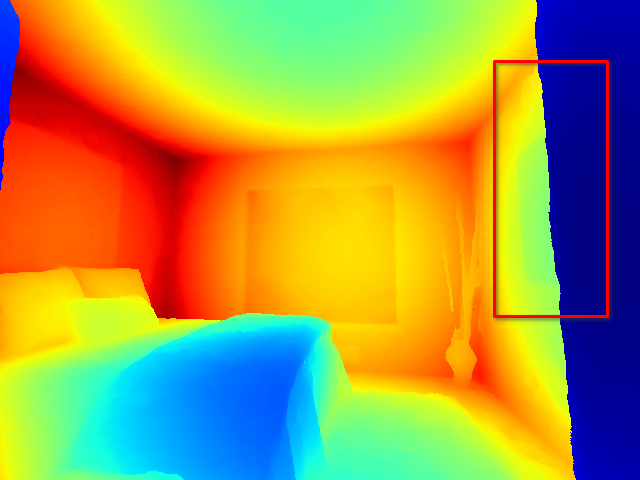}
    \end{subfigure}
    \begin{subfigure}[t]{0.16\linewidth}
        \centering
        \includegraphics[width=\linewidth]{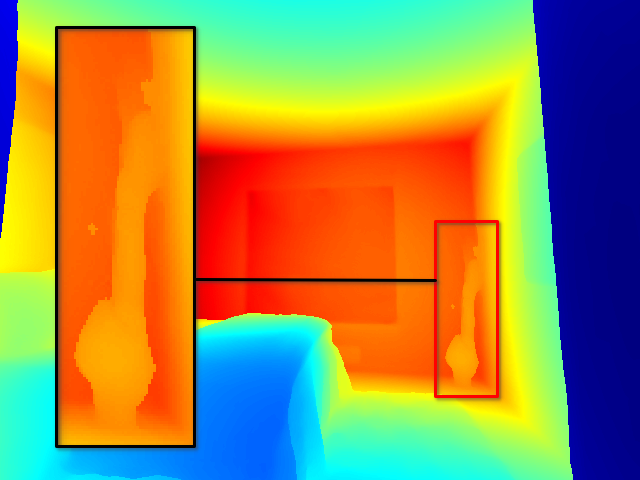}
    \end{subfigure}
    \begin{subfigure}[t]{0.16\linewidth}
        \centering
        \includegraphics[width=\linewidth]{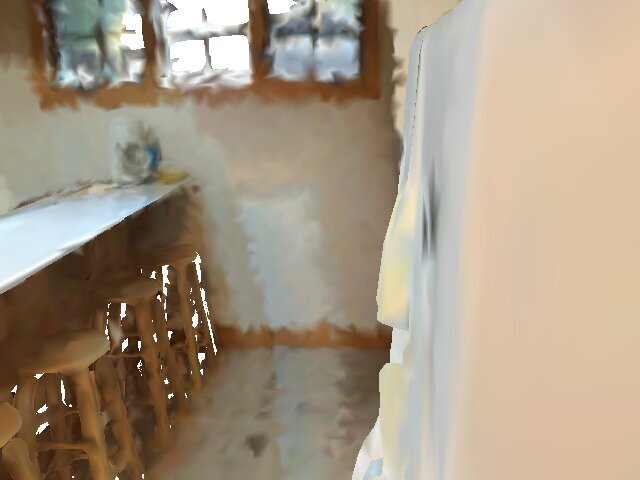}
        \caption{Scene}
    \end{subfigure}
    \begin{subfigure}[t]{0.16\linewidth}
        \centering
        \includegraphics[width=\linewidth]{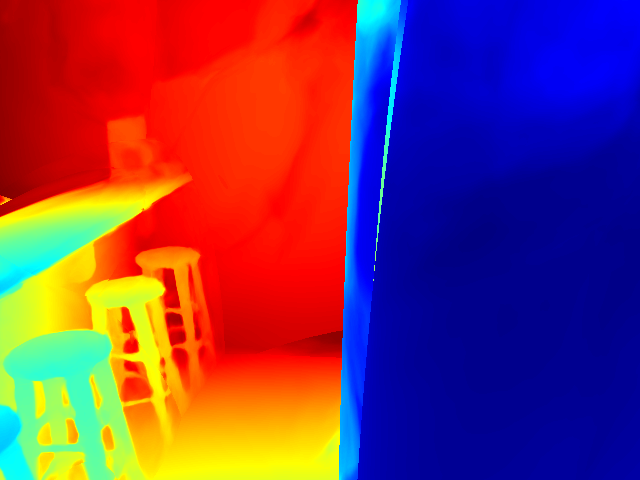}
        \caption{SDDF (ours)}
        \label{fig:comp_sddf}
    \end{subfigure}
    \begin{subfigure}[t]{0.16\linewidth}
        \centering
        \includegraphics[width=\linewidth]{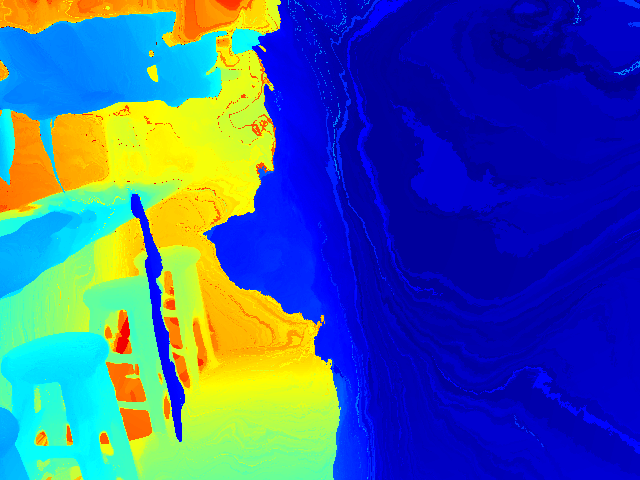}
        \caption{RaDe-GS~\cite{radegs2024}}
        \label{fig:comp_radegs}
    \end{subfigure}
    \begin{subfigure}[t]{0.16\linewidth}
        \centering
        \includegraphics[width=\linewidth]{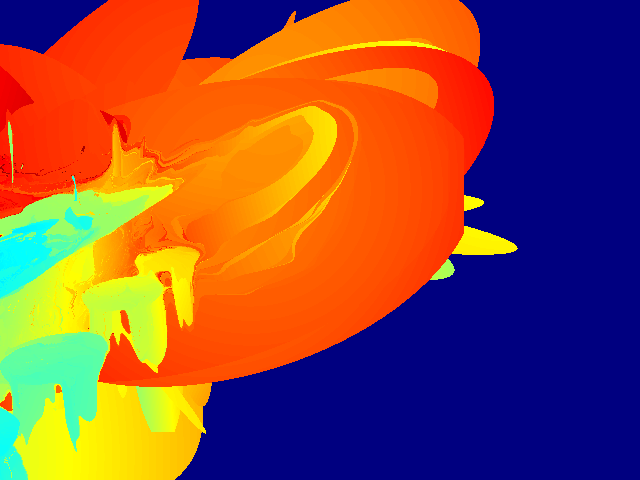}
        \caption{w/o RGB~\cite{radegs2024}}
        \label{fig:comp_radegs_depth_only}
    \end{subfigure}
    \begin{subfigure}[t]{0.16\linewidth}
        \centering
        \includegraphics[width=\linewidth]{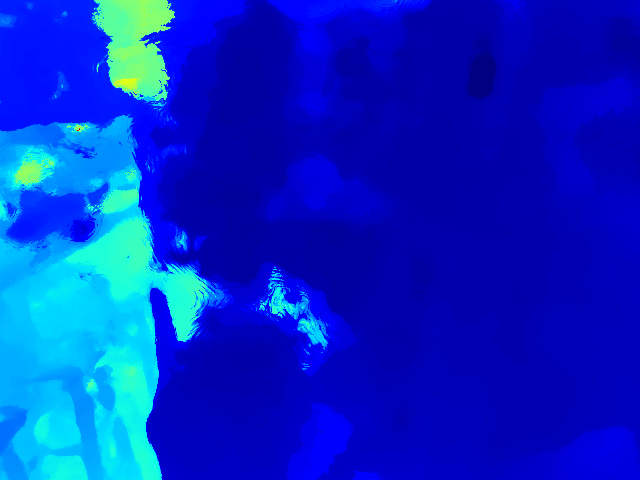}
        \caption{Nerfacto~\cite{nerfstudio2023}}
        \label{fig:comp_nerf}
    \end{subfigure}
    \begin{subfigure}[t]{0.16\linewidth}
        \centering
        \includegraphics[width=\linewidth]{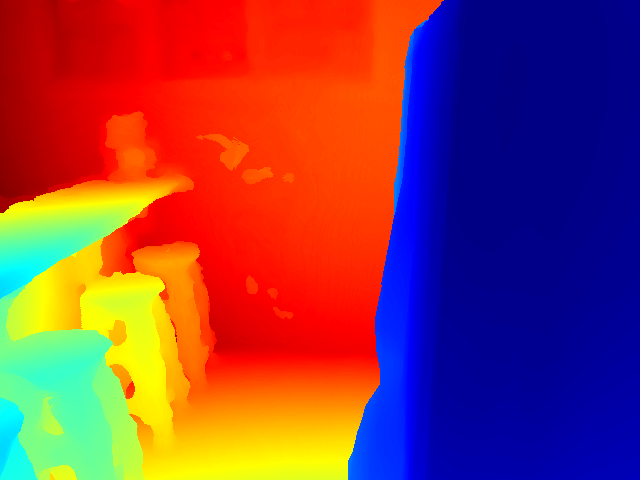}
        \caption{SDF~\cite{instantngp2022}}
        \label{fig:comp_sdf}
    \end{subfigure}
    \caption{Qualitative comparison of SDDF predictions. Row 1: Replica Hotel (synthesized). Row 2: Gibson Allensville (synthesized). Row 3: ScanNet scene 0000-00 (real). In areas with limited sensor measurements, RaDe-GS~\cite{radegs2024} fails to learn the geometry, with RGB (c) or without RGB (d), yielding artifacts. Nerfacto~\cite{nerfstudio2023} in e) shows significant artifacts and large distance prediction error. \sdfbaseline~\cite{instantngp2022} in f) tends to learn a smoother approximation, missing sharp boundaries.}
    \label{fig:sddf_vs_baselines}
\end{figure*}

\begin{table*}[t]
    \centering
    \caption{Mean absolute error (cm) of SDDF prediction}\label{tab:mae_sddf}
    \begin{tabular}{c|c|c|c|c|c|c|c|c}
        &              & Allensville & Hotel & Office 0 & Office 1 & Office 2 & Office 3 & Office 4 \\ \hline
\multirow{2}{*}{LiDAR} & \sdfbaseline~\cite{instantngp2022} & \textbf{1.137} & 1.224 & \textbf{0.825} & 0.724 & 1.342 & 1.608 & \textbf{1.037} \\
                       & \textbf{SDDF (ours)} & 1.350 & \textbf{0.997} & 1.092 & \textbf{0.694} & \textbf{1.236} & \textbf{1.588} & 1.132 \\ \hline
\multirow{2}{*}{RGB-D} & RaDe-GS~\cite{radegs2024} & \textbf{1.737} & \textbf{0.857} & \textbf{0.438} & \textbf{0.348} & \textbf{1.258} & \textbf{0.827} & \textbf{0.498} \\
                       & Nerfacto~\cite{nerfstudio2023} & 83.433 & 58.272 & 63.011 & 68.382 & 69.144 & 74.168 & 88.479 \\
                       & depth-Nerfacto~\cite{nerfstudio2023} & 85.655 & 61.706 & 65.074 & 69.767 & 74.422 & 79.227 & 76.535 \\\hline
\multirow{3}{*}{Depth Only} & RaDe-GS (w/o RGB)~\cite{radegs2024} & 25.002 & 45.948 & 1.208 & 6.654 & 3.582 & \textbf{0.955} & \textbf{0.500} \\
& \sdfbaseline~\cite{instantngp2022} & \textbf{1.106} & 1.297 & \textbf{0.737} & 0.711 & 1.120  & {1.395} & {0.953} \\
                       & \textbf{SDDF (ours)} & 1.490 & \textbf{1.247} & 1.120 & \textbf{0.696} & \textbf{1.081} & 1.568 & 1.206
    \end{tabular}
    \vspace{-1ex}
\end{table*}

\textbf{Accuracy of SDDF Predictions:} We first compare the SDDF predictions made using different methods.
Qualitative results are shown in Figs.~\ref{fig:sddf_prediction} and~\ref{fig:sddf_vs_baselines} (more results in Figs.~\ref{fig:more_sddf_qualitative_results}--\ref{fig:more_sddf_comp_scannet} in Supplemental~\ref{sec:appdx_more_vis_sddf_prediction}).
Our SDDF model accurately learns the geometry of the scene, as shown in the three qualitative examples in Fig.~\ref{fig:qualitative_sddf_results}.
The overall model remains robust to varying viewing distances because the latent feature for the residual network is constructed from the intersection point and the viewing direction in a way that satisfies the Eikonal constraint (Proposition~\ref{prop:eikonal_equation_with_residual}).
The ellipsoidal prior (Fig.~\ref{fig:sddf_prior_prediction}) is optimized to approximate the shapes of objects in the scene, based on which the residual network recovers fine details with high fidelity, as seen in Fig.~\ref{fig:sddf_prediction}.
This ellipsoid-based network design allows readily combining different sensor measurements of the same surface into a neural representation.

Fig.~\ref{fig:sddf_vs_baselines} shows qualitative comparisons against RaDe-GS~\cite{radegs2024} (RGB-D and depth-only), Nerfacto~\cite{nerfstudio2023}, and \sdfbaseline~\cite{instantngp2022} with sphere tracing.
In Fig.~\ref{fig:comp_radegs}, RaDe-GS exhibits erroneous artifacts around the toilet, at the corner, or near the fridge, where there are limited RGB-D observations.
Meanwhile, our method accurately reconstructs these areas.
Nerfacto also exhibits artifacts due to insufficient data as shown in Fig.~\ref{fig:comp_nerf}.
Since no explicit representations like our ellipsoids or the Gaussians in RaDe-GS are used, Nerfacto predicts SDDF based on the learned volume density, which is optimized for photometric rendering rather than scene geometry, leading to large distance prediction errors.

Meanwhile, sphere tracing on SDF-Instant-NGP~\cite{instantngp2022} does not exhibit significant artifacts.
However, this baseline tends to learn smoother shapes that lack sharper details, such as the plant shown in the second row or the chairs in the third row of Fig.~\ref{fig:comp_sdf}.
Moreover, the first row of Fig.~\ref{fig:comp_sdf} shows that the errors accumulated during sphere tracing become significant at boundaries when the SDF model does not predict accurate SDF consistently.
These qualitative observations are consistent in \emph{real} ScanNet data~\cite{dai2017scannet}, as shown in Fig.~\ref{fig:sddf_vs_baselines} and Fig.~\ref{fig:more_sddf_comp_scannet} in Supplemental~\ref{sec:appdx_more_vis_sddf_prediction}.

The quantitative results in Table~\ref{tab:mae_sddf} show that our method reaches the state of the art for both LiDAR and RGB-D datasets.
Our method has slightly higher errors on RGB-D datasets because our method only uses depth information from the RGB-D datasets, whereas RaDe-GS uses both RGB and depth data.
Moreover, our method is trained with 20$\times$ downsampling, whereas RaDe-GS uses 2$\times$ downsampling and hence approximately 10$\times$ more depth data.
Consequently, RaDe-GS uses a significantly higher number of parameters as seen in Table~\ref{tab:gpu_cost_model_size}.
Nerfacto and \sdfbaseline{} do not downsample the data, but generally perform worse than ours.
Training with depth loss~\cite{kangle2021dsnerf} does not improve the distance predictions by Nerfacto. We examine the results of Nerfacto and depth-Nerfacto further in Supplemental~\ref{sec:appdx_training_of_nerfacto}.
With RGB-D datasets, RaDe-GS~\cite{radegs2024} has the lowest MAE among models that use both RGB and depth.
However, as shown in Table~\ref{tab:mae_sddf} and in Fig.~\ref{fig:comp_radegs_depth_only}, when RaDe-GS is trained using depth data only, it shows significantly larger errors in larger and more complicated scenes.
We omit quantitative evaluation on ScanNet data (Fig.~\ref{fig:more_sddf_comp_scannet} in Supplemental~\ref{sec:appdx_more_vis_sddf_prediction}), due to the unavailability of a reliable ground-truth mesh.

\begin{table}
    \centering
    \caption{Max. GPU memory usage, model size, and inference time on Replica-Hotel dataset~\cite{replica19arxiv} (Intel 14900K CPU and NVIDIA RTX-3090 GPU). Training time is the time taken by the model to reach MAE $< 1.5$ cm.}\label{tab:gpu_cost_model_size}
    \begin{tabular}{c|c|c|c|c}
                           & SDDF & RaDe-GS & Nerfacto & SDF \\ \hline
         Training GRAM (G) & \textbf{3.3}  & 18.7    & 4.7      & 6.2 \\
          Testing GRAM (G) & \textbf{1.7}  & 6.7     & 3.6      & 1.8 \\
        No. Parameters (M) & 2.7  & 113.7   & 16.4     & \textbf{1.7} \\
       Training Time (min) & 100  & 30      & N/A      & \textbf{3} \\
       Inference Time (ms/frame) & 69   & \textbf{7}       & 254      & 103
    \end{tabular}
\vspace{-1ex}
\end{table}

\begin{table*}
    \centering
    \caption{Ablation experiments on LiDAR datasets. Mean absolute error (cm) of SDDF prediction is reported.} \label{tab:ablation_study}
    \begin{tabular}{c|c|c|c|c|c|c|c}
                             & Allensville & Hotel & Office 0 & Office 1 & Office 2 & Office 3 & Office 4 \\ \hline
       twice number of ellipsoids & 1.414 & 1.013 & \textbf{1.034} & 0.716 & 1.291 & 1.649 & 1.250 \\
         half number of ellipsoids & 1.477 & 1.013 & 1.072 & 0.706 & \textbf{1.218} & 1.727 & 1.192 \\
                 without negative samples & 5.379 & 2.639 & 1.954 & 0.867 & 2.734 & 8.598 & 2.406 \\
        K-means++~\cite{kmeanspp2007} only ellipsoid initialization & 1.405 & 1.033 & 1.060 & 0.703 & 1.279 & 1.636 & 1.248 \\
           Fourier embedding & 1.842 & 1.303 & 1.544 & 0.930 & 1.771 & 2.513 & 1.631 \\
                     default & \textbf{1.350} & \textbf{0.997} & 1.092 & \textbf{0.694} & 1.236 & \textbf{1.588} & \textbf{1.132}
    \end{tabular}
\end{table*}

\textbf{Computational Efficiency:} Table~\ref{tab:gpu_cost_model_size} shows the computational cost of all methods, in terms of training time, inference time, number of parameters and maximum GPU memory usage during training and testing.
Our method exhibits the lowest memory usage for training and testing, and the second lowest number of parameters after \sdfbaseline~\cite{instantngp2022}.
This is because we use much fewer ellipsoids than Gaussians in RaDe-GS~\cite{radegs2024}, which has the highest GPU memory usage.
The lower GPU memory usage of our method is beneficial for scaling up to larger scenes, or deployment on mobile robots with limited hardware resources. Although Nerfacto has relatively low GPU memory usage, it has the longest inference time, and its prediction error is much higher than other methods.

In terms of inference time, our method is the second best, with RaDe-GS significantly outperforming our method.
We attribute this to the highly optimized, full-CUDA implementation of RaDe-GS, and expect that it is possible to achieve similar performance with improvements in the implementation.
Our method is 1.5$\times$ faster to infer the directional distance than the \sdfbaseline-based sphere-tracing method, and 4$\times$ faster than Nerfacto.
This is because our method only requires a single forward pass, whereas the other methods require multiple passes.
It is also worth noting that our method is faster than \sdfbaseline{} despite their CUDA optimized implementation.
However, because of the extra direction input, our model needs more time to train, which is a future direction for improvement.

\subsection{Ablation Study}
We examine the performance of our model in four variations.
We first consider using twice or half the number of ellipsoids for each scene.
To investigate the contribution of negative sample augmentation, we compare against an SDDF model trained without negative samples.
In addition, we train SDDF models with ellipsoids initialized by K-means++~\cite{kmeanspp2007} only.
To demonstrate the benefit of our polynomial embedding, we compare it against Fourier embedding~\cite{fourier2020}:
\begin{equation}\label{eq:fourier_embedding}
    \bfm \!=\!\left[\sin(2\pi \bfB_{\bfp} \bfp), \cos(2\pi \bfB_{\bfp}\bfp), \sin(2\pi \bfB_{\bfv}\bfv), \cos(2\pi \bfB_{\bfv}\bfv)\right],
\end{equation}
where $\bfB_{\bfp}, \bfB_{\bfv} \in \bbR^{25\times 3}$ are random Fourier features sampled from $\calN(0, 1)$. We use~\eqref{eq:fourier_embedding} in place of~\eqref{eq:monomial_gen} and~\eqref{eq:pv_monomial_gen}, while~\eqref{eq:polynomial_embedding} is kept so that the model size remains unchanged.
The results are summarized in Table~\ref{tab:ablation_study}. Lastly, we illustrate the advantages of our SDDF definition compared to related formulations~\cite{zobeidi2021, pddf2022}.

Table~\ref{tab:ablation_study} shows that although our method is robust to the number of ellipsoids, using too few or too many ellipsoids causes problems.
With too few ellipsoids, the model cannot approximate the shape of objects in the scene well, whereas with too many ellipsoids, the model learns spurious occlusions between the ellipsoids.
During testing, we also found that too many ellipsoids cause the optimizer to move some redundant ellipsoids out of the scene.

Negative SDDF sample augmentation plays a key role in optimizing ellipsoids. Fig.~\ref{fig:ablation_study_no_neg} shows that with negative SDDF samples, the ellipsoid priors are optimized to cover the scene well.
In contrast, without negative samples, the resulting ellipsoids are generally smaller and regress into the interior of objects, which leads to missing ray intersections.
Missing ray intersections cause a wrong ellipsoid to be selected at a further location, leaving the burden to the residual decoder $R$ to learn larger residuals.

Appropriate ellipsoid initialization is also important.
As shown in Table~\ref{tab:ablation_study}, initializing the ellipsoids solely with K-means++~\cite{kmeanspp2007} causes larger errors, as K-means++ generally uses an excessive number of ellipsoids for planar surfaces.

We also consider using the Fourier embedding~\cite{fourier2020} in place of the polynomial embedding introduced in this work.
The Fourier embedding may seem useful because it can learn high-frequency features.
However, the results in Table~\ref{tab:ablation_study} show that polynomial embedding provides better generalization for different scenes with smaller errors. Our investigation indicates that this is because the Fourier embedding requires different spatial frequency hyperparameters for each scene to achieve comparable performance, whereas the polynomial embedding does not require such hyperparameter tuning.

Direct comparison against DDF~\cite{pddf2022} was not possible due to the unavailability of source code, while the SDDF method of~\cite{zobeidi2021} can handle only object-level reconstruction. We illustrate the advantage of our SDDF formulation in a simple 2D scene in Fig.~\ref{fig:compare_with_ddf}. The first row shows that our SDDF definition has fewer discontinuities than DDF~\cite{pddf2022}, while the SDDF of~\cite{zobeidi2021} fails to capture the two circles due to its object-level representation of the box. The discontinuities in DDF~\cite{pddf2022} translate into a larger learning error than our SDDF, as we demonstrate in the second row of Fig.~\ref{fig:compare_with_ddf} using an MLP model with hidden dimensions $[16, 8, 8, 1]$. Faster convergence to lower error in both training and validation with the same model indicates that our SDDF formulation is more amenable to learning than DDF.

\begin{figure}[t]
    \centering
    \begin{subfigure}[t]{0.48\linewidth}
        \centering
        \includegraphics[width=\linewidth]{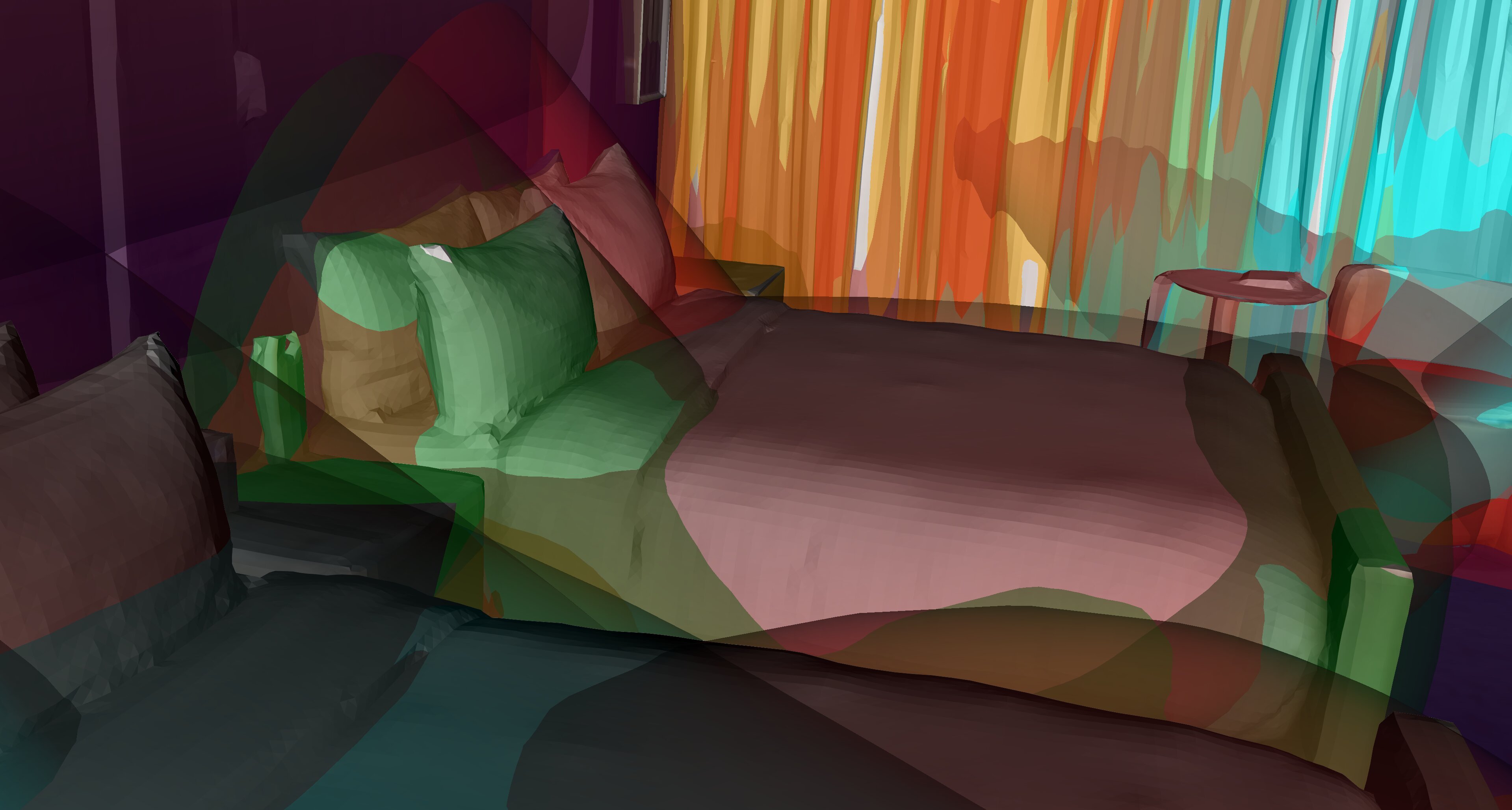}
        \caption{Default}
    \end{subfigure}
    \begin{subfigure}[t]{0.48\linewidth}
        \centering
        \includegraphics[width=\linewidth]{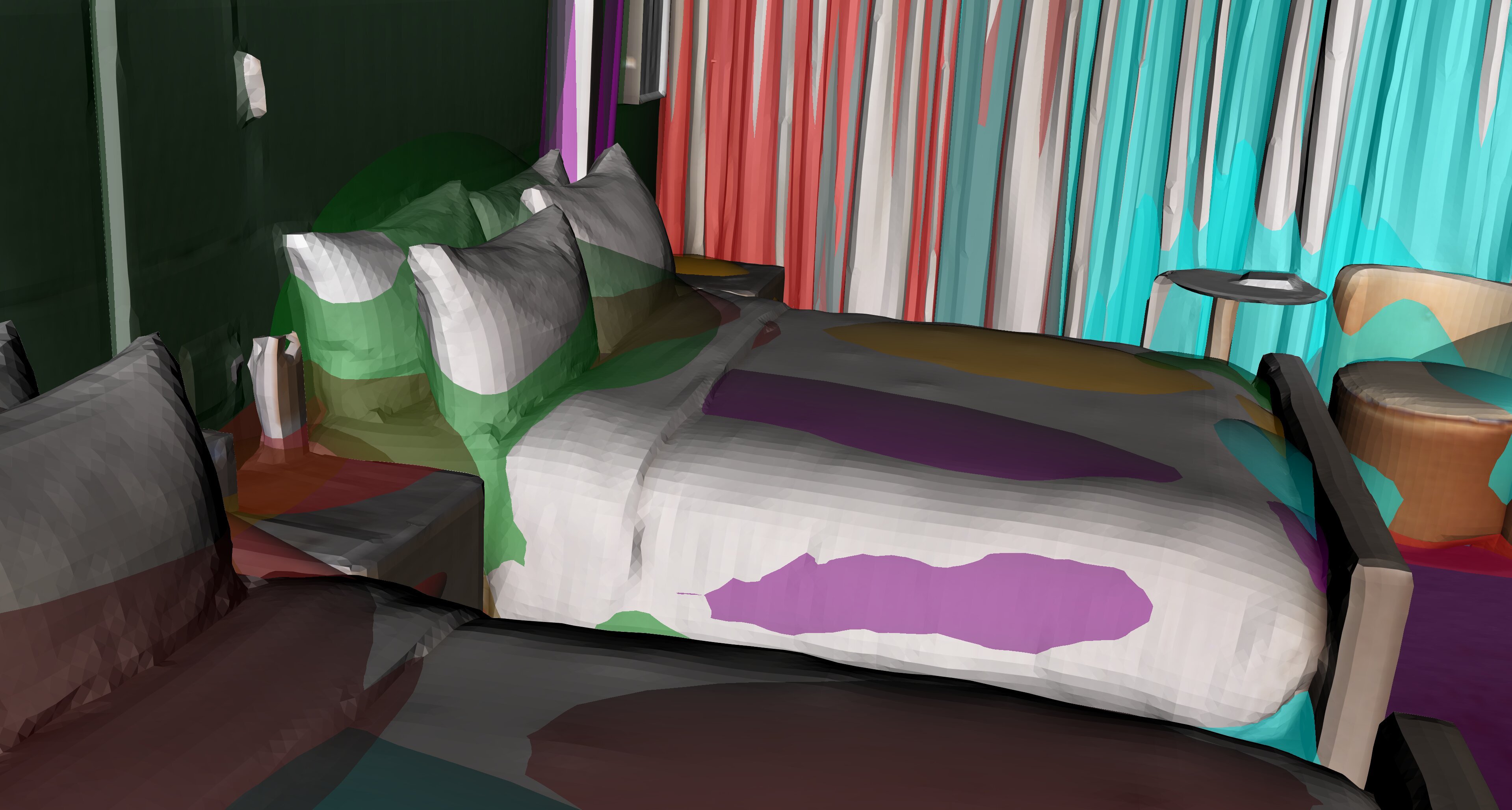}
        \caption{W/o negative samples}
    \end{subfigure}
    \vspace{-0.5ex}
    \caption{Ablation study of negative samples augmentation. With negative sample augmentation in a), the ellipsoids cover the objects in the scene. Without negative sample augmentation (b), the objects are not covered by ellipsoids in many places.}
    \label{fig:ablation_study_no_neg}
    \vspace{-1.5ex}
\end{figure}

\begin{figure}[t]
\centering
\begin{subfigure}[t]{0.33\linewidth}
    \centering
    \includegraphics[width=\linewidth]{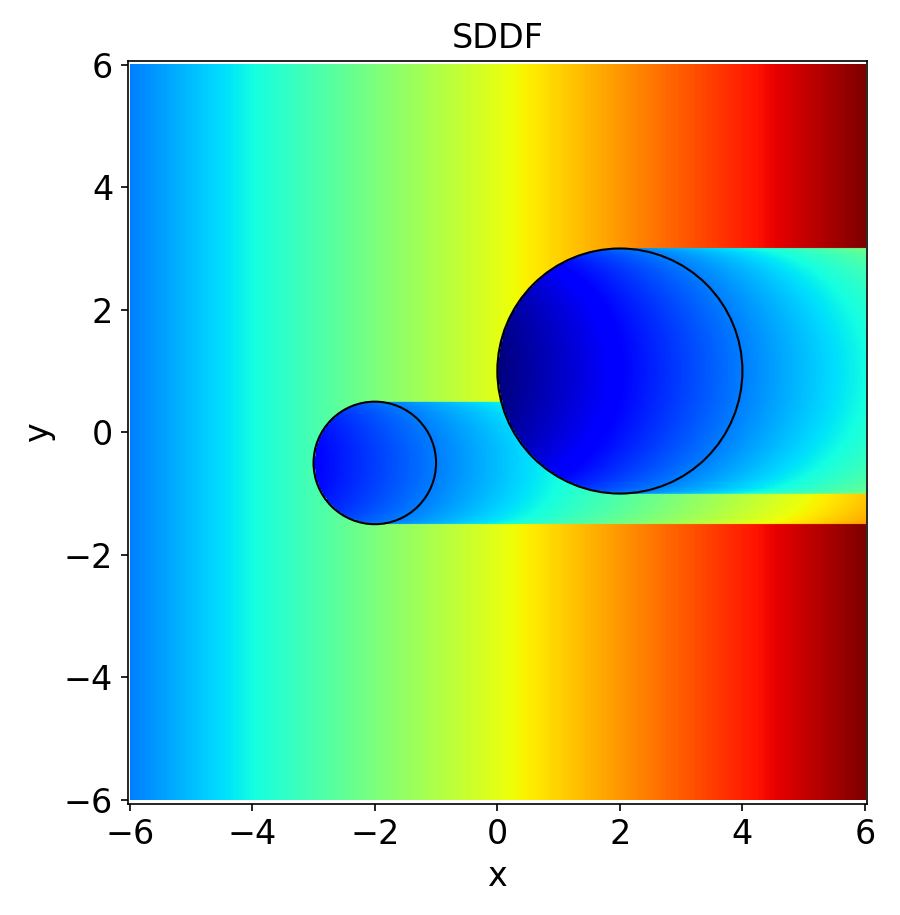}
\end{subfigure}%
\hfill%
\begin{subfigure}[t]{0.33\linewidth}
    \centering
    \includegraphics[width=\linewidth]{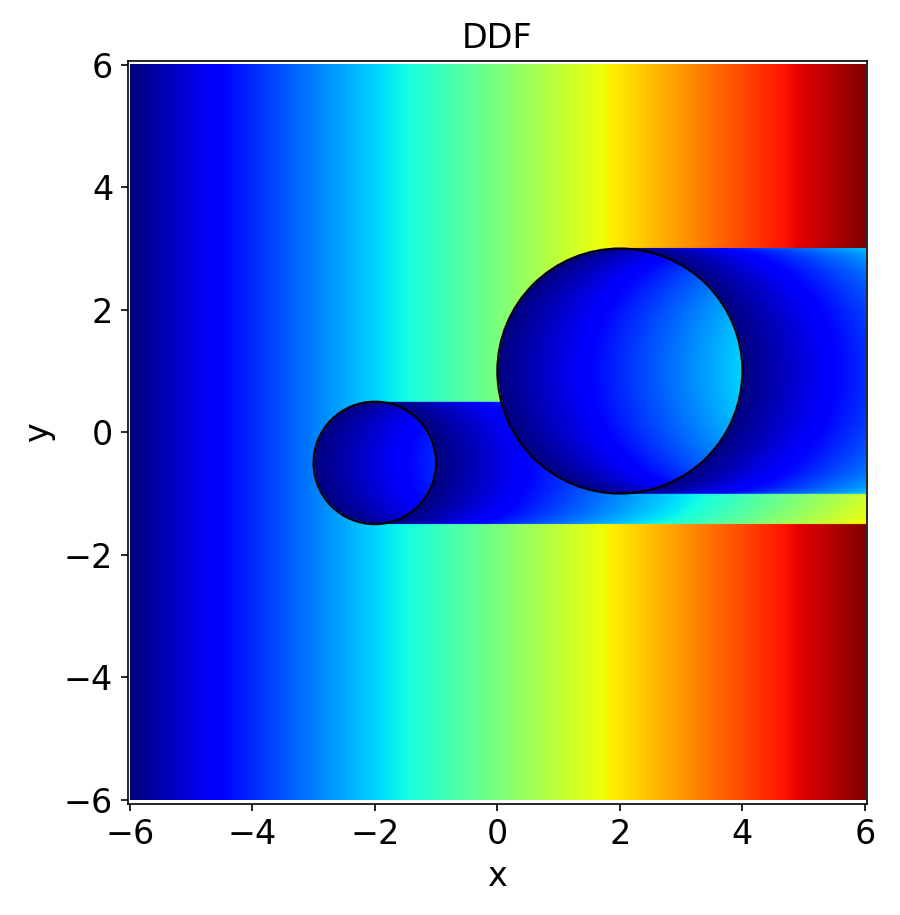}
\end{subfigure}%
\hfill%
\begin{subfigure}[t]{0.33\linewidth}
    \centering
    \includegraphics[width=\linewidth]{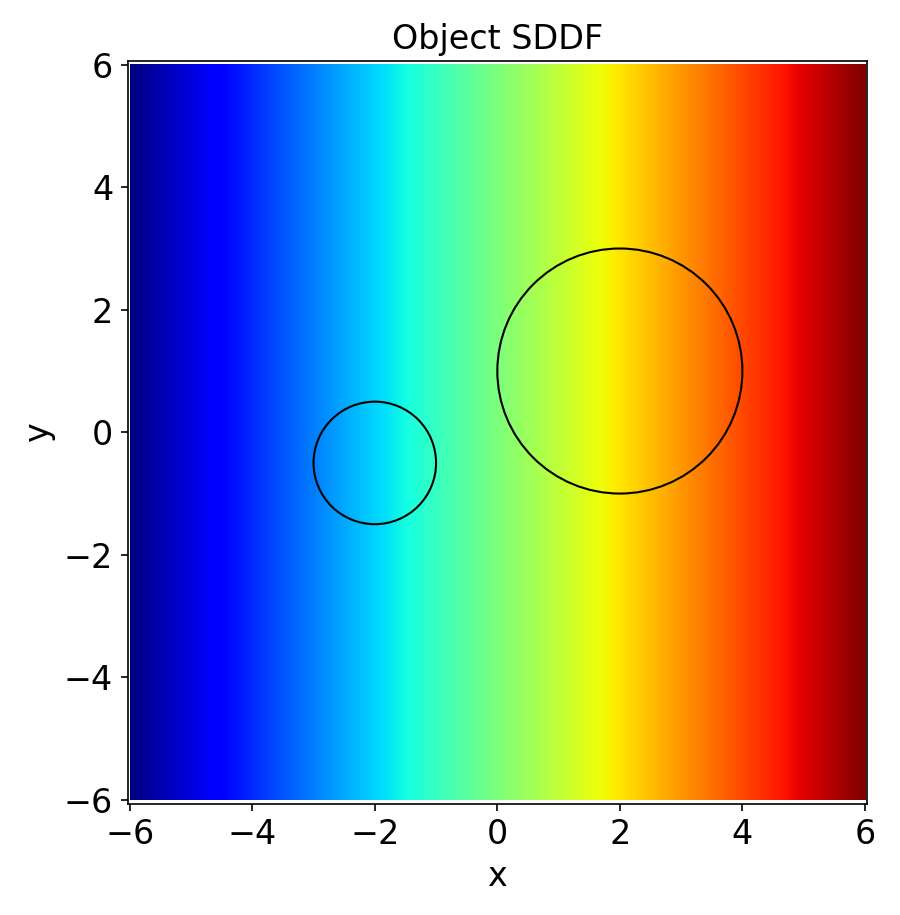}
\end{subfigure}\vspace*{-1ex}
\begin{subfigure}[t]{0.48\linewidth}
    \centering
    \includegraphics[width=\linewidth]{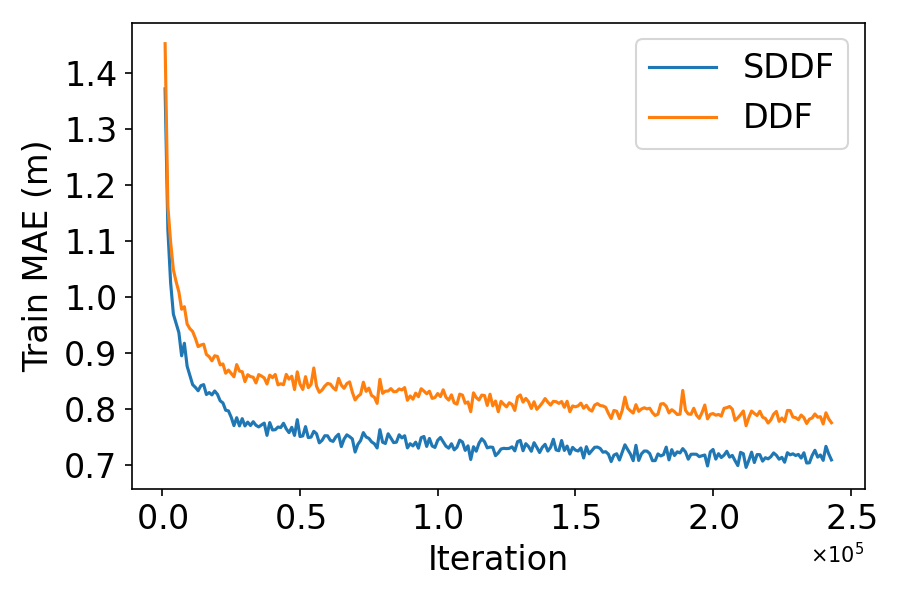}
\end{subfigure}%
\hfill%
\begin{subfigure}[t]{0.48\linewidth}
    \centering
    \includegraphics[width=\linewidth]{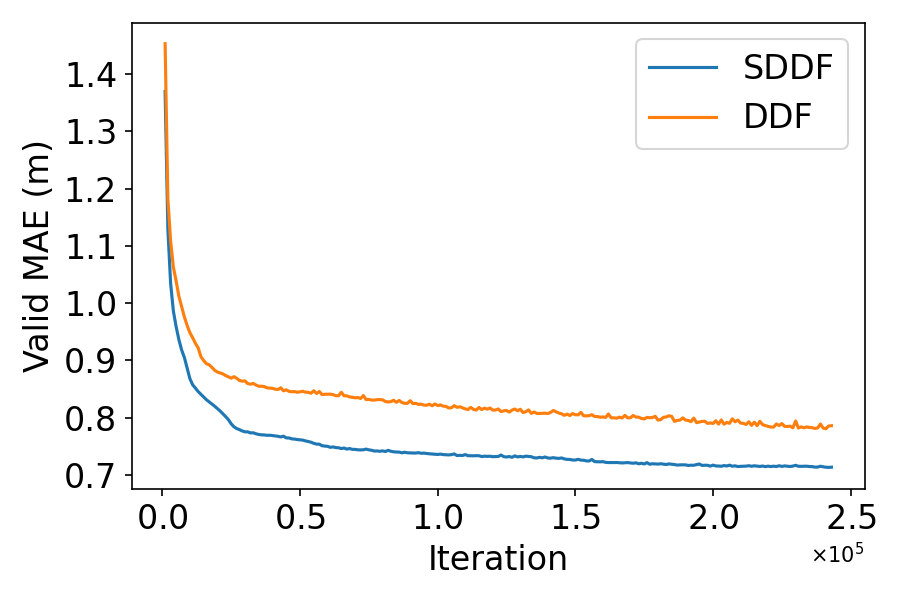}
\end{subfigure}
\vspace{-0.5ex}
\caption{Comparison of SDDF against DDF~\cite{pddf2022} and object SDDF~\cite{zobeidi2021} in a 2D scene with two circles in a square box. Top row: ground-truth targets with a fixed viewing direction $\bfv=(-1, 0)$. Bottom row: training and validation MAE (m) when learning using an MLP.} \label{fig:compare_with_ddf}
\vspace{-1.5ex}
\end{figure}

\subsection{Application to Viewpoint Optimization}

The first row of Fig.~\ref{fig:demo_view_opt} demonstrates gradient-based next-best view optimization using our SDDF model.
The second row shows the combined surface area observed by the two camera views and the overlapped area between the two views.
The visualization shows that our method can reduce the overlap between the two camera views and increase the observed area.
In this example, the area observed by the two camera views increases to $36.15 \mathrm{m}^2\ (+90.3\%)$ from $19.00 \mathrm{m}^2$ with the initial viewpoints.
In the third row of Fig.~\ref{fig:demo_view_opt}, we show an example of scaling up to a trajectory.
By optimizing the waypoints, the observed area increases from $124.55 \mathrm{m}^2$ to $177.50 \mathrm{m}^2\ (+37.7\%)$. We show more results of viewpoint optimization in different scenes in Supplemental~\ref{sec:appdx_more_vis_view_optimization}. These results illustrate a capability that SDDF enables for the first time as a proof of concept. While the same trajectory optimization can be done with SDF and differentiable sphere tracing, this will be much slower and less stable due to the multiple sphere tracing iterations.

\begin{figure}
    \centering
    \begin{subfigure}[t]{0.32\linewidth}
        \centering
        \includegraphics[width=\linewidth,trim={250pt 50pt 250pt 50pt},clip]{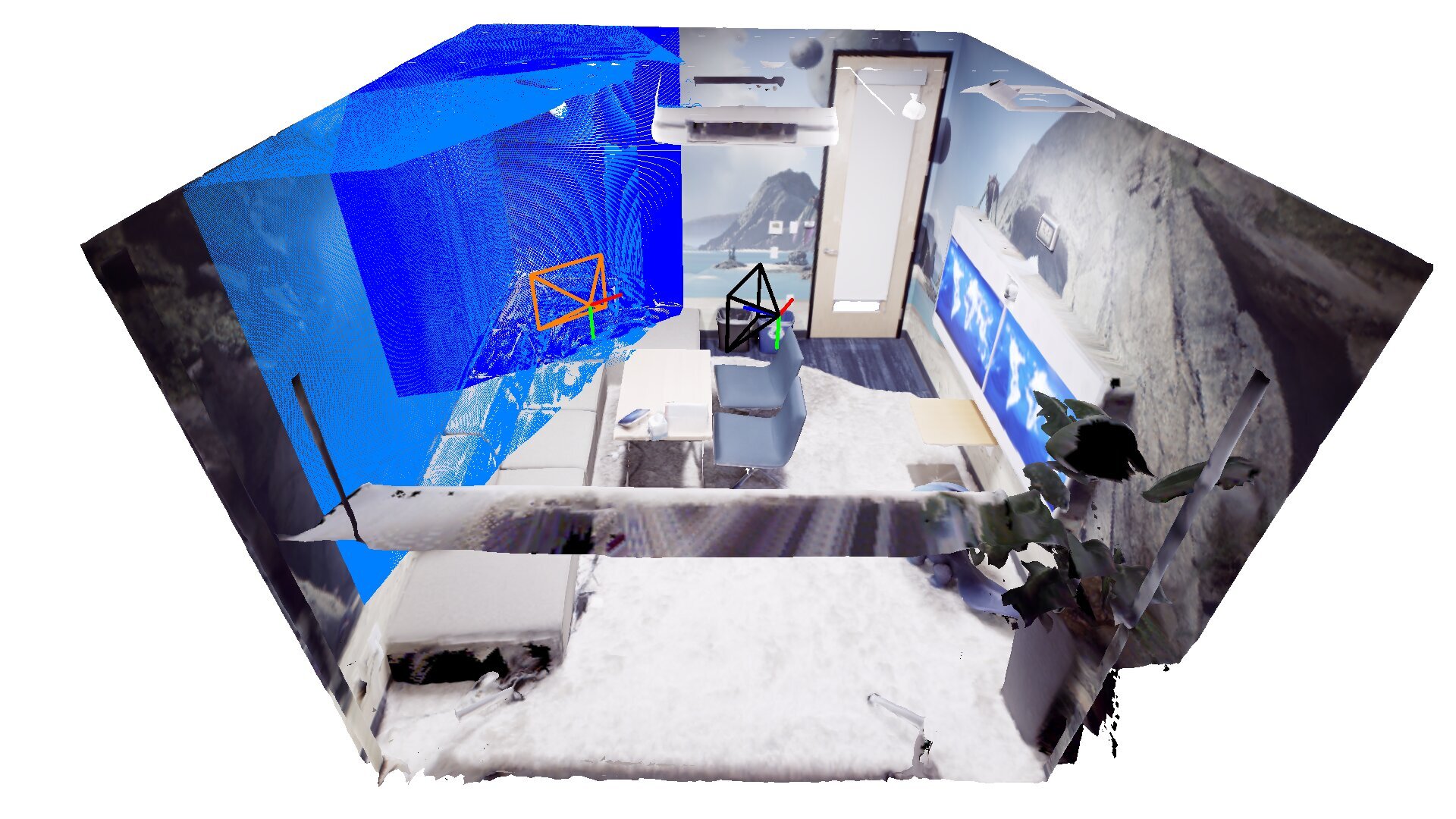}
        \caption{Initial}
        \label{fig:view_opt_initial}
    \end{subfigure}%
    \hfill%
    \begin{subfigure}[t]{0.32\linewidth}
        \centering
        \includegraphics[width=\linewidth,trim={250pt 50pt 250pt 50pt},clip]{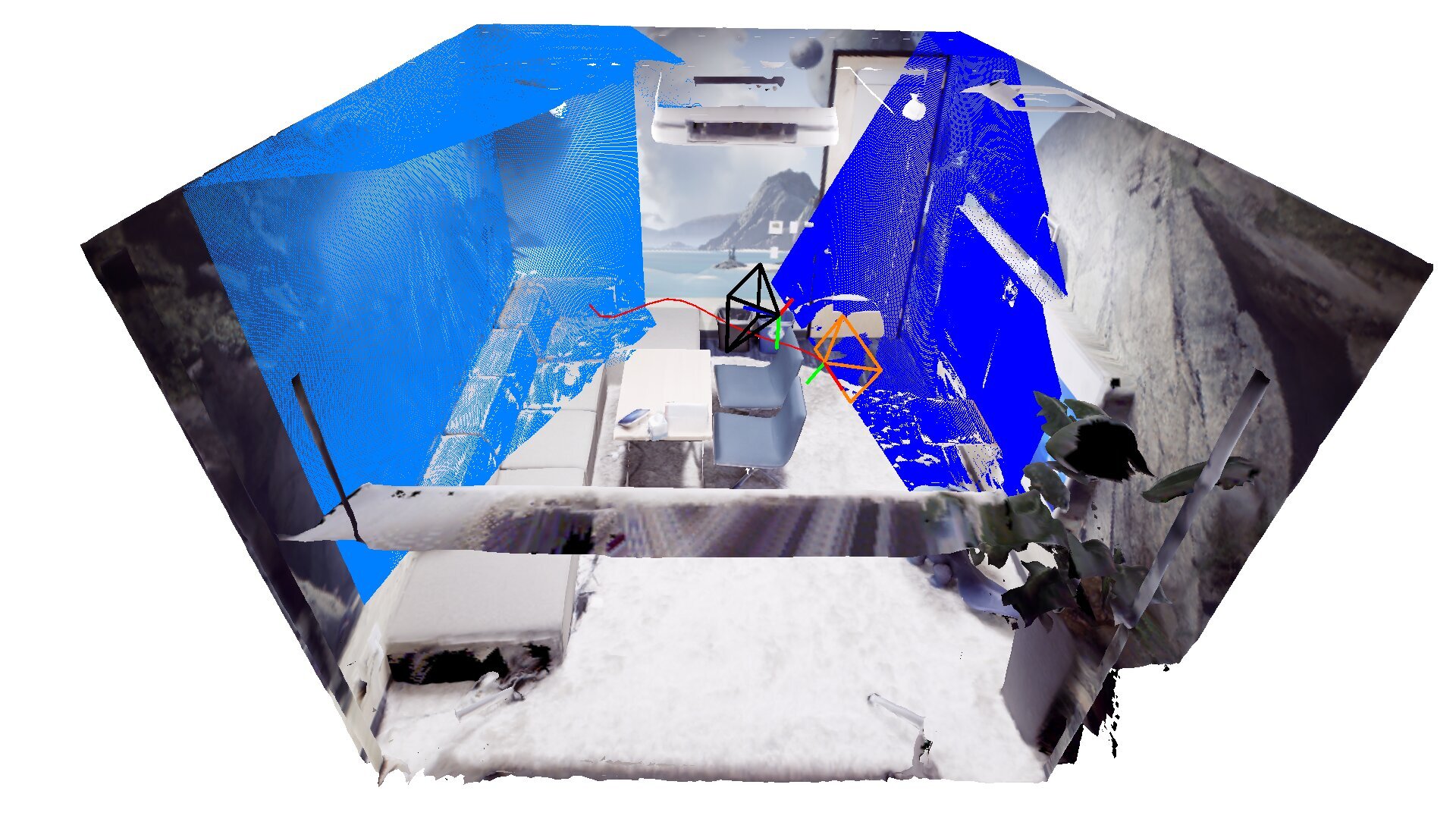}
        \caption{Reduce Overlap}
        \label{fig:view_opt_reduce_overlap}
    \end{subfigure}%
    \hfill%
    \begin{subfigure}[t]{0.32\linewidth}
        \centering
        \includegraphics[width=\linewidth,trim={250pt 50pt 250pt 50pt},clip]{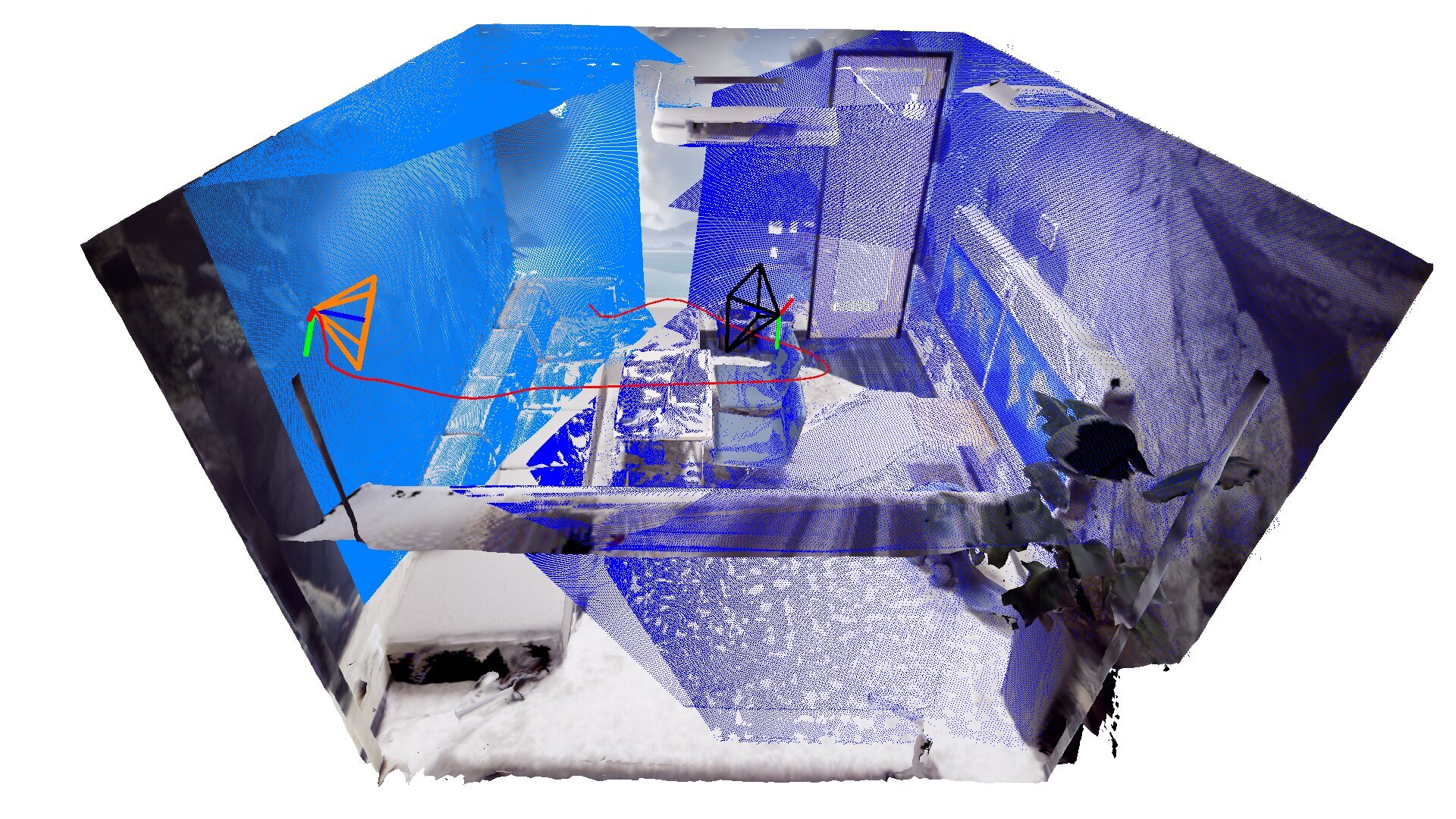}
        \caption{Increase Visibility}
        \label{fig:view_opt_increase_visibility}
    \end{subfigure}
    \begin{subfigure}[t]{0.48\linewidth}
        \centering
        \includegraphics[width=\linewidth,trim={0pt 0pt 0pt 5pt},clip]{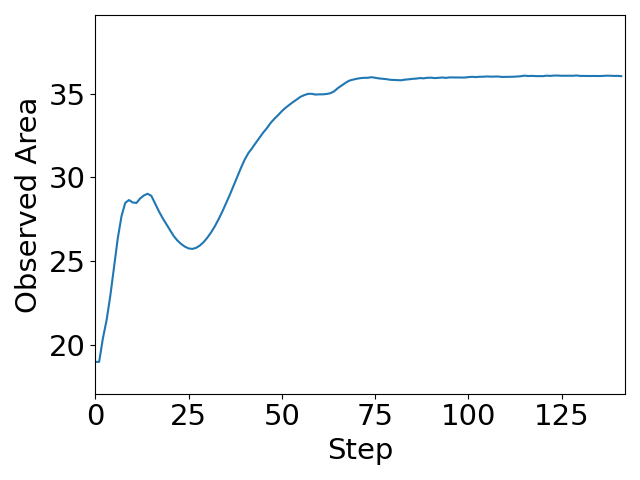}
        \caption{Step vs Observed Area}
        \label{fig:view_opt_curve_observed_area}
    \end{subfigure}
    \hfill
    \begin{subfigure}[t]{0.48\linewidth}
        \centering
        \includegraphics[width=\linewidth,trim={0pt 0pt 0pt 5pt},clip]{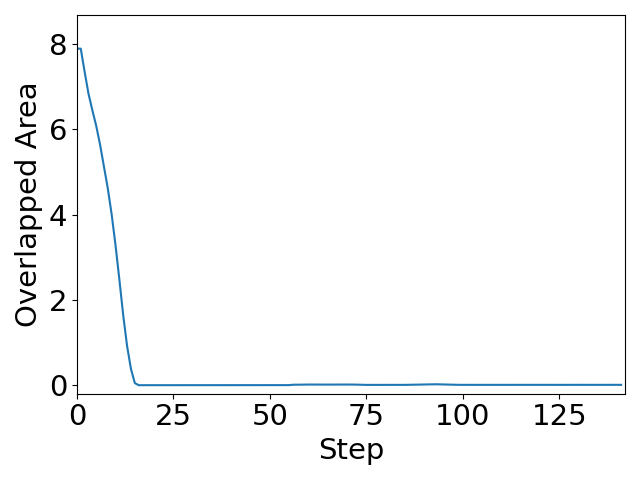}
        \caption{Step vs Overlapped Area}
        \label{fig:view_opt_curve_overlapped_area}
    \end{subfigure}
    \begin{subfigure}[t]{0.48\linewidth}
        \centering
        \includegraphics[width=\linewidth,trim={350pt 10pt 350pt 10pt},clip]{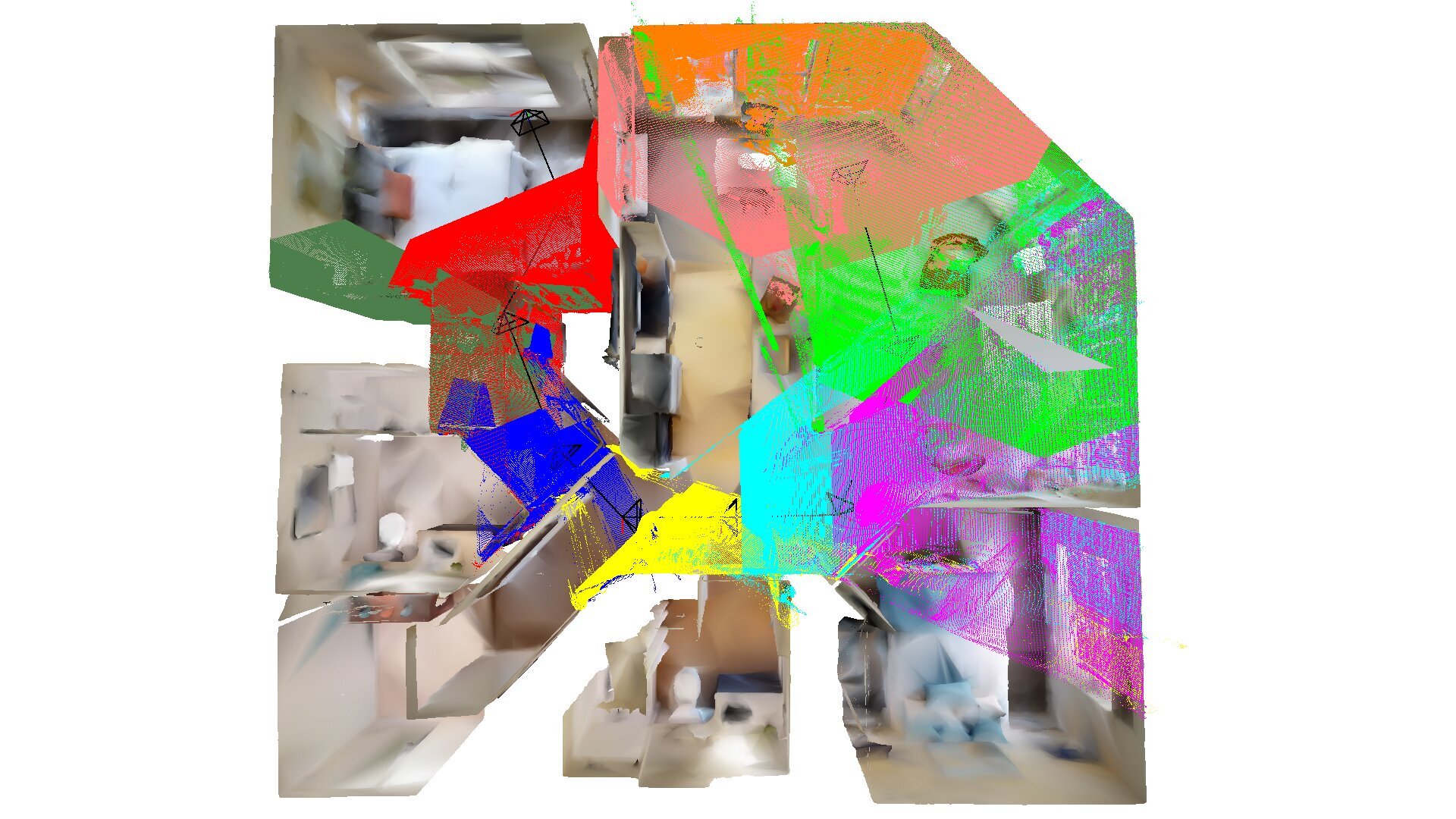}
        \caption{Initial}
        \label{fig:traj_opt_initial}
    \end{subfigure}%
    \hfill%
    \begin{subfigure}[t]{0.48\linewidth}
        \centering
        \includegraphics[width=\linewidth,trim={350pt 10pt 350pt 10pt},clip]{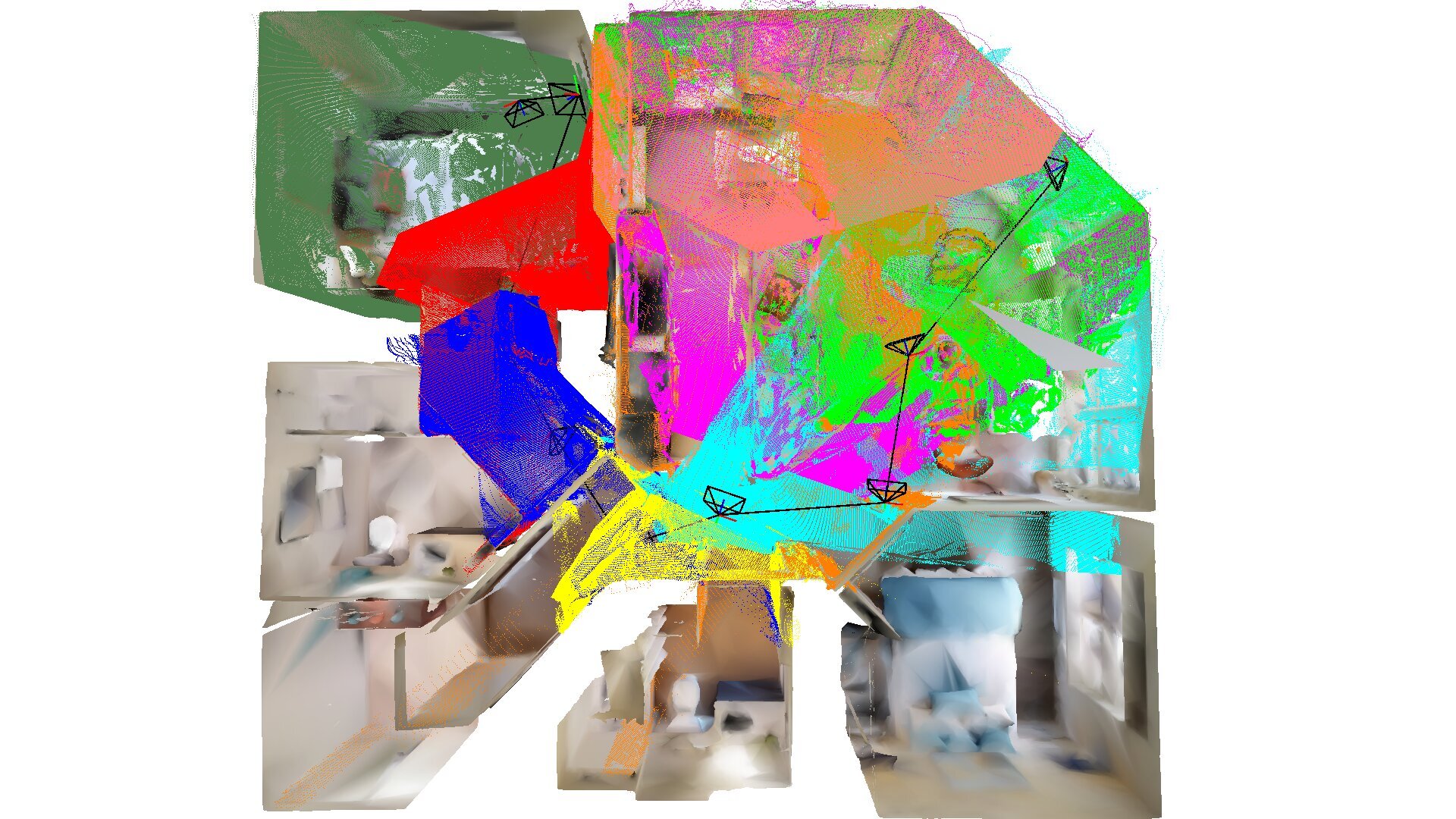}
        \caption{Optimized}
        \label{fig:traj_opt_optimized}
    \end{subfigure}
    \vspace{-0.5ex}
    \caption{Visualization of differentiable view optimization. In (a), (b) and (c), the black camera $C_t$ is at time $t$. The {\color{orange}orange} camera $C_{t+1}$ is at time $t+1$ with optimized pose. The {\color{cyan}cyan} point cloud is $\calP_t$ and the {\color{blue}blue} one is $\calP_{t+1}$. The trajectory of $C_{t+1}$ during the optimization process is colored in {\color{red}red}. Starting from the setup in (a), $C_{t+1}$ is optimized to reduce the overlap between $\calP_t$ and $\calP_{t+1}$ as shown in (b). In (c), the camera $C_{t+1}$ gets a bigger view not overlapped with $C_t$'s than (b). The risk loss ensures that $C_{t+1}$ stays away from the wall while trying to get a larger view. (d) and (e) show the curves of total observed area and overlapped area, respectively. (f) and (g) show an example of optimizing multiple trajectory waypoints. With our method, the observed area along the trajectory is significantly larger.}
    \vspace{-1.5ex}
    \label{fig:demo_view_opt}
\end{figure}

\subsection{Limitations}

Our method has several limitations.
The training of SDDF requires enough diversity in ray directions, which necessitates data augmentation techniques when insufficient sensor observations are available.
Diversity in viewing directions is especially important for our model to excel with sparser measurements, as evident in the downsampling in our experiments.
As shown in the results, our model performs worse with RGB-D measurements, although the RGB-D camera generates denser observations than the LiDAR in our experiments. This is due to the RGB-D camera having a smaller vertical field of view and, hence, less diversity in viewing directions.
This suggests a future research direction of efficiently enhancing the variety of viewing directions subject to a limited sensing budget.
One approach is to synthesize additional viewing directions based on the existing measurements using methods like hidden point removal (HPR)~\cite{hpr2007}. However, HPR requires computing the convex hull of the point cloud, which can be computationally expensive for large-scale scenes. It is important to consider more efficient ways to collect diverse viewing directions in future work. Another direction is to train the model with losses that encourage generalization to unseen viewing directions. For example, we can synthesize viewing rays hitting the same surface point but from different directions and enforce consistency in the predicted SDDF values for these rays.

It is also important to speed up the training of the SDDF model to make it scale to even larger scenes and be applicable in real-time reconstruction scenarios. One possible direction is to explore the combination of differentiable sphere tracing~\cite{dist2020} with scene-level SDF and SDDF learning, where the SDF and SDDF models improve each other to accelerate convergence.

Although our method is not very sensitive to the number of ellipsoids, the algorithmic addition and pruning of ellipsoids will obviate the need to choose the correct number and improve robustness to different scenes. The current implementation is not well optimized, although it is faster than SDF-based sphere tracing. Addressing these limitations is an interesting direction for future work.

\section{Conclusion}
\label{sec:conclusion}

In this work, we introduced a new definition of SDDF suitable for scene-level representation and differentiable rendering, and developed a hybrid explicit-implicit model to learn SDDF. Our method uses ellipsoids to obtain a coarse geometric prior and a residual network with latent features, obtained from the ellipsoids, to correct the differences between the prior and the ground truth. Our experiments demonstrate that SDDF is a promising scene representation for fast novel view rendering and gradient-based viewpoint optimization.

{
    \small
    \bibliographystyle{IEEEtran}
    \bibliography{bib/main.bib}
}

\newcommand{\biospacing}{-5.5ex}

\begin{IEEEbiography}[{\includegraphics[width=1in,height=1.25in,clip,keepaspectratio]{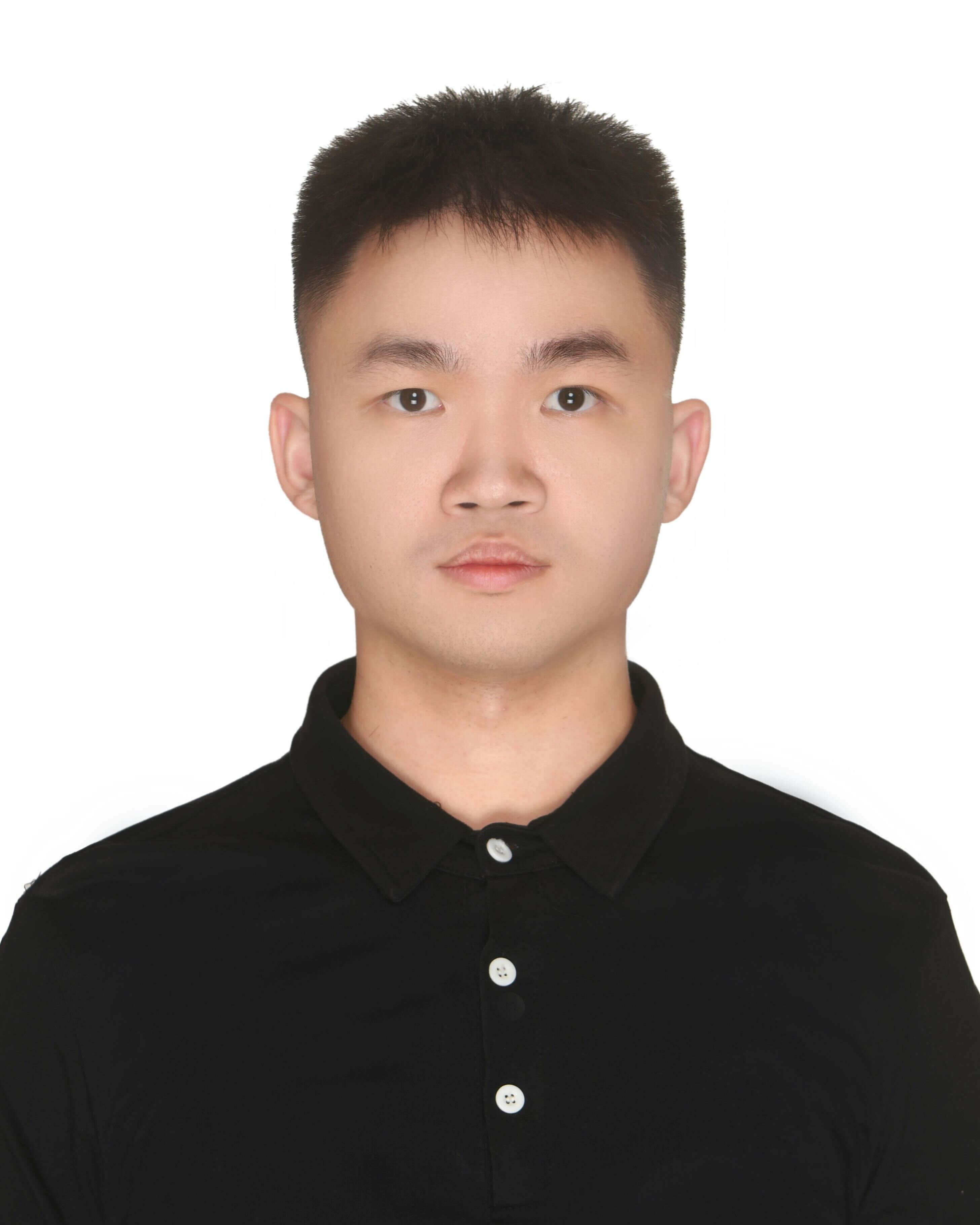}}]{Zhirui Dai} (Graduate Student Member, IEEE)
    is a Ph.D. student of Electrical and Computer Engineering at the University of California San Diego, La Jolla, CA, USA. He obtained a B.S. in Physics from Fudan University, Shanghai, China, in 2019 and an M.S. in Electrical and Computer Engineering from the University of California San Diego in 2021. His research focuses on mobile robot autonomy and particularly on mapping and task planning.
\end{IEEEbiography}

\vspace{\biospacing}

\begin{IEEEbiography}[{\includegraphics[width=1in,height=1.25in,clip,keepaspectratio]{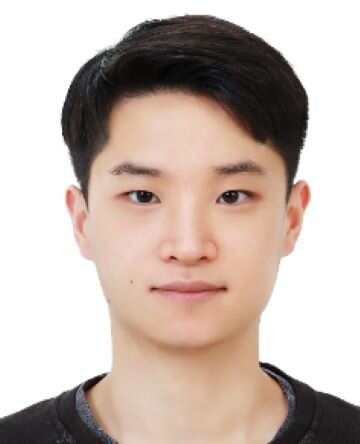}}]{Hojoon Shin}
    is an autonomy engineer at Brain Corporation. He obtained a B.S. in Mechanical and Aerospace Engineering from Seoul National University, Seoul, South Korea, in 2021 and an M.S. in Mechanical Engineering from the University of California San Diego in 2023. His current work focuses on autonomous robotics, with particular interest in navigation and motion planning.
\end{IEEEbiography}

\vspace{\biospacing}

\begin{IEEEbiography}[{\includegraphics[width=1in,height=1.25in,clip,keepaspectratio]{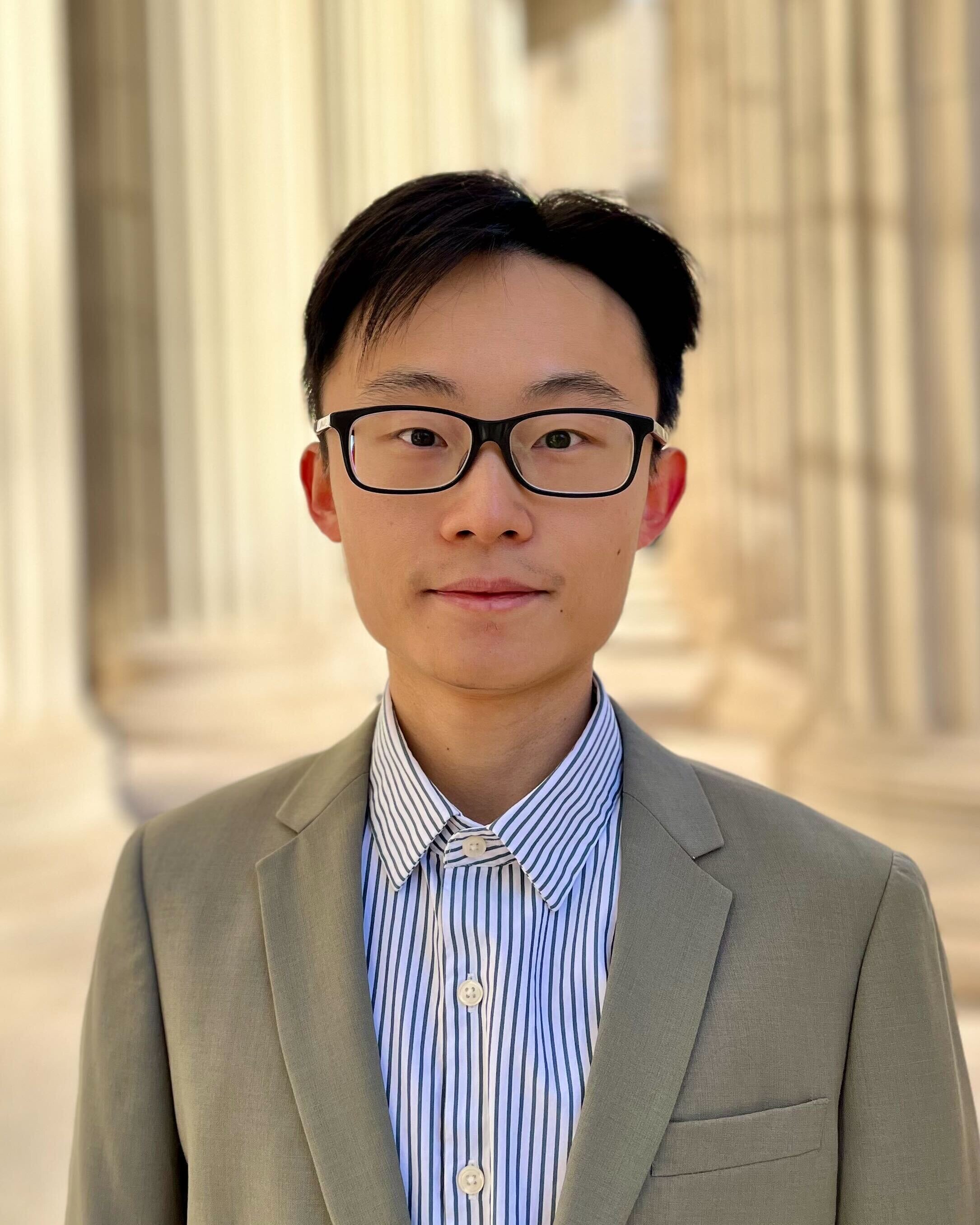}}]{Yulun Tian} (Member, IEEE) is an Assistant Professor of Robotics at the University of Michigan, Ann Arbor, MI, USA.
    Prior to this appointment, he was a Postdoctoral Scholar at the Contextual Robotics Institute, University of California San Diego, La Jolla, CA, USA.
    He received the B.A. degree in Computer Science from UC Berkeley, Berkeley, CA, USA, in 2017, and the S.M. and Ph.D. degrees
    in Aeronautics and Astronautics from Massachusetts Institute of Technology, Cambridge, MA, USA (2019 and 2023).
    His work received the 2024 Best Dissertation Award from the IEEE RAS Technical Committee for Multi-Robot Systems, the 2022 King-Sun Fu Memorial Best Paper Award from the IEEE Transactions on Robotics,
    a 2021 Honorable Mention from the IEEE Transactions on Robotics,
    and a 2020 Honorable Mention from the IEEE Robotics and Automation Letters.
    His current research interest includes spatial perception, trustworthy autonomy, and multi-agent systems.
\end{IEEEbiography}

\vspace{\biospacing}

\begin{IEEEbiography}[{\includegraphics[width=1in,height=1.25in,clip,keepaspectratio]{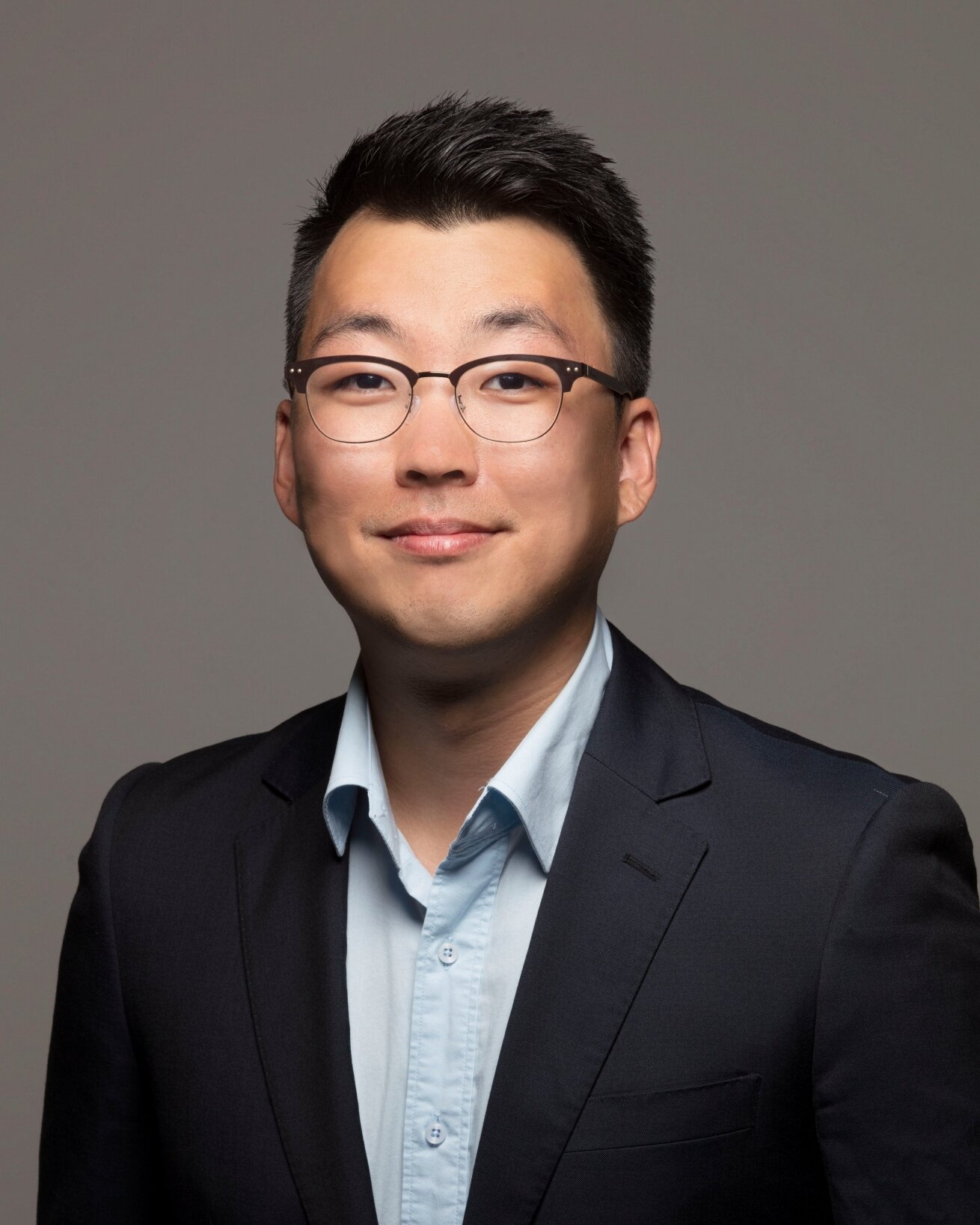}}]{Ki Myung Brian Lee}
    (Member, IEEE) is a Postdoctoral Scholar at the Contextual Robotics Institute, University of California San Diego, La Jolla, CA, USA.
    He received the B.Eng. (Hons I) degree in mechatronics (space) from the University of Sydney, Camperdown, NSW, Australia, and the Ph.D. degree in robotics from the University of Technology Sydney, Ultimo, NSW, Australia, in 2023.
    He was recognized as an RSS Pioneer of 2023, and was the recipient of the UTS Research Excellence Scholarship.
    His current research aims to develop novel representations of environments and tasks that accelerate planning and control. More broadly, he is interested in mobile robot autonomy in previously unseen environments.
\end{IEEEbiography}

\vspace{\biospacing}

\begin{IEEEbiography}[{\includegraphics[width=1in,height=1.25in,clip,keepaspectratio]{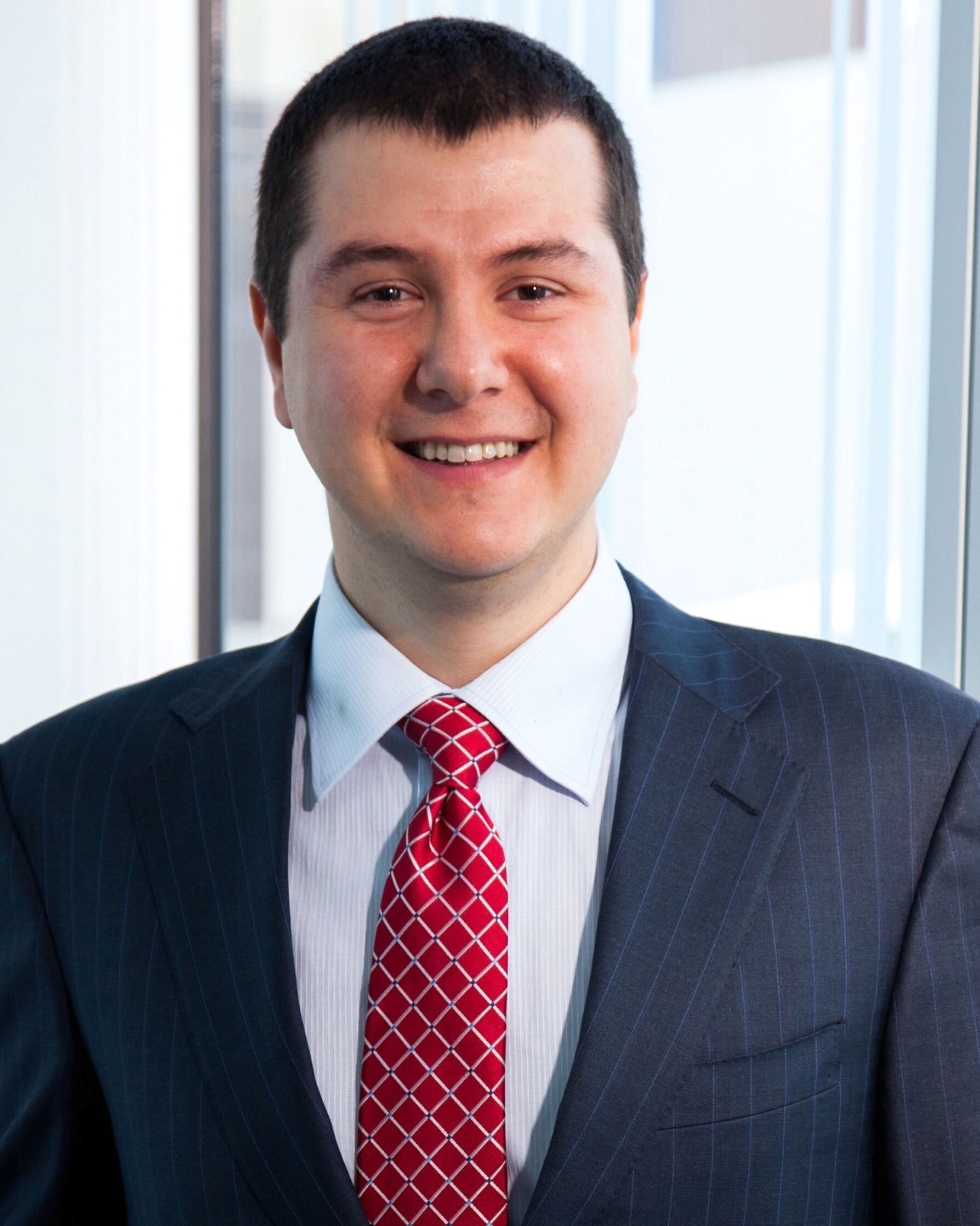}}]{Nikolay Atanasov}
    (S'07-M'16-SM'23) is an Associate Professor of Electrical and Computer Engineering at the University of California San Diego, La Jolla, CA, USA. He obtained a B.S. degree in Electrical Engineering from Trinity College, Hartford, CT, USA in 2008 and M.S. and Ph.D. degrees in Electrical and Systems Engineering from the University of Pennsylvania, Philadelphia, PA, USA in 2012 and 2015, respectively. Dr. Atanasov's research focuses on robotics, control theory, and machine learning with emphasis on active perception problems for autonomous mobile robots. He works on probabilistic models for simultaneous localization and mapping (SLAM) and on optimal control and reinforcement learning algorithms for minimizing probabilistic model uncertainty. Dr. Atanasov's work has been recognized by the Joseph and Rosaline Wolf award for the best Ph.D. dissertation in Electrical and Systems Engineering at the University of Pennsylvania in 2015, the Best Conference Paper Award at the IEEE International Conference on Robotics and Automation (ICRA) in 2017, the NSF CAREER Award in 2021, and the IEEE RAS Early Academic Career Award in Robotics and Automation in 2023.
\end{IEEEbiography}

\clearpage
\twocolumn[%
    \begin{@twocolumnfalse}
        \begin{center}
            {\huge Supplemental Material of \papertitle\par}
            \vspace{2em}
        \end{center}
    \end{@twocolumnfalse}%
]

\setcounter{section}{0}

\renewcommand{\theHsection}{S\arabic{section}}
\renewcommand{\theHsubsection}{S\arabic{section}.\arabic{subsection}}
\renewcommand{\theHsubsubsection}{S\arabic{section}.\arabic{subsection}.\arabic{subsubsection}}

\section{Extension to Unsigned Directional Distance Function}
\label{sec:appdx_ddf_extension}

For both SDF and SDDF, the sign divides the space into free space (outside objects) and occupied space (inside objects). Hence, both SDF and SDDF encounter an ill-posed definition of sign for non-watertight objects, such as a wall or a cup.

From an alternative point of view, if the obstacle set has measure zero both SDF and SDDF reduce to their corresponding unsigned versions. Our ellipsoidal prior can handle this case by setting one of the axis radii to zero, obtaining a 2D disk. The ellipsoid SDDF definition in Proposition~\ref{prop:ellipsoid_sddf} still holds for rays intersected with a 2D disk. Letting $r_z = 0$, we get the ellipsoid SDDF for a 2D disk, which is non-watertight in 3D:
\begin{equation}
    f(\bfp,\bfv;\calE) = -\frac{p'_z}{v'_z},
\end{equation}
where $\bfp'$ and $\bfv'$ are the ray origin and direction in the ellipsoid's local frame. The sign indicator in~\eqref{eq:ellipsoid_sign} becomes $s(\bfp,\bfv;\calE) = (r_1 r_2)^2 {p'}_z^2 \ge 0$, showing that the ellipsoid SDDF reduces to its unsigned version. When $v'_z = 0$, $f(\bfp,\bfv;\calE)=\infty$. There is no intersection. When $v'_z > 0$ and $p'_z \le 0$, or $v'_z < 0$ and $p'_z \ge 0$, the above equation gives a valid ellipsoid SDDF between the ray and the 2D disk when they intersect. Otherwise, the ray points in a direction away from the 2D disk and the above equation gives invalid negative values. To check the intersection, we can reuse the intersection indicator in~\eqref{eq:ellipsoid_intersection} with $r_z=0$:
\begin{equation}
    i(\bfp,\bfv;\calE) = (r_x r_y v'_z)^2 - r_x^2 {w'_x}^2 - r_y^2 {w'_y}^2,
\end{equation}
where $\bfw' = \bfp' \times \bfv'$. If $i(\bfp,\bfv;\calE) \ge 0$, the ray intersects with the 2D disk.

In practice, our method handles thin objects due to the unsigned DDF compatibility of the ellipsoid SDDF. As shown in Fig.~\ref{fig:vis-ellipsoids}, thin objects like walls, doors, or curtains are covered by flat ellipsoids.

It is an interesting future direction to unify the signed and unsigned versions of SDF/SDDF within a single framework. For areas around non-watertight objects, a proper transition between the signed and the unsigned versions would be an elegant way to keep the benefit of having sign for applications like collision detection.

\begin{figure}[t]
    \centering
    \includegraphics[width=\linewidth]{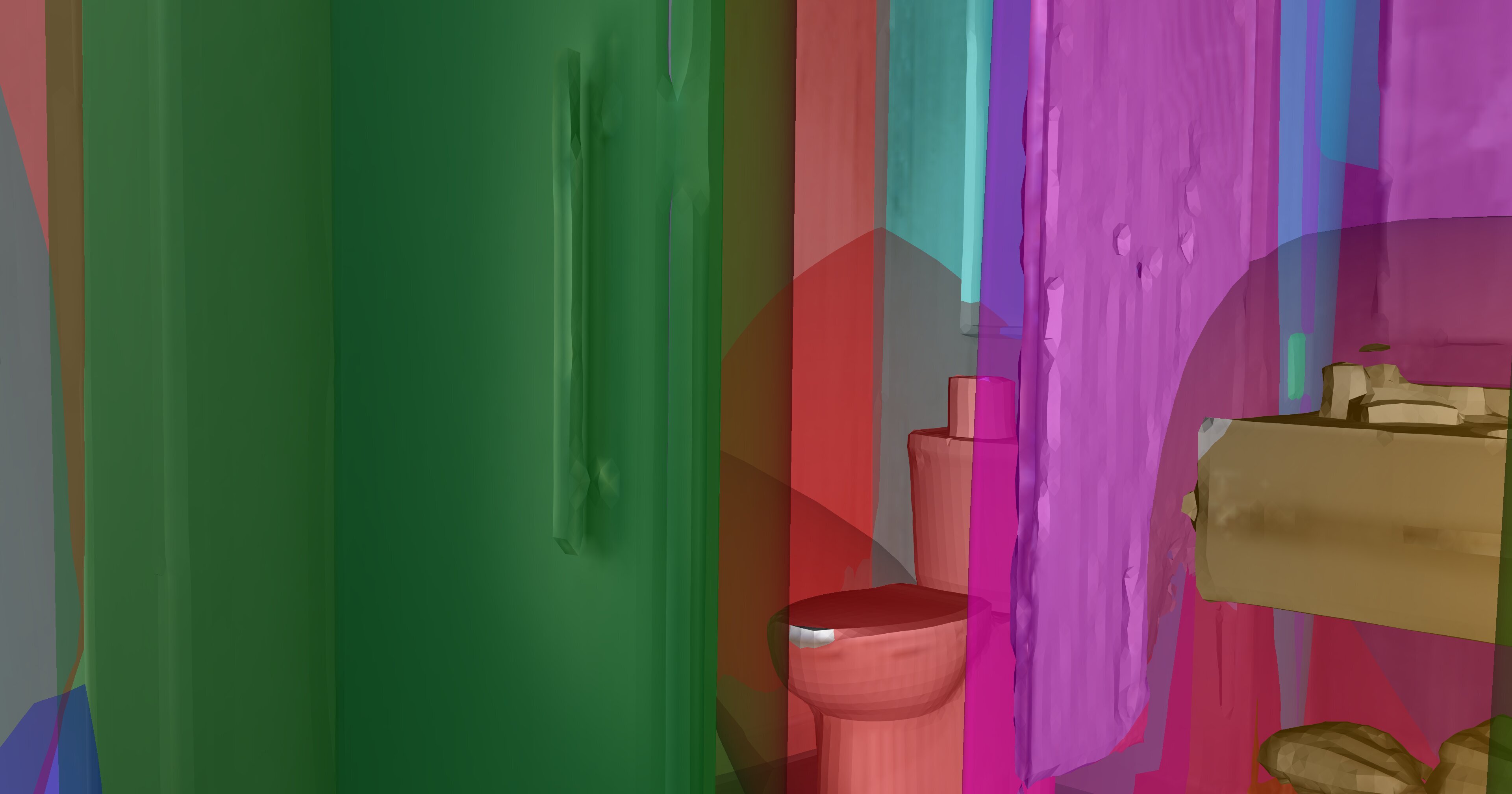}
    \caption{\small Ellipsoids covering non-watertight objects. The curtain is contained by a flat {\color{magenta}magenta} ellipsoid, while the door is covered by another flat {\color{ForestGreen} green} ellipsoid.}
    \label{fig:vis-ellipsoids}
\end{figure}

\section{Ellipsoid-based Prior Network}
\label{sec:appdx_geometry_prior_net}

\subsection{
\texorpdfstring
{Proof of Proposition~\ref{prop:ellipsoid_sddf}}  
{Proof of Proposition~\ref{prop:ellipsoid_sddf}}  
}
\label{sec:appdx_proof_prop_ellipsoid_sddf}

The intersection between the ray $(\bfp,\bfv)$ and the ellipsoid $\calE$ gives the following equations:
\begin{align}
    \bfq &= \bfR^\top \left(\bfp + f(\bfp,\bfv;\calE) \bfv - \bfc\right), \label{eq:intersection_point}\\
    0 &= \bfq^\top \bfQ_1^2 \bfq - \det\bfQ_0^2 \label{eq:ellipsoid},
\end{align}
where $\bfq$ is the intersection point expressed in coordinates in the local ellipsoid frame. Plugging~\eqref{eq:intersection_point} into~\eqref{eq:ellipsoid} leads to the following quadratic equation:
\begin{equation}
    \frac{\bfv'^\top \bfQ_1^2 \bfv'}{2}f^2 + \bfp'^\top \bfQ_1^2 \bfv' f + \frac{\bfp'^\top\bfQ_1^2\bfp'-\det\bfQ_0^2}{2} = 0.
\end{equation}
Solving the above equation for $f$, we get two roots:
\begin{equation} \label{eq:quadratic_eq_sols}
    f_{1,2} = -\frac{\sqrt{\alpha} \pm \bfp'^\top \bfQ_1^2 \bfv'}{\bfv'^\top \bfQ_1^2 \bfv'}
\end{equation}
where $\alpha$ is related to the intersection indicator $i(\bfp, \bfv)$ in~\eqref{eq:ellipsoid_intersection}:
\begin{equation*}
    \begin{aligned}
        \alpha &= \left(\bfp'^\top \bfQ_1^2 \bfv'\right)^2 - \bfv'^\top \bfQ_1^2 \bfv'\left(\bfp'^\top\bfQ_1^2\bfp'-\det\bfQ_0^2\right) \\
        &= \bfp'^\top\bfQ_1^2\left(\bfv'\bfp'^\top - \bfp'\bfv'^\top\right)\bfQ_1^2\bfv' + \bfv'^\top\bfQ_1^2\bfv'\det\bfQ_0^2 \\
        &= \bfp'^\top\bfQ_1^2 \bfw'^\wedge \bfQ_1^2 \bfv' + \bfv'^\top\bfQ_1^2\bfv'\det\bfQ_0^2 \\
        &= \left(\bfp'^\top \left(\bfQ_0^2\bfw'\right)^\wedge \bfv' + \bfv'^\top\bfQ_1^2\bfv'\right) \det\bfQ_0^2 \\
        &= \left(\bfv'^\top\bfQ_1^2\bfv' - \bfw'^\top\bfQ_0^2\bfw' \right) \det\bfQ_0^2 = i(\bfp,\bfv)\det\bfQ_0^2.
    \end{aligned}
\end{equation*}
According to the SDDF definition in~\eqref{eq:sddf_definition}, we keep $f_1$ as the solution, which is exactly~\eqref{eq:intersection_sol}.

\subsection{Backward Propagation} \label{sec:appdx_bp_ellipsoid_sddf}

To train the prior network efficiently, we derive and implement the backward propagation of computation in Sec.~\ref{sec:ellipsoid} and Sec.~\ref{sec:fuse_ellipsoid} explicitly in Python instead of relying on the computation graph of PyTorch.
Without loss of generality, we compute the backward propagation in the ellipsoid local frame.
To efficiently compute the gradients, let $t_0={\bfv'}^\top\bfQ_1^2\bfv'$, $t_1={\bfw'}^\top \bfQ_0^2 \bfw'$, $t_2=\bfp'^\top \bfQ_1^2 \bfv'$, then the intersection indicator becomes $i(\bfp,\bfv)=t_0-t_1$ and the SDDF prior is $f(\bfp,\bfv) = -({\det\bfQ_0\sqrt{\max(i,0)+\epsilon} + t_2})/{t_0}$.
Given the loss $\calL_P$ presented in Sec.~\ref{sec:loss_function_for_prior}, we derive the gradients $\nabla_{\bfp'}\calL_P$, $\nabla_{\bfv'}\calL_P$ and $\nabla_{\bfr}\calL_P$ via the chain rule of derivatives:
\begin{align}
    \nabla_{\bfp'}\calL_P &= \frac{\partial \calL_P}{\partial i}\nabla_{\bfp'}i + \frac{\partial \calL_P}{\partial f}\nabla_{\bfp'}f + \frac{\partial \calL_P}{\partial s}\nabla_{\bfp'}s, \\
    \nabla_{\bfv'}\calL_P &= \frac{\partial \calL_P}{\partial i}\nabla_{\bfv'}i + \frac{\partial \calL_P}{\partial f}\nabla_{\bfv'}f + \frac{\partial \calL_P}{\partial s}\nabla_{\bfv'}s, \\
    \nabla_{\bfr}\calL_P &= \frac{\partial \calL_P}{\partial i}\nabla_{\bfr}i + \frac{\partial \calL_P}{\partial f}\nabla_{\bfr}f + \frac{\partial \calL_P}{\partial s}\nabla_{\bfr}s, \\
    \nabla_{\bfr} f &= \frac{\partial f}{\partial i}\nabla_{\bfr}i + \frac{\partial f}{\partial t_0}\nabla_{\bfr}t_0
    + \frac{\partial f}{\partial t_2} \nabla_{\bfr}t_2 + \frac{\partial f}{\partial \bfr^\top}, \\
    \nabla_{\bfp'} f &= \frac{\partial f}{\partial i}\nabla_{\bfp'}i + \frac{\partial f}{\partial t_0}\nabla_{\bfp'}t_0
    + \frac{\partial f}{\partial t_2} \nabla_{\bfp'}t_2, \\
    \nabla_{\bfv'} f &= \frac{\partial f}{\partial i}\nabla_{\bfv'}i +
    \frac{\partial f}{\partial t_0}\nabla_{\bfv'}t_0 + \frac{\partial f}{\partial t_2} \nabla_{\bfv'}t_2.
\end{align}
And the related intermediate results are shown in Table~\ref{tab:ellipsoid_gradients} and Table~\ref{tab:ellipsoid_partial_derivatives}.

To compute the gradients $\nabla_{\bfp}\calL_P$, $\nabla_{\bfv}\calL_P$, $\nabla_{\bfR}\calL_P$, and $\nabla_{\bfc}\calL_P$, we can apply the chain rule of derivatives again.
Since
\begin{align}
    \bfp' &= \bfR^\top \left( \bfp - \bfc \right), \\
    \bfv' &= \bfR^\top \bfv,
\end{align}
we can get
\begin{align}
    \nabla_{\bfp}\calL_P &= \bfR \nabla_{\bfp'} \calL_P, \\
    \nabla_{\bfv}\calL_P &= \bfR \nabla_{\bfv'} \calL_P, \\
    \nabla_{\bfR}\calL_P &= \left(\bfp - \bfc\right)\nabla_{\bfp'}^\top \calL_P, \\
    \nabla_{\bfc}\calL_P &= -\bfR\nabla_{\bfp'}\calL_P,
\end{align}
which can be computed by the computation graph of PyTorch with little effort.
However, a custom implementation may save GPU memory and compute the results faster.

Then, to compute the gradients with respect to the pose parameters $\bfxi = (\bfrho, \bftheta) \in \bbR^6$, we refer to Theseus~\cite{pineda2022theseus}, which provides an efficient implementation of gradient computation for Lie algebra.
For convenience and readers' interest, we present the result of $\bfT=\bfT_0\exp\left(\bfxi^\wedge\right)$ (for simplicity, we assume $\bfT_0=\bfI$ here.) and gradient $\nabla_{\bfxi} \calL_P$ here:
\begin{equation}
\begin{aligned}
    \bfT &= \begin{bmatrix}
        \bfR & \bft \\
        \mathbf{0}^\top & 1
    \end{bmatrix} = \exp(\hat{\bfxi}) = \begin{bmatrix}
        \exp(\hat{\bftheta}) & J_R(\bftheta)^\top \bfrho \\
        \mathbf{0} & 1
    \end{bmatrix},
\end{aligned}
\end{equation}
\begin{align}
    \bfR &= \exp(\hat{\bftheta}) = \bfI + \frac{\sin\left\|\bftheta\right\|}{\left\|\bftheta\right\|}\hat{\bftheta}
    + \frac{1-\cos\left\|\bftheta\right\|}{\left\|\bftheta\right\|^2}\hat{\bftheta}^2, \\
    J_R(\bftheta) &= \bfI - \frac{1-\cos\left\|\bftheta\right\|}{\left\|\bftheta\right\|^2}\hat{\bftheta} +
    \frac{\left\|\bftheta\right\|-\sin\left\|\bftheta\right\|}{\left\|\bftheta\right\|^3}\hat{\bftheta}^2,
\end{align}
\begin{align}
    \nabla_{\bfxi}\calL_P &= \calJ_R \left( \bfxi \right)^\top \begin{bmatrix}
        \bfR^\top \nabla_{\bfc} \calL_P \\
        \bfW_{32} - \bfW_{23} \\
        \bfW_{13} - \bfW_{31} \\
        \bfW_{21} - \bfW_{12}
    \end{bmatrix},
\end{align}
where
\begin{align}
    \bfW &= \bfR^\top \nabla_{\bfR}\calL_P, \\
        \calJ_R(\bfxi) &= \begin{bmatrix}
        J_R(\bftheta) & Q_R(\bfxi) \\
        \mathbf{0}_{3\times3} & J_R(\bftheta)
    \end{bmatrix}, \\
    Q_R(\bfxi) &= \bfR^\top \bfL.
\end{align}

\begin{table}[t]
    \centering
    \caption{Gradients used in backward propagation of $G$, $\otimes$ denotes element-wise multiplication, and $\bfU = \bfr^\wedge \otimes \bfr^\wedge$.}
    \label{tab:ellipsoid_gradients}
    \setlength{\tabcolsep}{5pt}
    \begin{tabular}{c|c|c|c}
               & $\nabla_{\bfp'}$ & $\nabla_{\bfv'}$ & $\nabla_{\bfr}$ \\ \hline
         $t_0$ & $\mathbf{0}$ & $2\bfQ_1^2\bfv'$ & $2\bfQ_0\bfU(\bfv'\otimes\bfv')$ \\
         $t_1$ & $2 \bfv'\times \bfQ_0^2\bfw'$ & $-2\bfp'\times\bfQ_0^2\bfw'$ & $2 \left(\diag{\bfw'}\right)^2\bfr$ \\
         $t_2$ & $ \bfQ_1^2\bfv'$ & $\bfQ_1^2\bfp'$ & $2\bfQ_0\bfU (\bfp'\otimes\bfv')$ \\
         $i$ & $\nabla_{\bfp'}t_0-\nabla_{\bfp'}t_1$ & $\nabla_{\bfv'}t_0-\nabla_{\bfv'}t_1$ & $\nabla_{\bfr}t_0-\nabla_{\bfr}t_1$\\
         $s$ & $2\bfQ_1^2\bfp'$ & $\mathbf{0}$ & $2\bfQ_0\left(\bfU(\bfp'\otimes\bfp')-\bfQ_1^2\mathbf{1}\right)$
    \end{tabular}
    \setlength{\tabcolsep}{6pt}
\end{table}
\begin{table}[t]
    \centering
    \caption{Partial derivatives used in backward propagation of $G$. When $t_0\approx 0$, which is a degraded configuration, $\frac{\partial f}{\partial i}$, $\frac{\partial f}{\partial t_0}$ and $\frac{\partial f}{\partial t_2}$ are set to 0.}
    \label{tab:ellipsoid_partial_derivatives}
    \begin{tabular}{c|c|c|c|c}
          & ${\partial}/{\partial i}$ & ${\partial}/{\partial t_0}$ & ${\partial}/{\partial t_2}$ & ${\partial}/{\partial\bfr^\top}$ \\ \hline
        $f$ & $-\frac{\det{\bfQ_0}}{2 t_0 \sqrt{i + \epsilon}}$ & $-\frac{f}{t_0}$ & $-\frac{1}{t_0}$ & $-\frac{\sqrt{\max(i,0)+\epsilon}}{t_0}\bfQ_1 \mathbf{1}$
    \end{tabular}
\end{table}
And
\begin{equation}
\begin{aligned}
    &\bfL=\bfR Q_R(\bfxi) = \frac{\cos\left\|\bftheta\right\|-1}{\left\|\bftheta\right\|^2}\hat{\bfrho}\\
    &-\frac{\left\|\bftheta\right\|-\sin\left\|\bftheta\right\|}{\left\|\bftheta\right\|^3}
    \left(\left(\bftheta\times\bfrho\right)^\wedge + \bfrho\bftheta^\top - \bftheta^\top\bfrho\bfI\right)\\
    &+\left(\frac{\sin\left\|\bftheta\right\|}{\left\|\bftheta\right\|^3}+\frac{2(\cos\left\|\bftheta\right\|-1)}{\left\|\bftheta\right\|^4}\right) (\bftheta\times\bfrho)\bftheta^\top \\
    &-\left(\frac{\cos\left\|\bftheta\right\|-1}{\left\|\bftheta\right\|^4}+\frac{3(\left\|\bftheta\right\|-\sin\left\|\bftheta\right\|)}{\left\|\bftheta\right\|^5}\right) \left(\bftheta\times(\bftheta\times\bfrho)\bftheta^\top\right),
\end{aligned}
\end{equation}
where many intermediate results from computing $\bfR$ and $J_R(\bftheta)$ can be reused.
In our experiments, we find that this custom implementation of gradient computation is essential for using the GPU efficiently.
Otherwise, the PyTorch computation graph will take too much GPU memory and compute the gradients slowly.
Currently, this backward propagation is implemented in Python, which could be dramatically accelerated by a CUDA alternative.

\begin{table*}[t]
    \caption{\small Comparison of SDDF Prediction by Nerfacto trained w/ and w/o DS-NeRF loss. The units of min, max and mean are meters.}\label{tab:sddf_vis_comp_nerfacto}
    \newcommand{\rowlabel}[1]{\rotatebox[origin=l]{90}{#1}}
    \centering
    \begin{tabular}{c|ccccc}
        \hline
        & Allensville & Hotel & Office 0 & Office 1 & Office 2
        \\
        \hline
        \rowlabel{\quad \quad \quad w/o} &
        \includegraphics[width=0.16\linewidth]{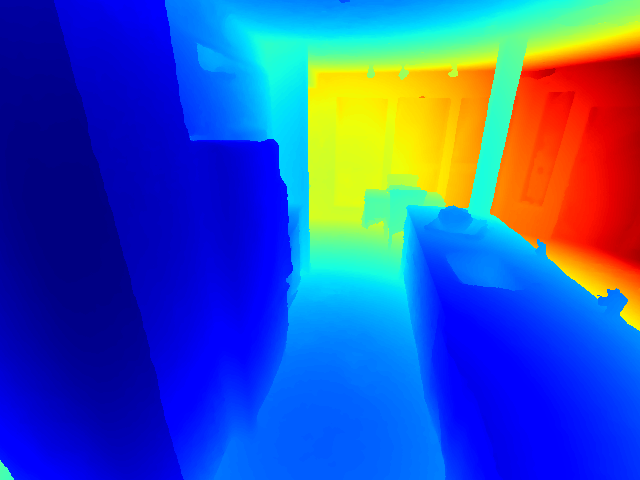} &
        \includegraphics[width=0.16\linewidth]{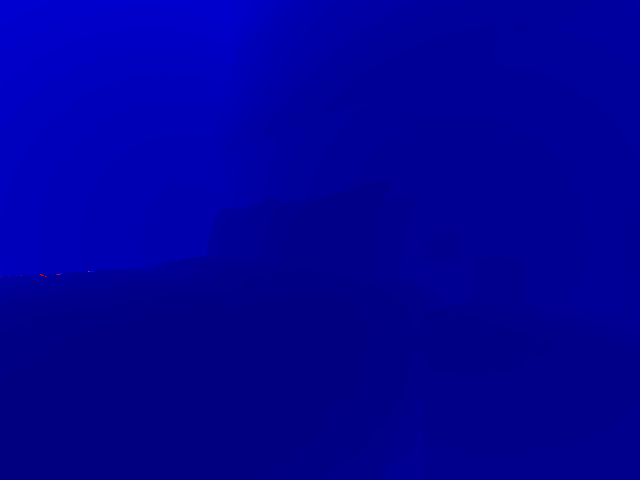}       &
        \includegraphics[width=0.16\linewidth]{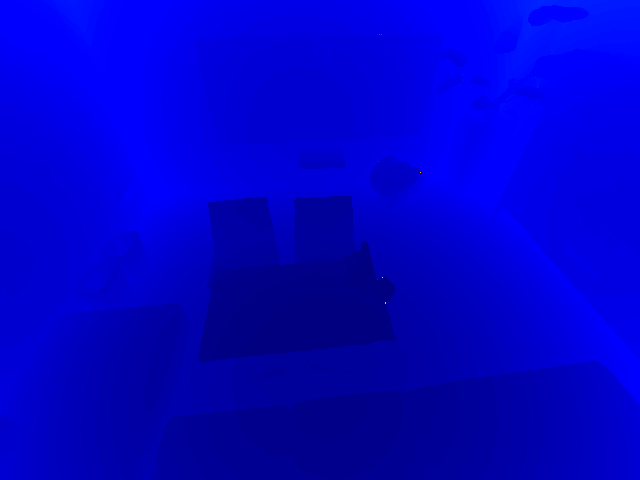}    &
        \includegraphics[width=0.16\linewidth]{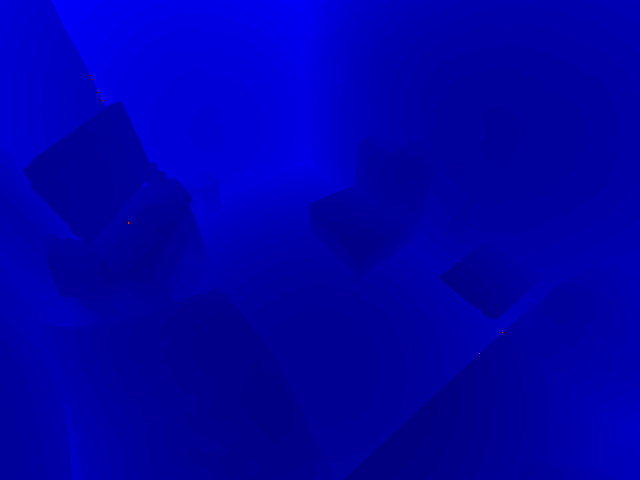}    &
        \includegraphics[width=0.16\linewidth]{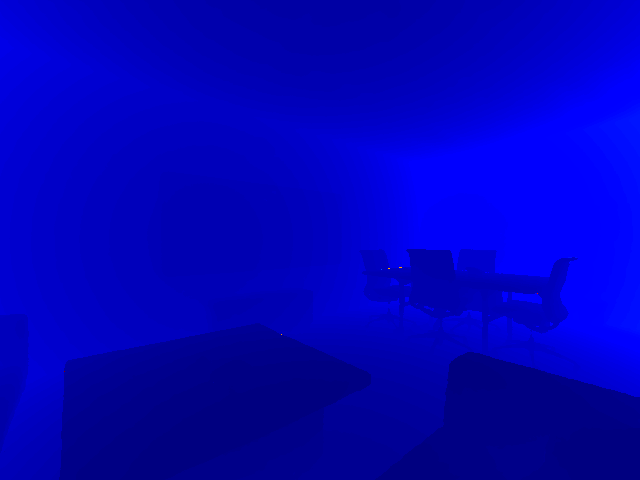}    \\
        {\small Min}  & 0.628 & 0.494  & 1.939  & 1.637  & 1.000  \\
        {\small Max}  & 6.170 & 37.747 & 21.298 & 19.184 & 36.283 \\
        {\small Mean} & 2.537 & 1.462  & 3.244  & 2.482  & 3.469  \\
        \hline\hline
        \rowlabel{\quad\quad \quad w/} &
        \includegraphics[width=0.16\linewidth]{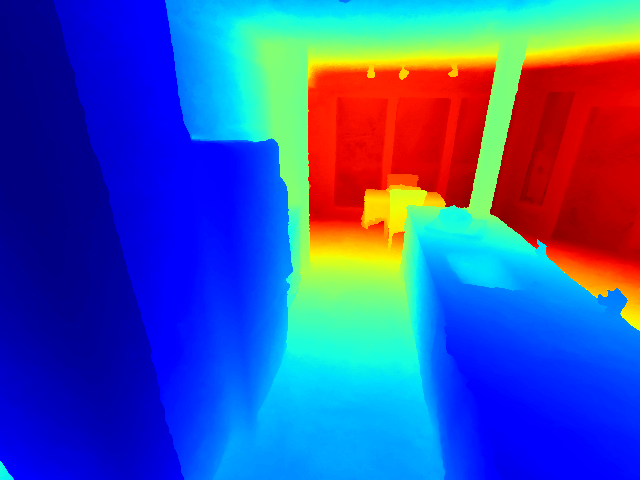} &
        \includegraphics[width=0.16\linewidth]{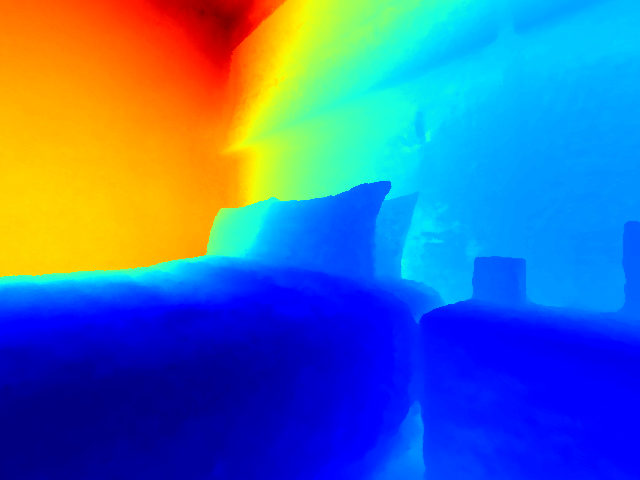}       &
        \includegraphics[width=0.16\linewidth]{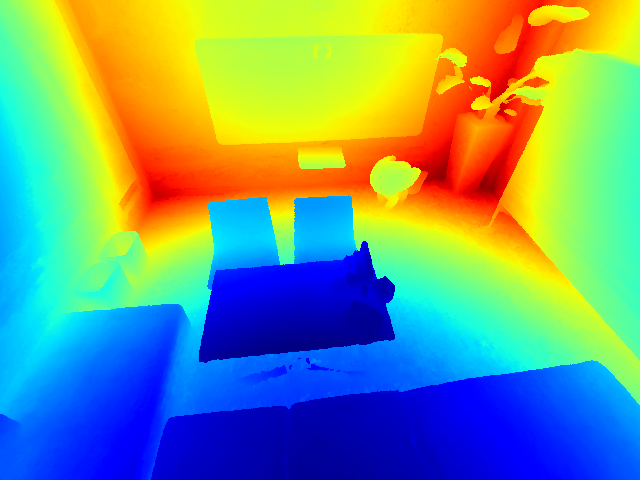}    &
        \includegraphics[width=0.16\linewidth]{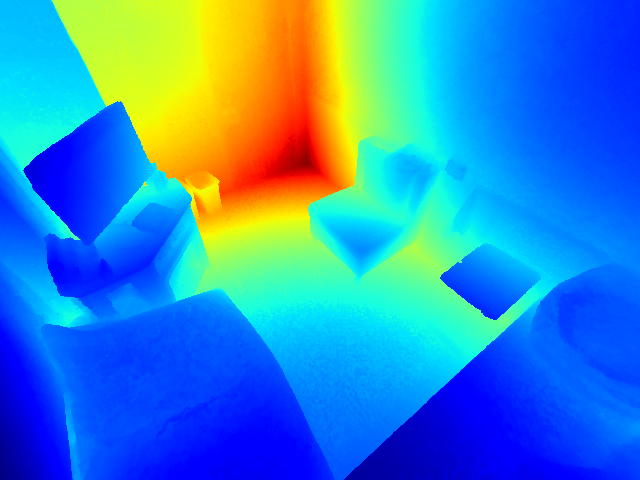}    &
        \includegraphics[width=0.16\linewidth]{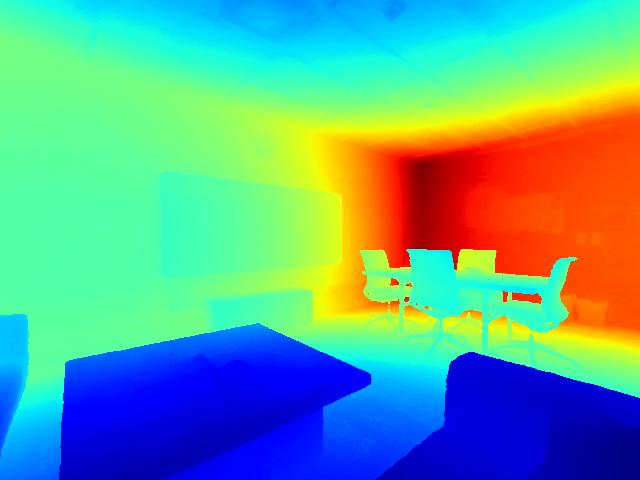}    \\
        {\small Min}  & 0.632 & 0.500 & 1.940 & 1.638 & 1.003 \\
        {\small Max}  & 6.177 & 3.416 & 4.900 & 4.258 & 6.739 \\
        {\small Mean} & 2.549 & 1.478 & 3.264 & 2.511 & 3.516 \\
        \hline
    \end{tabular}
\end{table*}

\subsection{
\texorpdfstring
{Proof of Proposition~\ref{prop:prior_eikonal_equation}}
{Proof of Proposition~\ref{prop:prior_eikonal_equation}}
}
\label{sec:appdx_proof_prop_eikonal_equation}

Without loss of generality, we can prove Proposition~\ref{prop:prior_eikonal_equation} in the local frame of the selected ellipsoid. According to Table~\ref{tab:ellipsoid_gradients} and Table~\ref{tab:ellipsoid_partial_derivatives}, we have
\begin{align}
    {\bfv'}^\top \nabla_{\bfp'} t_0 & = 0, \\
    {\bfv'}^\top \nabla_{\bfp'} t_1 & = 0, \\
    {\bfv'}^\top \nabla_{\bfp'} t_2 &= t_0,
\end{align}
where $t_0={\bfv'}^\top\bfQ_1^2\bfv'$, $t_1={\bfw'}^\top \bfQ_0^2 \bfw'$, $t_2=\bfp'^\top \bfQ_1^2 \bfv'$, then the intersection indicator becomes $i(\bfp,\bfv)=t_0-t_1$ and the SDDF prior is $f(\bfp,\bfv) = -({\det\bfQ_0\sqrt{\max(i,0)+\epsilon} + t_2})/{t_0}$.
So, we get:
\begin{align}
    {\bfv'}^\top \nabla_{\bfp'} f + 1 &= 0, \label{eq:eikonal_prior}
\end{align}
which is the directional Eikonal equation for SDDF.

\section{Residual Network}

\subsection{
\texorpdfstring
{Proof of Proposition~\ref{prop:eikonal_equation_with_residual}}
{Proof of Proposition~\ref{prop:eikonal_equation_with_residual}}
}
\label{sec:appdx_proof_prop_eikonal_equation_with_residual}

The SDDF $\hat{f}(\bfp,\bfv)$ in~\eqref{eq:final_sddf} computed by the combination of the prior and the residual networks satisfies the SDDF directional Eikonal equation in~\eqref{eq:sddf_eikonal_equation}.

Since the ellipsoid prior network $P$ satisfies~\eqref{eq:sddf_eikonal_equation} according to Proposition~\ref{prop:prior_eikonal_equation}, we have ${\bfv'}^\top\nabla_{\bfp'}f=-1$.
Then, because
\begin{equation}
\begin{aligned}
    {\bfv'}^\top \nabla_{\bfp'} \delta_f &= {\bfv'}^\top \nabla_{\bfp'}\bfq \nabla_{\bfq}\delta_f \\
    &= {\bfv'}^\top \left( \bfI + \nabla_{\bfp'}f {\bfv'}^\top \right) \nabla_{\bfq}\delta_f \\
    &= \mathbf{0}^\top \nabla_{\bfq}\delta_f = 0,
\end{aligned}
\end{equation}
we get ${\bfv'}^\top \nabla_{\bfp'}\hat{f} = {\bfv'}^\top\nabla_{\bfp'}f + {\bfv'}^\top\nabla_{\bfp'}\delta_f = -1$.

\section{Training Details} \label{sec:appdx_training_details}

\subsection{Training of Our SDDF Model}
\label{sec:appdx_training_of_sddf}

We use $256$ ellipsoids for the Allensville scene and $128$ for other scenes.
The dimension of the latent feature is $m=256$.
The decoder is a 7-layer MLP of dimensions $[256, 256, 512, 512, 256, 128, 64]$ with skip connections at the first and third layers.
LeakyReLU is used as the nonlinear layer.

Our model is optimized by Adam with a learning rate $0.001$ for the first $150$ epochs and $0.0001$ for another $150$ epochs.
Each batch for training has $512$k rays.
The prior network is pretrained for $19$k batch iterations and then trained together with the residual network for another $1$k batch iterations.
Then, the prior network is frozen, and we optimize the residual network only for the remaining epochs.

\subsection{Training of Nerfacto}
\label{sec:appdx_training_of_nerfacto}

We train Nerfacto~\cite{nerfstudio2023} with the default parameter settings from Nerfstudio. Nerfacto is trained without the DS-NeRF depth loss~\cite{kangle2021dsnerf}, while depth-Nerfacto is trained with this depth loss enabled. The weight of the depth loss is $0.001$ by default. The other hyperparameters are set to default values. Both settings were trained for $50$k iterations without camera pose optimization because ground-truth camera poses were provided.

As shown in Table~\ref{tab:sddf_vis_comp_nerfacto}, adding the DS-NeRF loss helps to prevent unreasonably larger SDDF predictions. However, the min and mean of the predictions are similar with and without the DS-NeRF loss. We think that this is due to the volume density approach designed for photometric rendering in Nerfacto, which leads to poor geometric consistency and inaccurate distance predictions at different viewpoints.

\subsection{Training of RaDe-GS}
\label{sec:appdx_training_of_gs}

We train RaDe-GS~\cite{radegs2024} with the default parameter settings from the official implementation. To initialize the Gaussians, 5\% of the points in the training set are sampled (3\% for the Allensville scene). For each scene, RaDe-GS is trained for 30k iterations with RGB-D images downsampled to 320$\times$240 resolution due to GPU memory limits.

\subsection{Training of \sdfbaseline}
\label{sec:appdx_training_of_sdf}

For SDDF, we augment the dataset with negative samples to improve the performance. Similarly, we generate two kinds of samples for training \sdfbaseline: SDF samples near the surface $\bfx_i \in \calX_\text{near}$ with distance labels $y_i \in \bbR$ and free space positions $\calX_\text{free}$ that provides only sign supervision. The SDF baseline is trained with SDF loss $l_\text{SDF}$, sign loss $l_\text{sign}$ and Eikonal loss $l_\text{eik}$:
\begin{align}
    l_\text{SDF} &= \frac{1}{\|\calX_\text{near}\|} \sum_{\bfx_i \in \calX_\text{near}} |f(\bfx_i)-y_i|, \\
    l_\text{sign} &= \frac{1}{\|\calX_\text{free}\|} \sum_{\bfx_j \in \calX_\text{free}} \exp\left(-\lambda f(\bfx_j)\right),
\end{align}
\begin{align}
    l_\text{eik} &= \frac{1}{\|\calX_\text{near}\| + \|\calX_\text{free}\|} \sum_{\bfx_k \in \calX_\text{near} \cup \calX_\text{free}} \left| \|\nabla f(\bfx_i)\| - 1 \right|.
\end{align}
The near-surface samples and the SDF loss help the model learn the surface geometry. Meanwhile, $l_\text{sign}$ encourages the model to predict large positive SDF at positions in the free space, while $l_\text{eik}$ guarantees that the learned SDF satisfies the Eikonal property.

The SDF model is trained with a batch size of $60$, Adam optimizer, and a learning rate of $0.0001$ for $300$ epochs.

\section{More Visualization} \label{sec:appdx_more_visualization}

\subsection{SDDF Prediction}
\label{sec:appdx_more_vis_sddf_prediction}

We present additional visualizations of SDDF predictions by our method for different scenes in Fig.~\ref{fig:more_sddf_qualitative_results}.
In Figs.~\ref{fig:more_sddf_comp} and~\ref{fig:more_sddf_comp_scannet}, more examples are shown to compare our method with the baselines.

\subsection{View Optimization Results}
\label{sec:appdx_more_vis_view_optimization}

We provide more qualitative results of view optimization in Figs.~\ref{fig:traj_opt_gibson_allensville},~\ref{fig:traj_opt_replica_hotel0}, and~\ref{fig:traj_opt_replica_office3}.

\begin{figure*}
    \centering
    \begin{subfigure}[t]{0.24\linewidth}
        \centering
        \includegraphics[width=\linewidth,trim={0pt 6pt 0pt 6pt},clip]{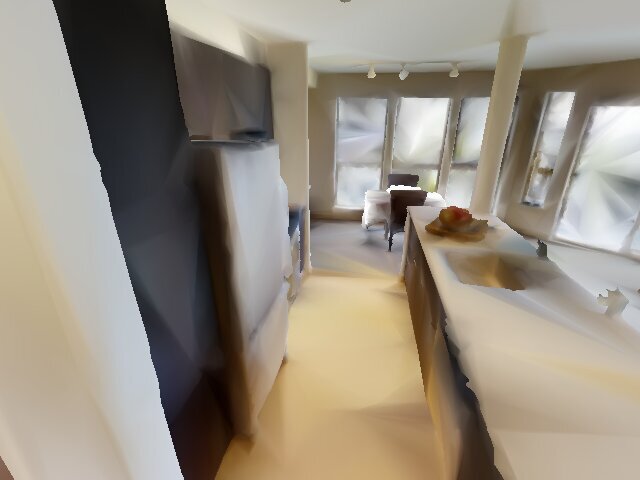}
    \end{subfigure}
    \begin{subfigure}[t]{0.24\linewidth}
        \centering
        \includegraphics[width=\linewidth,trim={0pt 6pt 0pt 6pt},clip]{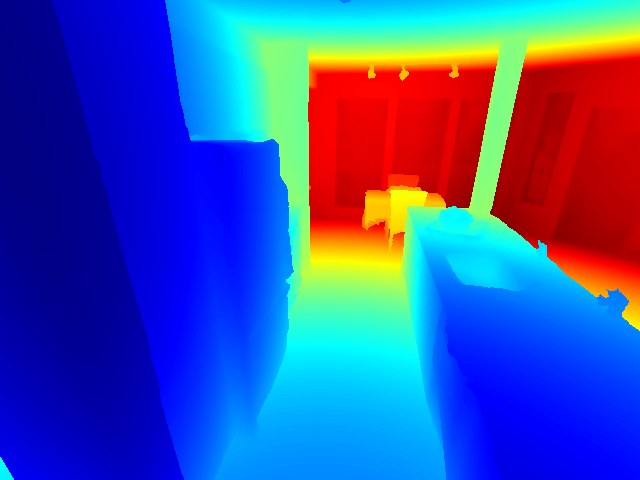}
    \end{subfigure}
    \begin{subfigure}[t]{0.24\linewidth}
        \centering
        \includegraphics[width=\linewidth,trim={0pt 6pt 0pt 6pt},clip]{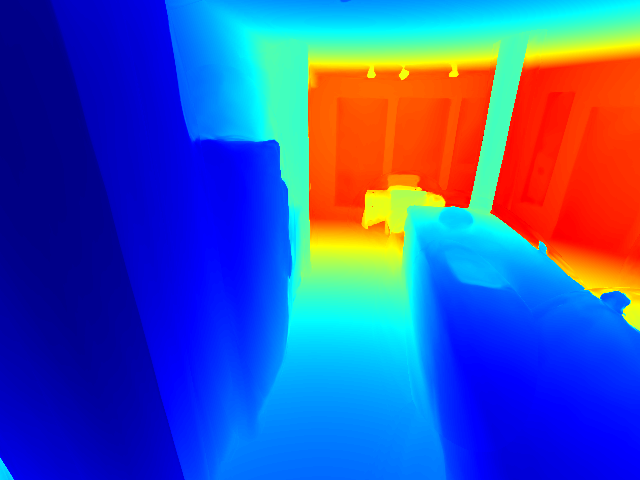}
    \end{subfigure}
    \begin{subfigure}[t]{0.24\linewidth}
        \centering
        \includegraphics[width=\linewidth,trim={0pt 6pt 0pt 6pt},clip]{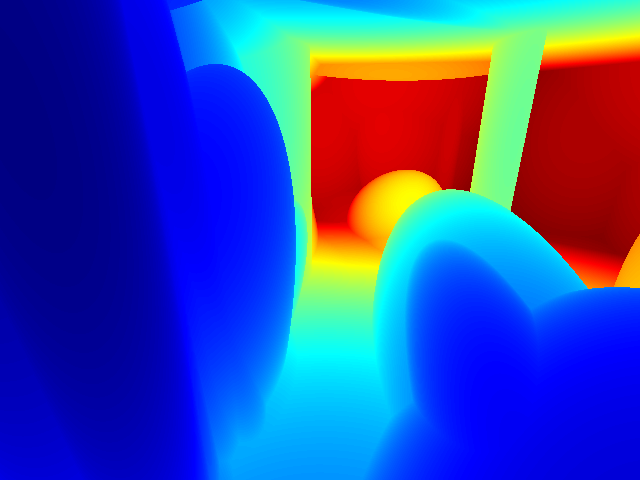}
    \end{subfigure}
    \begin{subfigure}[t]{0.24\linewidth}
        \centering
        \includegraphics[width=\linewidth,trim={0pt 6pt 0pt 6pt},clip]{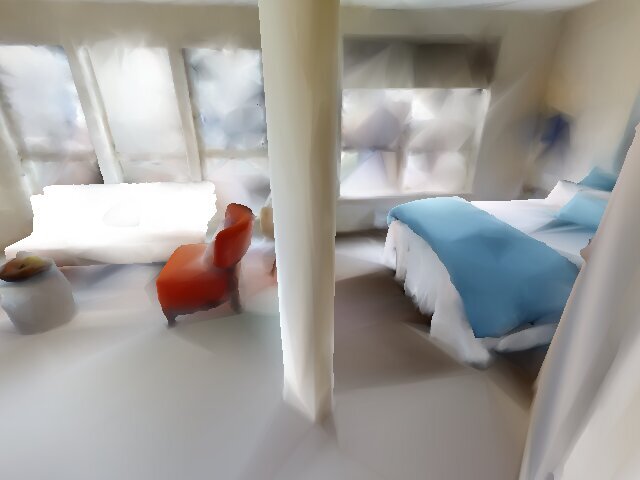}
    \end{subfigure}
    \begin{subfigure}[t]{0.24\linewidth}
        \centering
        \includegraphics[width=\linewidth,trim={0pt 6pt 0pt 6pt},clip]{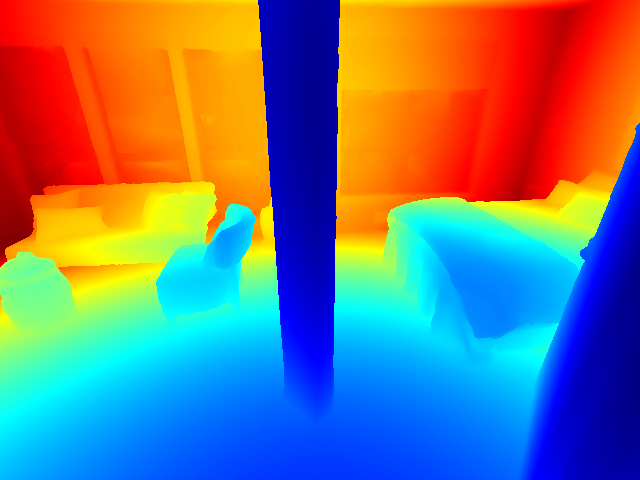}
    \end{subfigure}
    \begin{subfigure}[t]{0.24\linewidth}
        \centering
        \includegraphics[width=\linewidth,trim={0pt 6pt 0pt 6pt},clip]{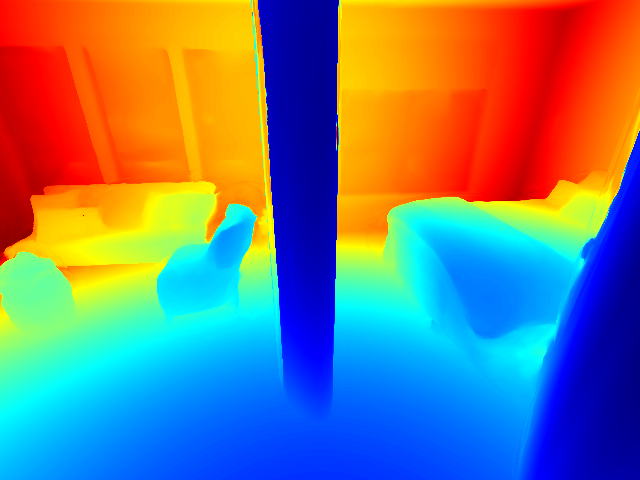}
    \end{subfigure}
    \begin{subfigure}[t]{0.24\linewidth}
        \centering
        \includegraphics[width=\linewidth,trim={0pt 6pt 0pt 6pt},clip]{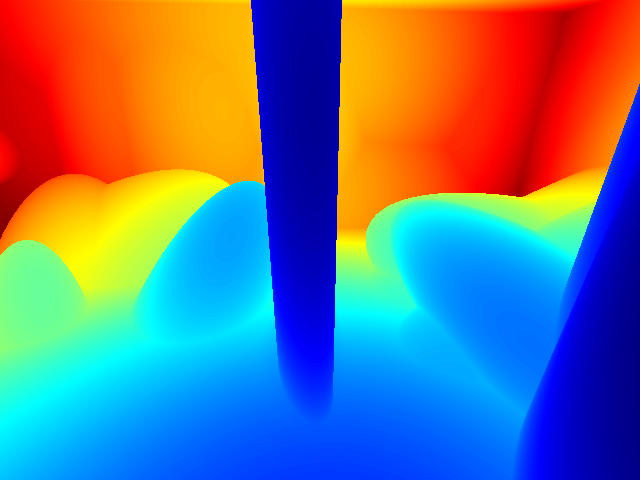}
    \end{subfigure}
    \begin{subfigure}[t]{0.24\linewidth}
        \centering
        \includegraphics[width=\linewidth,trim={0pt 6pt 0pt 6pt},clip]{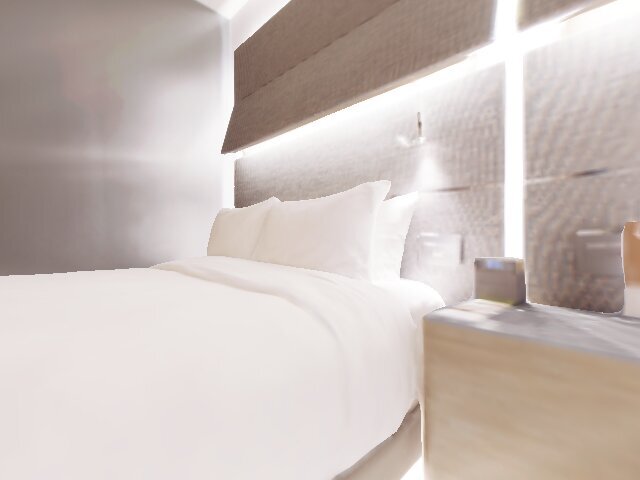}
    \end{subfigure}
    \begin{subfigure}[t]{0.24\linewidth}
        \centering
        \includegraphics[width=\linewidth,trim={0pt 6pt 0pt 6pt},clip]{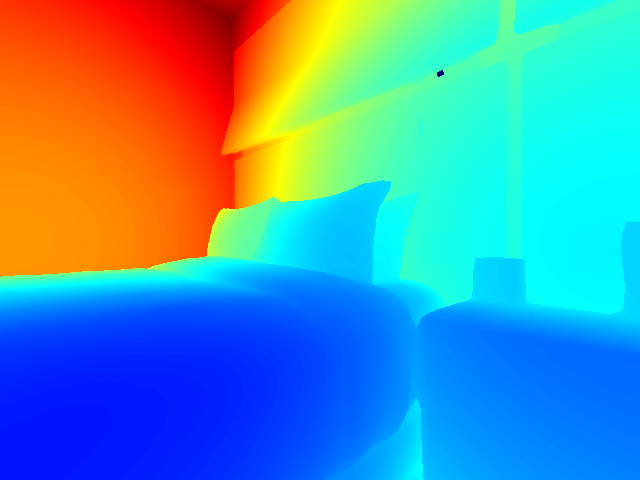}
    \end{subfigure}
    \begin{subfigure}[t]{0.24\linewidth}
        \centering
        \includegraphics[width=\linewidth,trim={0pt 6pt 0pt 6pt},clip]{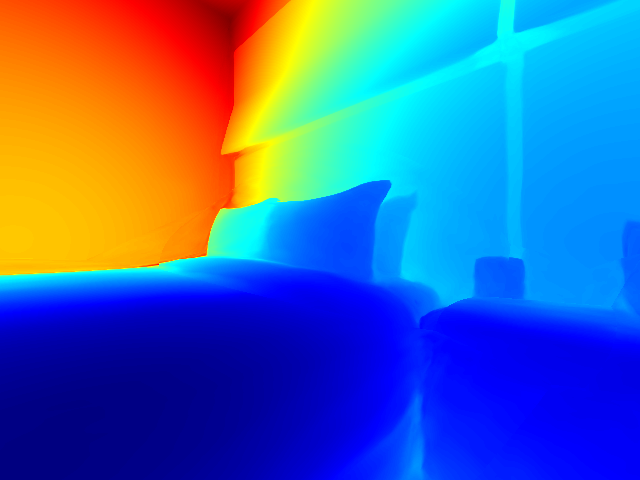}
    \end{subfigure}
    \begin{subfigure}[t]{0.24\linewidth}
        \centering
        \includegraphics[width=\linewidth,trim={0pt 6pt 0pt 6pt},clip]{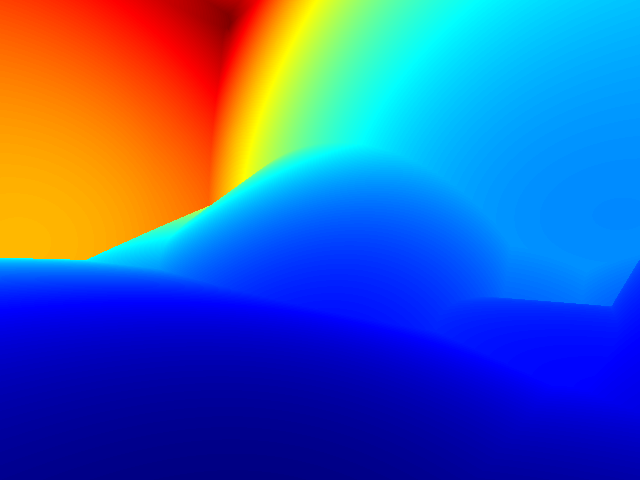}
    \end{subfigure}
    \begin{subfigure}[t]{0.24\linewidth}
        \centering
        \includegraphics[width=\linewidth,trim={0pt 6pt 0pt 6pt},clip]{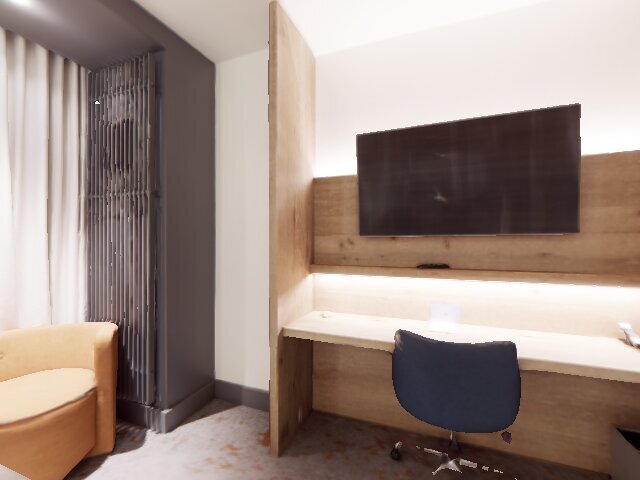}
    \end{subfigure}
    \begin{subfigure}[t]{0.24\linewidth}
        \centering
        \includegraphics[width=\linewidth,trim={0pt 6pt 0pt 6pt},clip]{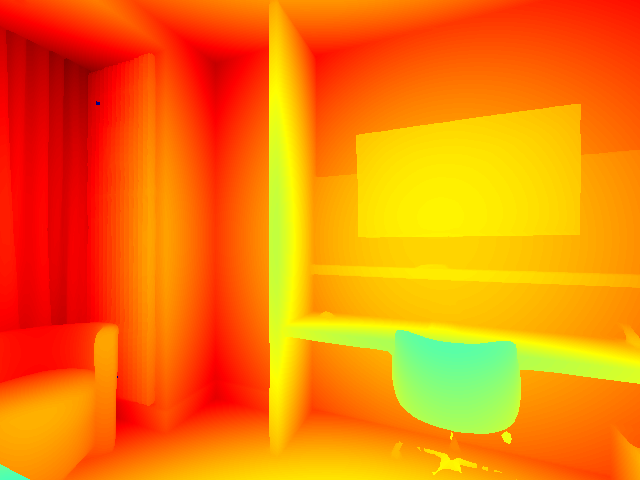}
    \end{subfigure}
    \begin{subfigure}[t]{0.24\linewidth}
        \centering
        \includegraphics[width=\linewidth,trim={0pt 6pt 0pt 6pt},clip]{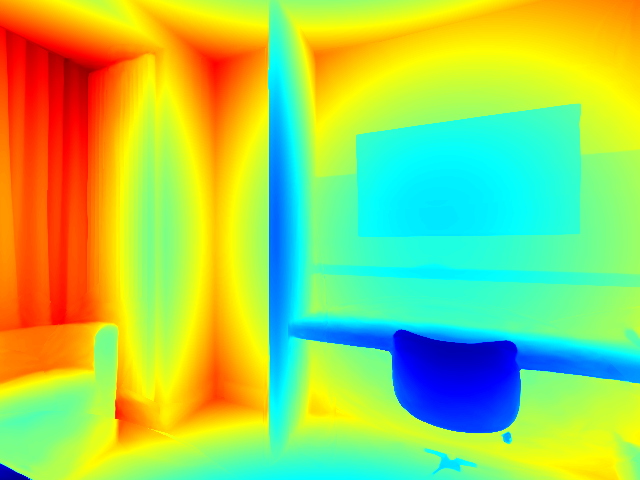}
    \end{subfigure}
    \begin{subfigure}[t]{0.24\linewidth}
        \centering
        \includegraphics[width=\linewidth,trim={0pt 6pt 0pt 6pt},clip]{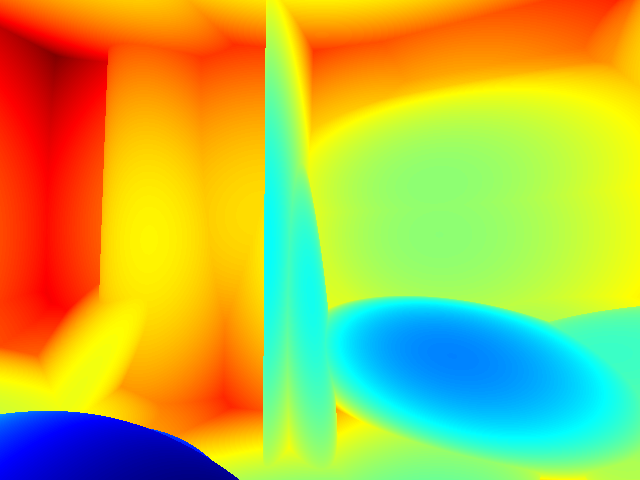}
    \end{subfigure}
    \begin{subfigure}[t]{0.24\linewidth}
        \centering
        \includegraphics[width=\linewidth,trim={0pt 6pt 0pt 6pt},clip]{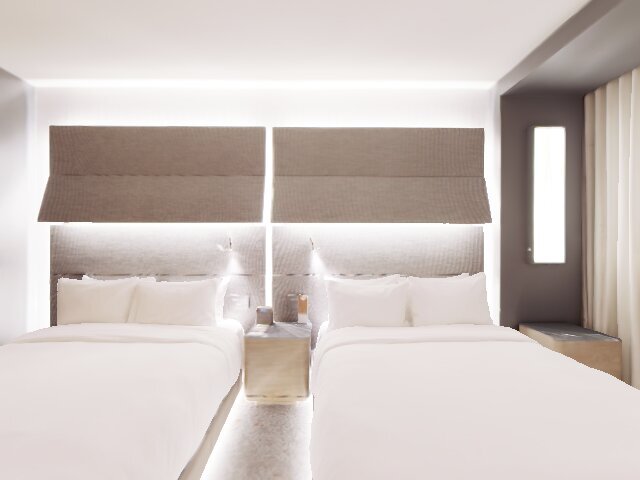}
    \end{subfigure}
    \begin{subfigure}[t]{0.24\linewidth}
        \centering
        \includegraphics[width=\linewidth,trim={0pt 6pt 0pt 6pt},clip]{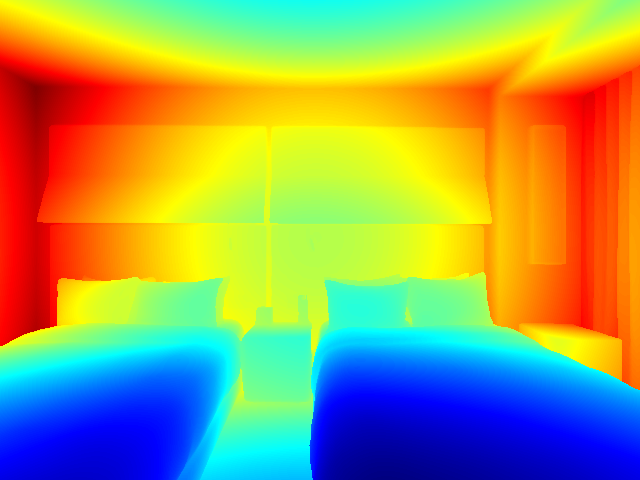}
    \end{subfigure}
    \begin{subfigure}[t]{0.24\linewidth}
        \centering
        \includegraphics[width=\linewidth,trim={0pt 6pt 0pt 6pt},clip]{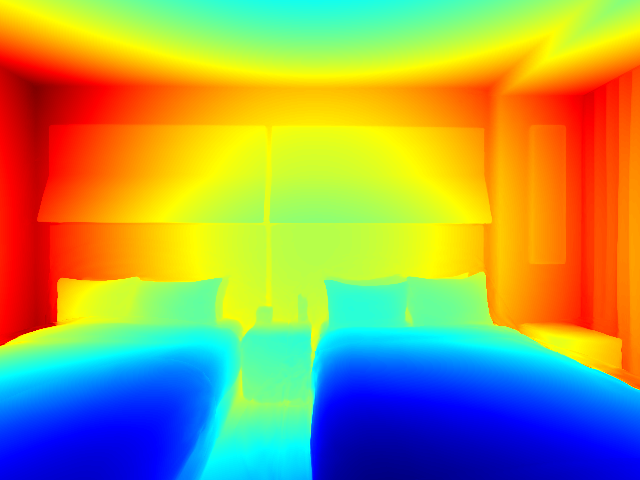}
    \end{subfigure}
    \begin{subfigure}[t]{0.24\linewidth}
        \centering
        \includegraphics[width=\linewidth,trim={0pt 6pt 0pt 6pt},clip]{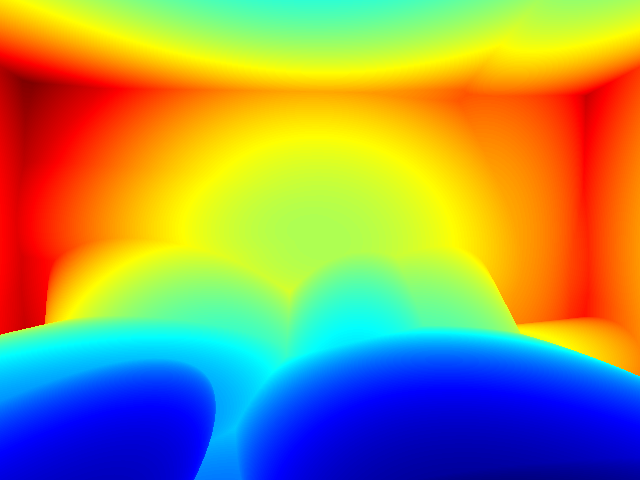}
    \end{subfigure}
    \begin{subfigure}[t]{0.24\linewidth}
        \centering
        \includegraphics[width=\linewidth,trim={0pt 6pt 0pt 6pt},clip]{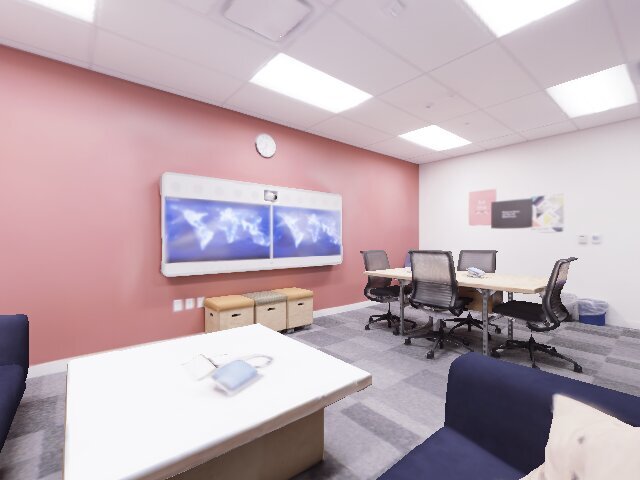}
    \end{subfigure}
    \begin{subfigure}[t]{0.24\linewidth}
        \centering
        \includegraphics[width=\linewidth,trim={0pt 6pt 0pt 6pt},clip]{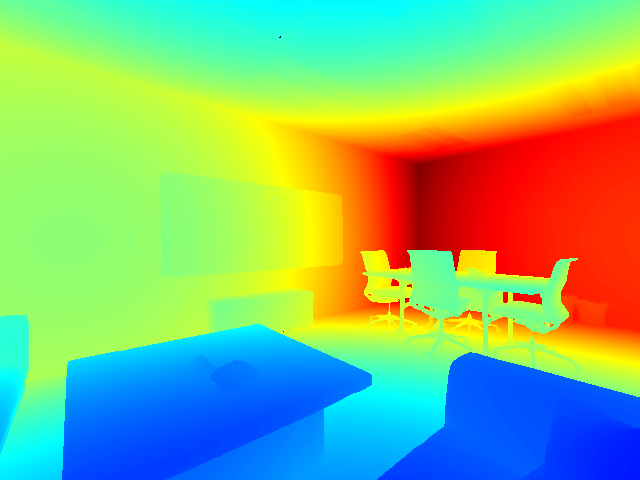}
    \end{subfigure}
    \begin{subfigure}[t]{0.24\linewidth}
        \centering
        \includegraphics[width=\linewidth,trim={0pt 6pt 0pt 6pt},clip]{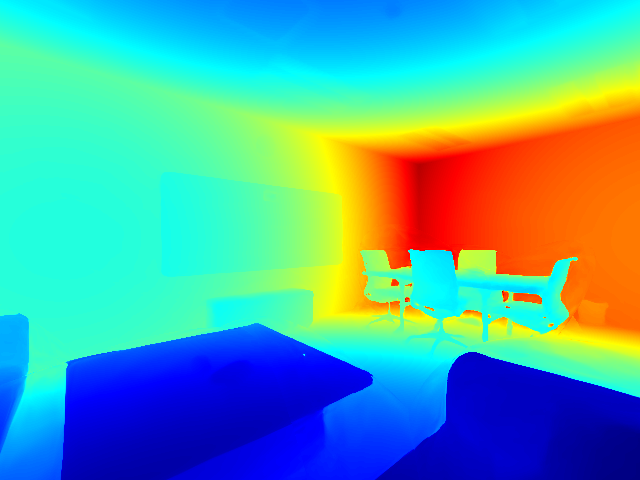}
    \end{subfigure}
    \begin{subfigure}[t]{0.24\linewidth}
        \centering
        \includegraphics[width=\linewidth,trim={0pt 6pt 0pt 6pt},clip]{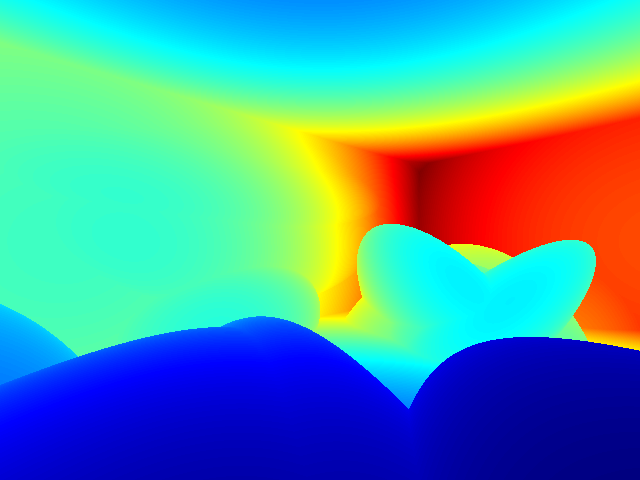}
    \end{subfigure}
    \begin{subfigure}[t]{0.24\linewidth}
        \centering
        \includegraphics[width=\linewidth,trim={0pt 6pt 0pt 6pt},clip]{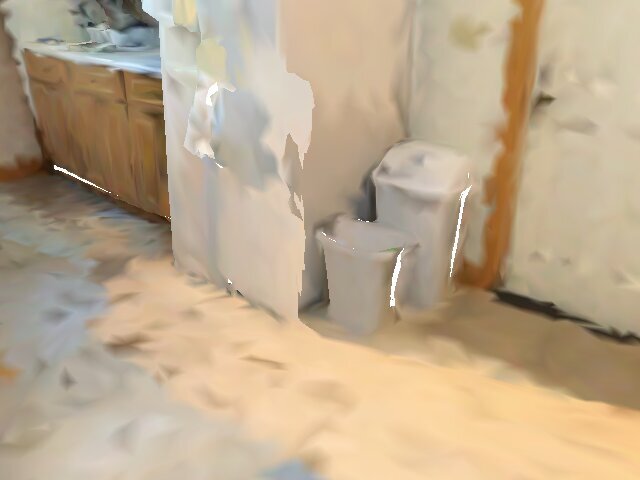}
        \caption{Scene}
    \end{subfigure}
    \begin{subfigure}[t]{0.24\linewidth}
        \centering
        \includegraphics[width=\linewidth,trim={0pt 6pt 0pt 6pt},clip]{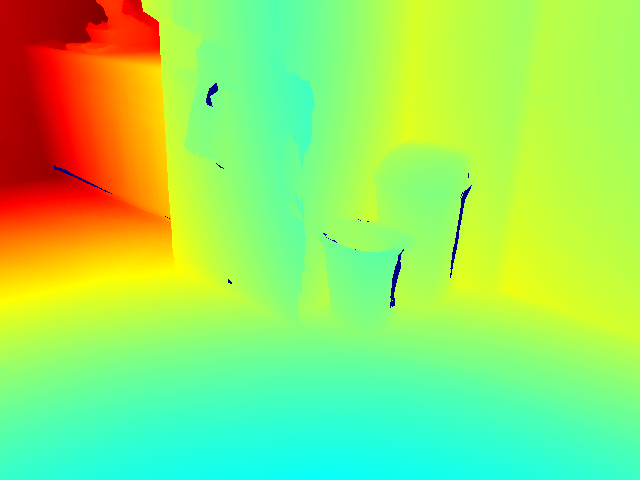}
        \caption{SDDF Ground Truth}
    \end{subfigure}
    \begin{subfigure}[t]{0.24\linewidth}
        \centering
        \includegraphics[width=\linewidth,trim={0pt 6pt 0pt 6pt},clip]{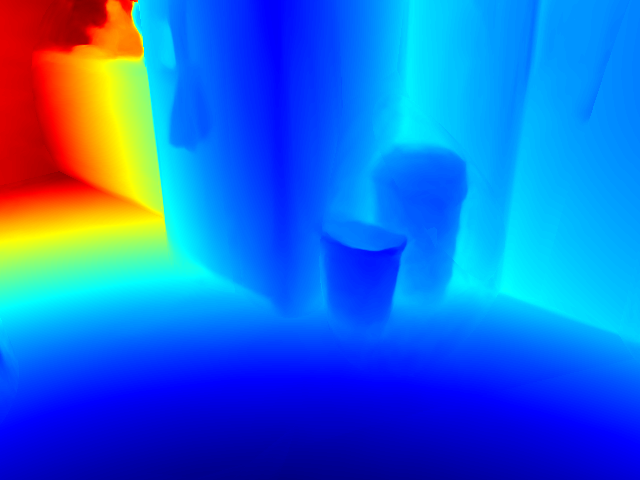}
        \caption{SDDF Prediction}
    \end{subfigure}
    \begin{subfigure}[t]{0.24\linewidth}
        \centering
        \includegraphics[width=\linewidth,trim={0pt 6pt 0pt 6pt},clip]{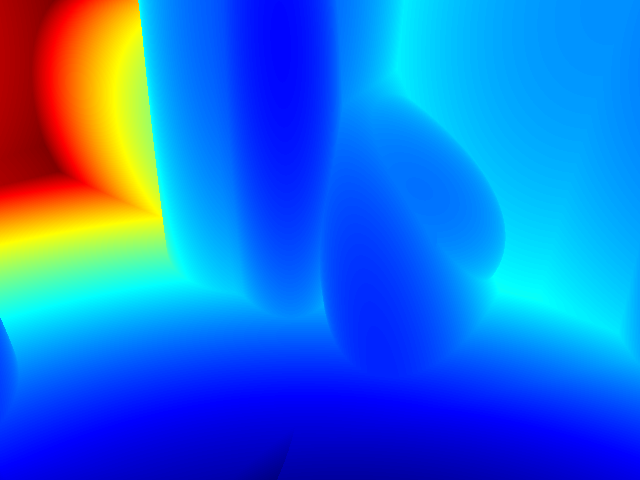}
        \caption{SDDF Prior}
    \end{subfigure}
    \caption{More qualitative results of our method. The four columns are the RGB images of the view, the ground-truth SDDF (color scale may be different due to the occasional zeros caused by mesh holes), the SDDF prediction by our method, and the SDDF prior generated by the ellipsoid-based prior network in our SDDF model.}
    \label{fig:more_sddf_qualitative_results}
\end{figure*}
\begin{figure*}
    \centering
    \begin{subfigure}[t]{0.16\linewidth}
        \centering
        \includegraphics[width=\linewidth,trim={0pt 3pt 0pt 3pt},clip]{fig/sddf_results/allensville-lidar-01-rgb.jpg}
    \end{subfigure}
    \begin{subfigure}[t]{0.16\linewidth}
        \centering
        \includegraphics[width=\linewidth,trim={0pt 3pt 0pt 3pt},clip]{fig/sddf_results/allensville-lidar-01-sddf.png}
    \end{subfigure}
    \begin{subfigure}[t]{0.16\linewidth}
        \centering
        \includegraphics[width=\linewidth,trim={0pt 3pt 0pt 3pt},clip]{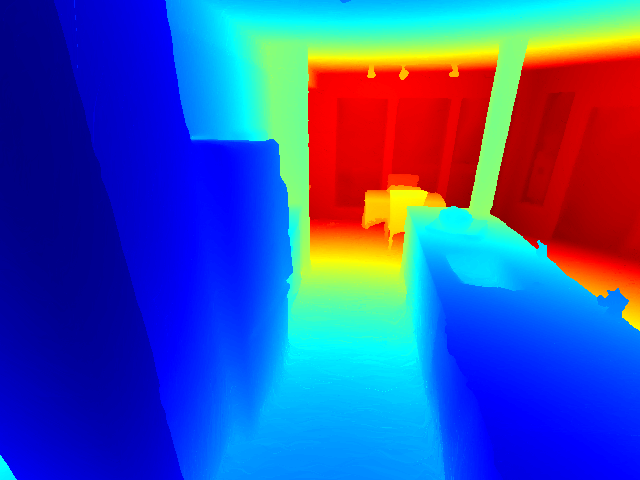}
    \end{subfigure}
    \begin{subfigure}[t]{0.16\linewidth}
        \centering
        \includegraphics[width=\linewidth,trim={0pt 3pt 0pt 3pt},clip]{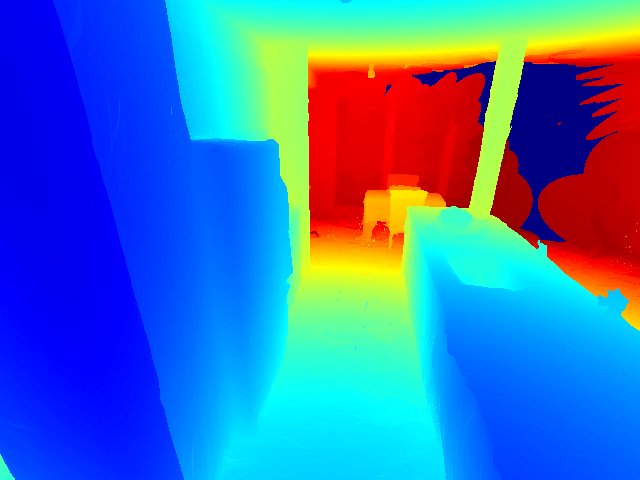}
    \end{subfigure}
    \begin{subfigure}[t]{0.16\linewidth}
        \centering
        \includegraphics[width=\linewidth,trim={0pt 3pt 0pt 3pt},clip]{fig/sddf_results/allensville-rgbd-01-nerf.png}
    \end{subfigure}
    \begin{subfigure}[t]{0.16\linewidth}
        \centering
        \includegraphics[width=\linewidth,trim={0pt 3pt 0pt 3pt},clip]{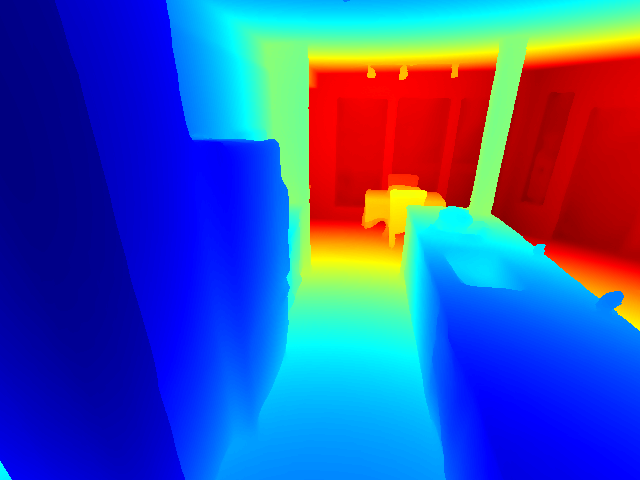}
    \end{subfigure}
    \begin{subfigure}[t]{0.16\linewidth}
        \centering
        \includegraphics[width=\linewidth,trim={0pt 3pt 0pt 3pt},clip]{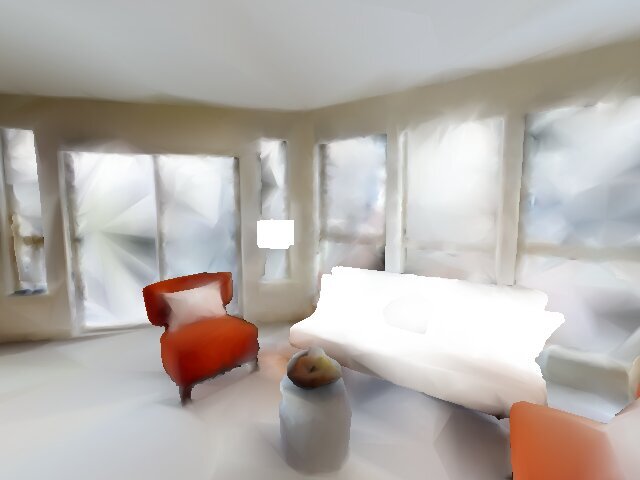}
    \end{subfigure}
    \begin{subfigure}[t]{0.16\linewidth}
        \centering
        \includegraphics[width=\linewidth,trim={0pt 3pt 0pt 3pt},clip]{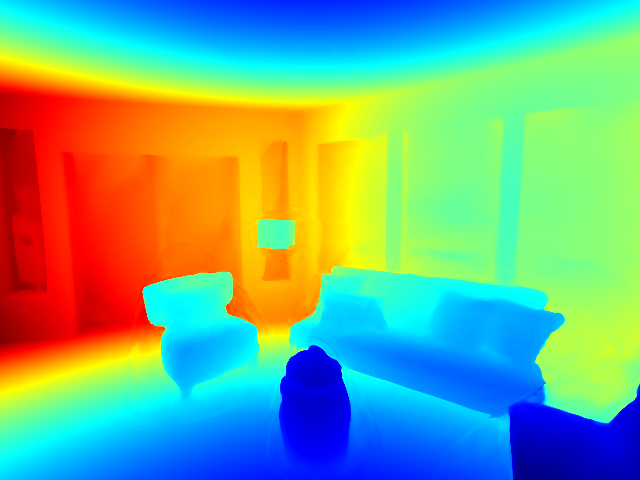}
    \end{subfigure}
    \begin{subfigure}[t]{0.16\linewidth}
        \centering
        \includegraphics[width=\linewidth,trim={0pt 3pt 0pt 3pt},clip]{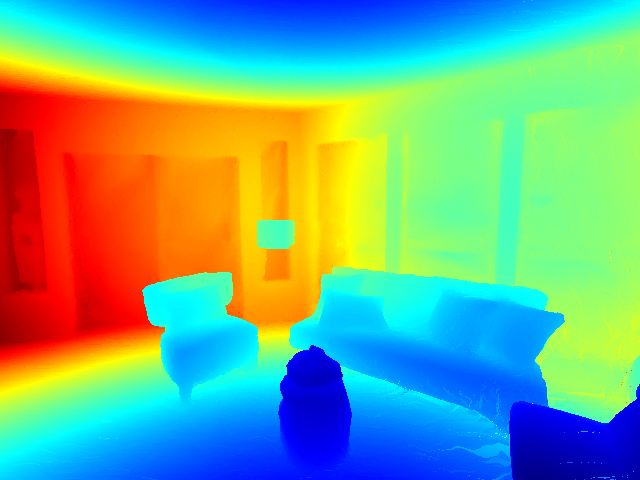}
    \end{subfigure}
    \begin{subfigure}[t]{0.16\linewidth}
        \centering
        \includegraphics[width=\linewidth,trim={0pt 3pt 0pt 3pt},clip]{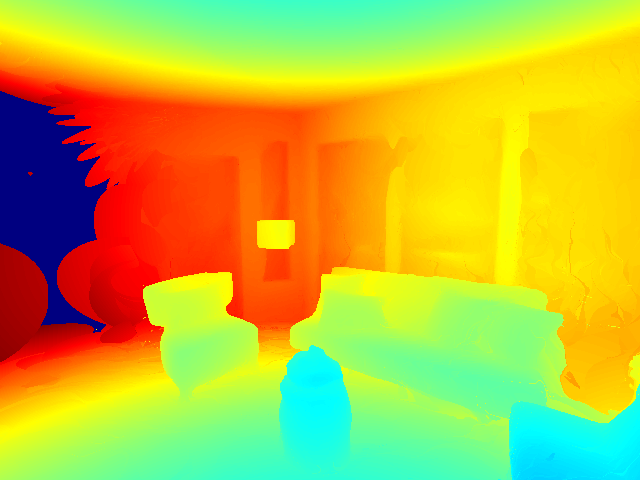}
    \end{subfigure}
    \begin{subfigure}[t]{0.16\linewidth}
        \centering
        \includegraphics[width=\linewidth,trim={0pt 3pt 0pt 3pt},clip]{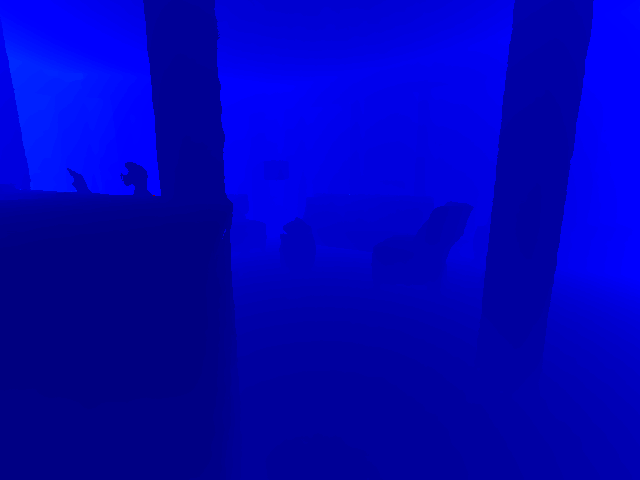}
    \end{subfigure}
    \begin{subfigure}[t]{0.16\linewidth}
        \centering
        \includegraphics[width=\linewidth,trim={0pt 3pt 0pt 3pt},clip]{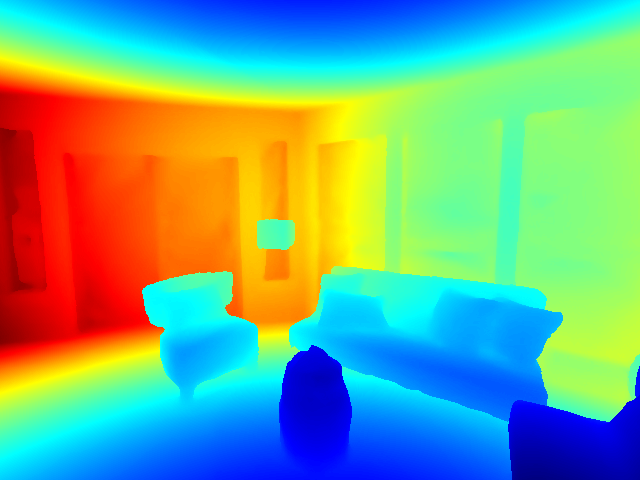}
    \end{subfigure}
    \begin{subfigure}[t]{0.16\linewidth}
        \centering
        \includegraphics[width=\linewidth,trim={0pt 3pt 0pt 3pt},clip]{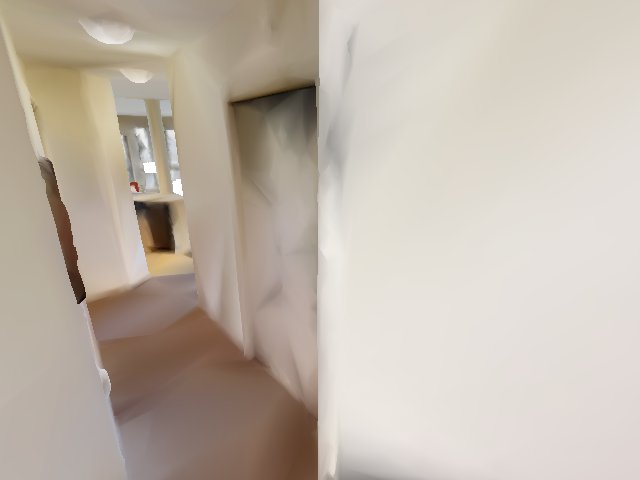}
    \end{subfigure}
    \begin{subfigure}[t]{0.16\linewidth}
        \centering
        \includegraphics[width=\linewidth,trim={0pt 3pt 0pt 3pt},clip]{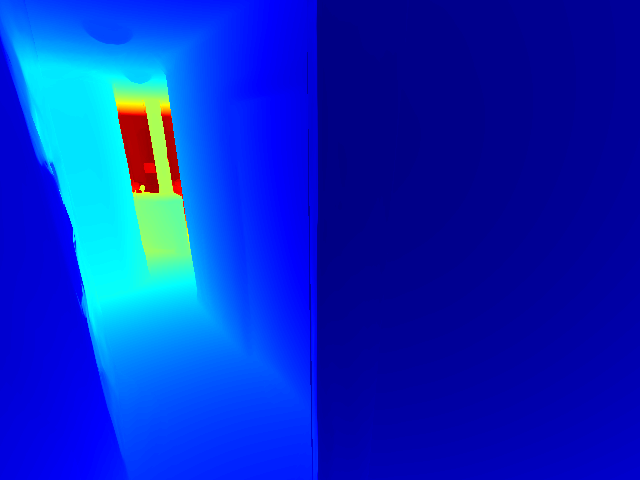}
    \end{subfigure}
    \begin{subfigure}[t]{0.16\linewidth}
        \centering
        \includegraphics[width=\linewidth,trim={0pt 3pt 0pt 3pt},clip]{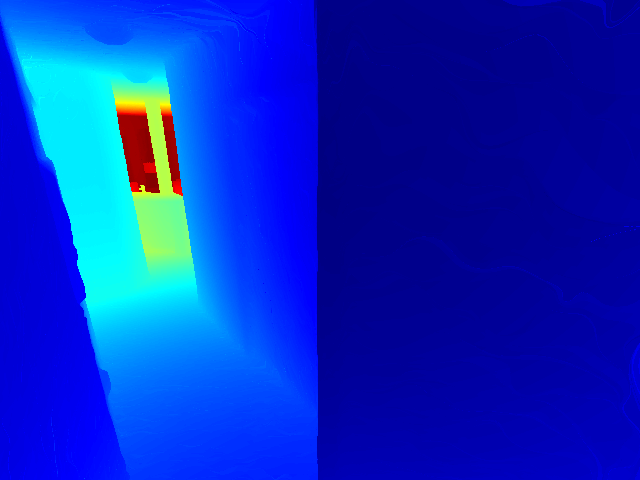}
    \end{subfigure}
    \begin{subfigure}[t]{0.16\linewidth}
        \centering
        \includegraphics[width=\linewidth,trim={0pt 3pt 0pt 3pt},clip]{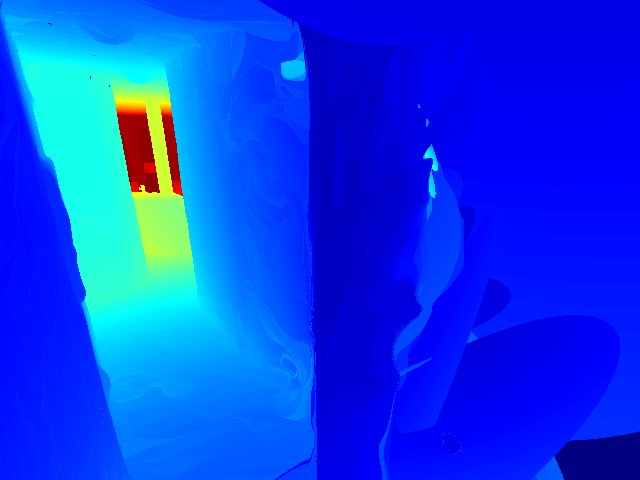}
    \end{subfigure}
    \begin{subfigure}[t]{0.16\linewidth}
        \centering
        \includegraphics[width=\linewidth,trim={0pt 3pt 0pt 3pt},clip]{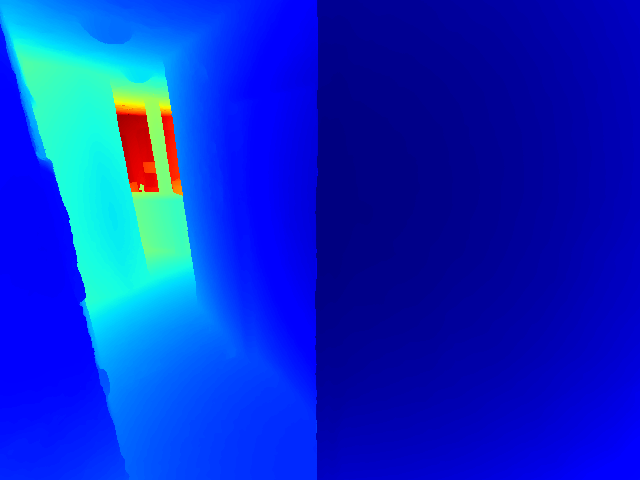}
    \end{subfigure}
    \begin{subfigure}[t]{0.16\linewidth}
        \centering
        \includegraphics[width=\linewidth,trim={0pt 3pt 0pt 3pt},clip]{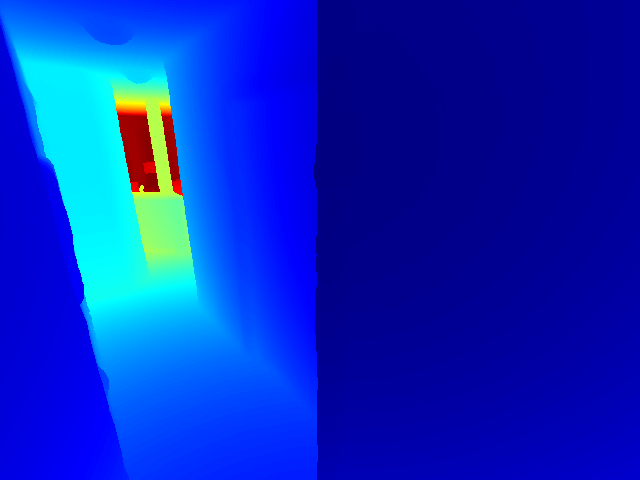}
    \end{subfigure}
    \begin{subfigure}[t]{0.16\linewidth}
        \centering
        \includegraphics[width=\linewidth,trim={0pt 3pt 0pt 3pt},clip]{fig/sddf_results/hotel-lidar-06-rgb.jpg}
    \end{subfigure}
    \begin{subfigure}[t]{0.16\linewidth}
        \centering
        \includegraphics[width=\linewidth,trim={0pt 3pt 0pt 3pt},clip]{fig/sddf_results/hotel-lidar-06-sddf.png}
    \end{subfigure}
    \begin{subfigure}[t]{0.16\linewidth}
        \centering
        \includegraphics[width=\linewidth,trim={0pt 3pt 0pt 3pt},clip]{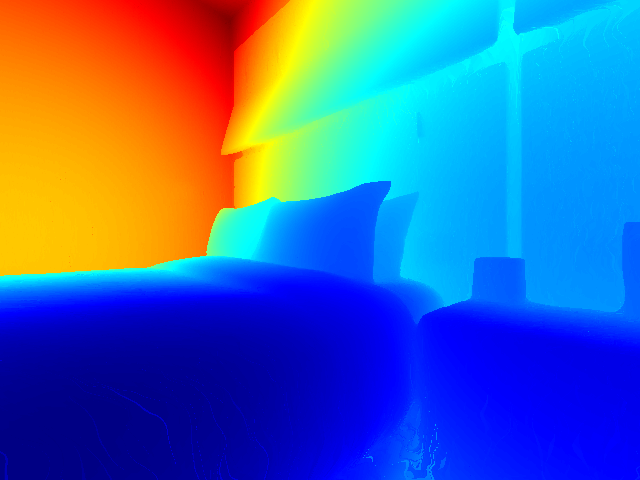}
    \end{subfigure}
    \begin{subfigure}[t]{0.16\linewidth}
        \centering
        \includegraphics[width=\linewidth,trim={0pt 3pt 0pt 3pt},clip]{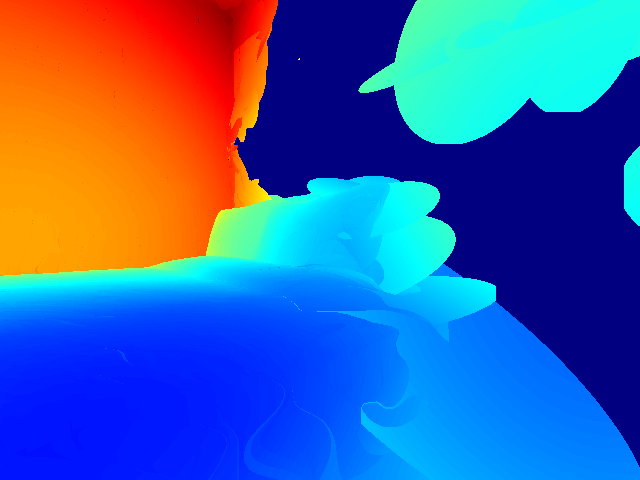}
    \end{subfigure}
    \begin{subfigure}[t]{0.16\linewidth}
        \centering
        \includegraphics[width=\linewidth,trim={0pt 3pt 0pt 3pt},clip]{fig/sddf_results/hotel-rgbd-06-nerf.png}
    \end{subfigure}
    \begin{subfigure}[t]{0.16\linewidth}
        \centering
        \includegraphics[width=\linewidth,trim={0pt 3pt 0pt 3pt},clip]{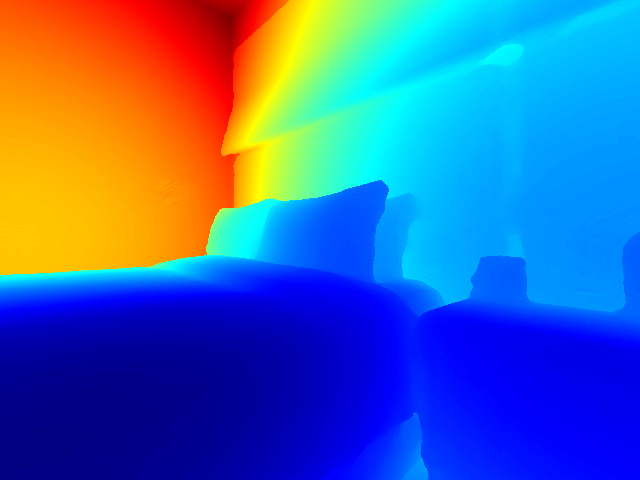}
    \end{subfigure}
    \begin{subfigure}[t]{0.16\linewidth}
        \centering
        \includegraphics[width=\linewidth,trim={0pt 3pt 0pt 3pt},clip]{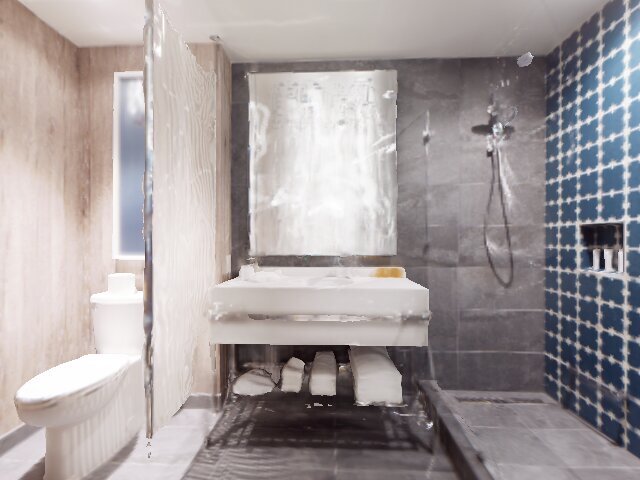}
    \end{subfigure}
    \begin{subfigure}[t]{0.16\linewidth}
        \centering
        \includegraphics[width=\linewidth,trim={0pt 3pt 0pt 3pt},clip]{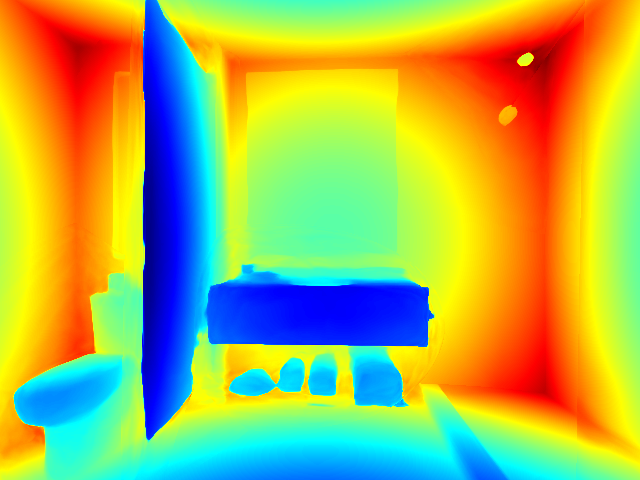}
    \end{subfigure}
    \begin{subfigure}[t]{0.16\linewidth}
        \centering
        \includegraphics[width=\linewidth,trim={0pt 3pt 0pt 3pt},clip]{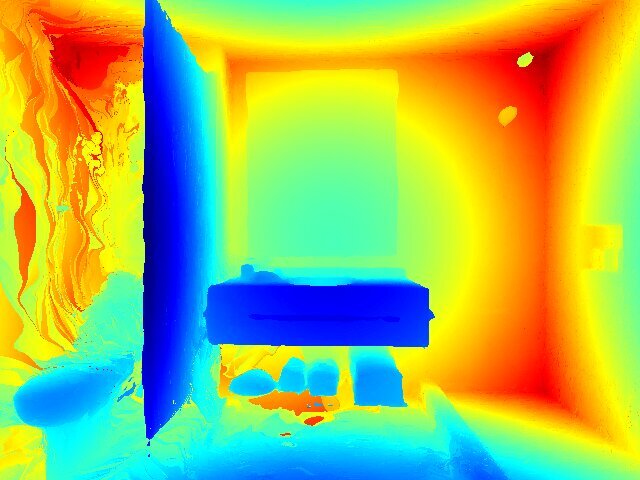}
    \end{subfigure}
    \begin{subfigure}[t]{0.16\linewidth}
        \centering
        \includegraphics[width=\linewidth,trim={0pt 3pt 0pt 3pt},clip]{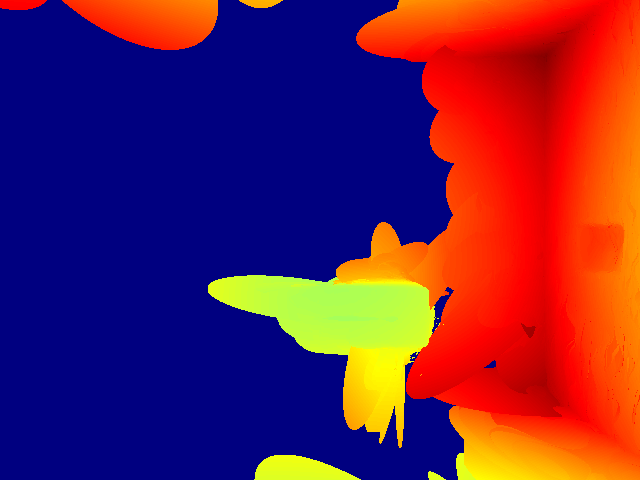}
    \end{subfigure}
    \begin{subfigure}[t]{0.16\linewidth}
        \centering
        \includegraphics[width=\linewidth,trim={0pt 3pt 0pt 3pt},clip]{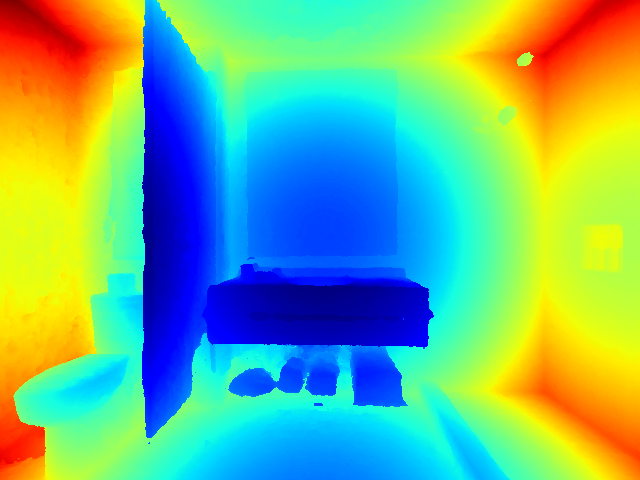}
    \end{subfigure}
    \begin{subfigure}[t]{0.16\linewidth}
        \centering
        \includegraphics[width=\linewidth,trim={0pt 3pt 0pt 3pt},clip]{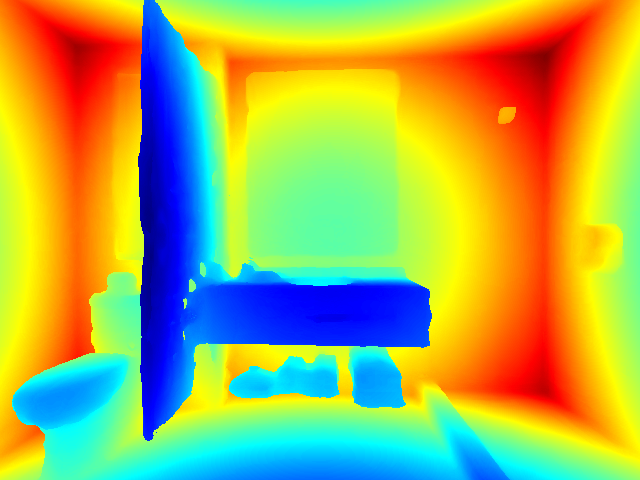}
    \end{subfigure}
    \begin{subfigure}[t]{0.16\linewidth}
        \centering
        \includegraphics[width=\linewidth,trim={0pt 3pt 0pt 3pt},clip]{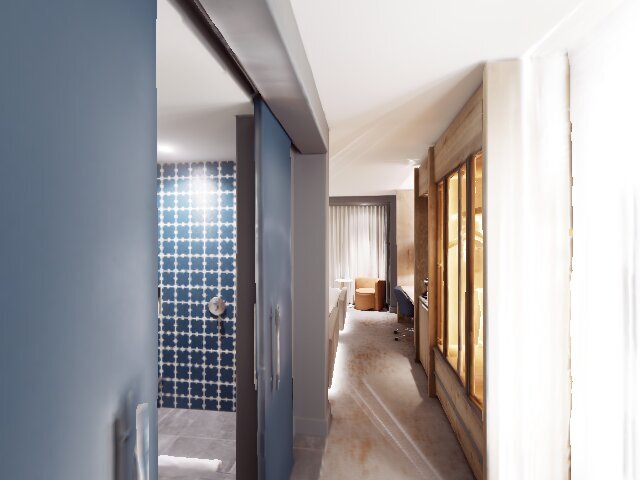}
    \end{subfigure}
    \begin{subfigure}[t]{0.16\linewidth}
        \centering
        \includegraphics[width=\linewidth,trim={0pt 3pt 0pt 3pt},clip]{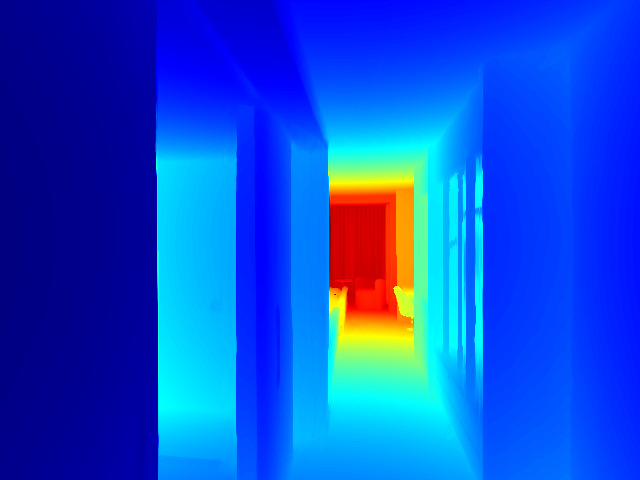}
    \end{subfigure}
    \begin{subfigure}[t]{0.16\linewidth}
        \centering
        \includegraphics[width=\linewidth,trim={0pt 3pt 0pt 3pt},clip]{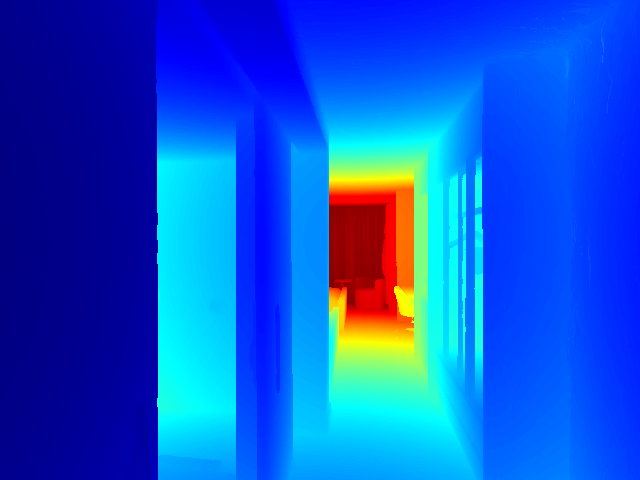}
    \end{subfigure}
    \begin{subfigure}[t]{0.16\linewidth}
        \centering
        \includegraphics[width=\linewidth,trim={0pt 3pt 0pt 3pt},clip]{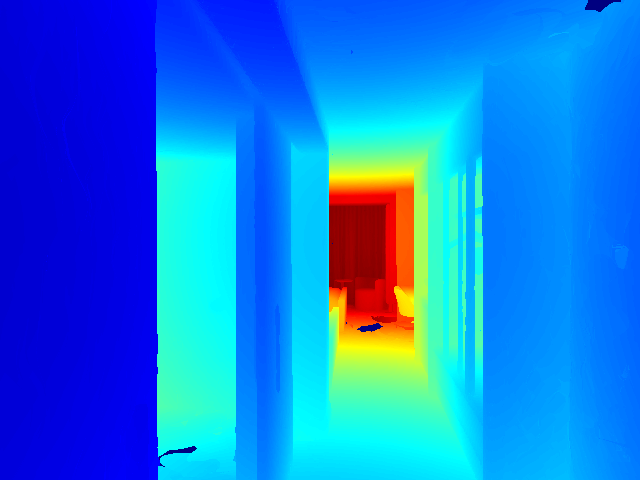}
    \end{subfigure}
    \begin{subfigure}[t]{0.16\linewidth}
        \centering
        \includegraphics[width=\linewidth,trim={0pt 3pt 0pt 3pt},clip]{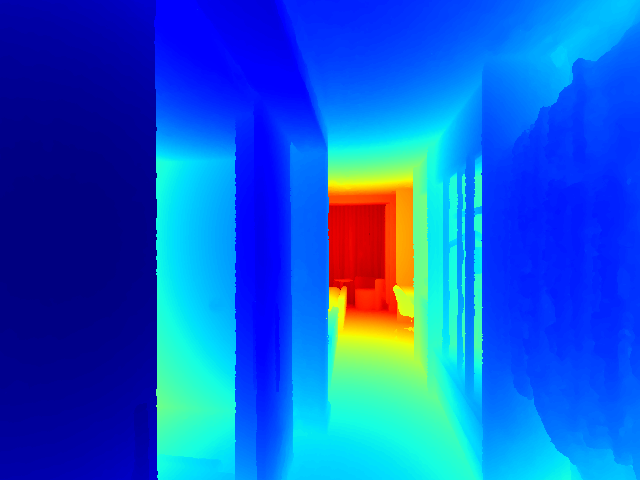}
    \end{subfigure}
    \begin{subfigure}[t]{0.16\linewidth}
        \centering
        \includegraphics[width=\linewidth,trim={0pt 3pt 0pt 3pt},clip]{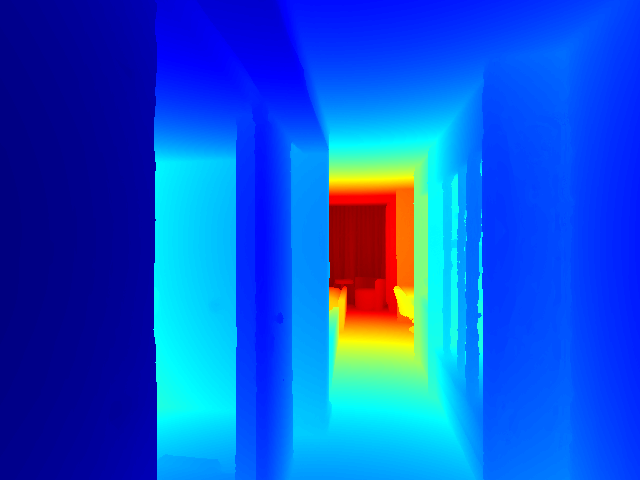}
    \end{subfigure}
    \begin{subfigure}[t]{0.16\linewidth}
        \centering
        \includegraphics[width=\linewidth,trim={0pt 3pt 0pt 3pt},clip]{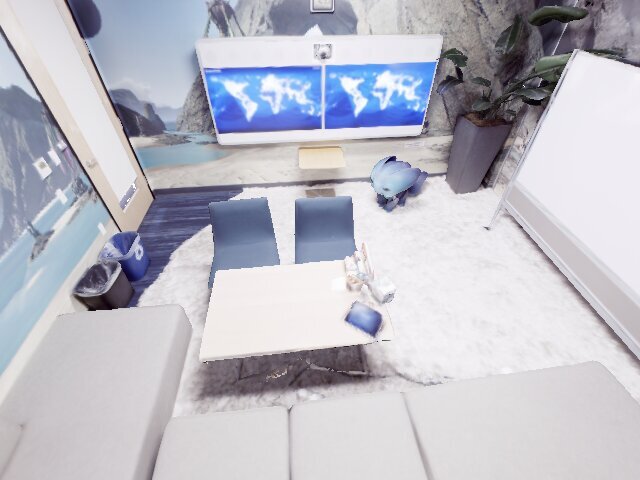}
    \end{subfigure}
    \begin{subfigure}[t]{0.16\linewidth}
        \centering
        \includegraphics[width=\linewidth,trim={0pt 3pt 0pt 3pt},clip]{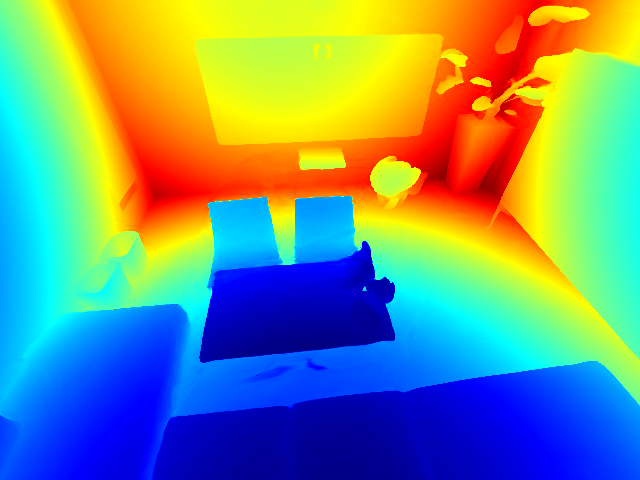}
    \end{subfigure}
    \begin{subfigure}[t]{0.16\linewidth}
        \centering
        \includegraphics[width=\linewidth,trim={0pt 3pt 0pt 3pt},clip]{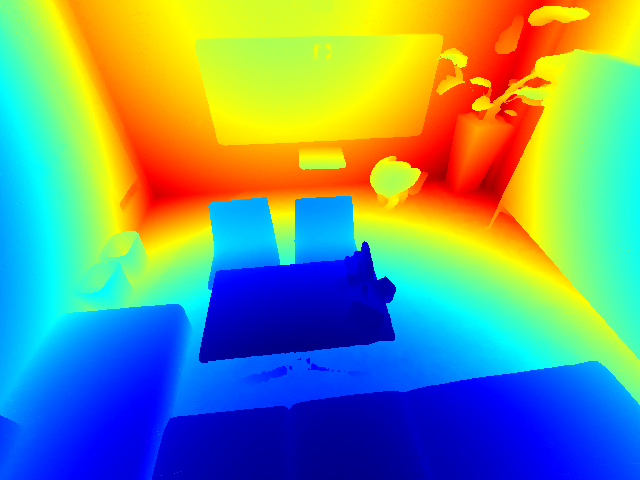}
    \end{subfigure}
    \begin{subfigure}[t]{0.16\linewidth}
        \centering
        \includegraphics[width=\linewidth,trim={0pt 3pt 0pt 3pt},clip]{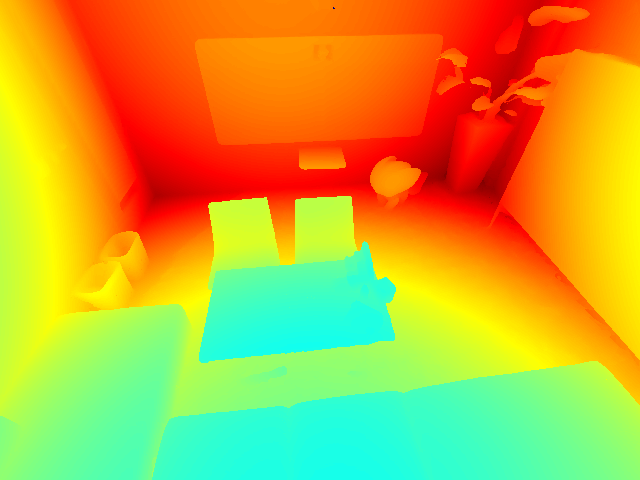}
    \end{subfigure}
    \begin{subfigure}[t]{0.16\linewidth}
        \centering
        \includegraphics[width=\linewidth,trim={0pt 3pt 0pt 3pt},clip]{fig/sddf_results/office-0-rgbd-06-nerf.png}
    \end{subfigure}
    \begin{subfigure}[t]{0.16\linewidth}
        \centering
        \includegraphics[width=\linewidth,trim={0pt 3pt 0pt 3pt},clip]{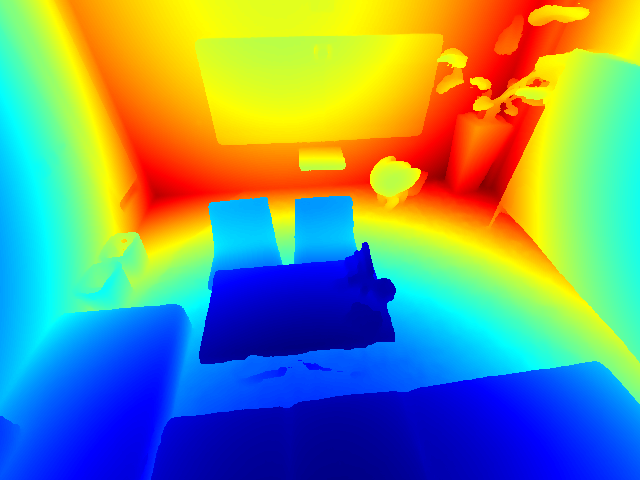}
    \end{subfigure}
    \begin{subfigure}[t]{0.16\linewidth}
        \centering
        \includegraphics[width=\linewidth,trim={0pt 3pt 0pt 3pt},clip]{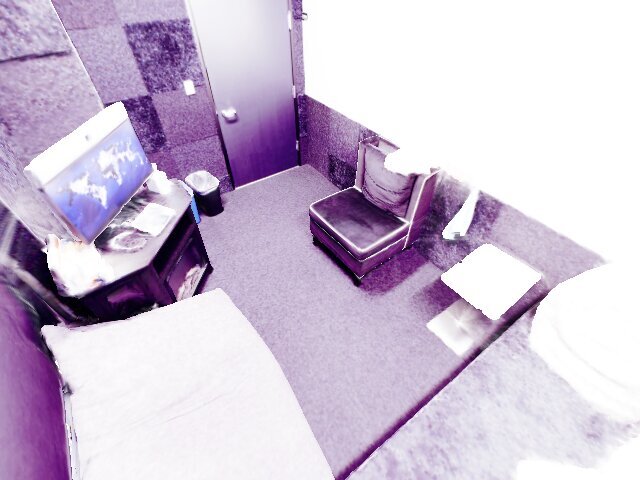}
    \end{subfigure}
    \begin{subfigure}[t]{0.16\linewidth}
        \centering
        \includegraphics[width=\linewidth,trim={0pt 3pt 0pt 3pt},clip]{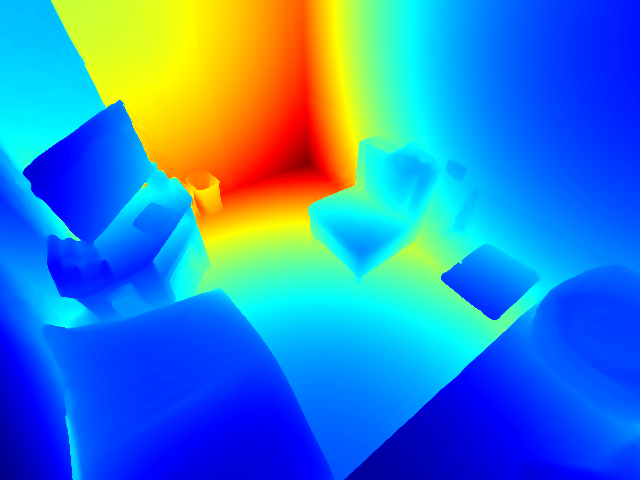}
    \end{subfigure}
    \begin{subfigure}[t]{0.16\linewidth}
        \centering
        \includegraphics[width=\linewidth,trim={0pt 3pt 0pt 3pt},clip]{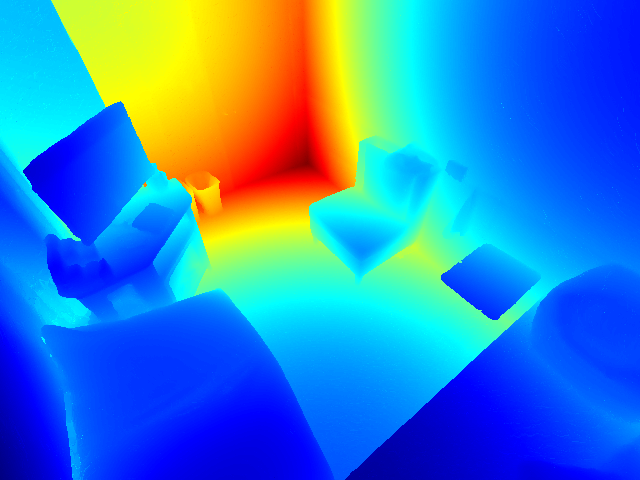}
    \end{subfigure}
    \begin{subfigure}[t]{0.16\linewidth}
        \centering
        \includegraphics[width=\linewidth,trim={0pt 3pt 0pt 3pt},clip]{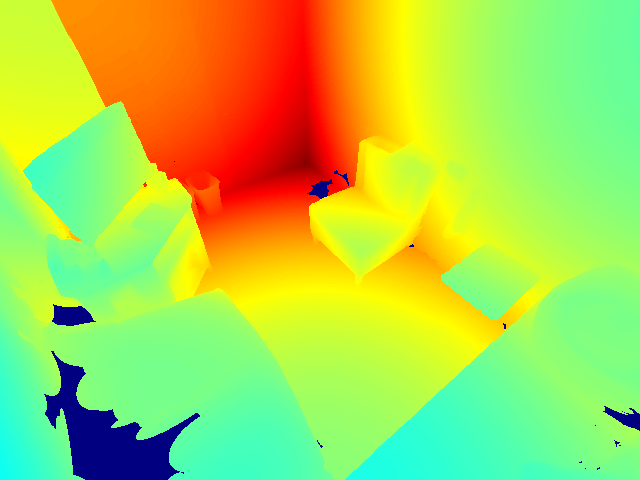}
    \end{subfigure}
    \begin{subfigure}[t]{0.16\linewidth}
        \centering
        \includegraphics[width=\linewidth,trim={0pt 3pt 0pt 3pt},clip]{fig/sddf_results/office-1-rgbd-06-nerf.png}
    \end{subfigure}
    \begin{subfigure}[t]{0.16\linewidth}
        \centering
        \includegraphics[width=\linewidth,trim={0pt 3pt 0pt 3pt},clip]{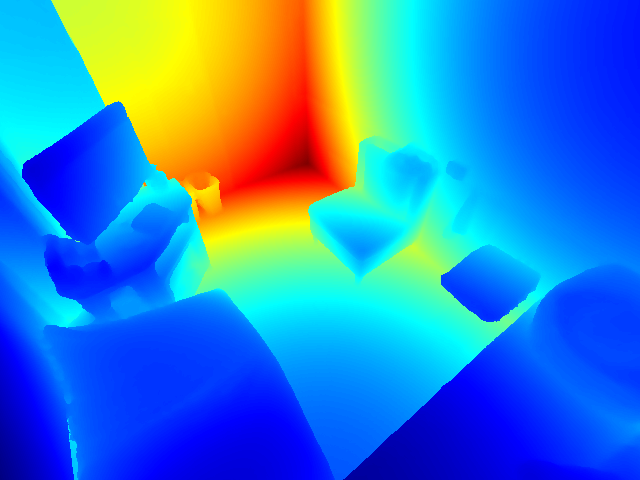}
    \end{subfigure}
    \begin{subfigure}[t]{0.16\linewidth}
        \centering
        \includegraphics[width=\linewidth,trim={0pt 3pt 0pt 3pt},clip]{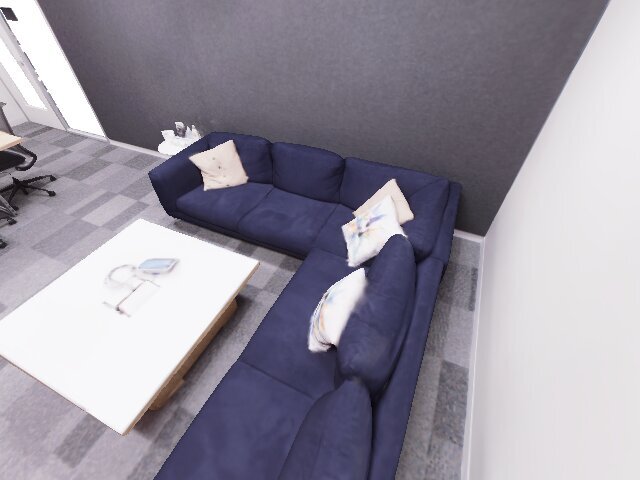}
        \caption{Scene}
    \end{subfigure}
    \begin{subfigure}[t]{0.16\linewidth}
        \centering
        \includegraphics[width=\linewidth,trim={0pt 3pt 0pt 3pt},clip]{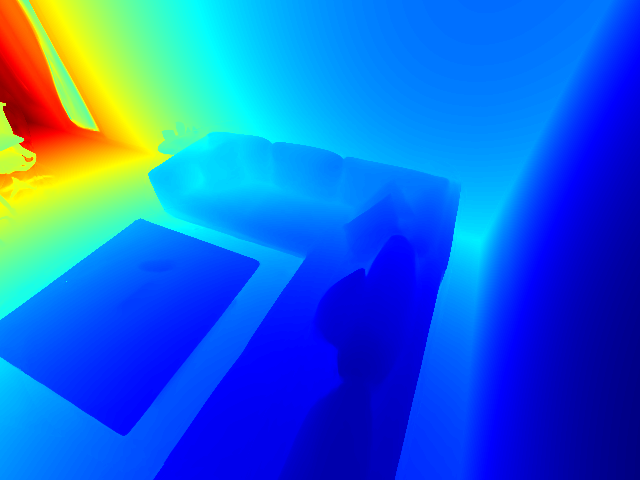}
        \caption{SDDF (ours)}
    \end{subfigure}
    \begin{subfigure}[t]{0.16\linewidth}
        \centering
        \includegraphics[width=\linewidth,trim={0pt 3pt 0pt 3pt},clip]{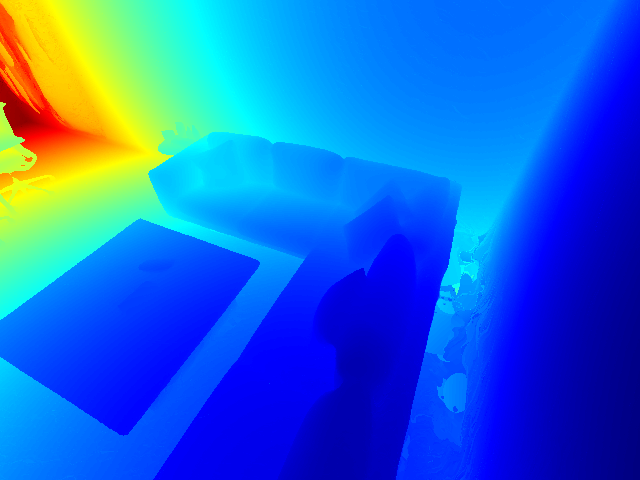}
        \caption{RaDe-GS~\cite{radegs2024}}
    \end{subfigure}
    \begin{subfigure}[t]{0.16\linewidth}
        \centering
        \includegraphics[width=\linewidth,trim={0pt 3pt 0pt 3pt},clip]{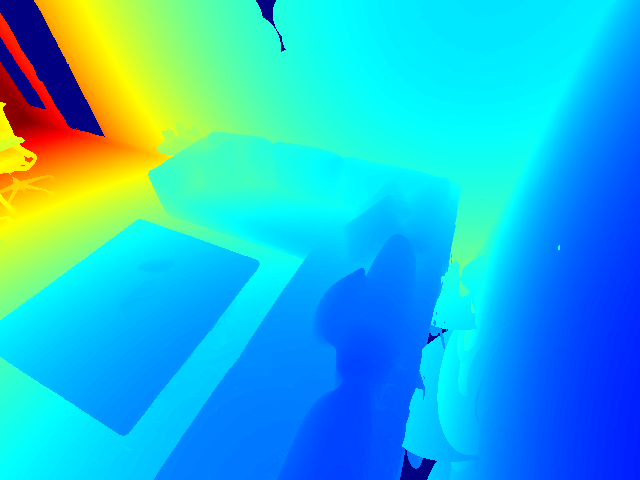}
        \caption{w/o RGB}
    \end{subfigure}
    \begin{subfigure}[t]{0.16\linewidth}
        \centering
        \includegraphics[width=\linewidth,trim={0pt 3pt 0pt 3pt},clip]{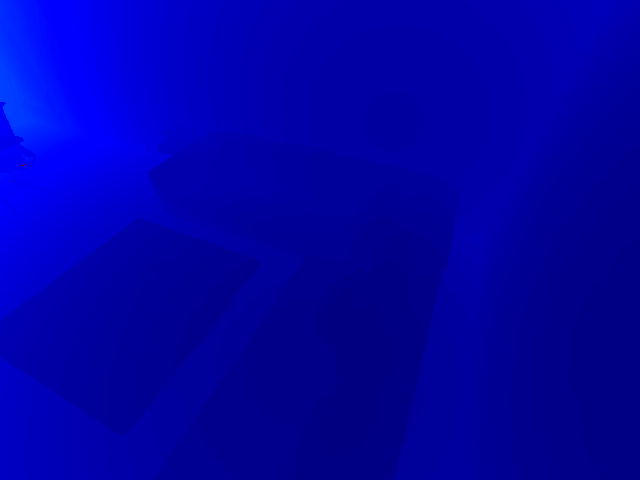}
        \caption{Nerfacto~\cite{nerfstudio2023}}
    \end{subfigure}
    \begin{subfigure}[t]{0.16\linewidth}
        \centering
        \includegraphics[width=\linewidth,trim={0pt 3pt 0pt 3pt},clip]{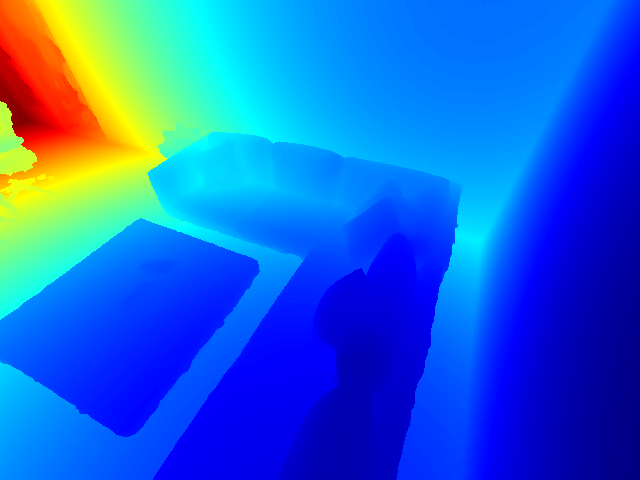}
        \caption{SDF~\cite{instantngp2022}}
    \end{subfigure}
    \caption{More qualitative SDDF prediction comparison between our method and baselines on synthesized datasets, Replica and Gibson.
    The six columns are the RGB images of the scene, SDDF prediction by our method, RaDe-GS~\cite{radegs2024}, RaDe-GS~\cite{radegs2024} trained without RGB images, Nerfacto~\cite{nerfstudio2023} and \sdfbaseline~\cite{instantngp2022} with sphere tracing~\cite{spheretracing1996} respectively.
    When only a few sensor measurements are available, RaDe-GS~\cite{radegs2024} fails to learn the geometry, giving splatting artifacts.
    When trained without RGB, RaDe-GS presents instability in certain cases shown above. The splatting artifacts are more significant. And even worse, it cannot learn the scene when available measurements become fewer.
    Nerfacto produces cloud-like artifacts when no sufficient data is provided such as the sixth row. And it has large distance errors in some cases like the second row. Some pixels have much larger distance values. Thus, the image looks blue globally.
    \sdfbaseline~\cite{instantngp2022} tends to learn smoother shapes that lack sharper details, like the fourth row. And the accumulated errors by sphere tracing may cause holes on objects or rough object boundaries, e.g., the curtain in the figure of the fifth row and the closet in the figure of the sixth row.}
    \label{fig:more_sddf_comp}
\end{figure*}
\begin{figure*}
    \centering
    \begin{subfigure}[t]{0.16\linewidth}
        \centering
        \includegraphics[width=\linewidth]{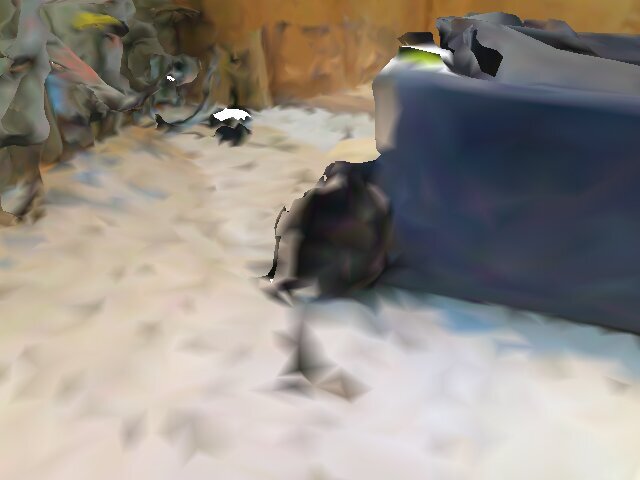}
    \end{subfigure}
    \begin{subfigure}[t]{0.16\linewidth}
        \centering
        \includegraphics[width=\linewidth]{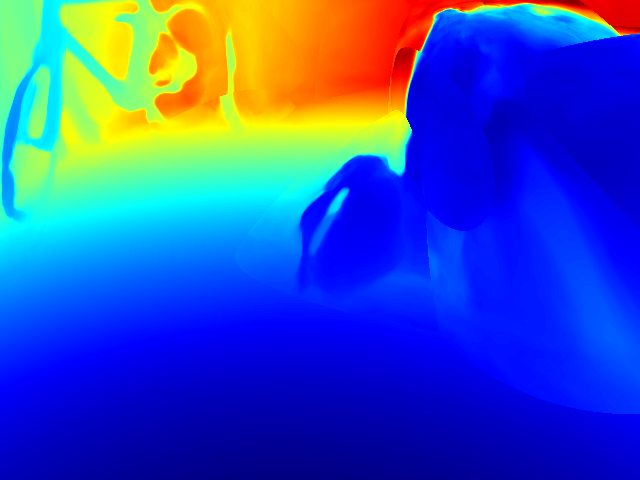}
    \end{subfigure}
    \begin{subfigure}[t]{0.16\linewidth}
        \centering
        \includegraphics[width=\linewidth]{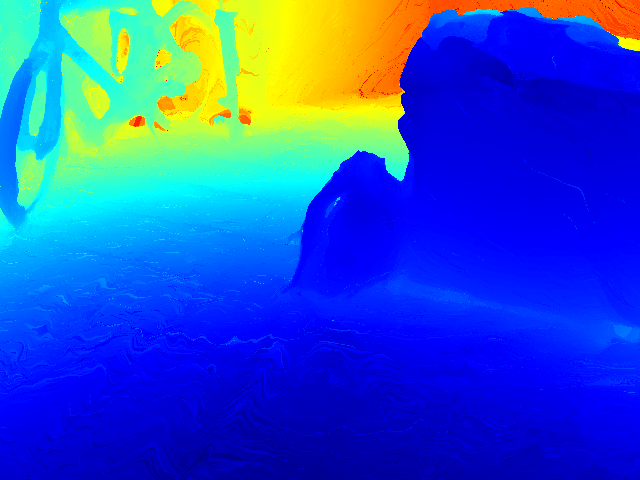}
    \end{subfigure}
    \begin{subfigure}[t]{0.16\linewidth}
        \centering
        \includegraphics[width=\linewidth]{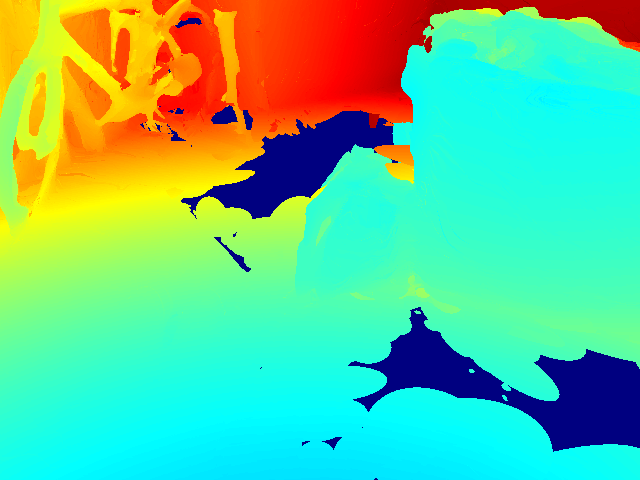}
    \end{subfigure}
    \begin{subfigure}[t]{0.16\linewidth}
        \centering
        \includegraphics[width=\linewidth]{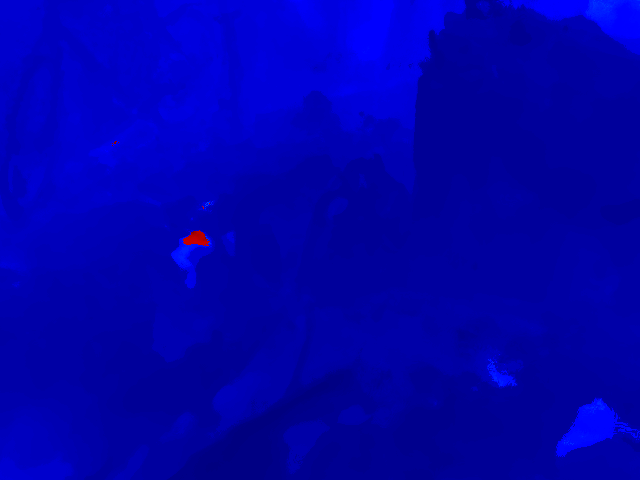}
    \end{subfigure}
    \begin{subfigure}[t]{0.16\linewidth}
        \centering
        \includegraphics[width=\linewidth]{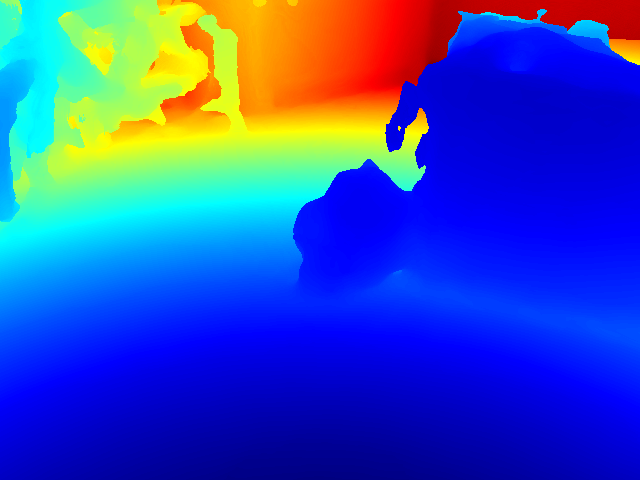}
    \end{subfigure}
    \begin{subfigure}[t]{0.16\linewidth}
        \centering
        \includegraphics[width=\linewidth]{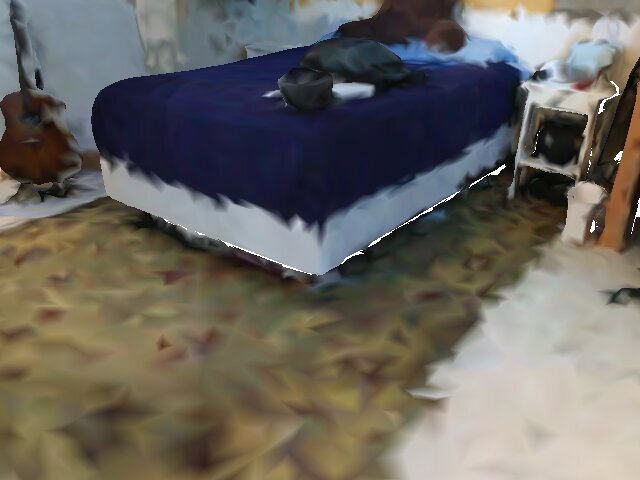}
    \end{subfigure}
    \begin{subfigure}[t]{0.16\linewidth}
        \centering
        \includegraphics[width=\linewidth]{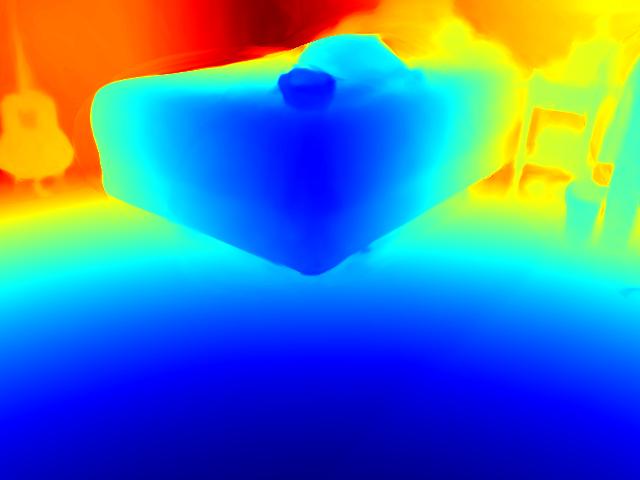}
    \end{subfigure}
    \begin{subfigure}[t]{0.16\linewidth}
        \centering
        \includegraphics[width=\linewidth]{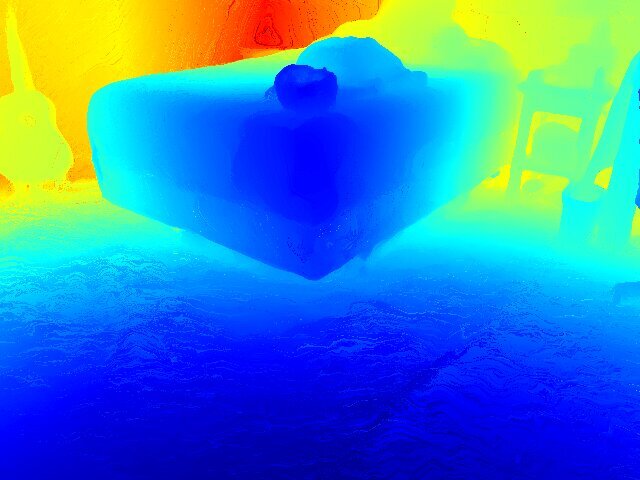}
    \end{subfigure}
    \begin{subfigure}[t]{0.16\linewidth}
        \centering
        \includegraphics[width=\linewidth]{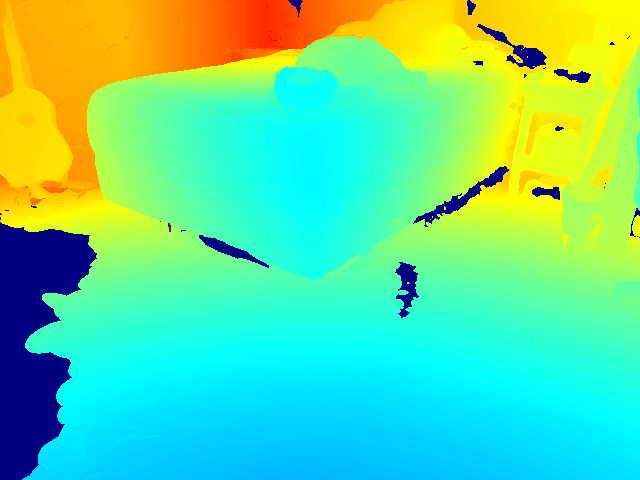}
    \end{subfigure}
    \begin{subfigure}[t]{0.16\linewidth}
        \centering
        \includegraphics[width=\linewidth]{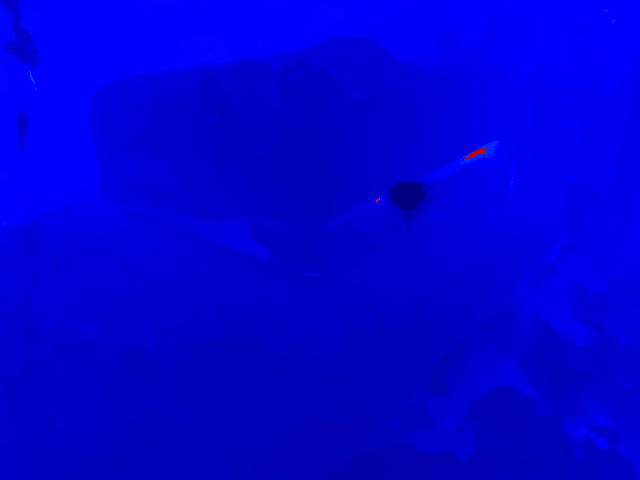}
    \end{subfigure}
    \begin{subfigure}[t]{0.16\linewidth}
        \centering
        \includegraphics[width=\linewidth]{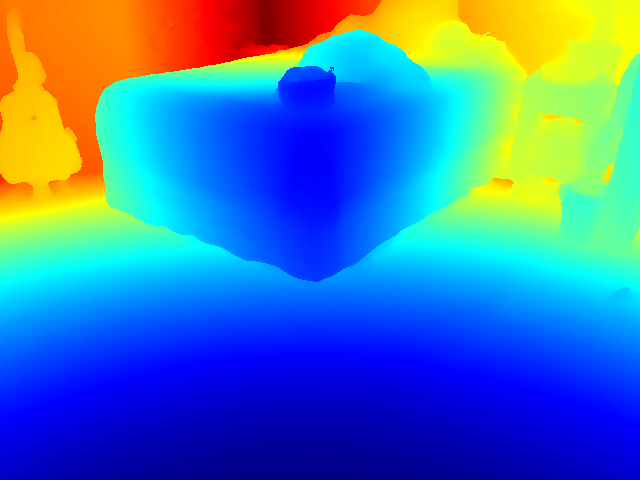}
    \end{subfigure}
    \begin{subfigure}[t]{0.16\linewidth}
        \centering
        \includegraphics[width=\linewidth]{fig/sddf_results/scannet-0000_00-rgbd-08-rgb.jpg}
    \end{subfigure}
    \begin{subfigure}[t]{0.16\linewidth}
        \centering
        \includegraphics[width=\linewidth]{fig/sddf_results/scannet-0000_00-rgbd-08-sddf.png}
    \end{subfigure}
    \begin{subfigure}[t]{0.16\linewidth}
        \centering
        \includegraphics[width=\linewidth]{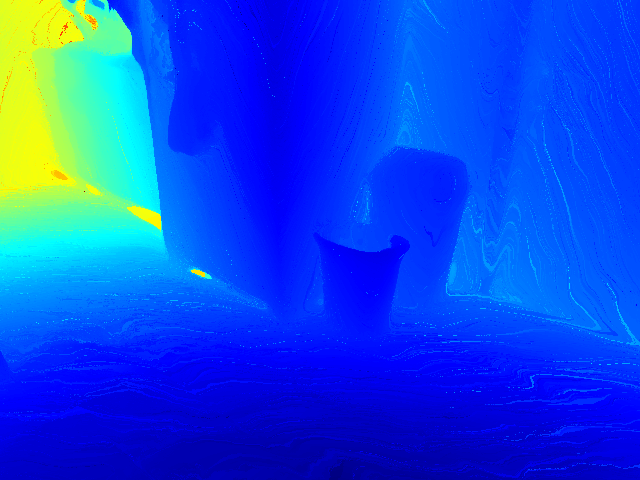}
    \end{subfigure}
    \begin{subfigure}[t]{0.16\linewidth}
        \centering
        \includegraphics[width=\linewidth]{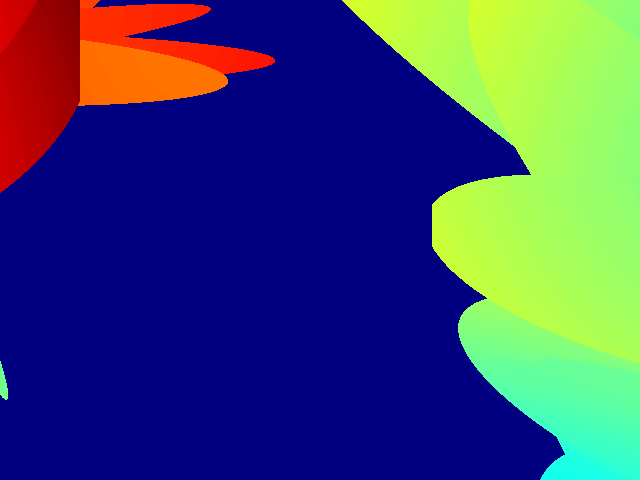}
    \end{subfigure}
    \begin{subfigure}[t]{0.16\linewidth}
        \centering
        \includegraphics[width=\linewidth]{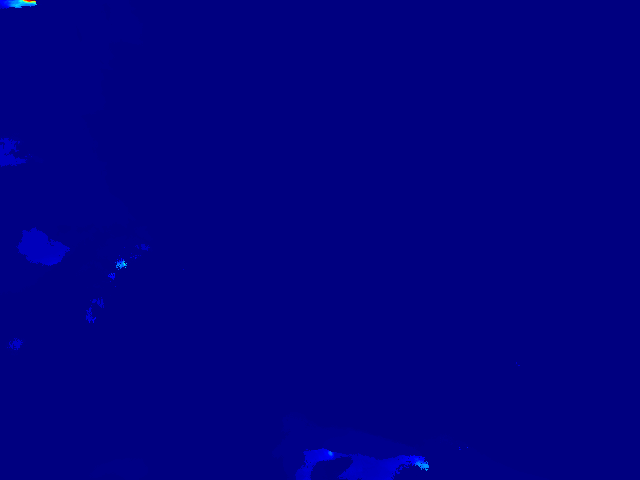}
    \end{subfigure}
    \begin{subfigure}[t]{0.16\linewidth}
        \centering
        \includegraphics[width=\linewidth]{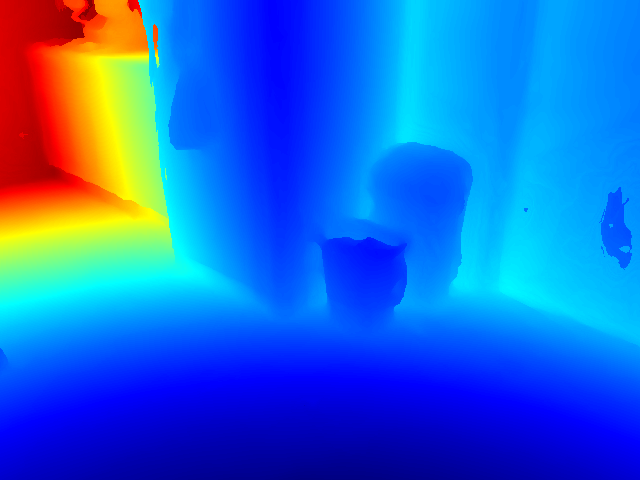}
    \end{subfigure}
    \begin{subfigure}[t]{0.16\linewidth}
        \centering
        \includegraphics[width=\linewidth]{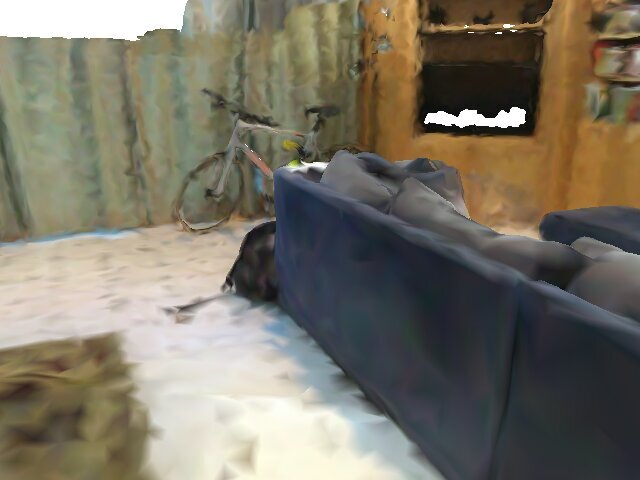}
    \end{subfigure}
    \begin{subfigure}[t]{0.16\linewidth}
        \centering
        \includegraphics[width=\linewidth]{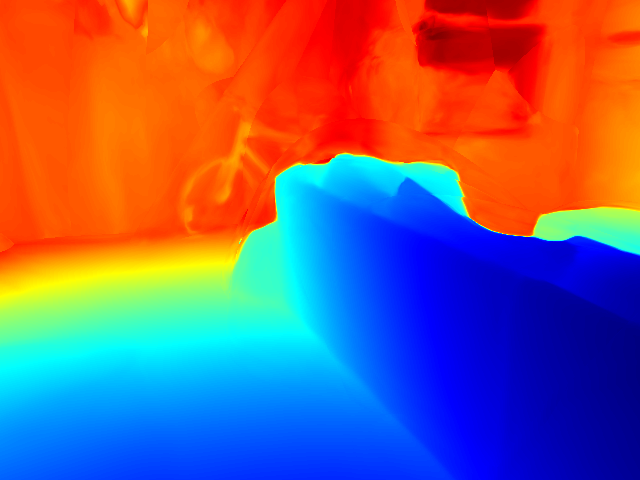}
    \end{subfigure}
    \begin{subfigure}[t]{0.16\linewidth}
        \centering
        \includegraphics[width=\linewidth]{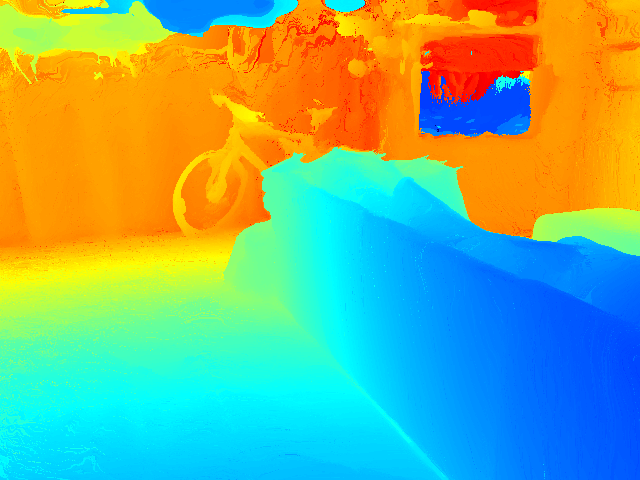}
    \end{subfigure}
    \begin{subfigure}[t]{0.16\linewidth}
        \centering
        \includegraphics[width=\linewidth]{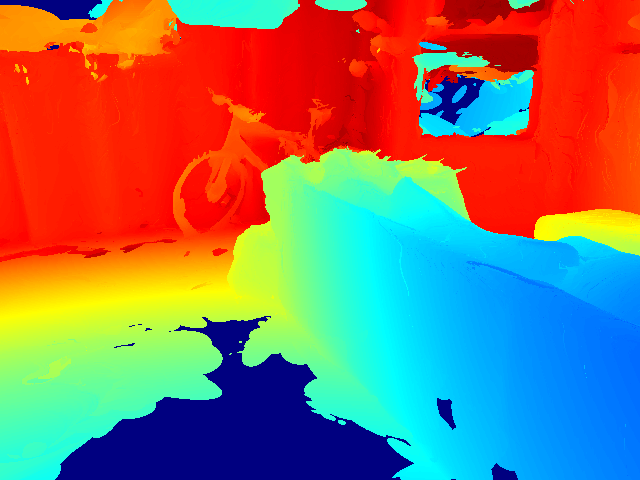}
    \end{subfigure}
    \begin{subfigure}[t]{0.16\linewidth}
        \centering
        \includegraphics[width=\linewidth]{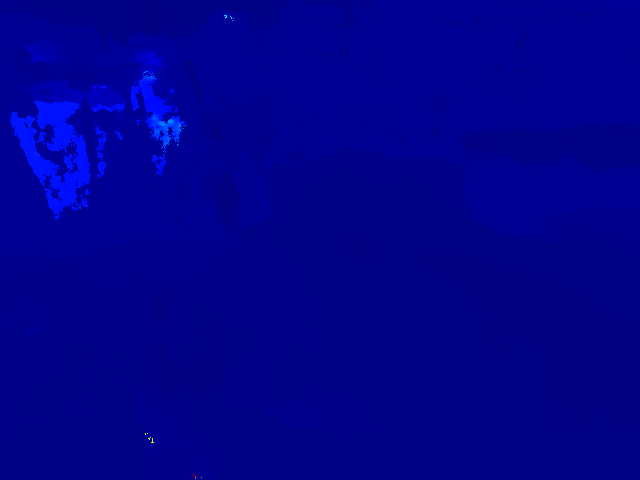}
    \end{subfigure}
    \begin{subfigure}[t]{0.16\linewidth}
        \centering
        \includegraphics[width=\linewidth]{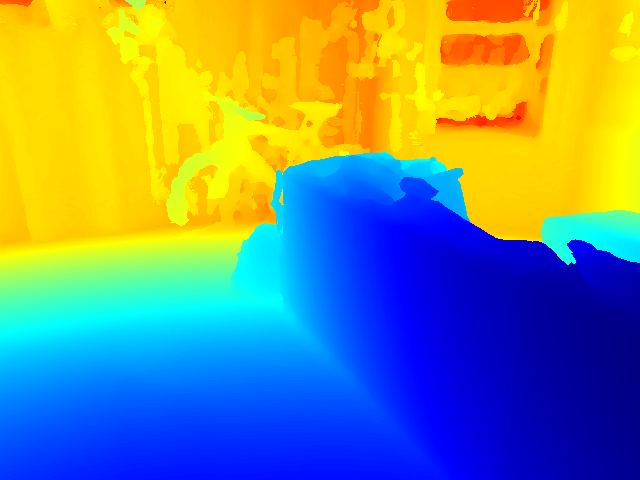}
    \end{subfigure}
    \begin{subfigure}[t]{0.16\linewidth}
        \centering
        \includegraphics[width=\linewidth]{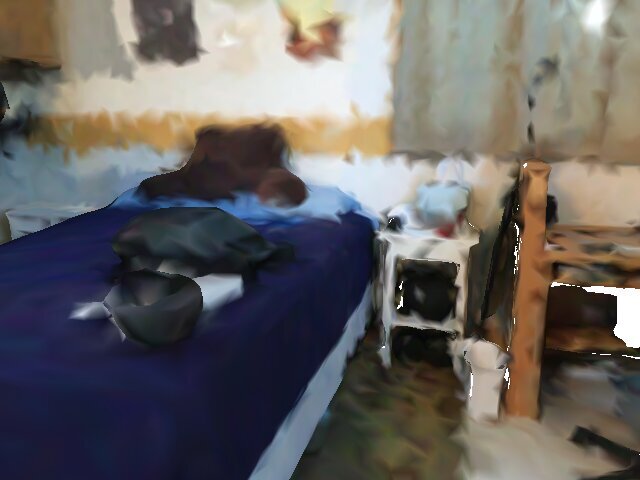}
    \end{subfigure}
    \begin{subfigure}[t]{0.16\linewidth}
        \centering
        \includegraphics[width=\linewidth]{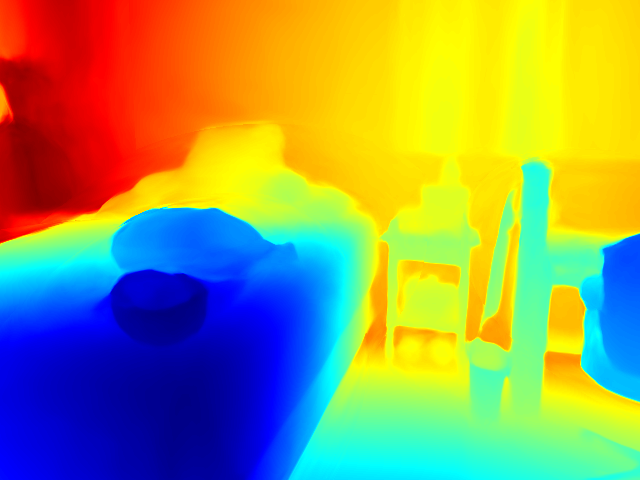}
    \end{subfigure}
    \begin{subfigure}[t]{0.16\linewidth}
        \centering
        \includegraphics[width=\linewidth]{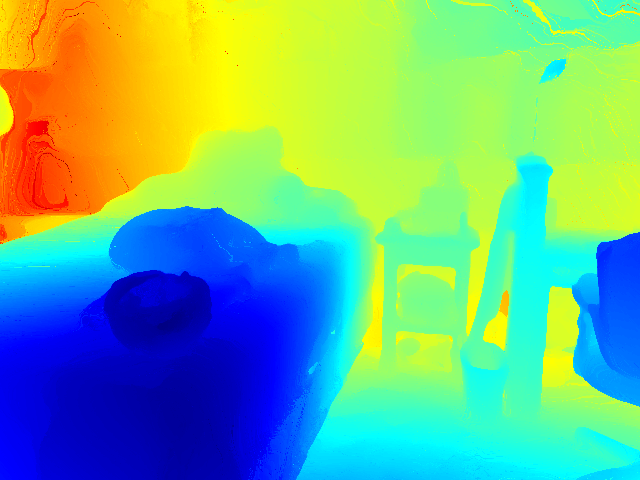}
    \end{subfigure}
    \begin{subfigure}[t]{0.16\linewidth}
        \centering
        \includegraphics[width=\linewidth]{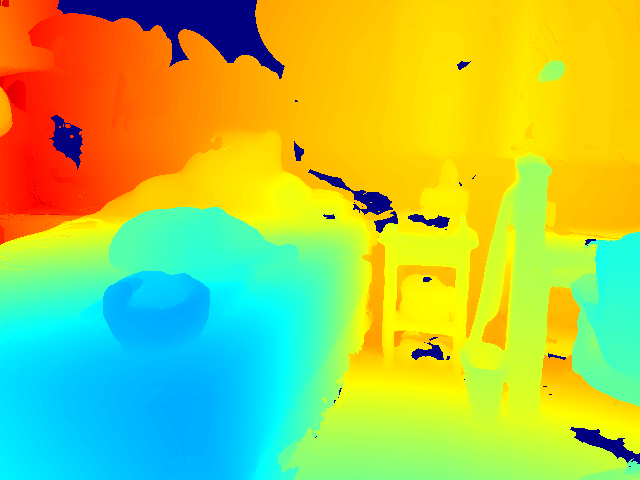}
    \end{subfigure}
    \begin{subfigure}[t]{0.16\linewidth}
        \centering
        \includegraphics[width=\linewidth]{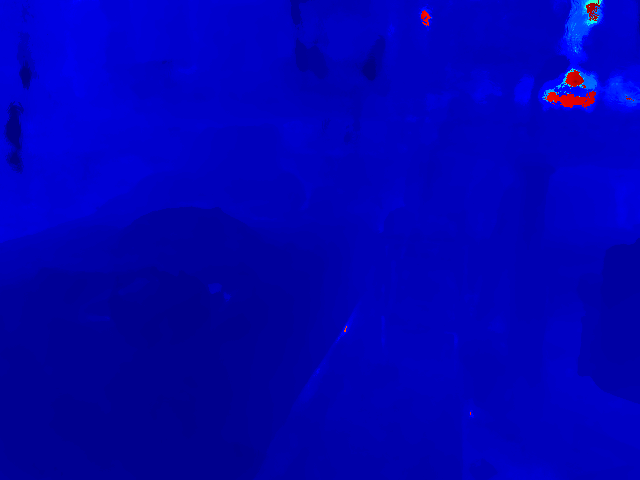}
    \end{subfigure}
    \begin{subfigure}[t]{0.16\linewidth}
        \centering
        \includegraphics[width=\linewidth]{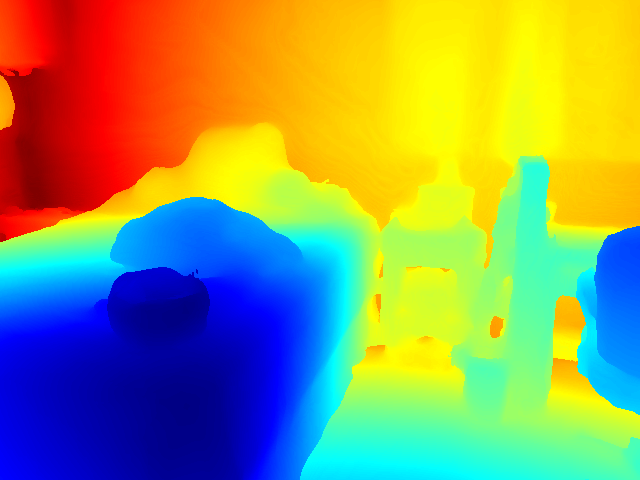}
    \end{subfigure}
    \begin{subfigure}[t]{0.16\linewidth}
        \centering
        \includegraphics[width=\linewidth]{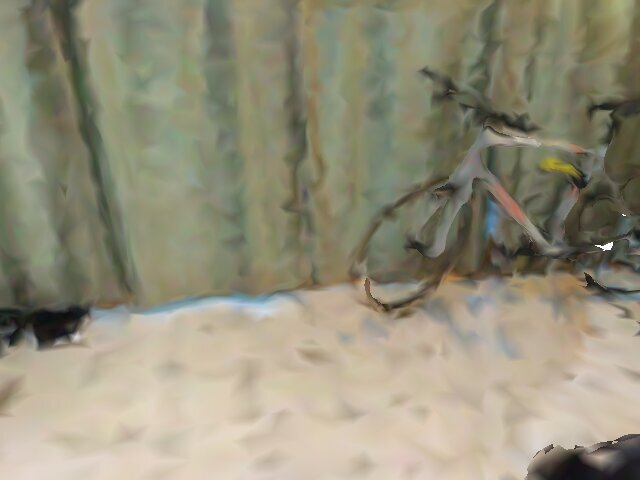}
    \end{subfigure}
    \begin{subfigure}[t]{0.16\linewidth}
        \centering
        \includegraphics[width=\linewidth]{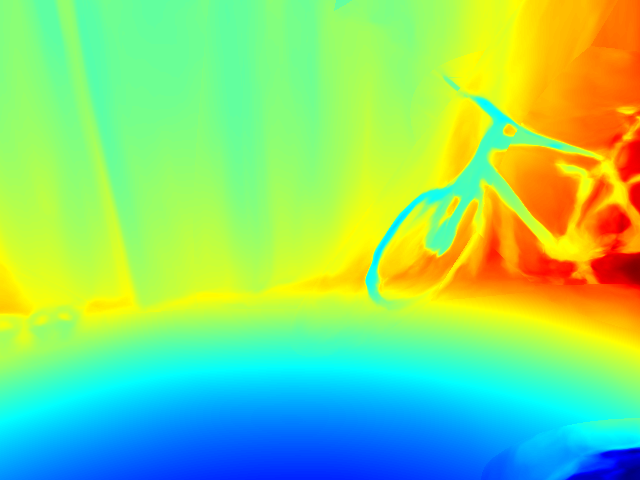}
    \end{subfigure}
    \begin{subfigure}[t]{0.16\linewidth}
        \centering
        \includegraphics[width=\linewidth]{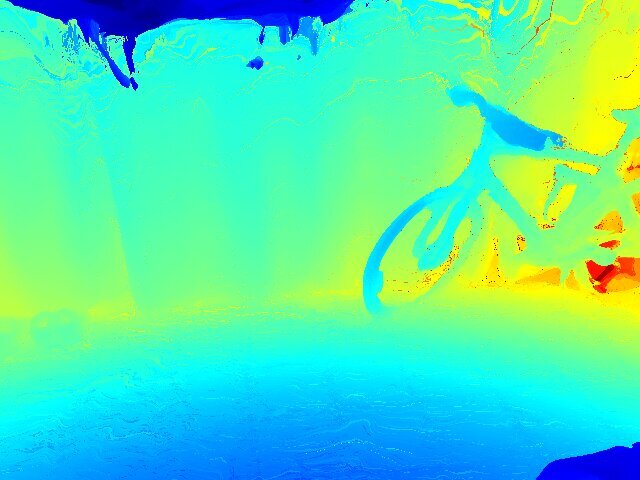}
    \end{subfigure}
    \begin{subfigure}[t]{0.16\linewidth}
        \centering
        \includegraphics[width=\linewidth]{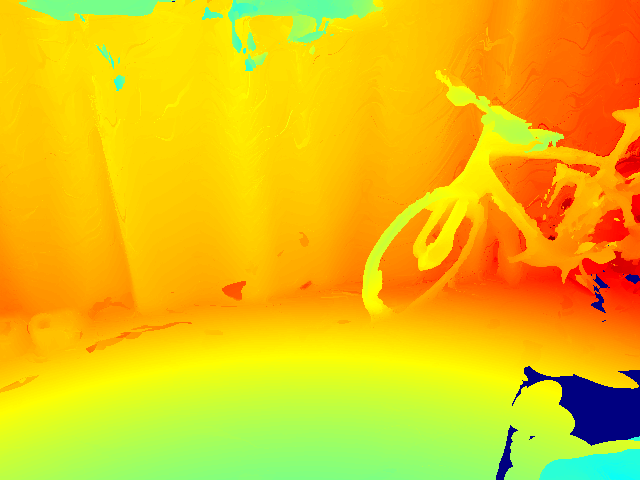}
    \end{subfigure}
    \begin{subfigure}[t]{0.16\linewidth}
        \centering
        \includegraphics[width=\linewidth]{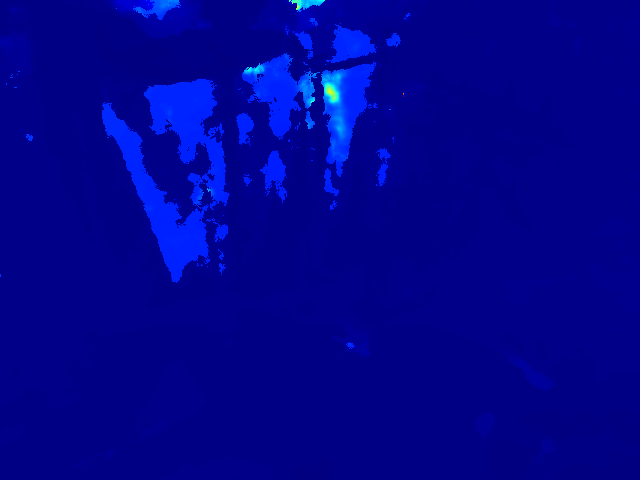}
    \end{subfigure}
    \begin{subfigure}[t]{0.16\linewidth}
        \centering
        \includegraphics[width=\linewidth]{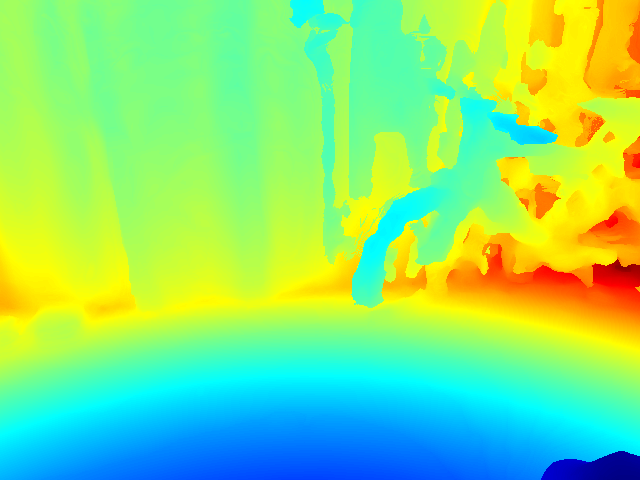}
    \end{subfigure}
    \begin{subfigure}[t]{0.16\linewidth}
        \centering
        \includegraphics[width=\linewidth]{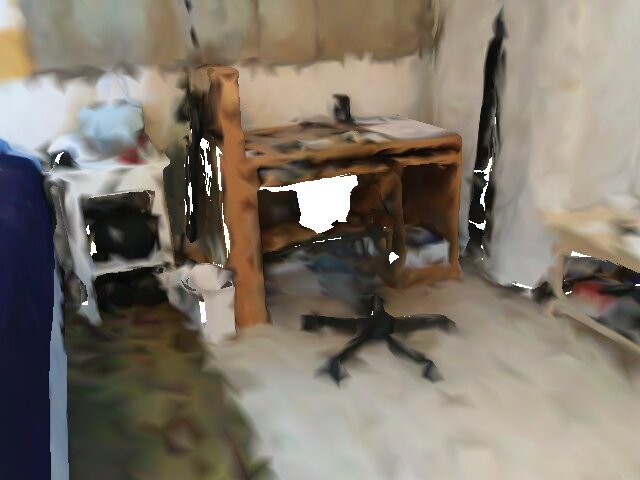}
    \end{subfigure}
    \begin{subfigure}[t]{0.16\linewidth}
        \centering
        \includegraphics[width=\linewidth]{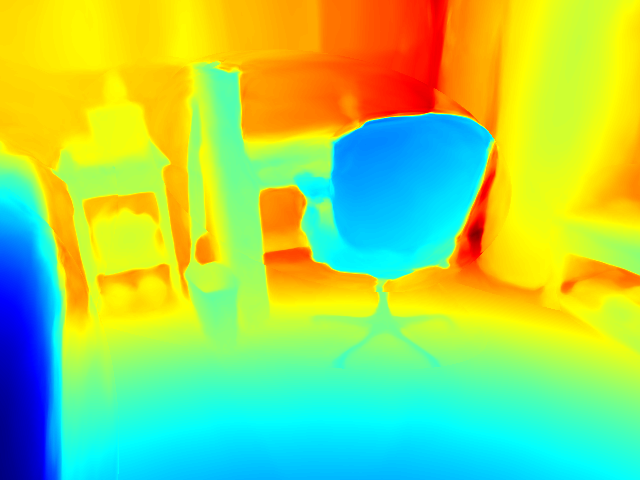}
    \end{subfigure}
    \begin{subfigure}[t]{0.16\linewidth}
        \centering
        \includegraphics[width=\linewidth]{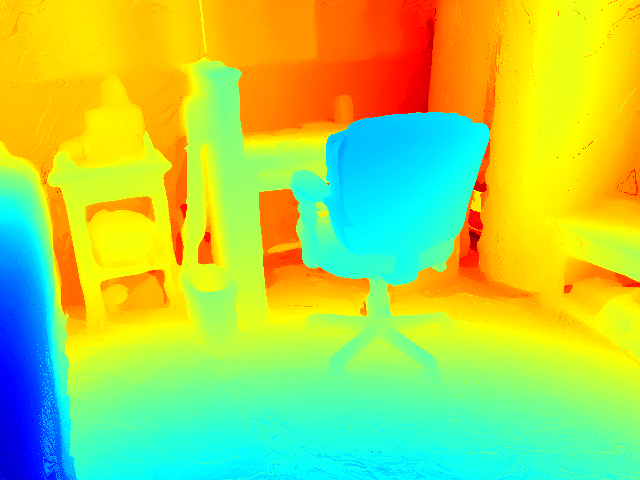}
    \end{subfigure}
    \begin{subfigure}[t]{0.16\linewidth}
        \centering
        \includegraphics[width=\linewidth]{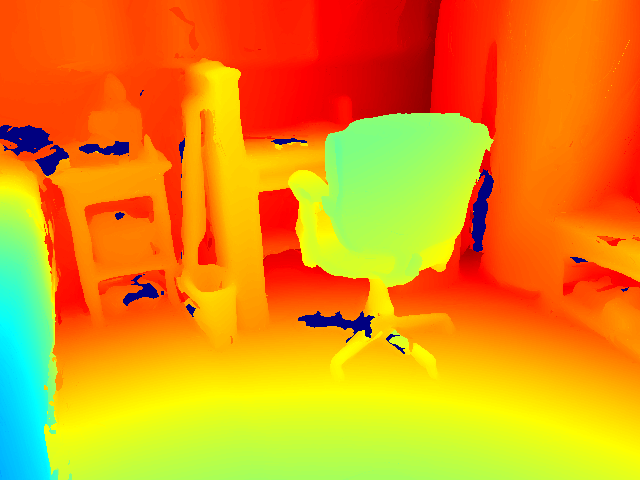}
    \end{subfigure}
    \begin{subfigure}[t]{0.16\linewidth}
        \centering
        \includegraphics[width=\linewidth]{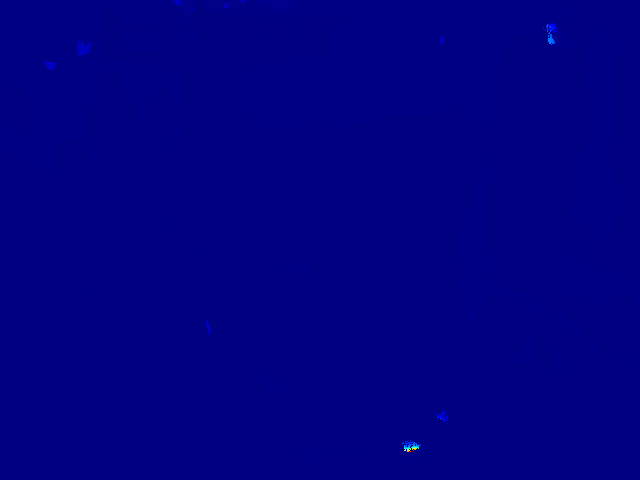}
    \end{subfigure}
    \begin{subfigure}[t]{0.16\linewidth}
        \centering
        \includegraphics[width=\linewidth]{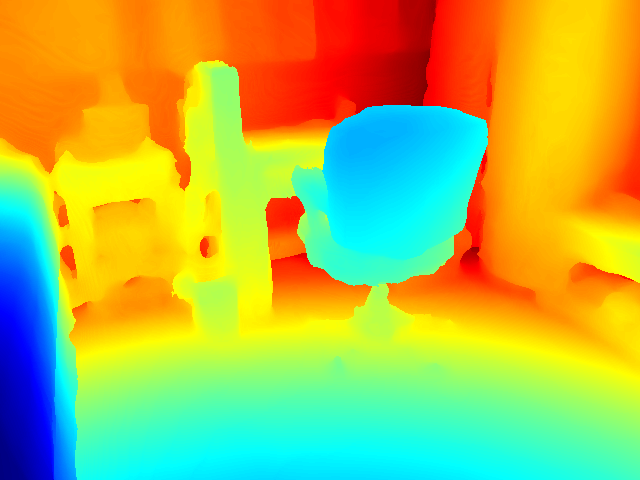}
    \end{subfigure}
    \begin{subfigure}[t]{0.16\linewidth}
        \centering
        \includegraphics[width=\linewidth]{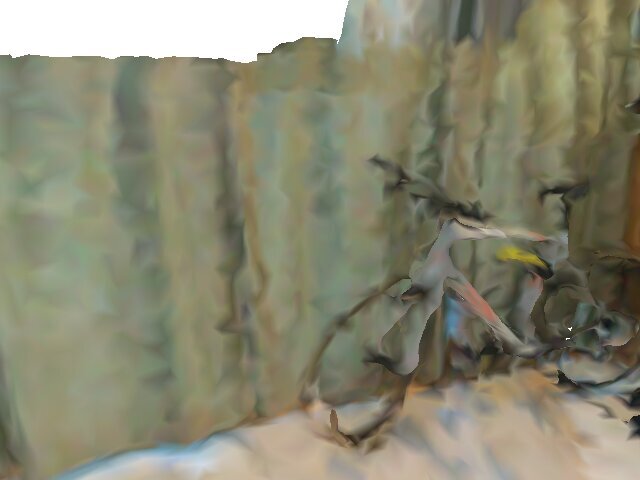}
        \caption{Scene}
    \end{subfigure}
    \begin{subfigure}[t]{0.16\linewidth}
        \centering
        \includegraphics[width=\linewidth]{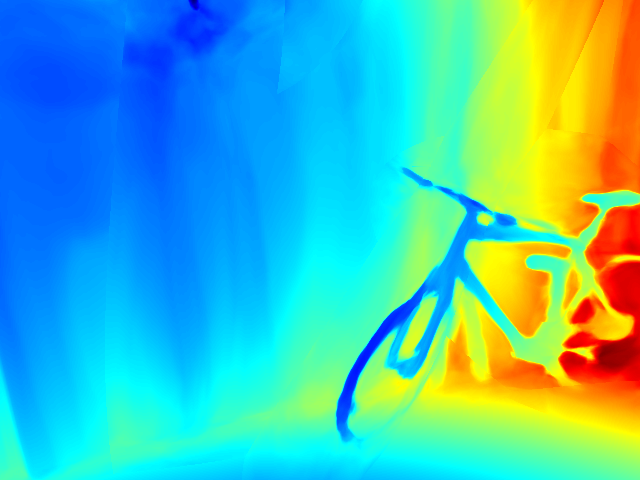}
        \caption{SDDF (ours)}
    \end{subfigure}
    \begin{subfigure}[t]{0.16\linewidth}
        \centering
        \includegraphics[width=\linewidth]{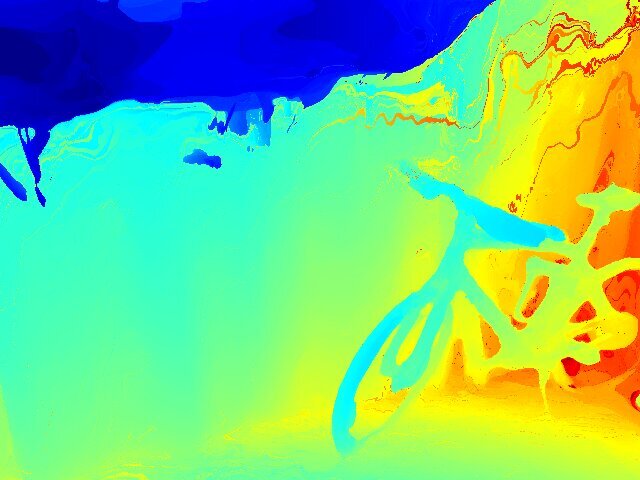}
        \caption{RaDe-GS~\cite{radegs2024}}
    \end{subfigure}
    \begin{subfigure}[t]{0.16\linewidth}
        \centering
        \includegraphics[width=\linewidth]{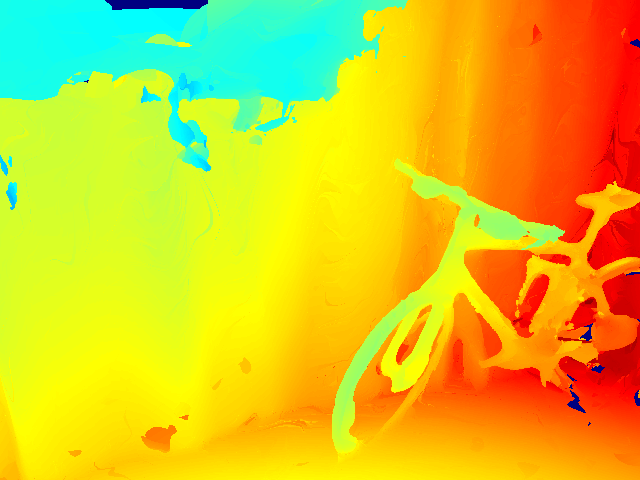}
        \caption{w/o RGB}
    \end{subfigure}
    \begin{subfigure}[t]{0.16\linewidth}
        \centering
        \includegraphics[width=\linewidth]{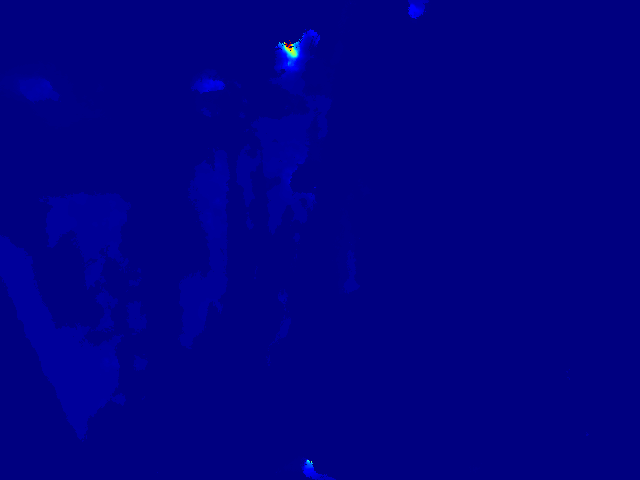}
        \caption{Nerfacto~\cite{nerfstudio2023}}
    \end{subfigure}
    \begin{subfigure}[t]{0.16\linewidth}
        \centering
        \includegraphics[width=\linewidth]{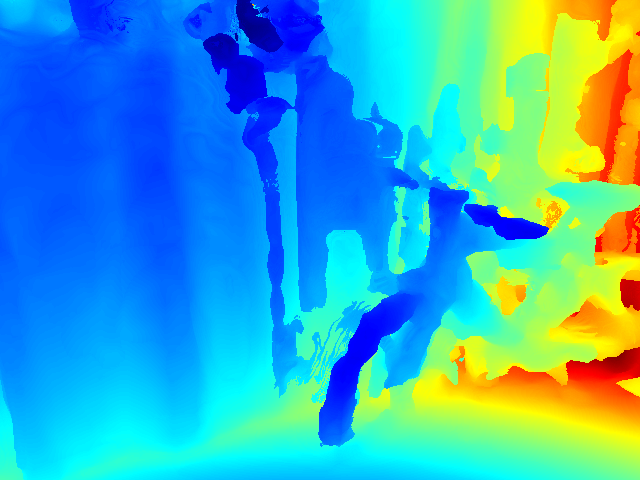}
        \caption{SDF~\cite{instantngp2022}}
    \end{subfigure}
    \caption{More qualitative SDDF prediction comparison between our method and baselines on real data from ScanNet~\cite{dai2017scannet} scene 0000-00. The six columns are the RGB images of the scene, SDDF prediction by our method, RaDe-GS~\cite{radegs2024}, RaDe-GS~\cite{radegs2024} trained without RGB images, Nerfacto~\cite{nerfstudio2023} and \sdfbaseline~\cite{instantngp2022} with sphere tracing~\cite{spheretracing1996}. The results lead to the same conclusion we draw from the synthesized datasets.}
    \label{fig:more_sddf_comp_scannet}
\end{figure*}

\begin{figure*}
\centering
\begin{tabular}{cccc}
    \includegraphics[width=0.2265\linewidth,trim={350pt 10pt 350pt 10pt},clip]{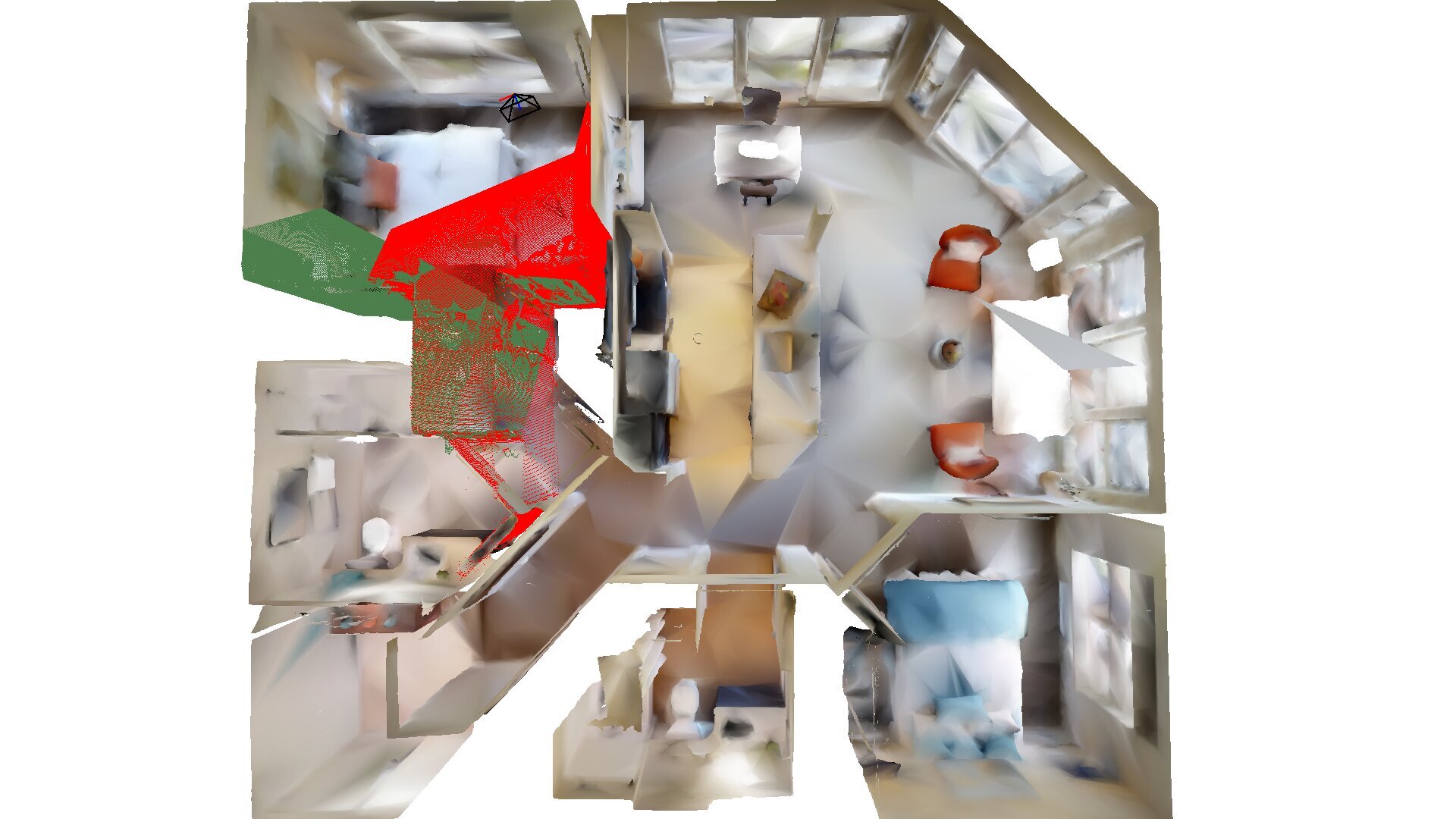} &
    \includegraphics[width=0.2265\linewidth,trim={350pt 10pt 350pt 10pt},clip]{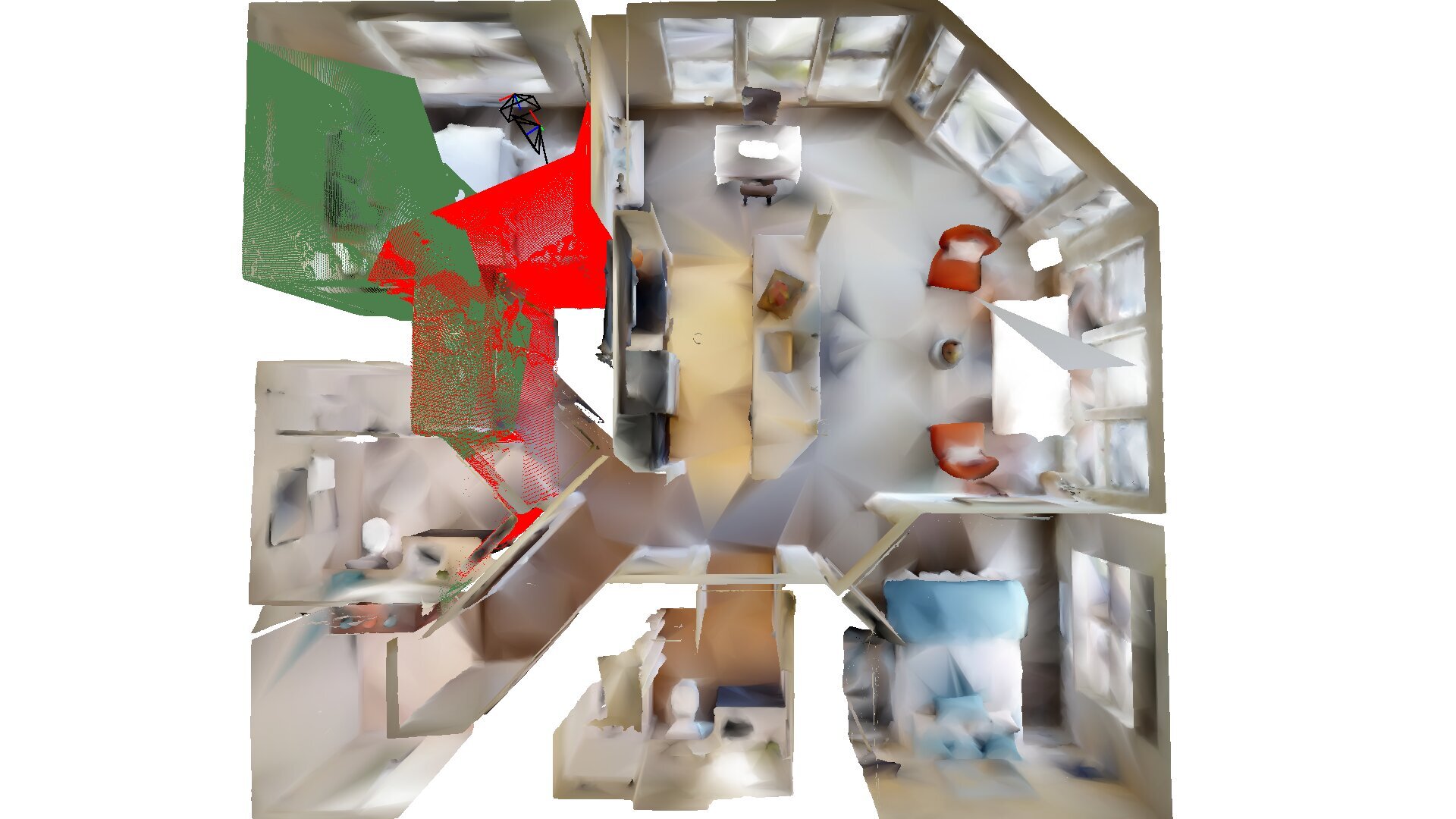} &
    \includegraphics[width=0.2265\linewidth,trim={350pt 10pt 350pt 10pt},clip]{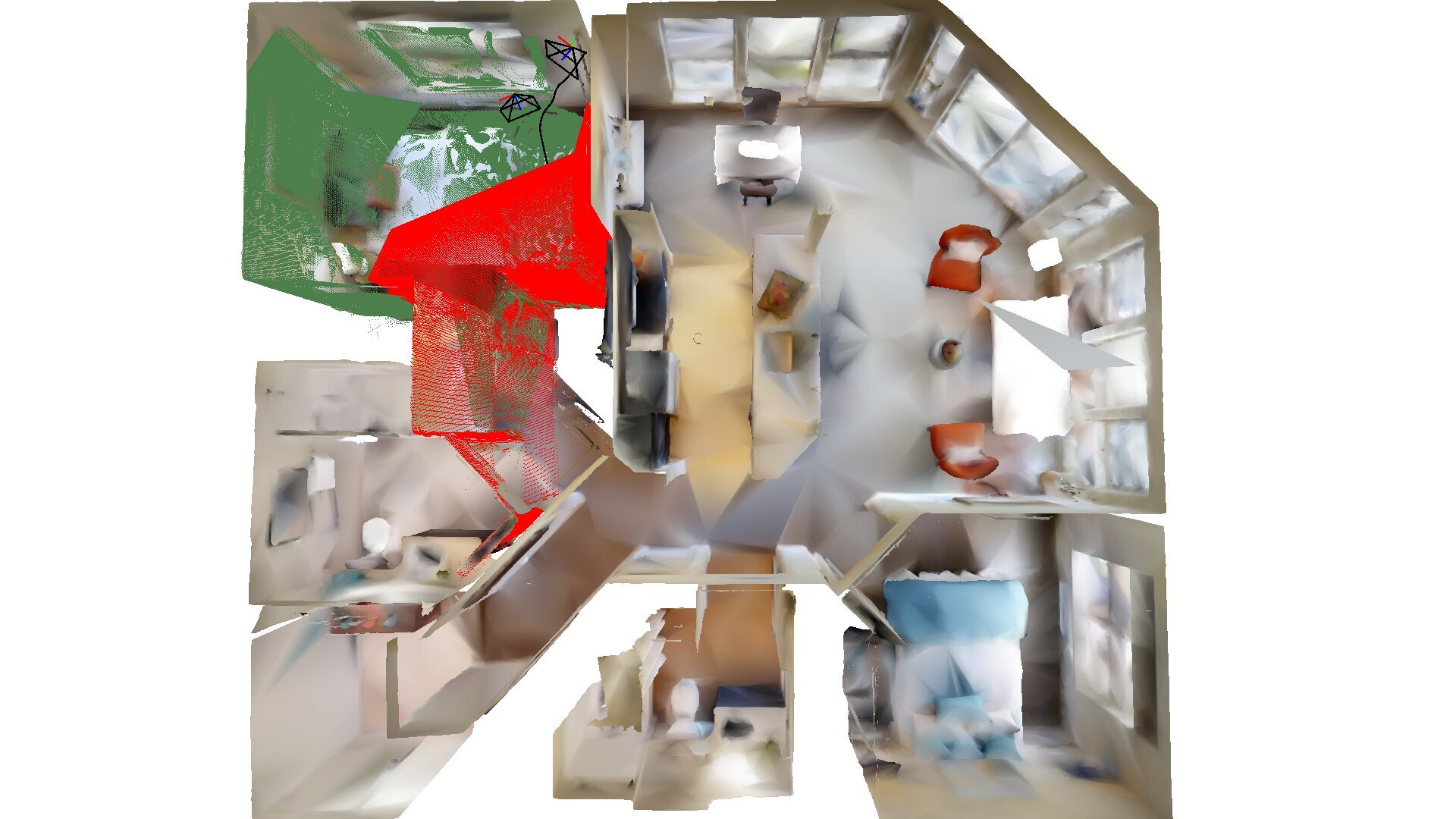} &
    \includegraphics[width=0.2265\linewidth,trim={350pt 10pt 350pt 10pt},clip]{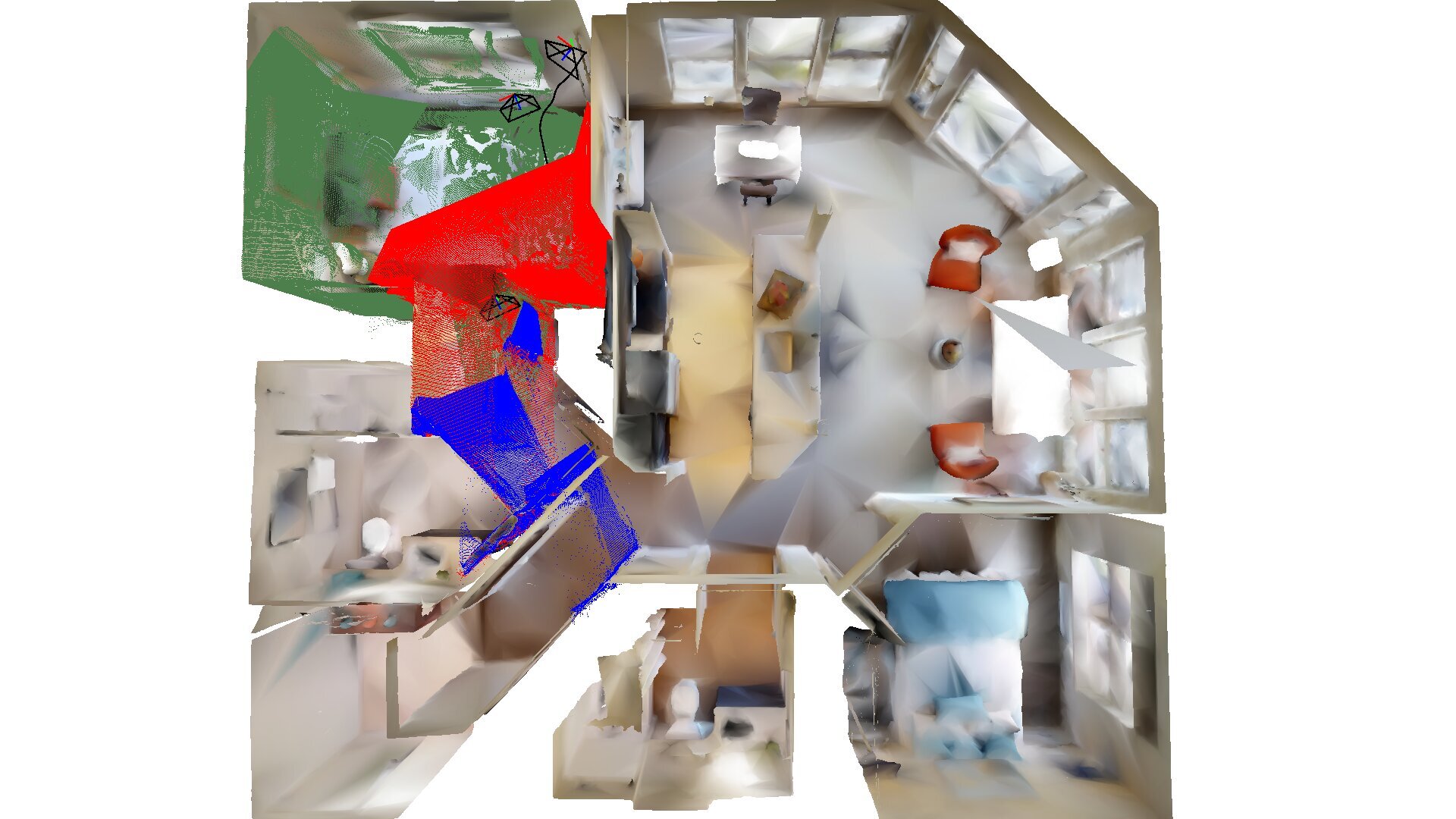} \\
    \includegraphics[width=0.2265\linewidth,trim={350pt 10pt 350pt 10pt},clip]{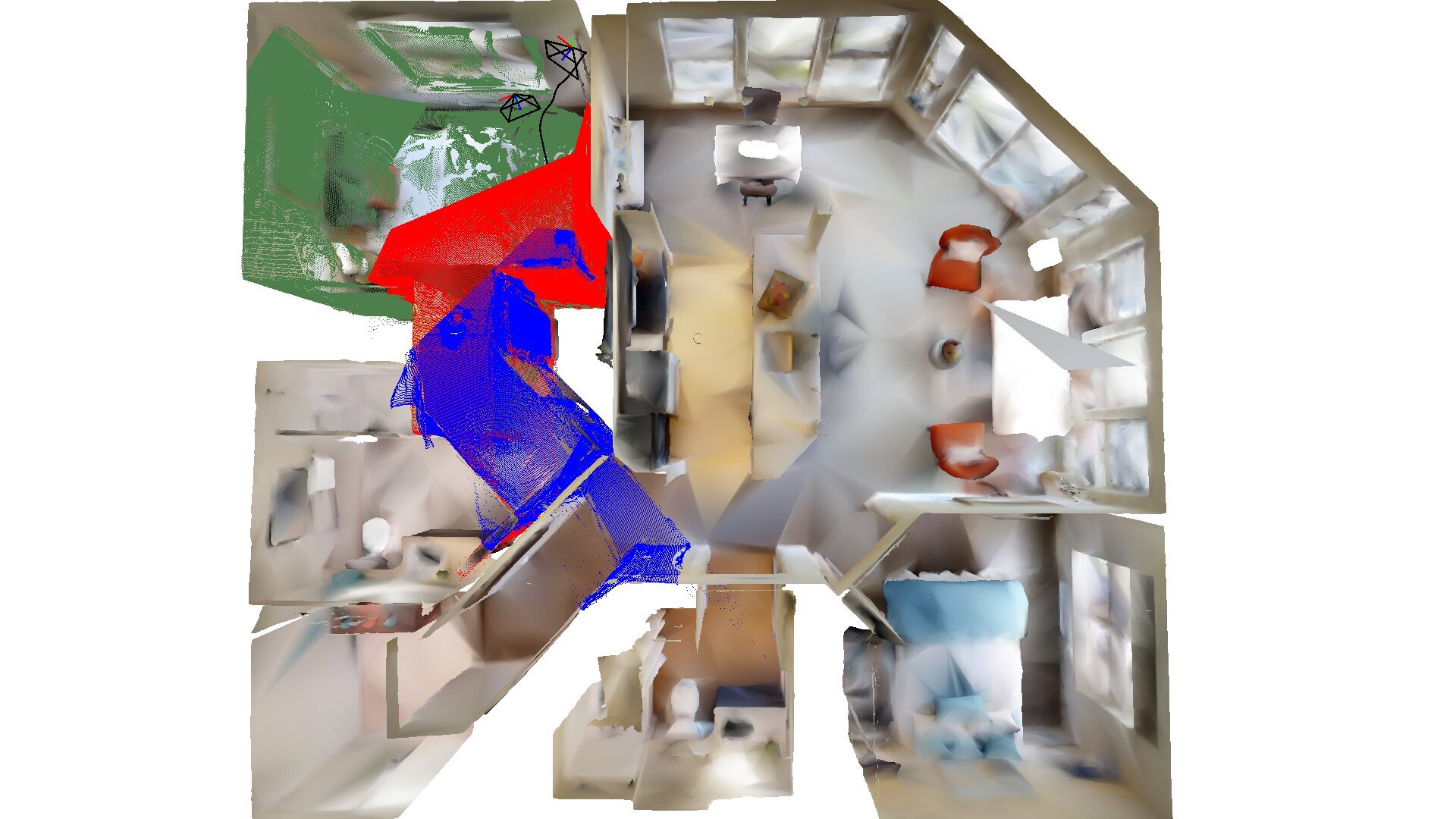} &
    \includegraphics[width=0.2265\linewidth,trim={350pt 10pt 350pt 10pt},clip]{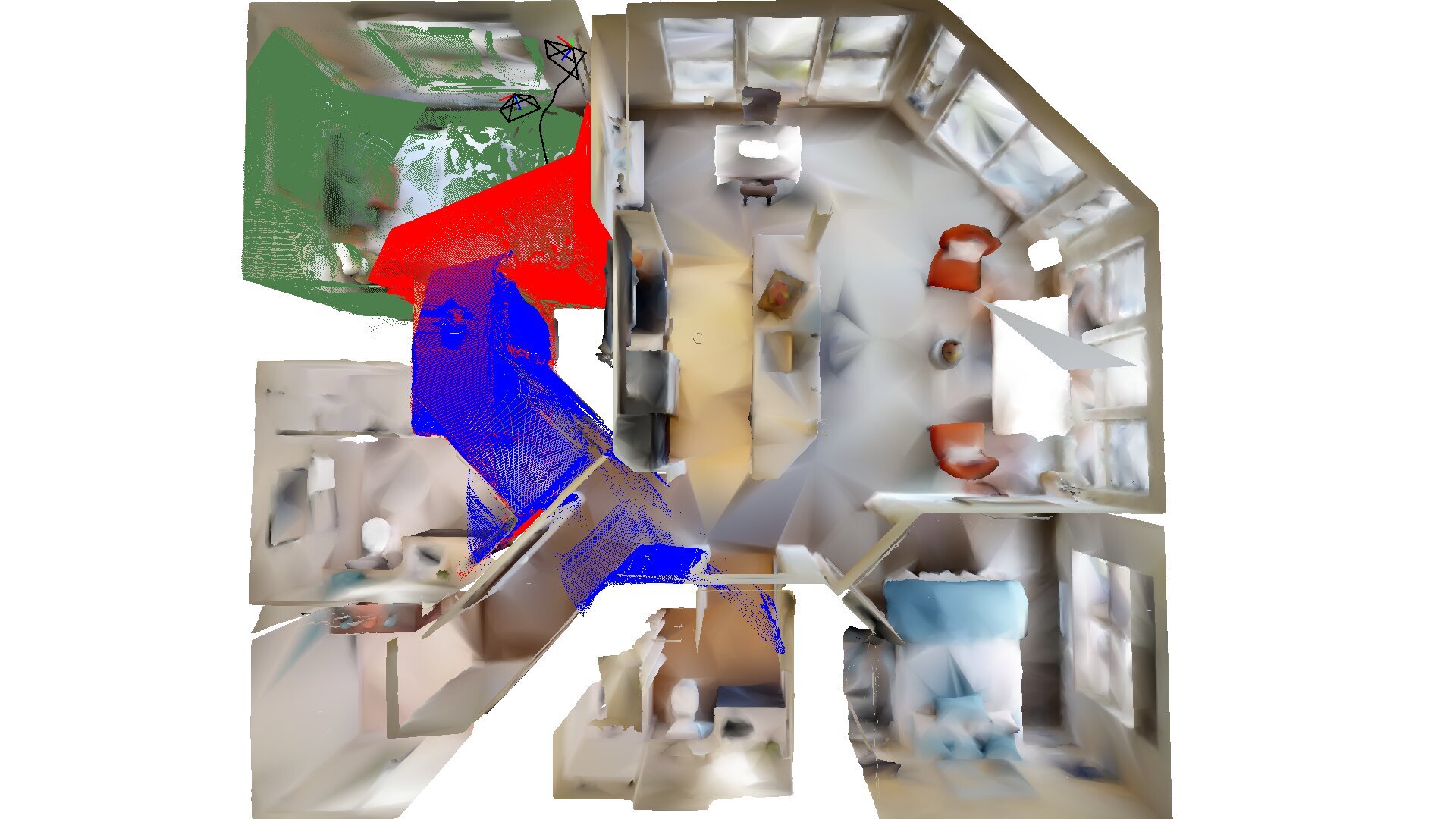} &
    \includegraphics[width=0.2265\linewidth,trim={350pt 10pt 350pt 10pt},clip]{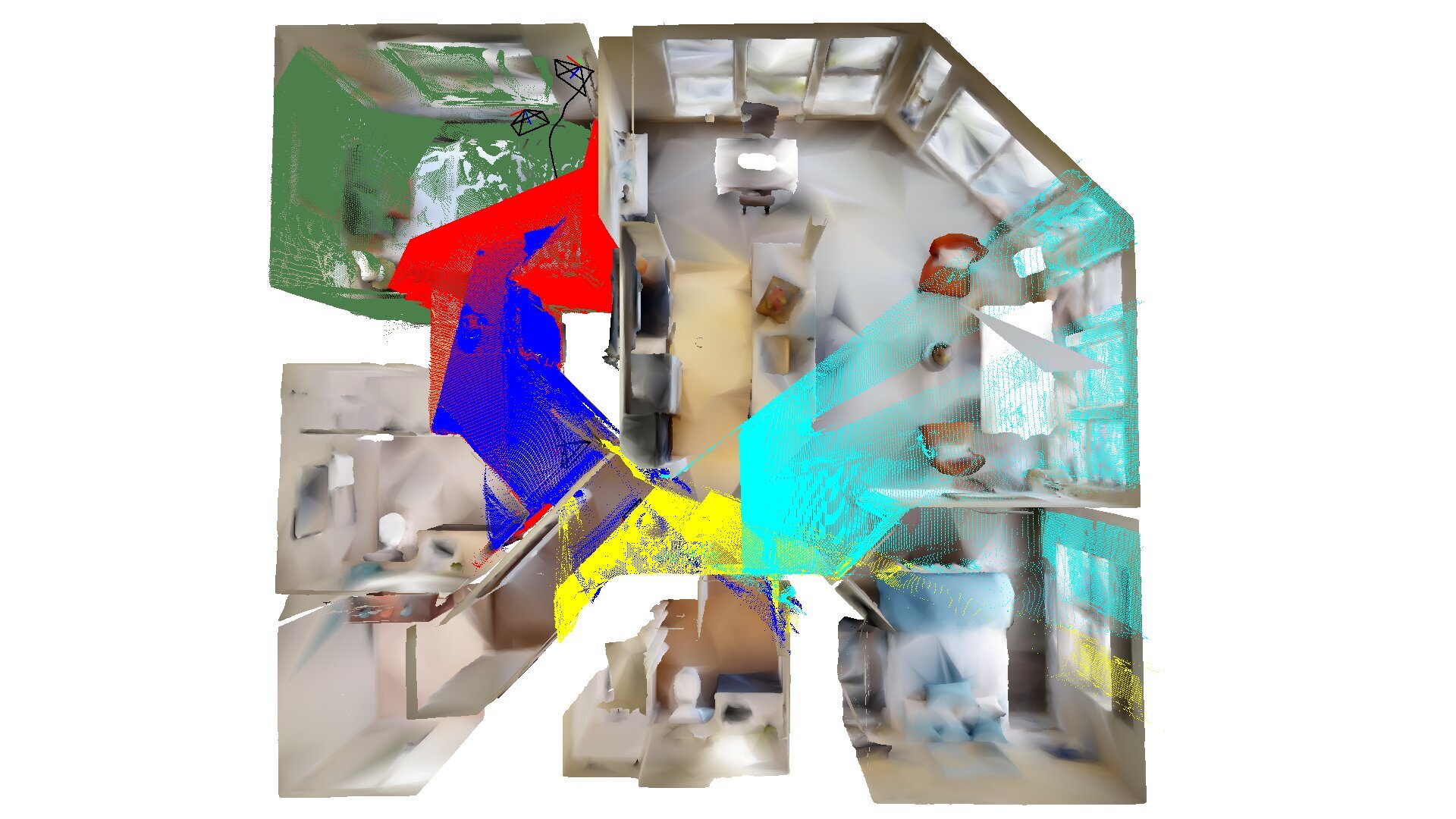} &
    \includegraphics[width=0.2265\linewidth,trim={350pt 10pt 350pt 10pt},clip]{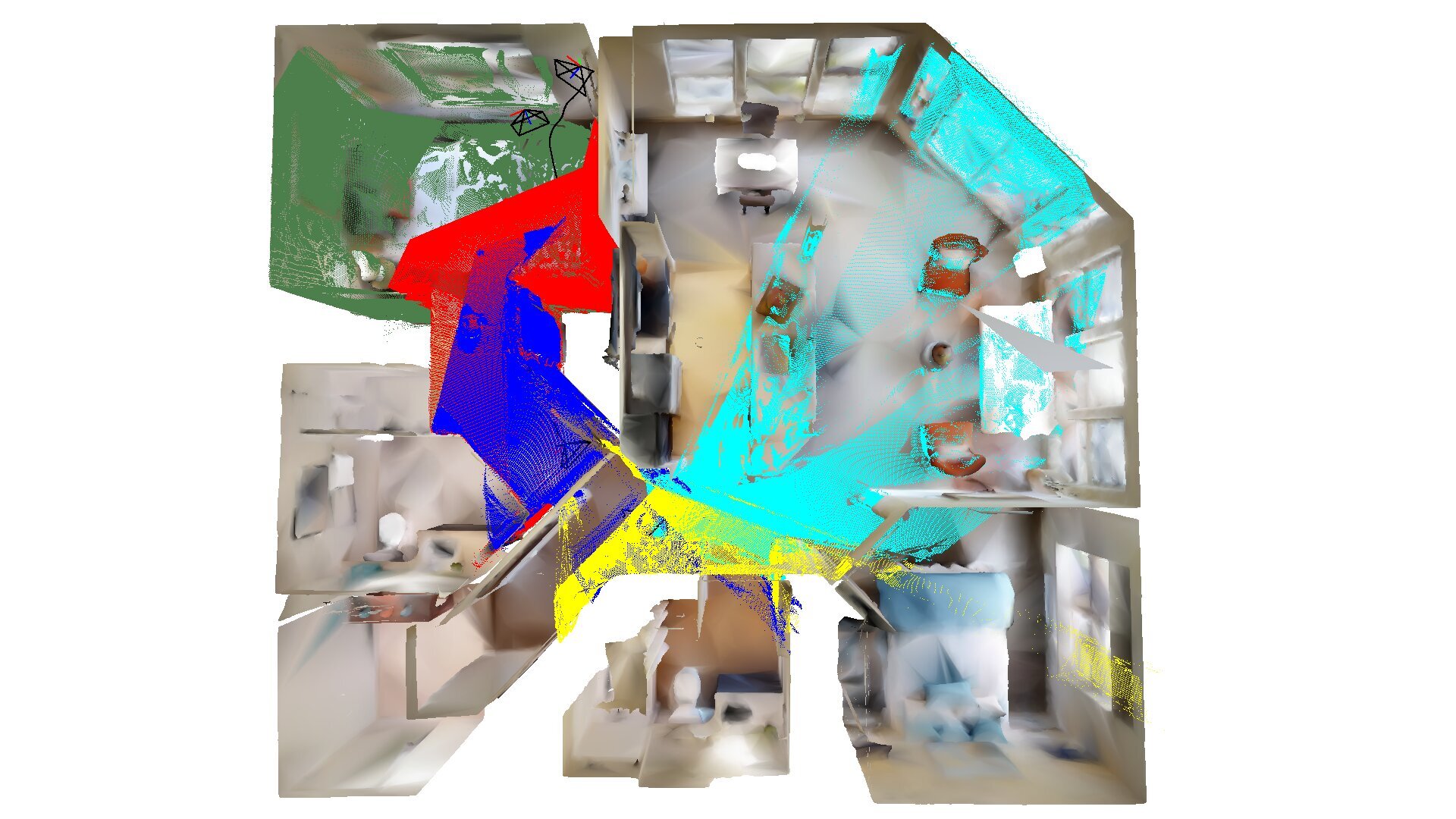} \\
    \includegraphics[width=0.2265\linewidth,trim={350pt 10pt 380pt 40pt},clip]{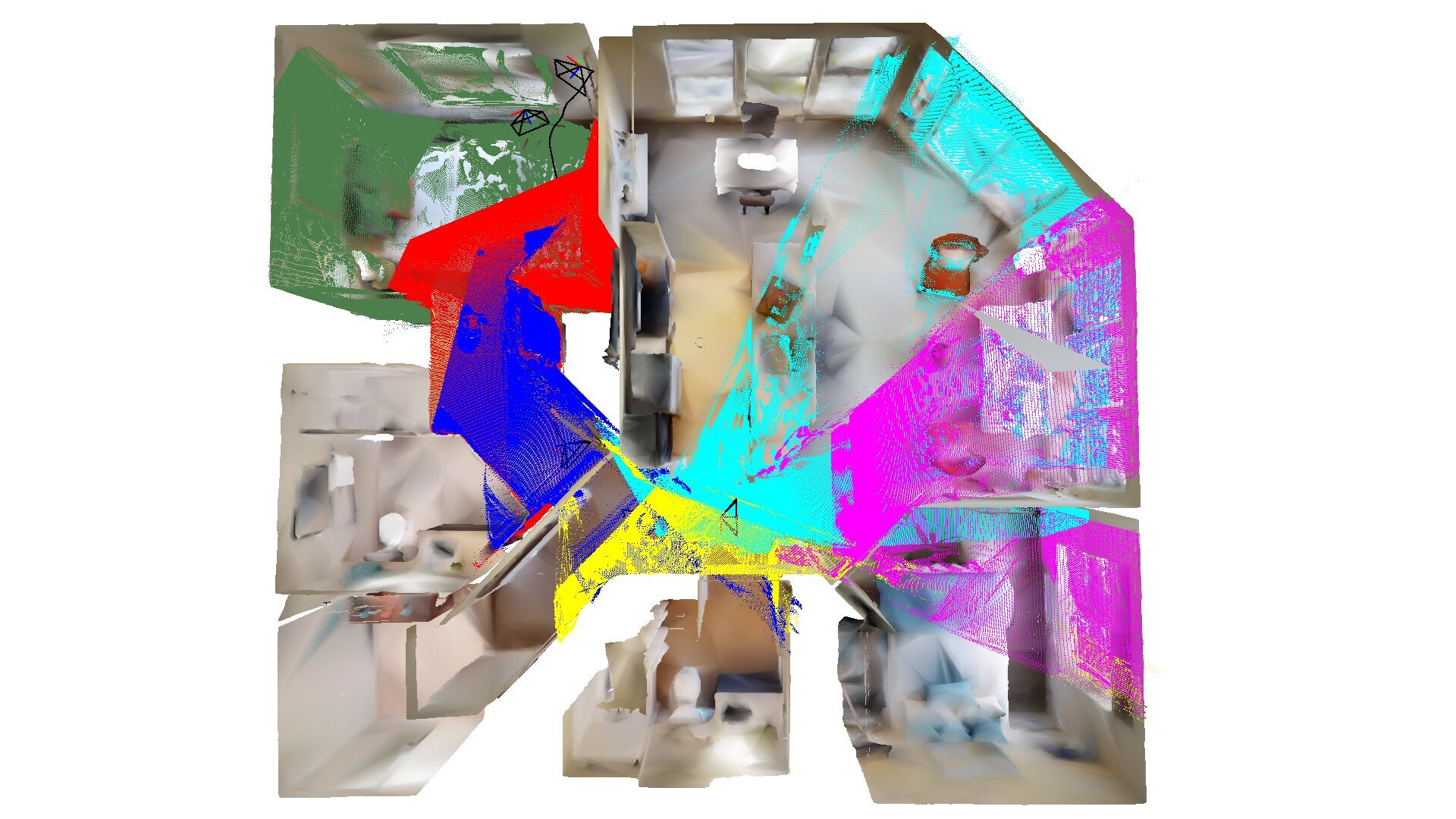} &
    \includegraphics[width=0.2265\linewidth,trim={330pt 10pt 400pt 40pt},clip]{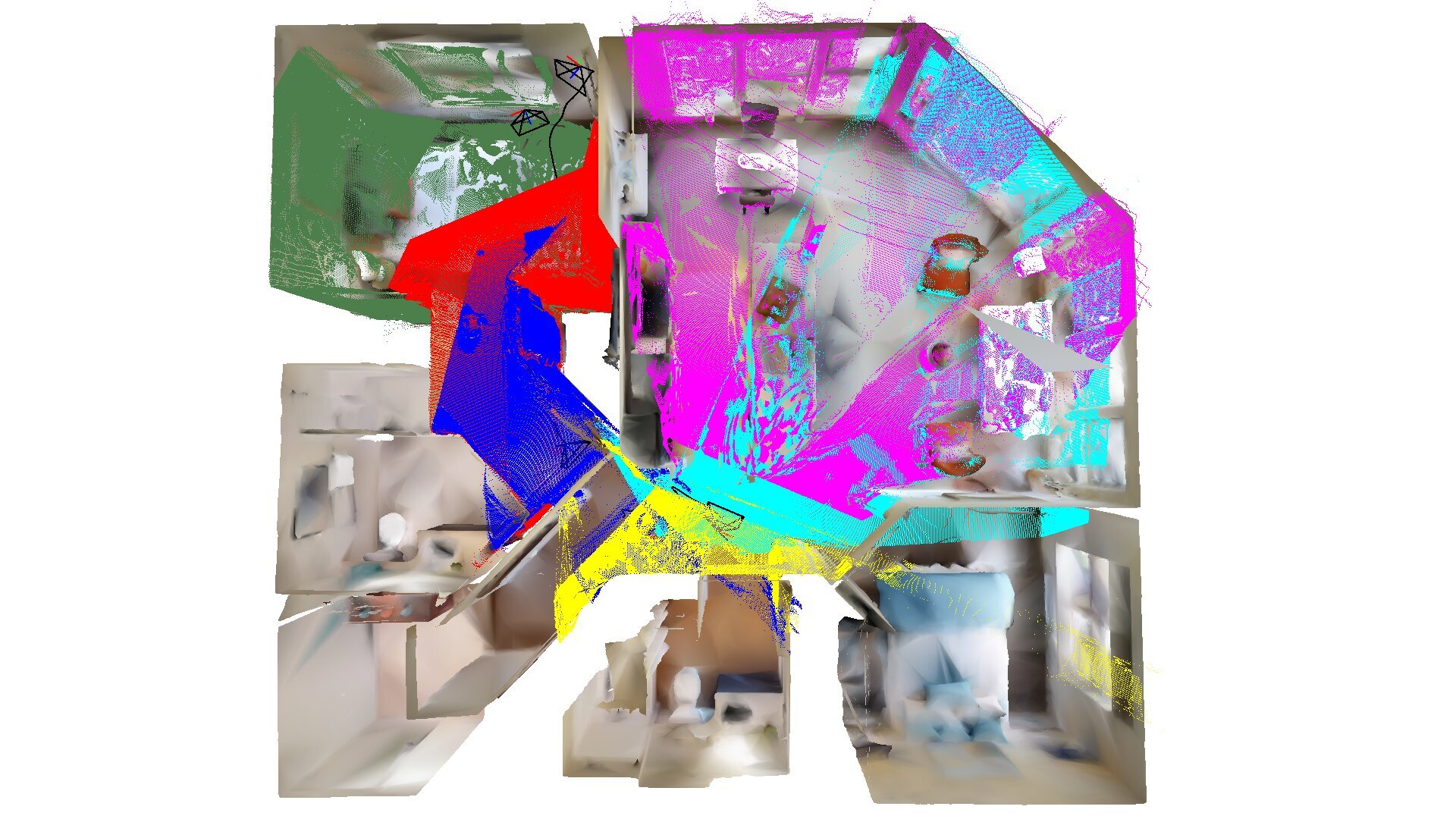} &
    \includegraphics[width=0.2265\linewidth,trim={330pt 10pt 400pt 40pt},clip]{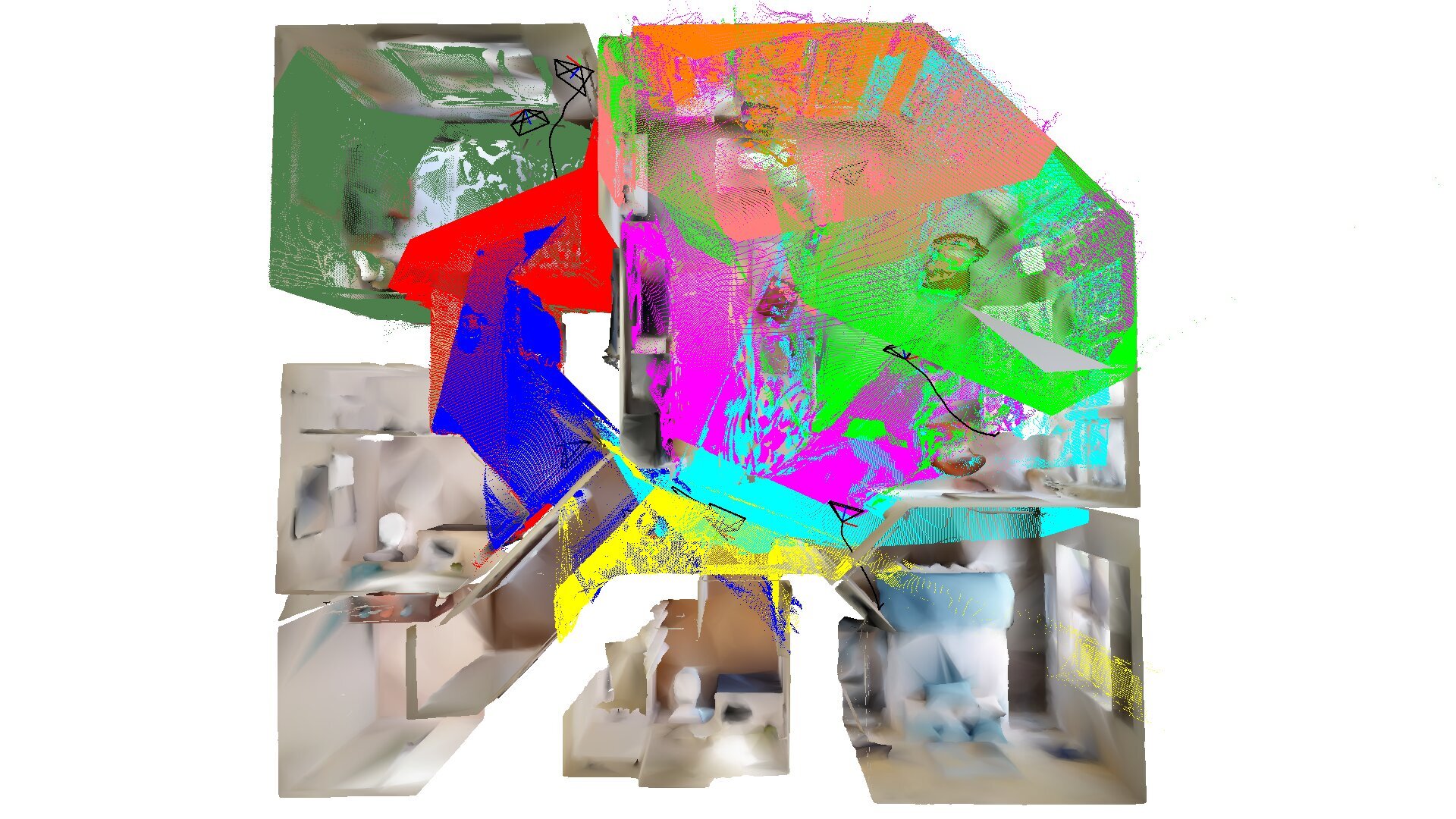} &
    \includegraphics[width=0.2265\linewidth,trim={330pt 10pt 400pt 40pt},clip]{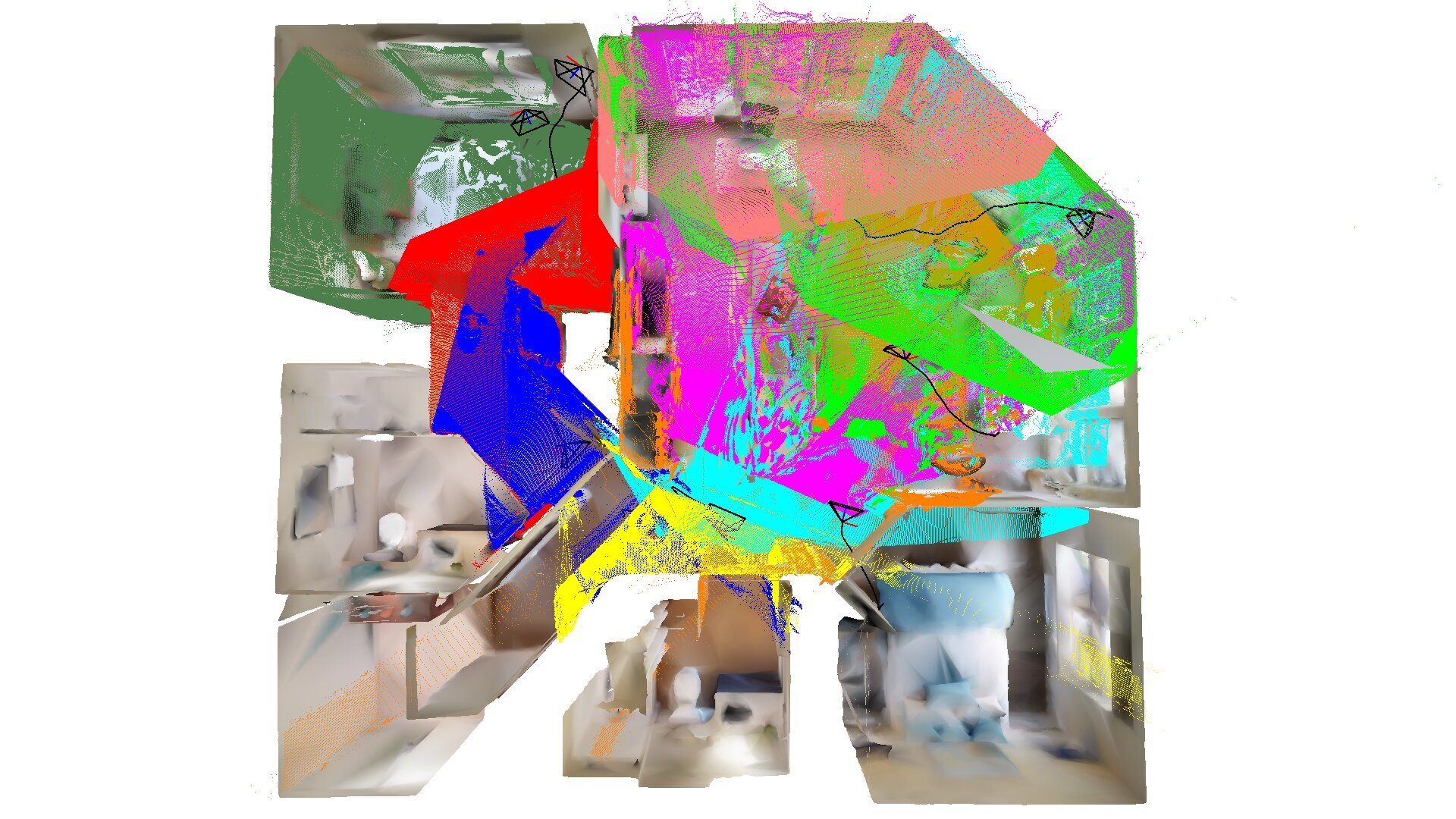}
\end{tabular}
\caption{\small Visualization of differentiable view optimization along a trajectory in the Gibson Allensville scene. The figures show selected frames at different iteration steps of the optimization process.}
\label{fig:traj_opt_gibson_allensville}
\end{figure*}
\begin{figure*}
\centering
\begin{tabular}{cccc}
    \includegraphics[width=0.2265\linewidth,trim={0pt 40pt 0pt 40pt},clip]{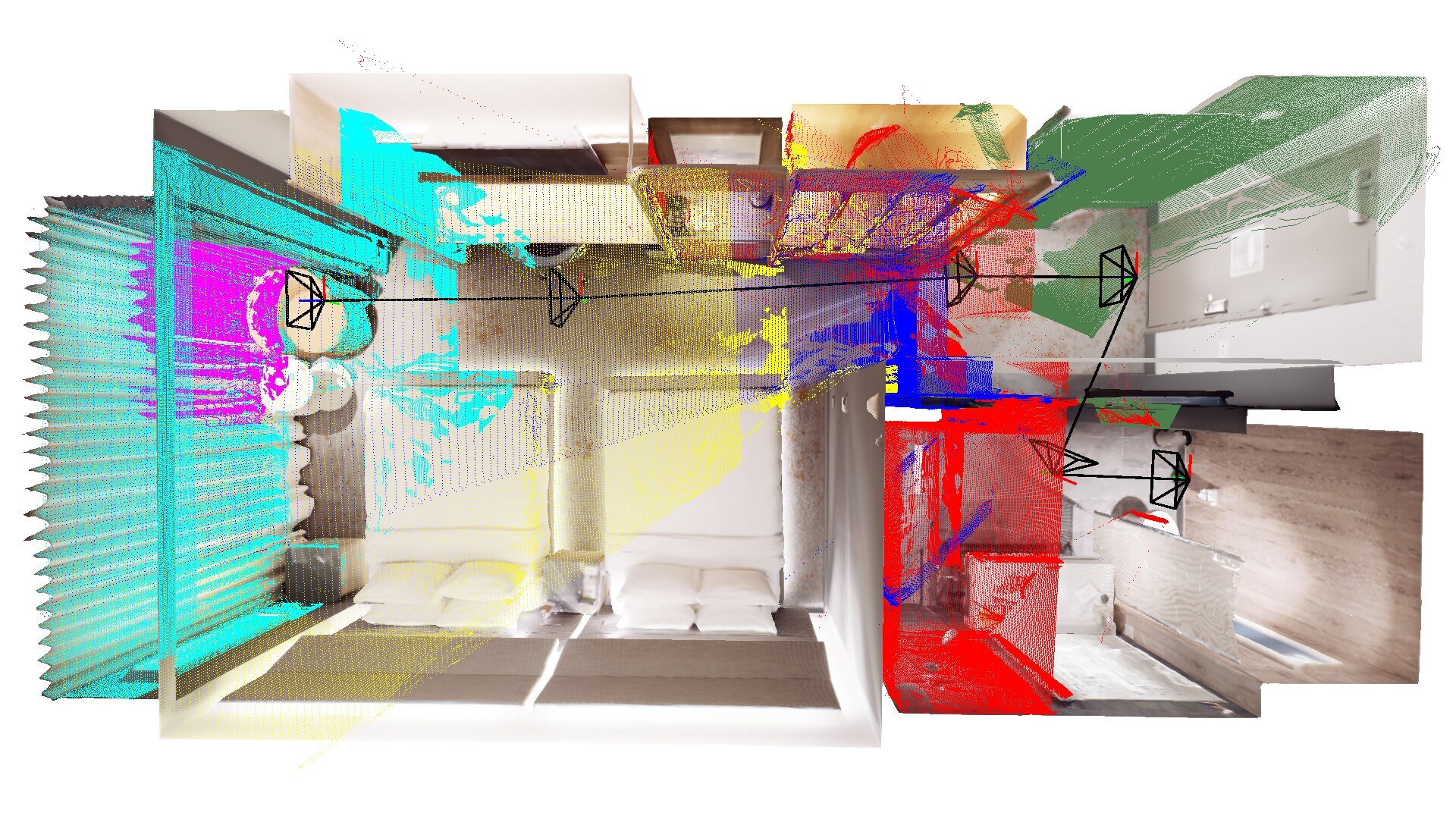} &
    \includegraphics[width=0.2265\linewidth,trim={0pt 40pt 0pt 40pt},clip]{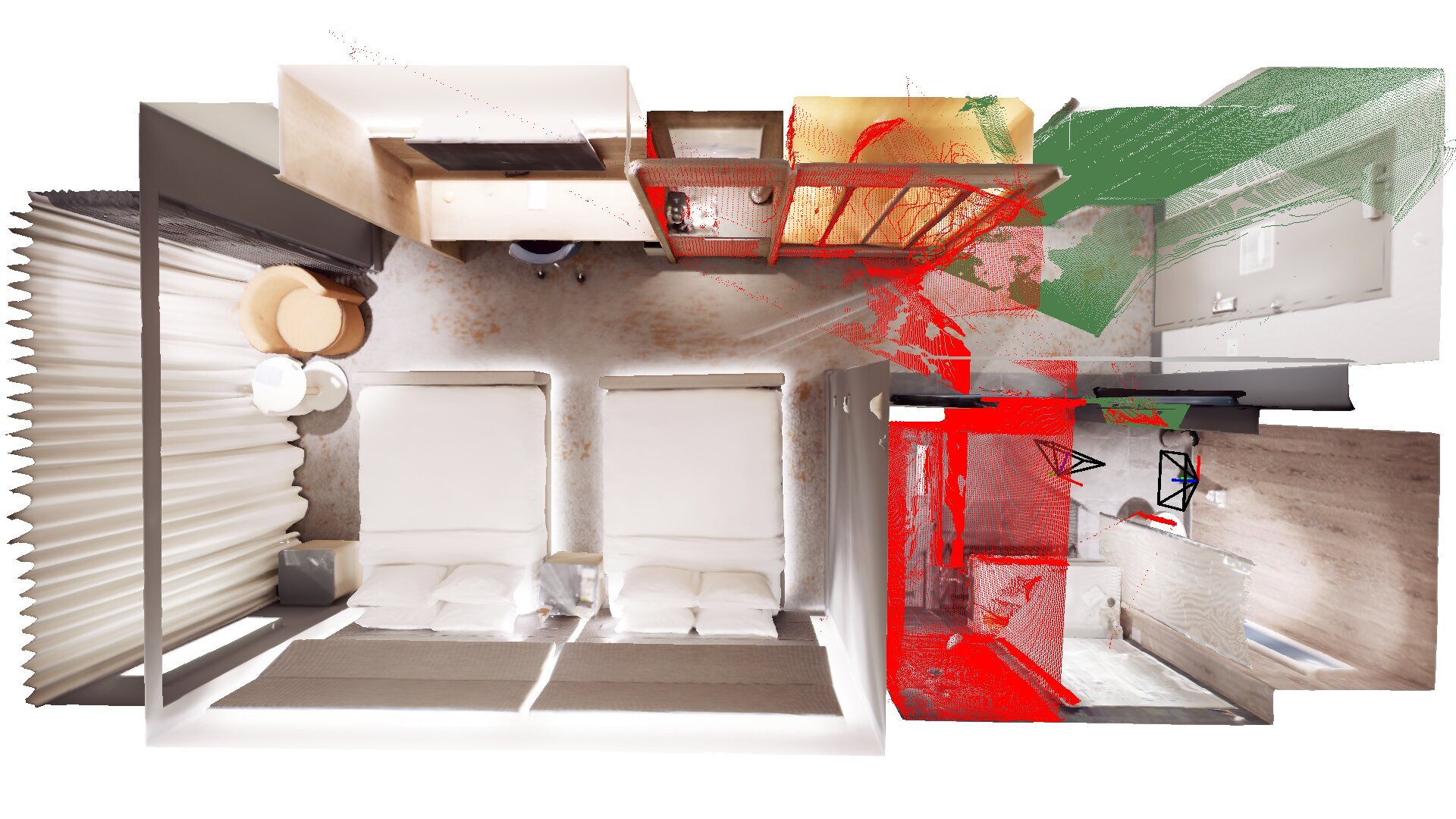} &
    \includegraphics[width=0.2265\linewidth,trim={0pt 40pt 0pt 40pt},clip]{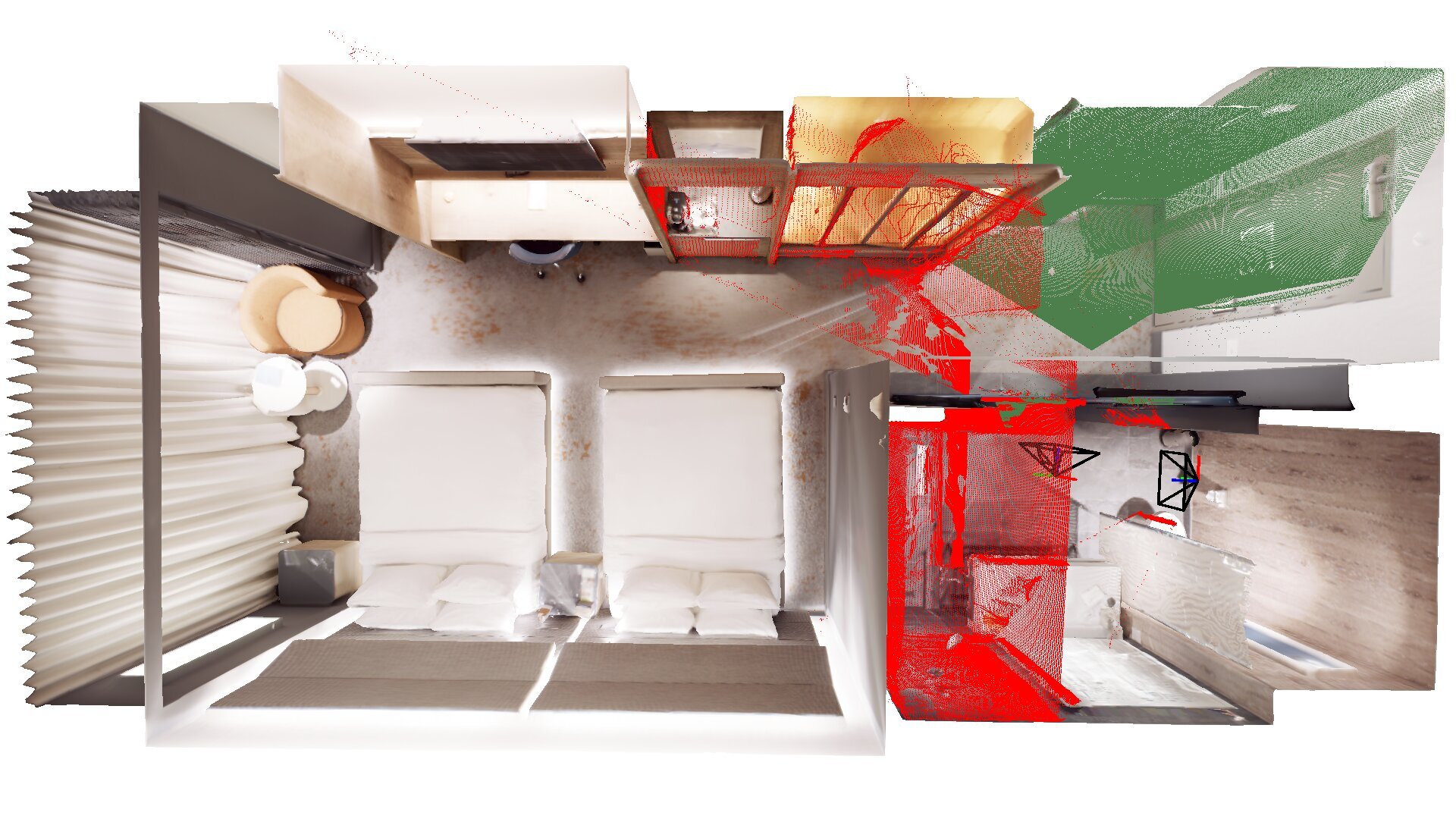} &
    \includegraphics[width=0.2265\linewidth,trim={0pt 40pt 0pt 40pt},clip]{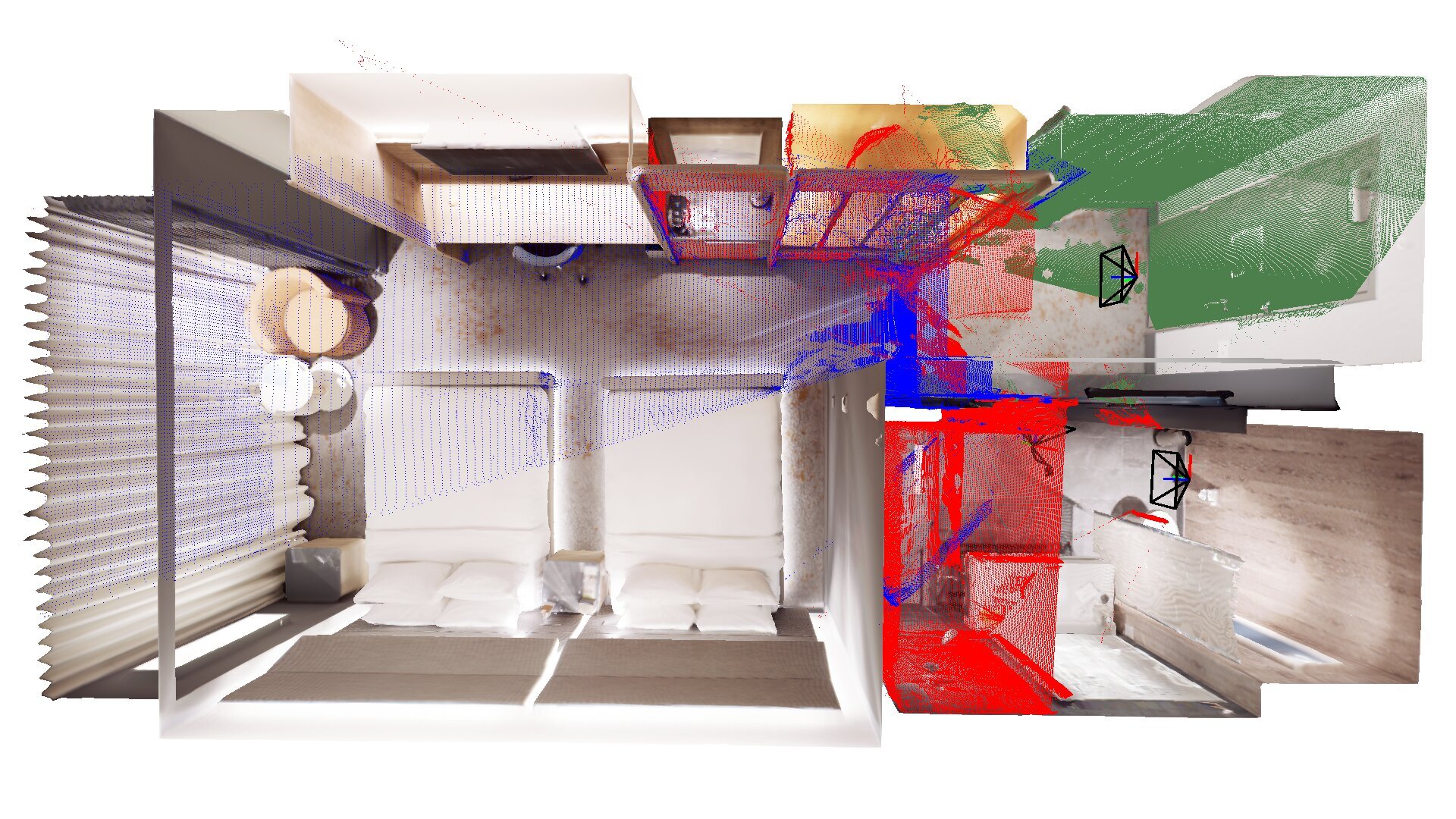} \\
    \includegraphics[width=0.2265\linewidth,trim={0pt 40pt 0pt 40pt},clip]{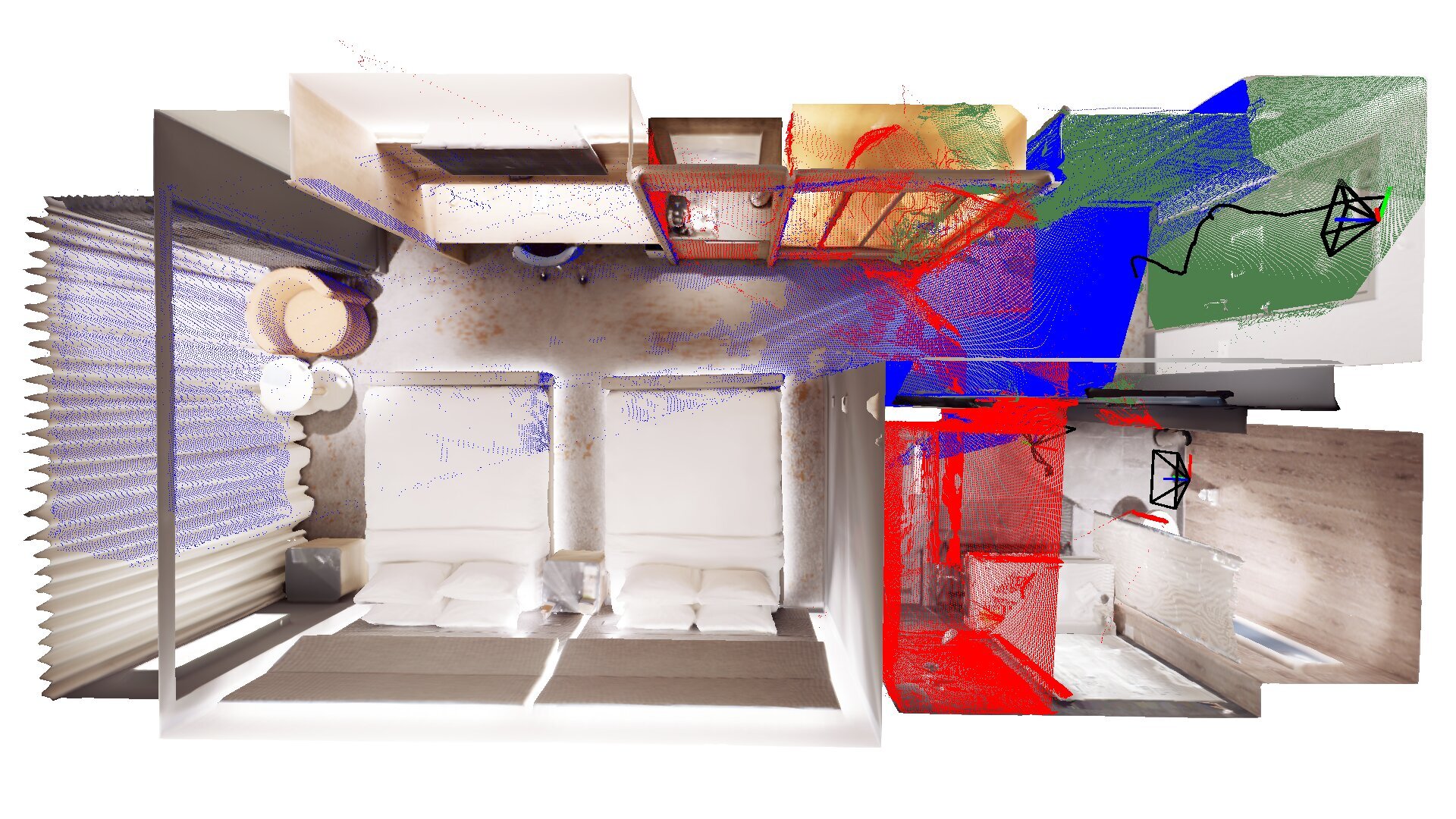} &
    \includegraphics[width=0.2265\linewidth,trim={0pt 40pt 0pt 40pt},clip]{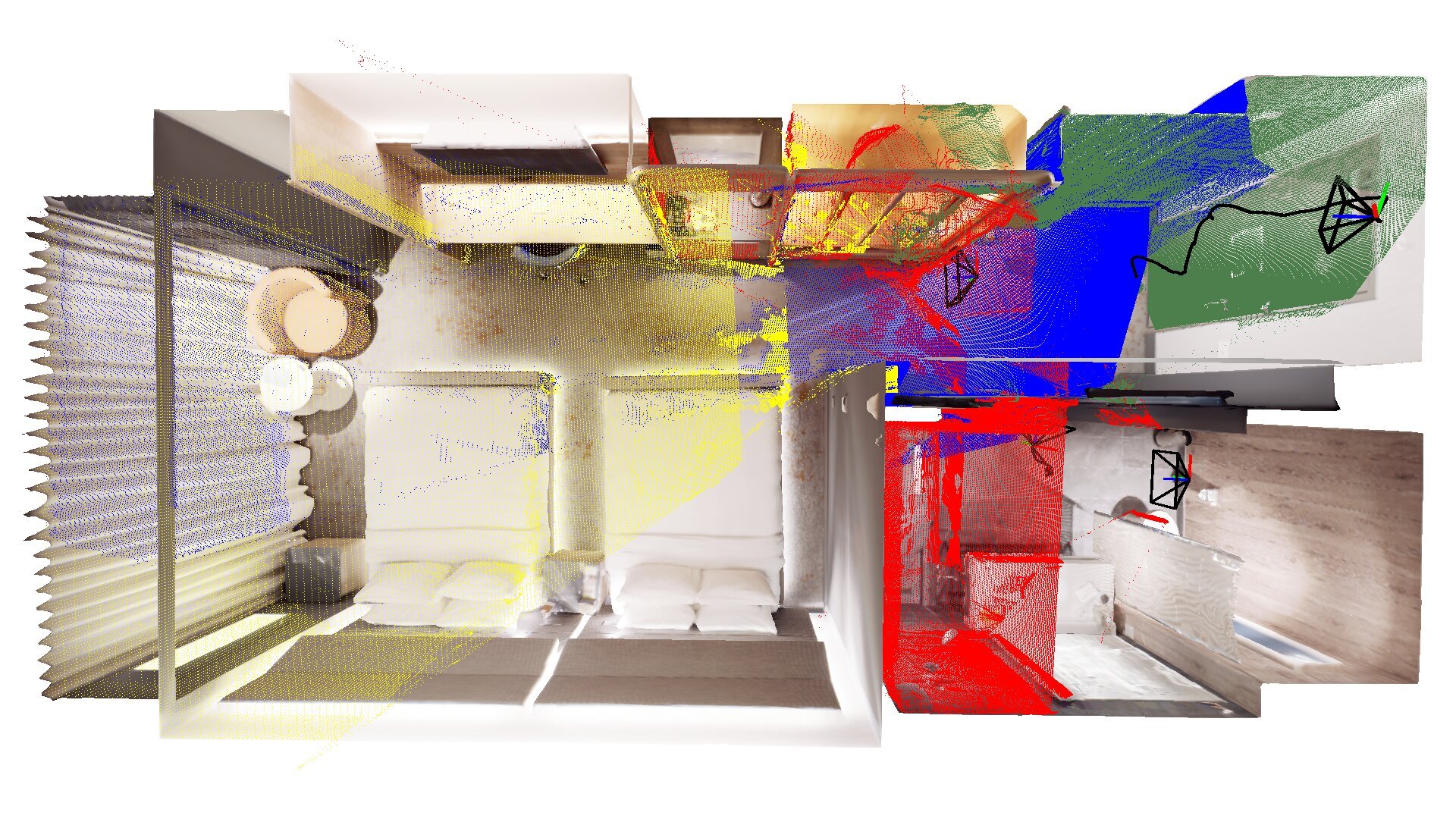} &
    \includegraphics[width=0.2265\linewidth,trim={0pt 40pt 0pt 40pt},clip]{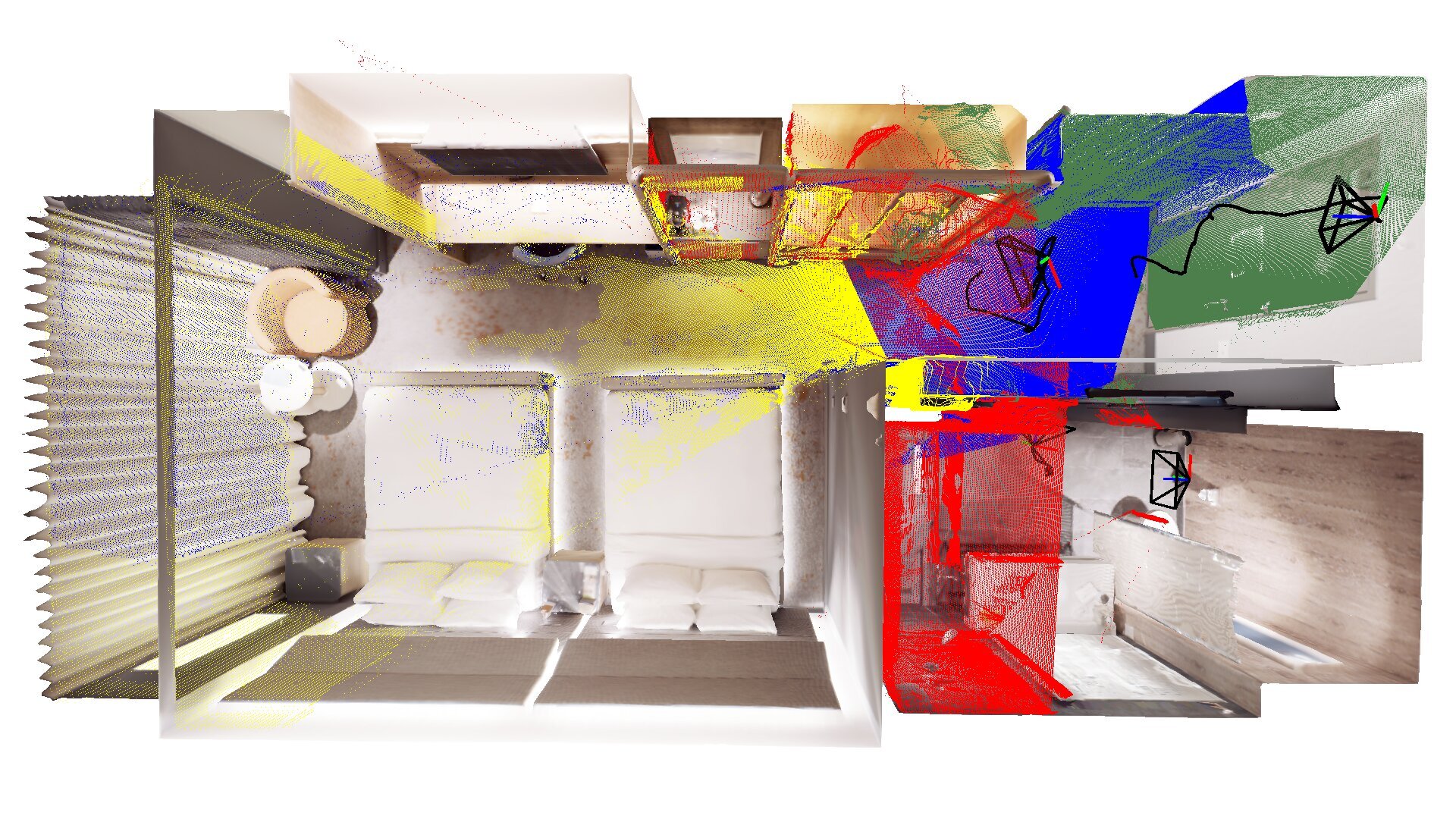} &
    \includegraphics[width=0.2265\linewidth,trim={0pt 40pt 0pt 40pt},clip]{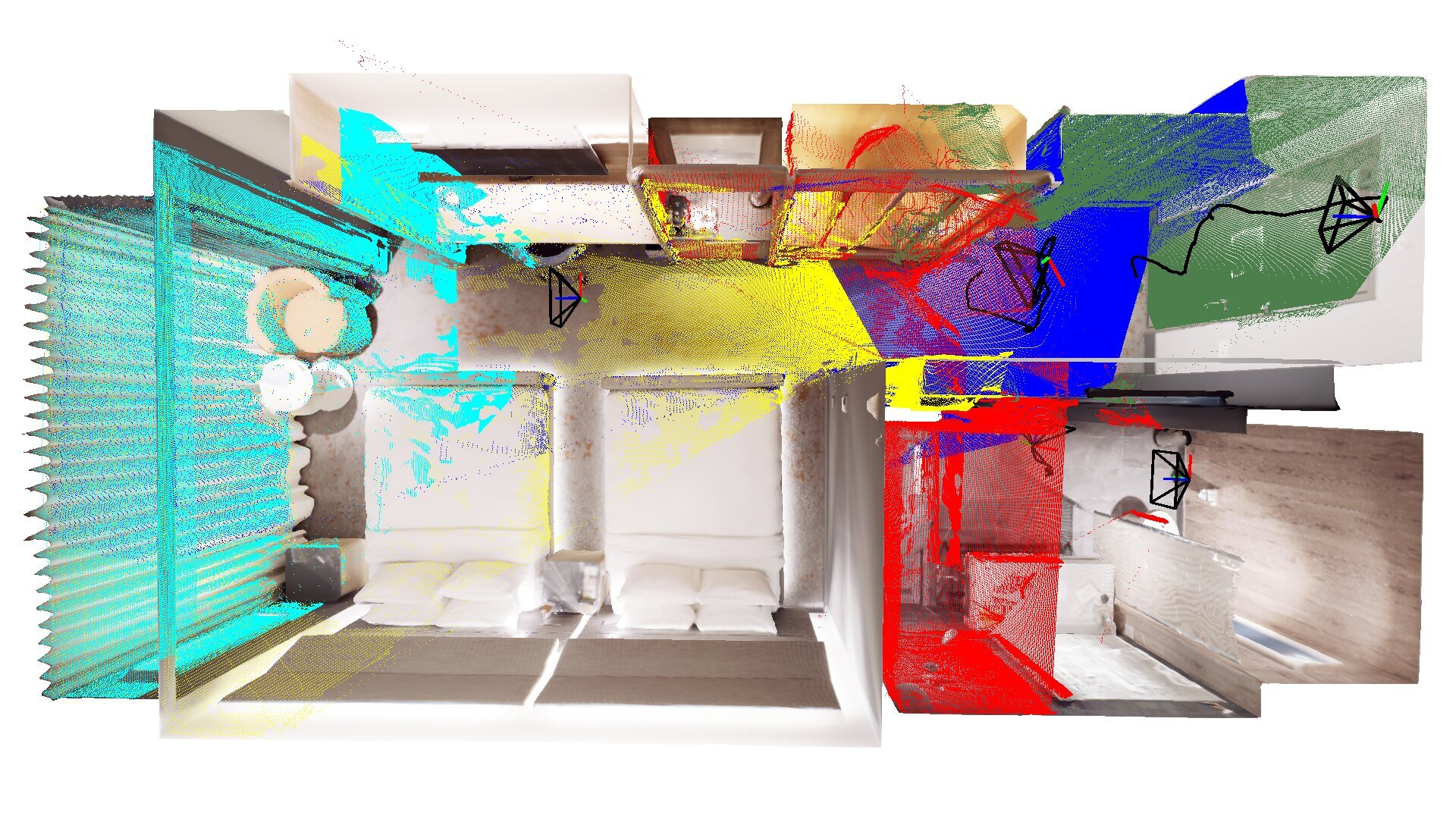} \\
    \includegraphics[width=0.2265\linewidth,trim={0pt 40pt 0pt 40pt},clip]{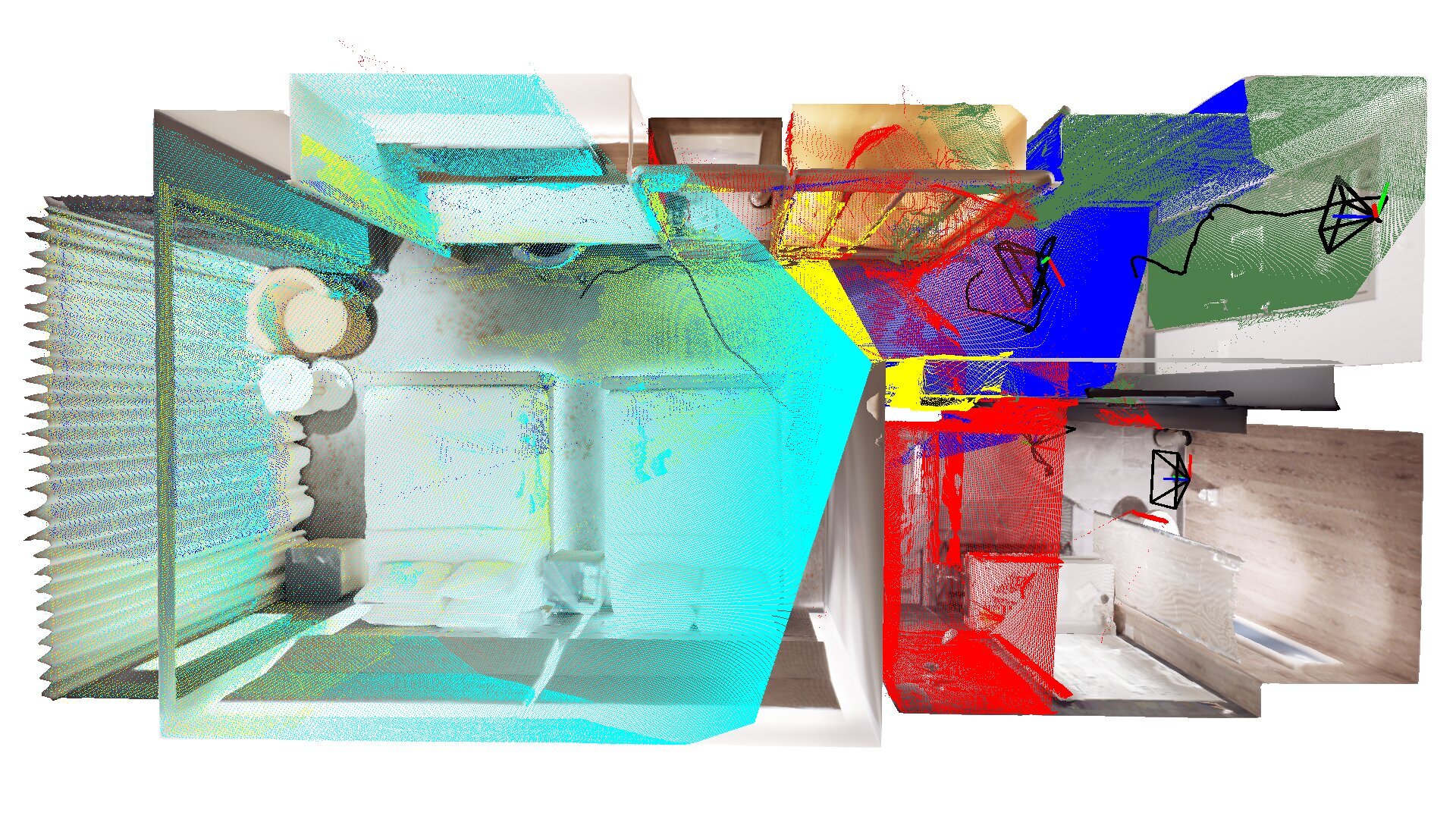} &
    \includegraphics[width=0.2265\linewidth,trim={0pt 40pt 0pt 40pt},clip]{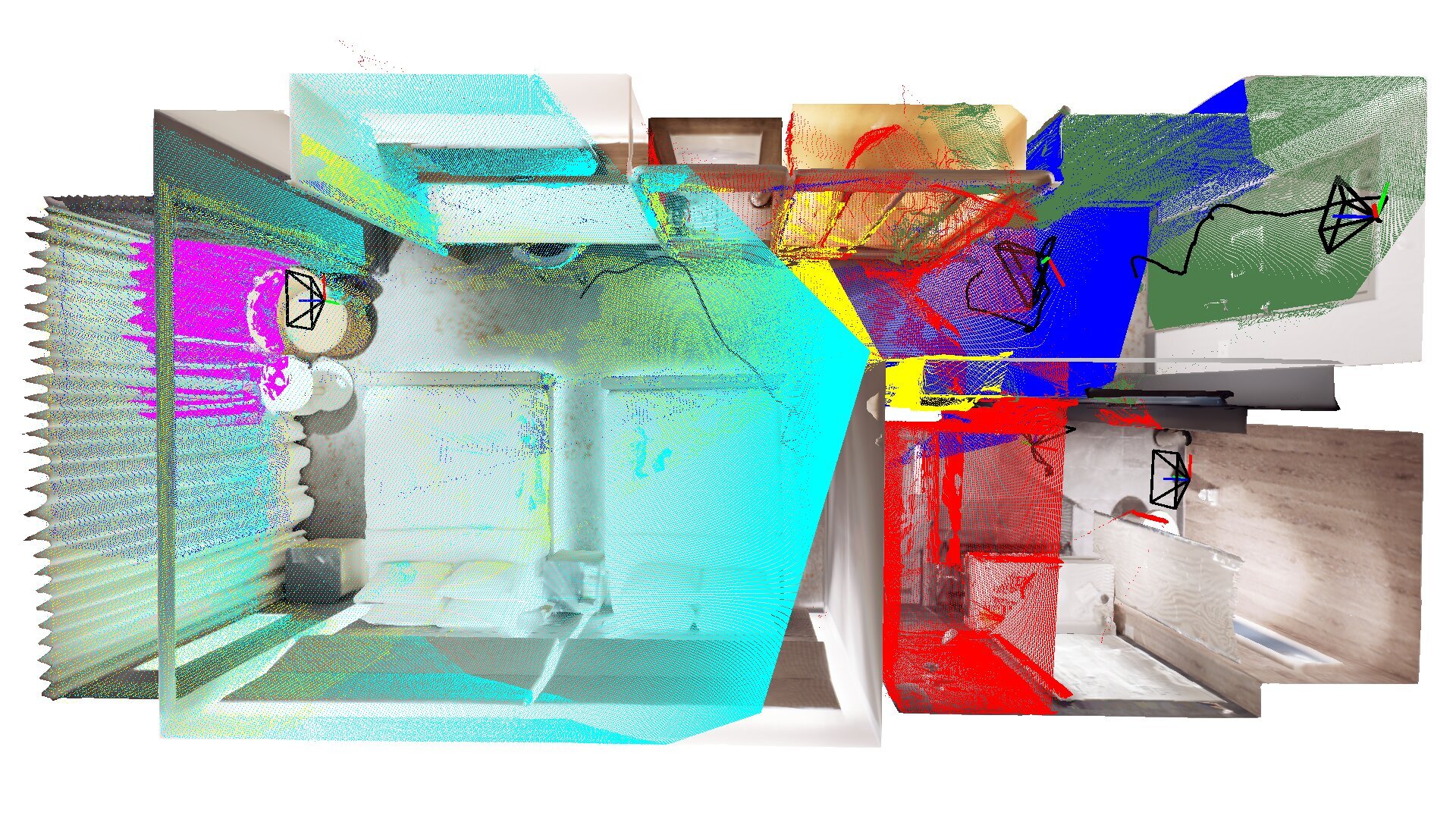} &
    \includegraphics[width=0.2265\linewidth,trim={0pt 40pt 0pt 40pt},clip]{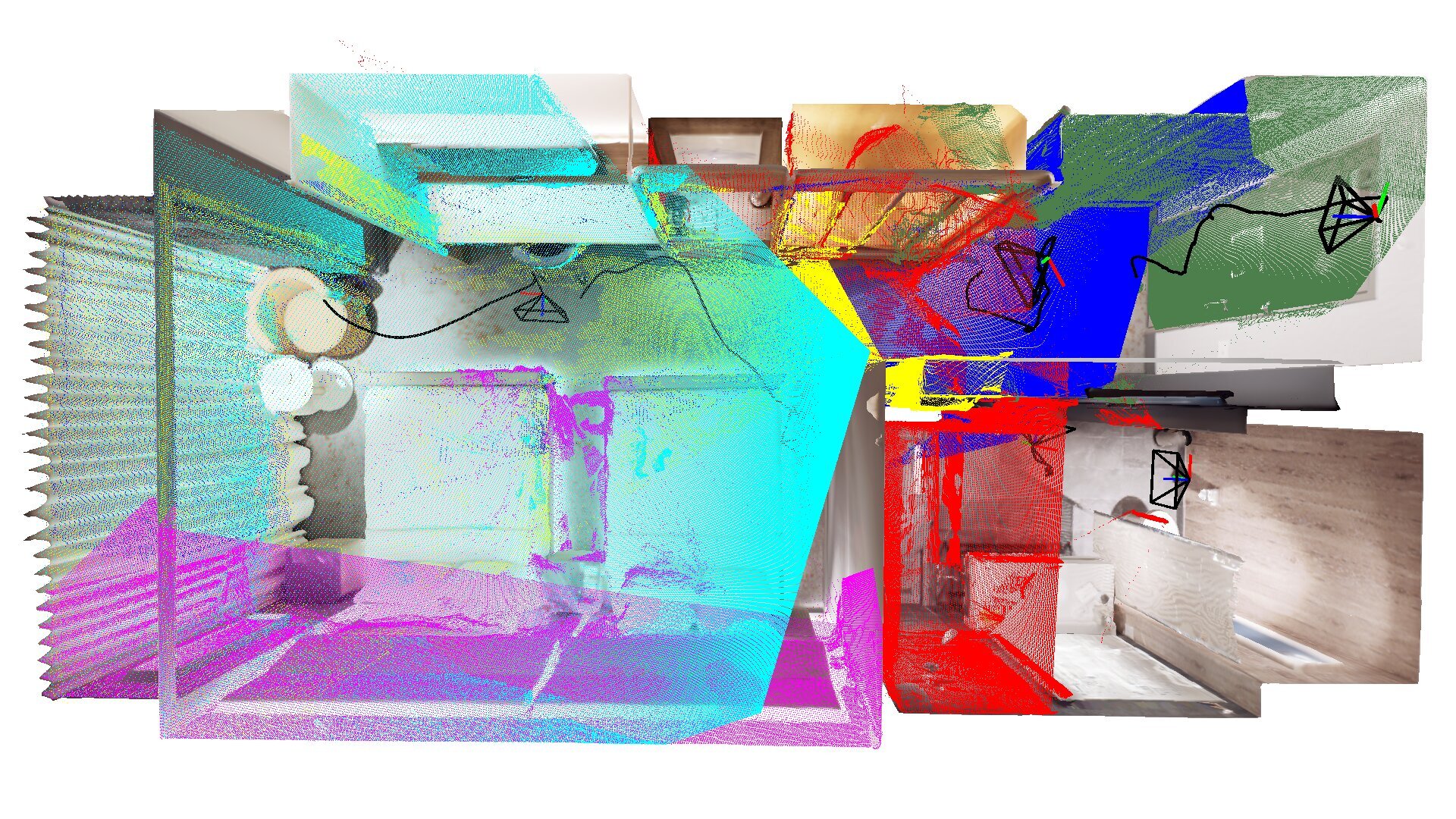} &
    \includegraphics[width=0.2265\linewidth,trim={0pt 40pt 0pt 40pt},clip]{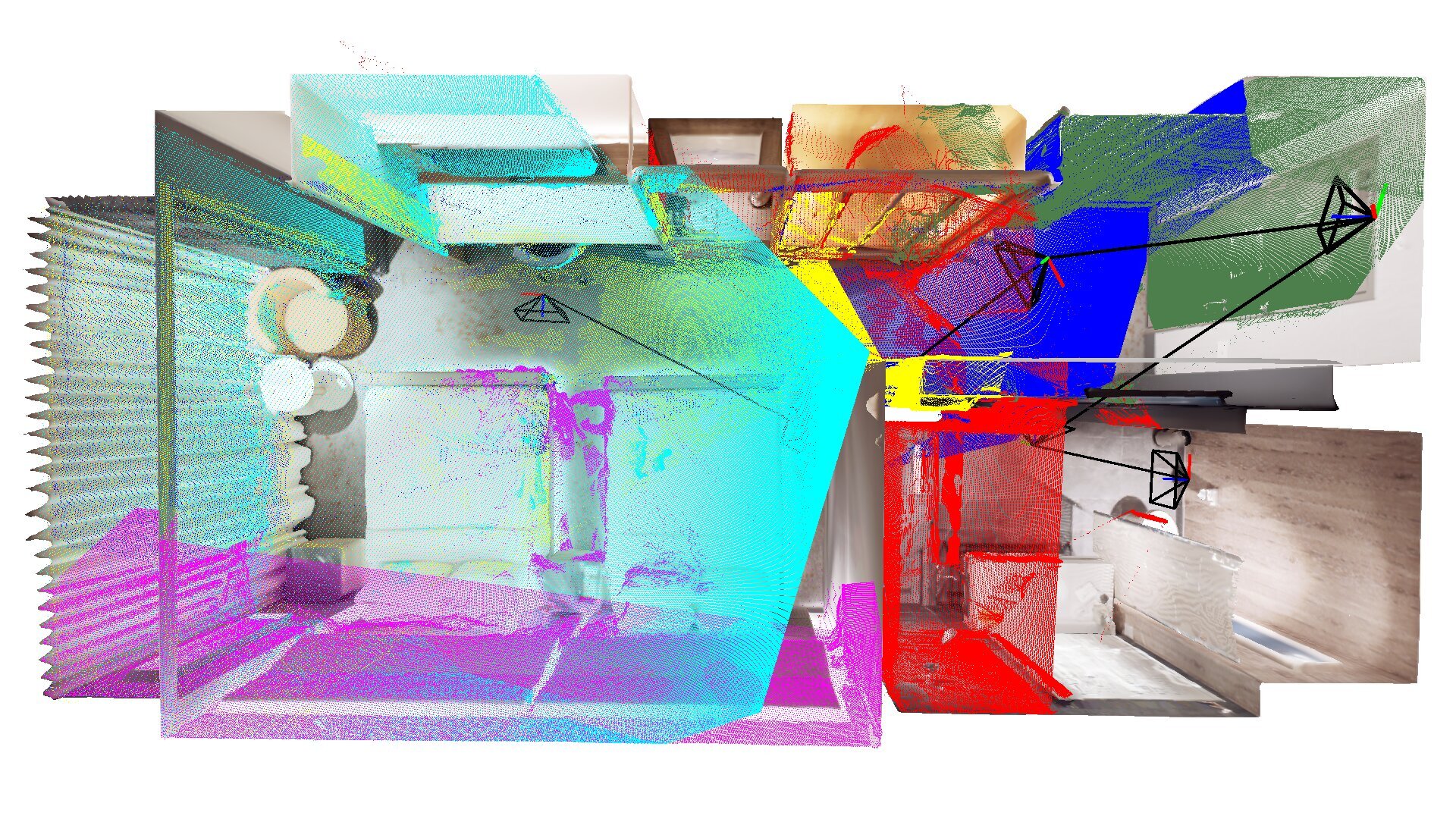}
\end{tabular}
\caption{\small Visualization of differentiable view optimization along a trajectory in the Replica Hotel 0 scene. The first and the last figures show the initial and optimized views, respectively. The other figures show selected frames at different iteration steps of the optimization process.}
\label{fig:traj_opt_replica_hotel0}
\end{figure*}
\begin{figure*}
\centering
\begin{tabular}{ccc}
    \includegraphics[width=0.31\linewidth,trim={0pt 160pt 20pt 110pt},clip]{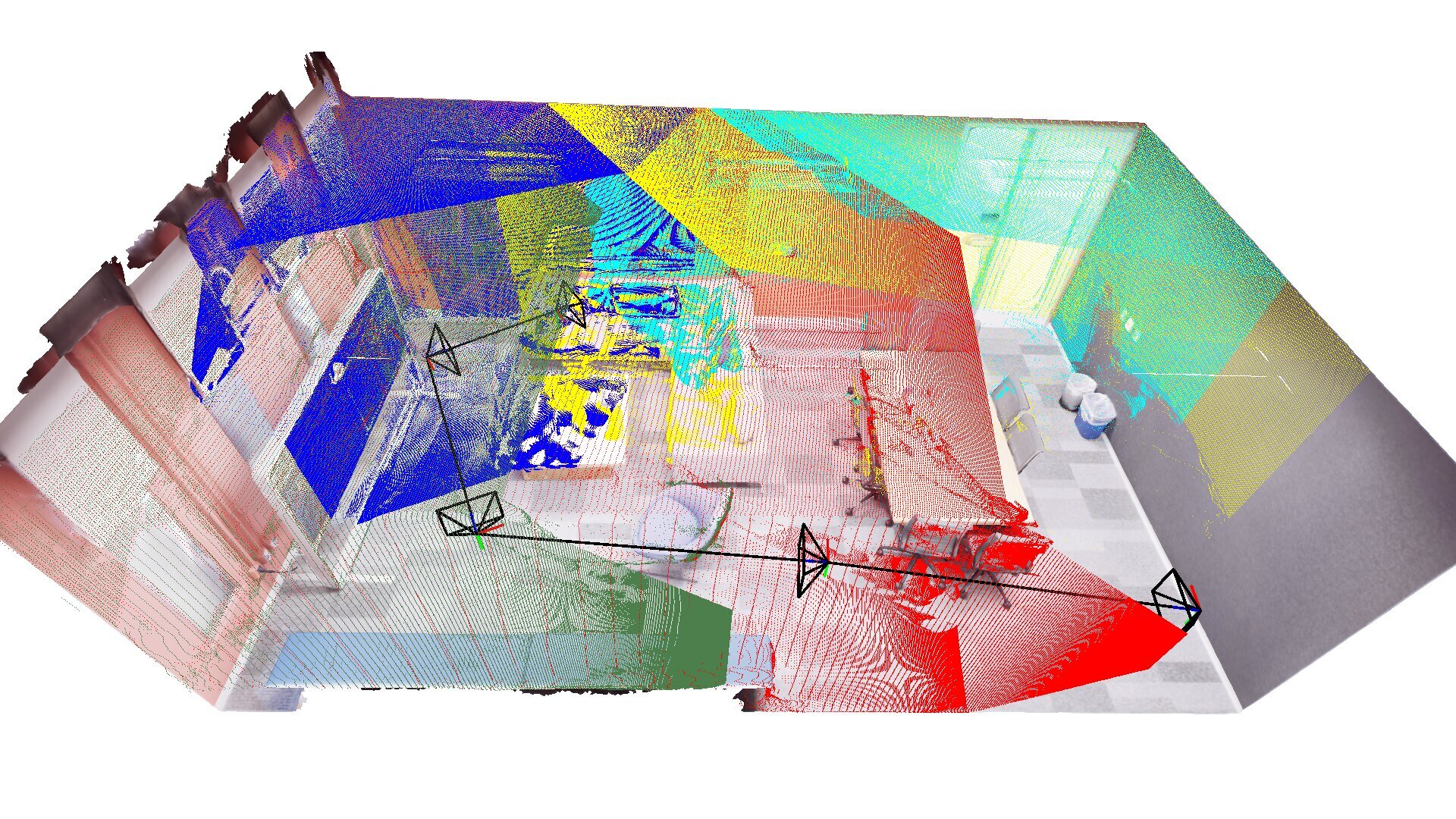} &
    \includegraphics[width=0.31\linewidth,trim={0pt 160pt 20pt 110pt},clip]{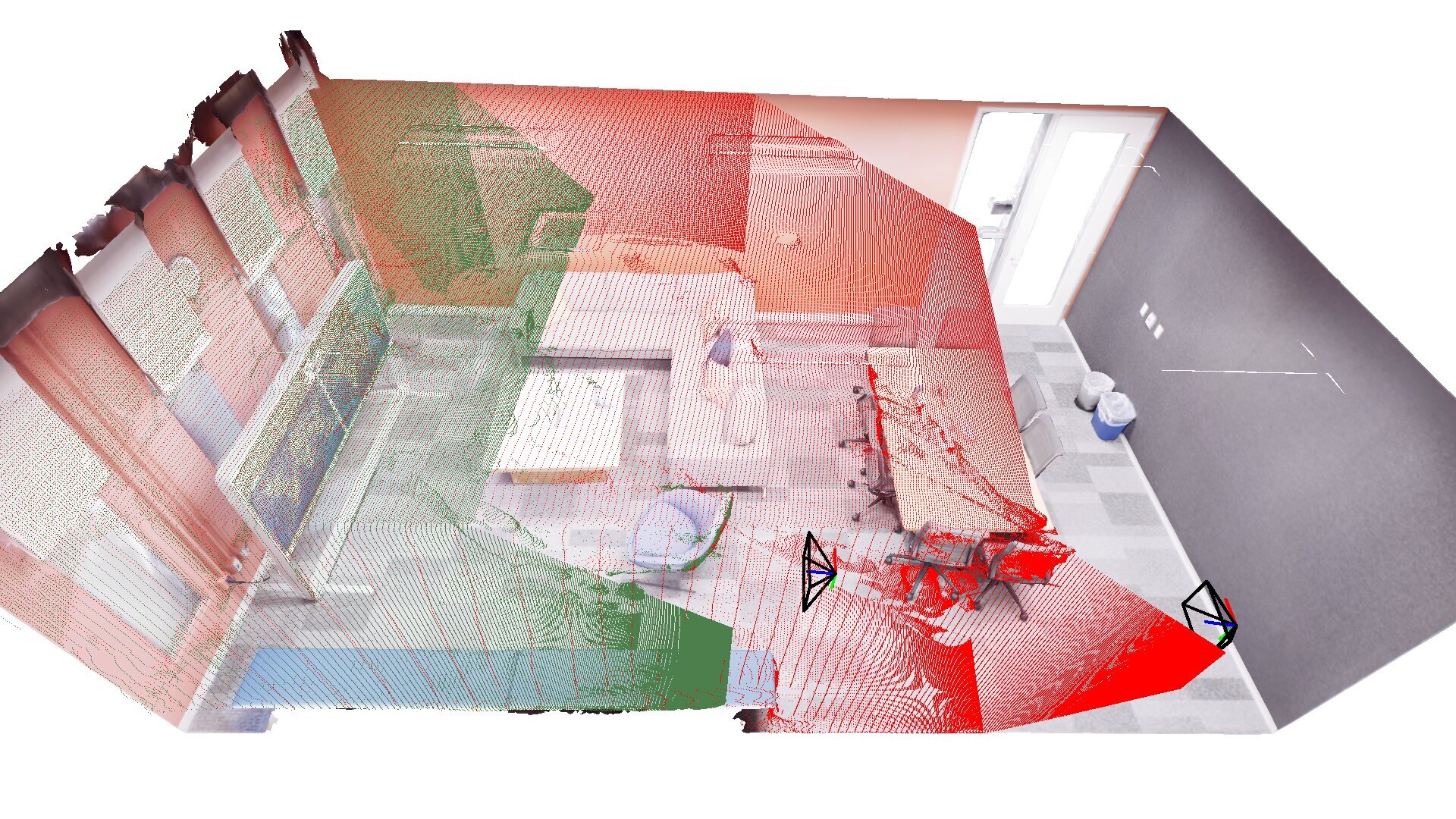} &
    \includegraphics[width=0.31\linewidth,trim={0pt 160pt 20pt 110pt},clip]{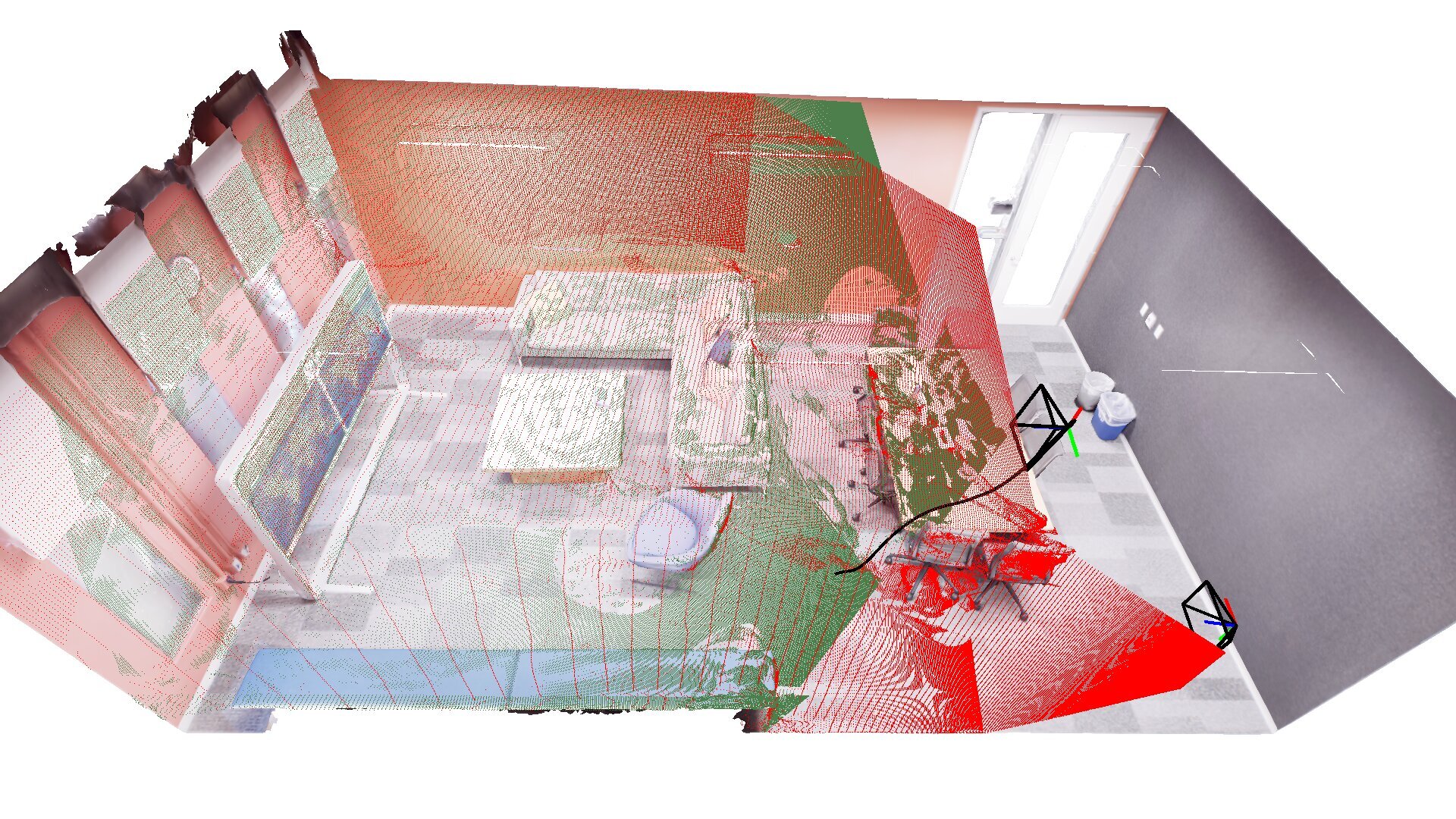} \\
    \includegraphics[width=0.31\linewidth,trim={0pt 160pt 20pt 110pt},clip]{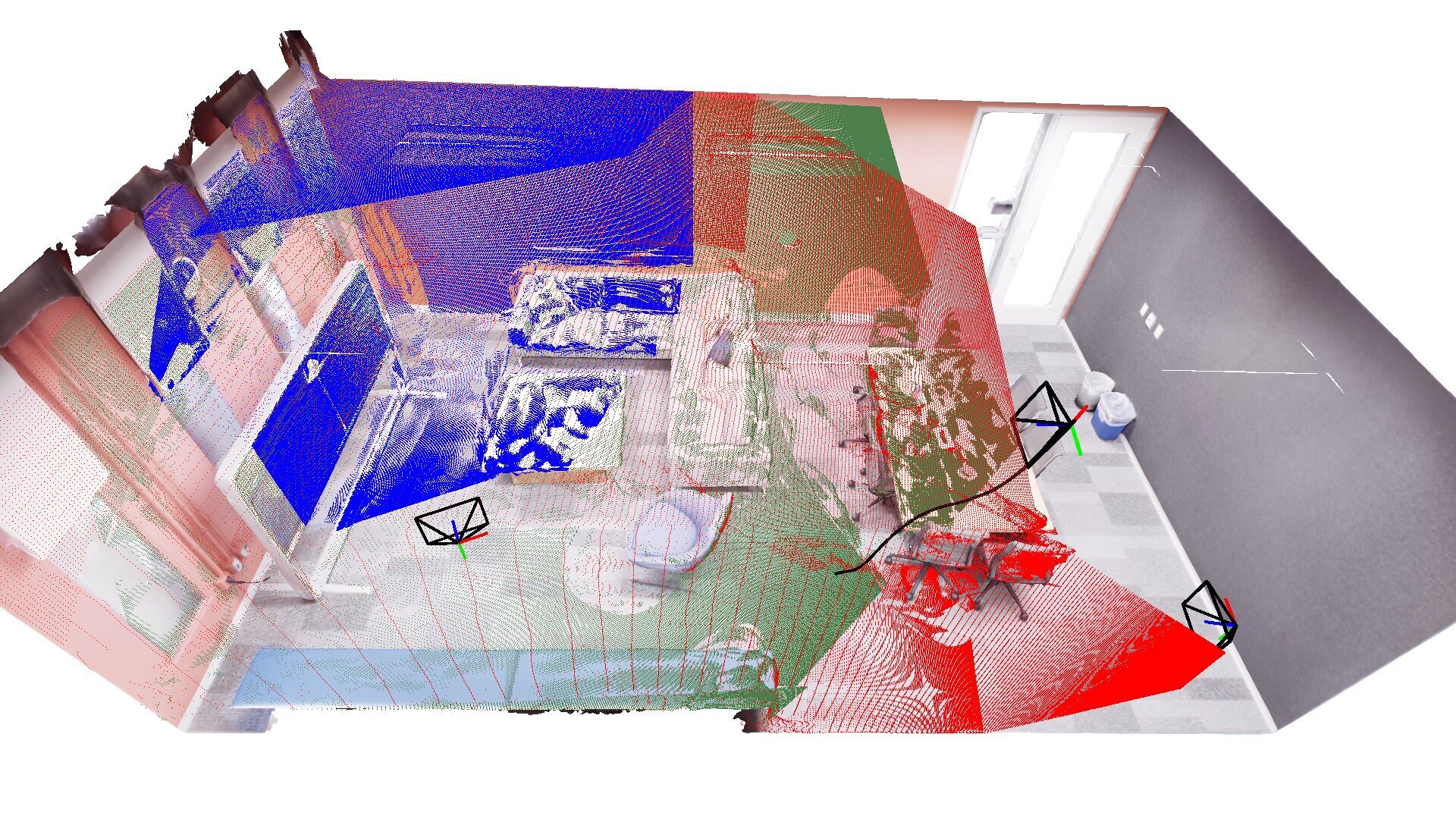} &
    \includegraphics[width=0.31\linewidth,trim={0pt 160pt 20pt 110pt},clip]{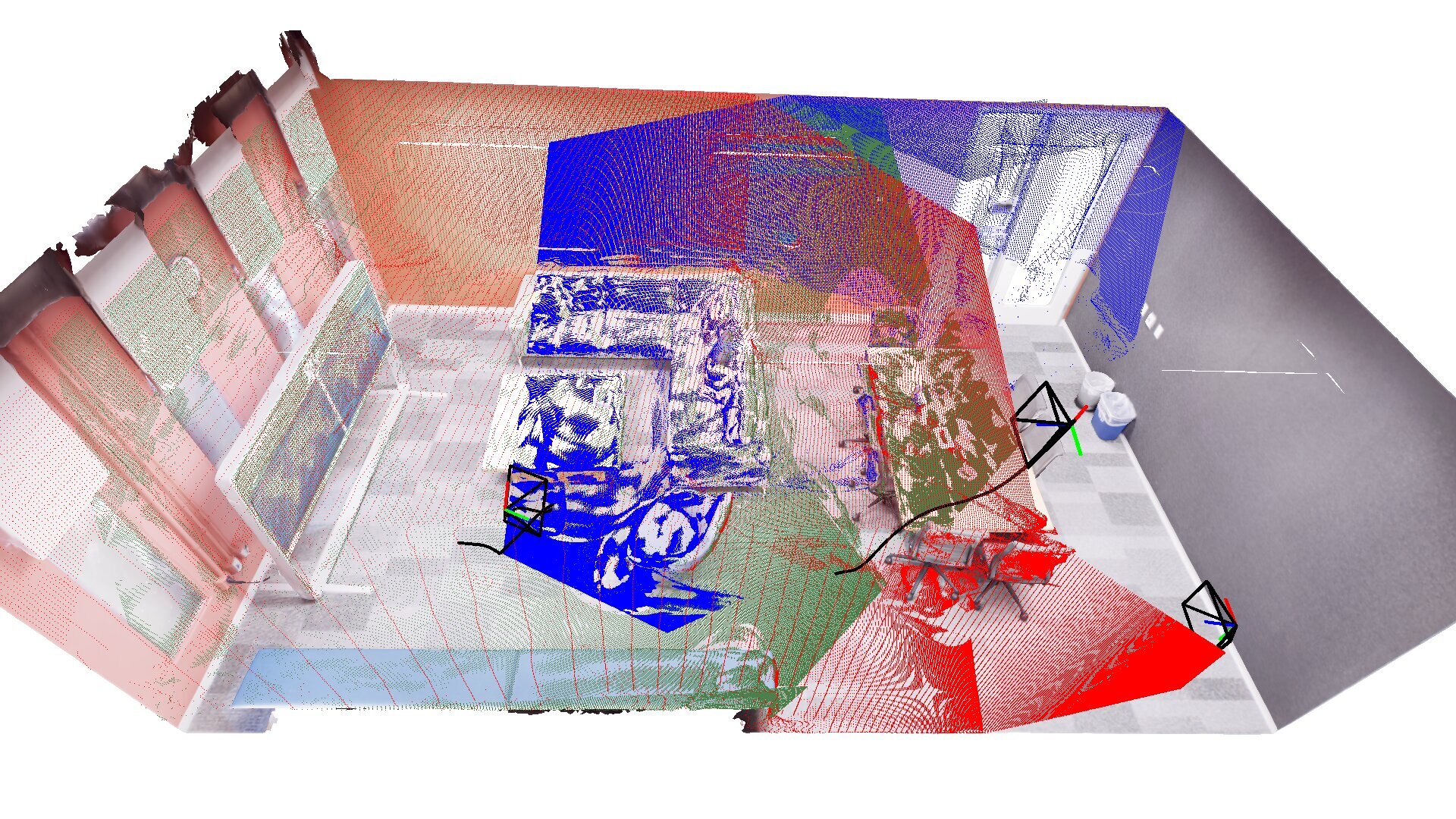} &
    \includegraphics[width=0.31\linewidth,trim={0pt 160pt 20pt 110pt},clip]{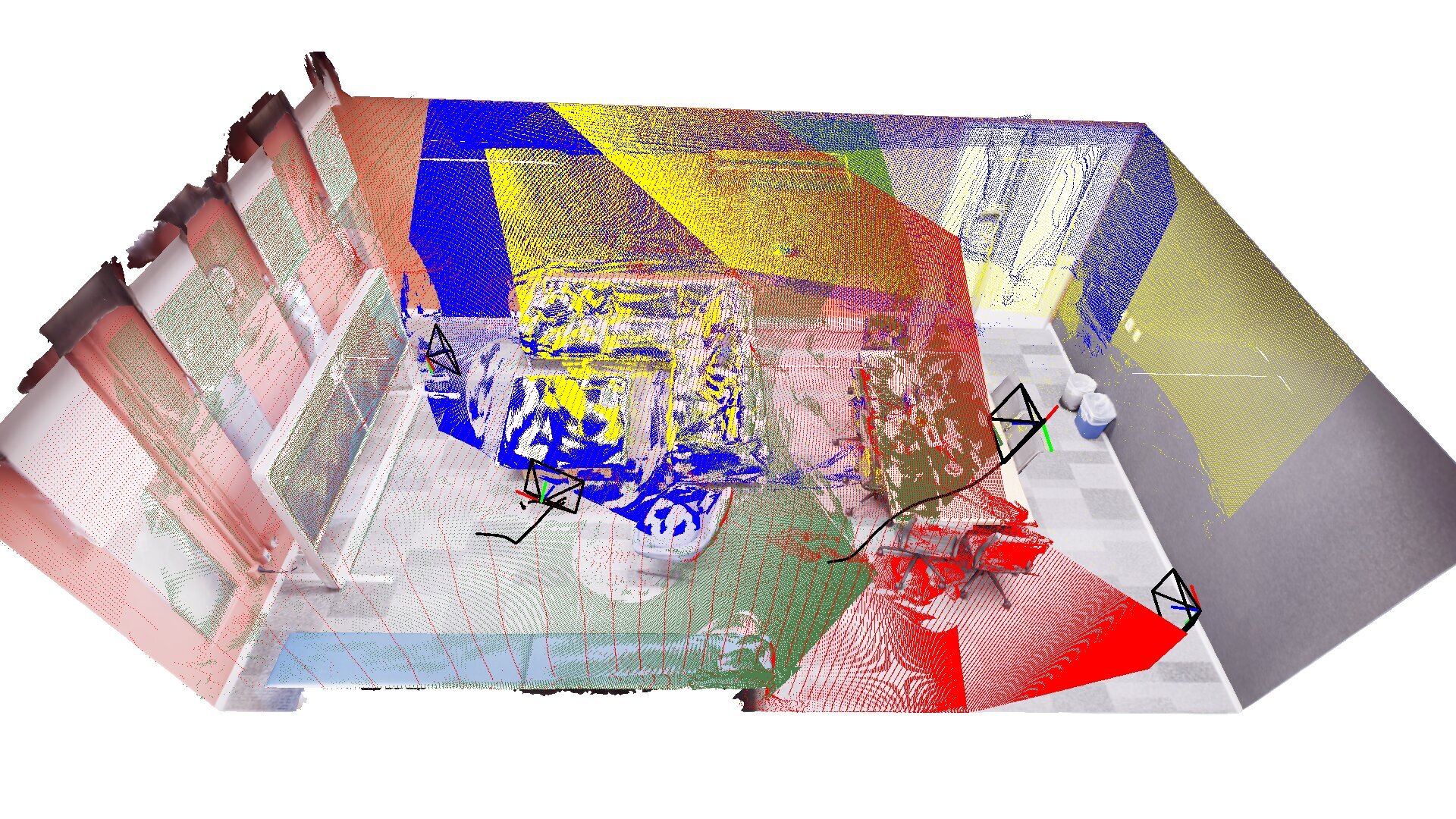} \\
    \includegraphics[width=0.31\linewidth,trim={0pt 160pt 20pt 110pt},clip]{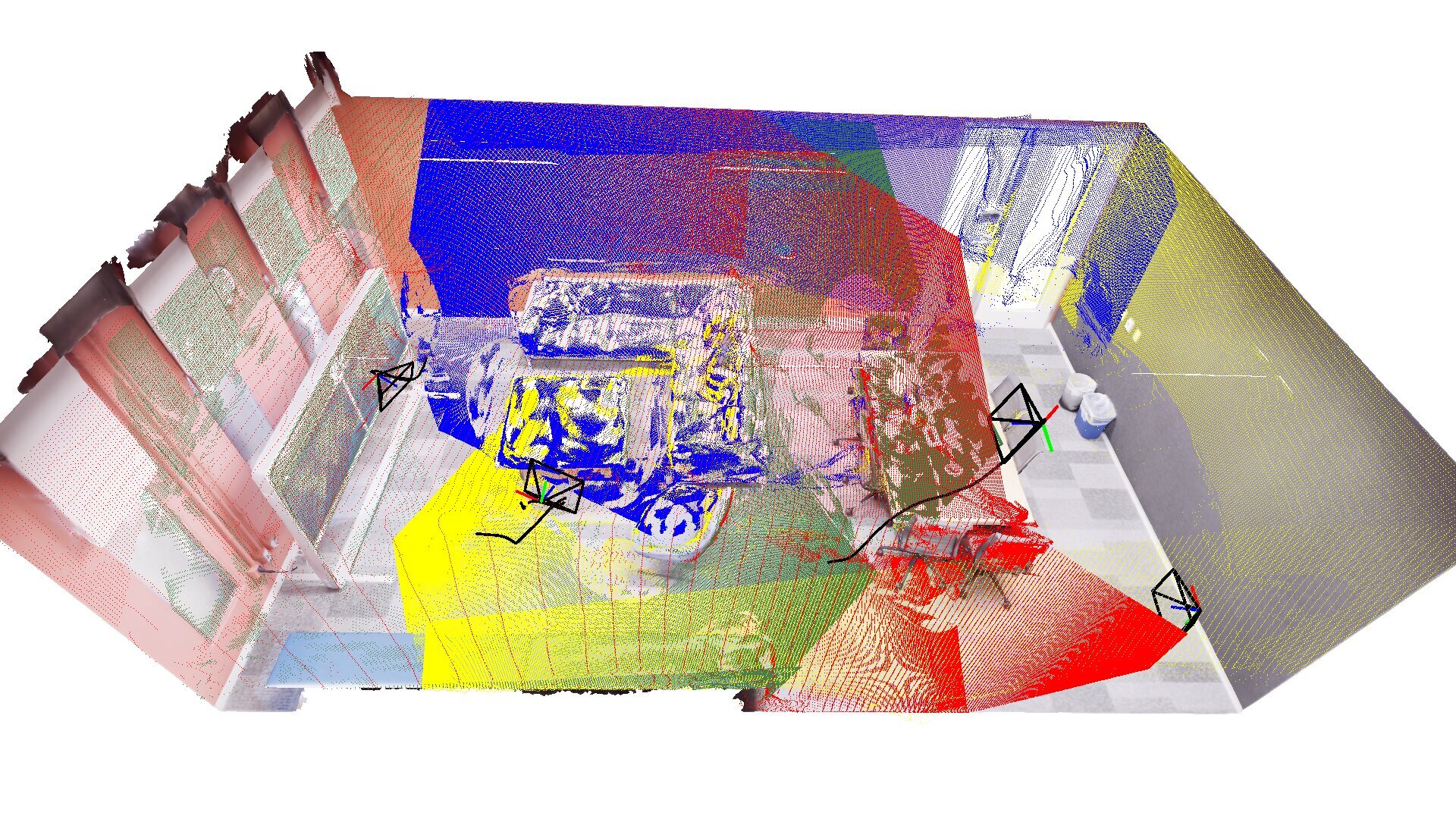} &
    \includegraphics[width=0.31\linewidth,trim={0pt 160pt 20pt 110pt},clip]{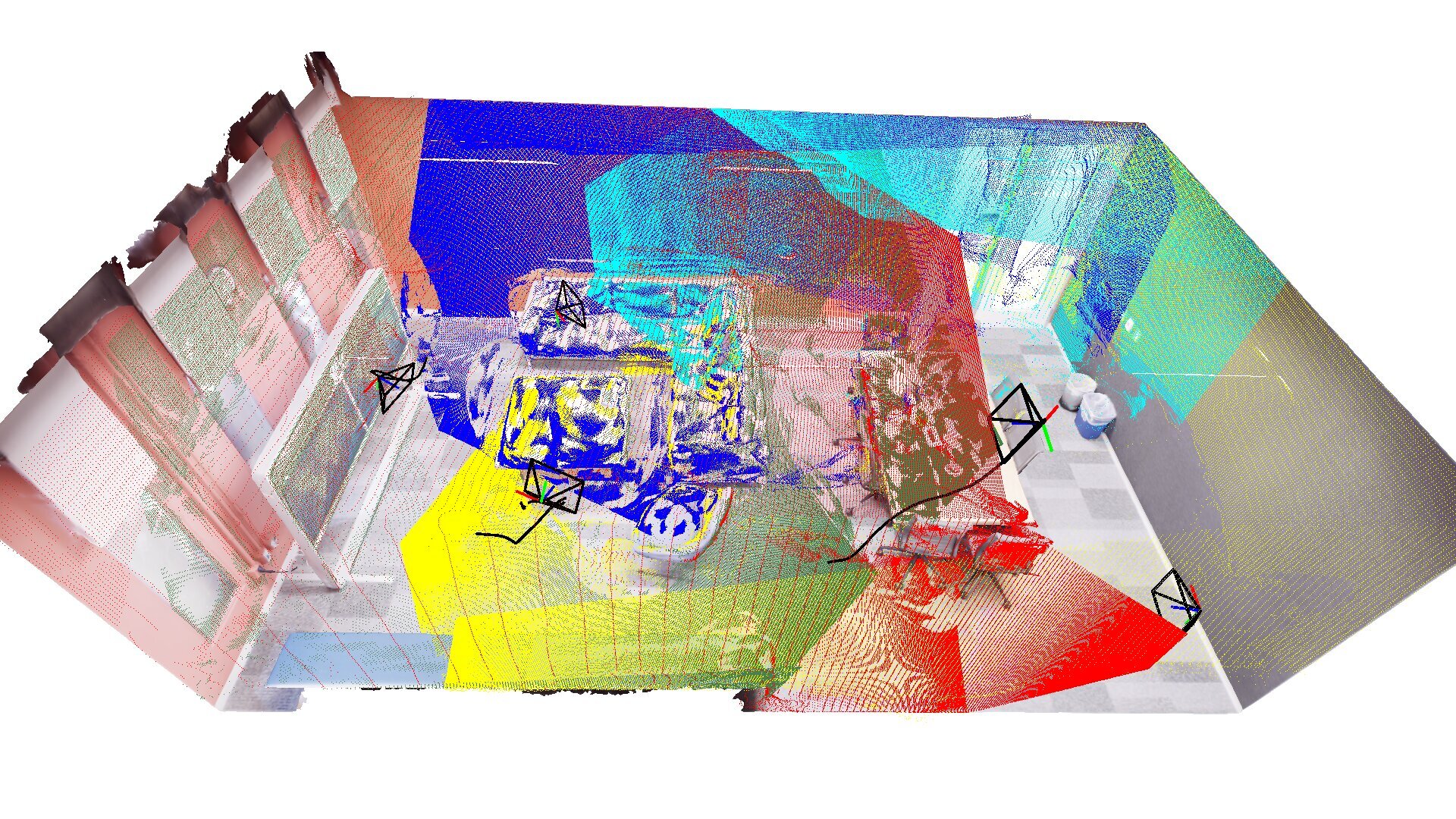} &
    \includegraphics[width=0.31\linewidth,trim={0pt 160pt 20pt 110pt},clip]{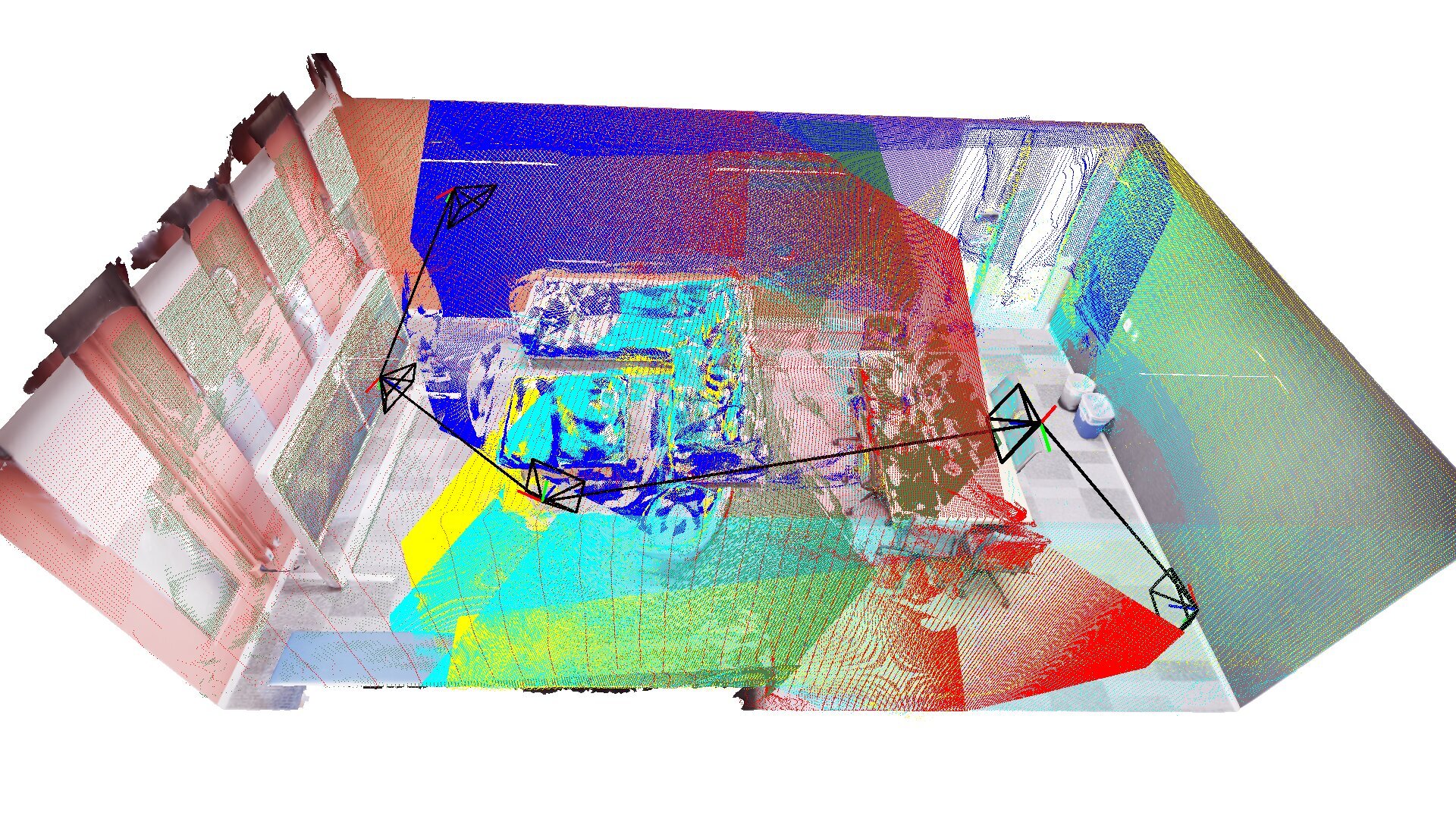}
\end{tabular}
\caption{\small Visualization of differentiable view optimization along a trajectory in the Replica Office 3 scene. The first and the last figures show the initial and optimized views, respectively. The other figures show selected frames at different iteration steps of the optimization process.}
\label{fig:traj_opt_replica_office3}
\end{figure*}

\end{document}